%% file: main.tex
\documentclass{article}

\PassOptionsToPackage{numbers,sort,compress}{natbib}

\usepackage[preprint]{neurips_2021}

\input{header}
\input{math_defs}

\title{Analytic Insights into Structure and Rank of \\Neural Network Hessian Maps}

\newcommand*\samethanks[1][\value{footnote}]{\footnotemark[#1]}

\usepackage{authblk}

\author[1,2]{Sidak Pal Singh\thanks{Detailed list of contributions are: Sidak first discovered that the Hessian rank formula, in an early form, holds experimentally to high fidelity, thus kick-starting the project. Sidak came up with the proof technique and proved Theorem~\ref{theorem:ub-outer}, Theorem~\ref{theorem:func-hess-cols},  Theorem~\ref{thm:relu-rank}, Theorem~\ref{theorem:bias}. Sidak wrote essentially the entire paper and noted the rank-deficiency interpretation. Gregor proved Lemma~\ref{lemma:evolve}, assisted in a part of Theorem~\ref{theorem:ub-outer}, and empirically observed the eventual formula for the Hessian rank. Gregor essentially ran all the experiments for the final submission and made the corresponding figures. Correspondence to \texttt{sidak.singh@inf.ethz.ch}.} \textsuperscript{$\,\,$}}

\author[1]{Gregor Bachmann\samethanks \textsuperscript{$\,\,$} }
\author[1,2]{Thomas Hofmann}
\affil[1]{ETH Zürich}
\affil[2]{Max Planck ETH Center for Learning Systems}

\begin{document}
\maketitle

\doparttoc % Tell to minitoc to generate a toc for the parts
\faketableofcontents % Run a fake tableofcontents command for the partocs, otherwise the table of contents in supplementary also includes the main text TOC

\let\saveFloatBarrier\FloatBarrier% save the original command
\let\FloatBarrier\relax

\begin{abstract}
The Hessian of a neural network captures parameter interactions through second-order derivatives of the loss. It is a fundamental object of study, closely tied to various problems in deep learning, including model design, optimization, and generalization. Most prior work has been empirical, typically focusing on low-rank approximations and heuristics that are blind to the network structure.  In contrast, we develop theoretical tools to analyze the range of the Hessian map, providing us with a precise understanding of its rank deficiency as well as the structural reasons behind it. This yields exact formulas and tight upper bounds for the Hessian rank of deep linear networks, allowing for an elegant interpretation in terms of rank deficiency. Moreover, we demonstrate that our bounds remain faithful as an estimate of the numerical Hessian rank, for a larger class of models such as rectified and hyperbolic tangent networks. Further, we also investigate the implications of model architecture (e.g.~width, depth, bias) on the rank deficiency. Overall, 
our work provides novel insights into the source and extent of redundancy in overparameterized networks. 
\end{abstract}

\section{Introduction}
Since the very infancy of neural networks, the Hessian matrix has been a central object of study. This is because the Hessian captures pairwise interactions of parameters via second-order derivatives of the loss function. As a result, the Hessian was productively employed, for instance, in (quasi-Newton) optimization methods~\cite{lbfgs,setiono1995use}, model design and pruning~\cite{lecun1990optimal,hassibiStork,hochreiter1995simplifying}, generalization~\cite{mackay}, network calibration~\cite{10.5555/2986766.2986882}, automatically tuning hyper-parameters~\cite{LeCun1992AutomaticLR}. 
But, from the outset the main practical challenge has been its size, scaling quadratically with 
the model dimensionality. This makes the problem severe for today's DNNs which have millions or even billions of parameters~\cite{he2015deep,devlin2019bert}. 

Consequently, most prior work has focused on designing scalable Hessian approximations, which either take the route of Hessian-vector products (R-operator) \cite{pearlmutter1994fast,martens2011learning,bastien2012theano} or employ positive definite approximations by appealing to the Fisher information matrix. Additional approximations --- without exception --- are needed on top, such as diagonal approximations \cite{lecun1990optimal,schaul2013pesky} in the former or  K-FAC~\cite{martens2020optimizing,botev2017practical,goldfarb2020practical} restricted to layerwise or arbitrary blocks on the diagonal~\cite{laurent2018an,dong2017learning,singh2020woodfisher}, in the latter.

The goal of this paper is to advance the analytical understanding of the Hessian map of a neural network. We pursue the fundamental question of how the model architecture induces structural properties of the Hessian. In particular, we analyze the dimension of its range (i.e., the rank) and identify the sources of rank deficiency.
Understanding the range of the Hessian map, in turn, delivers insights into the important aspect of how gradients change between iterations.

 \begin{wrapfigure}{R}{0.4\textwidth}
% \begin{SCfigure}
    \vspace{-2em}
    \centering
    \includegraphics[width=0.4\textwidth]{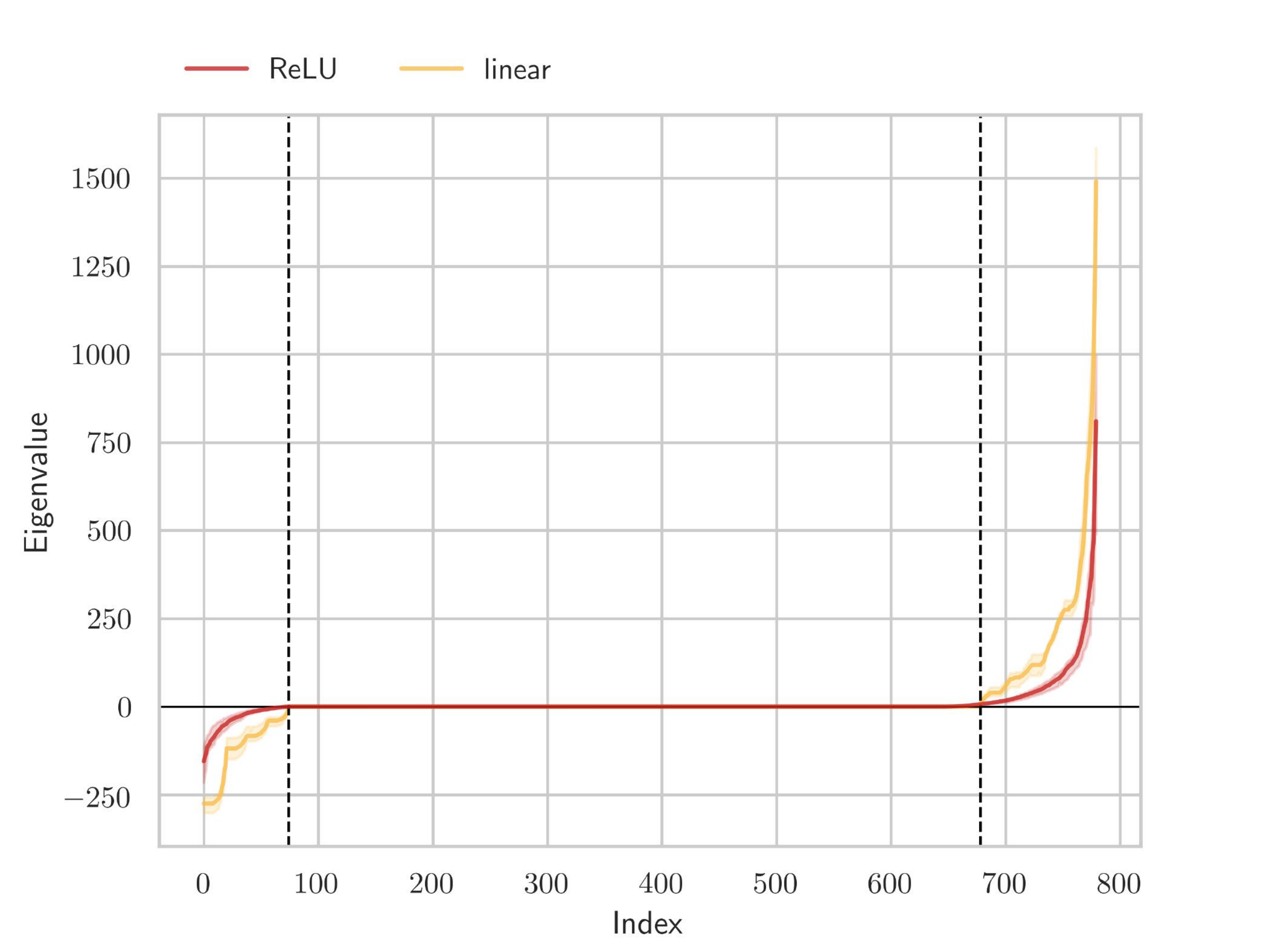}
    \caption{\small{Hessian spectrum of linear and ReLU networks at initialization. Dashed lines indicate our rank predictions. Results have been averaged over 3 runs.
    }}
    \label{fig:intro}
    \vspace{-1em}
% \end{SCfigure}
 \end{wrapfigure}
 
A reason why such an important direction currently remains sidelined is that non-linearities in a neural network result in an increased dependence on the data distribution, making a suitable theoretical analysis seem intractable. Following \citet{saxe2014exact,kawaguchi2016deep}, who delivered useful general insights on neural networks by looking at the linear case, we take a step back and rigorously characterize the range of the Hessian map and determine the resultant rank deficiency for deep linear networks. 
\textit{The key result of our paper is an exact formula along with tight upper bound on the rank of the Hessian --- which  effectively depend on the sum of hidden-layer widths.} This stands opposed to the total number of parameters which are proportional to the sum of squared layer widths, thus implying a significant redundancy in the parameterization of neural networks (see Figure~\ref{fig:intro}).

The exact quantification of the Hessian rank gives a precise yet interpretable ballpark on the inherent complexity of neural networks since rank naturally measures the effective number of parameters. This relationship is further reinforced by connection to the classical complexity measure of~\citet{Gull1989,NIPS1991_c3c59e5f_mackay}, which is equivalent to rank for a sufficiently small constant controlling the prior. Therefore, this sheds a novel perspective on the nature and degree of overparameterization in neural networks, and opens up interesting avenues for future investigation. 

\vspace{2mm}

\paragraph{Contributions.} The main contributions of our paper can be summarized as follows:
\begin{itemize}
    \item Section~\ref{sec:hessianstruct}: We characterize the structure of the Hessian range by exhibiting it in the form of matrix derivatives.
    \item Section~\ref{sec:rank}: We prove tight upper bounds and provide exact formulas on the Hessian rank that are neatly interpretable. \textit{To the best of our knowledge, this is the very first time that such formulas and bounds are made available for neural networks. }
    \item Section~\ref{sec:empirical-nonlinear}: In the non-linear case, we show that our (linear) rank formulae faithfully capture the numerical rank.
    \item Section~\ref{sec:rank-evolution}: We demonstrate via experiments and theory that such rank bounds also hold throughout the course of training.
    \item Section~\ref{sec:relu-bound}: In the non-linear case, we also provide a pessimistic yet non-trivial bound, which provably establishes degeneracy of the Hessian at the minimum.
    \item Section~\ref{sec:bias}, \ref{sec:verifyBehaviour}: We extend our rank results to the case of bias and investigate the effects of architectural components (such as width, depth, bias) on rank.
    \item Appendix~\ref{supp:spectra}: As a by-product, our analysis also reveals interesting properties of the Hessian spectrum, and we prove the presence of additional redundancies due to repeated eigenvalue plateaus (e.g., see Fig.~\ref{fig:func-spectra}). \looseness=-1
\end{itemize}

\begin{figure}[h]
    \centering
    {\includegraphics[trim={0.75cm 0.3cm 0 2.75cm},clip,width=0.35\textwidth]{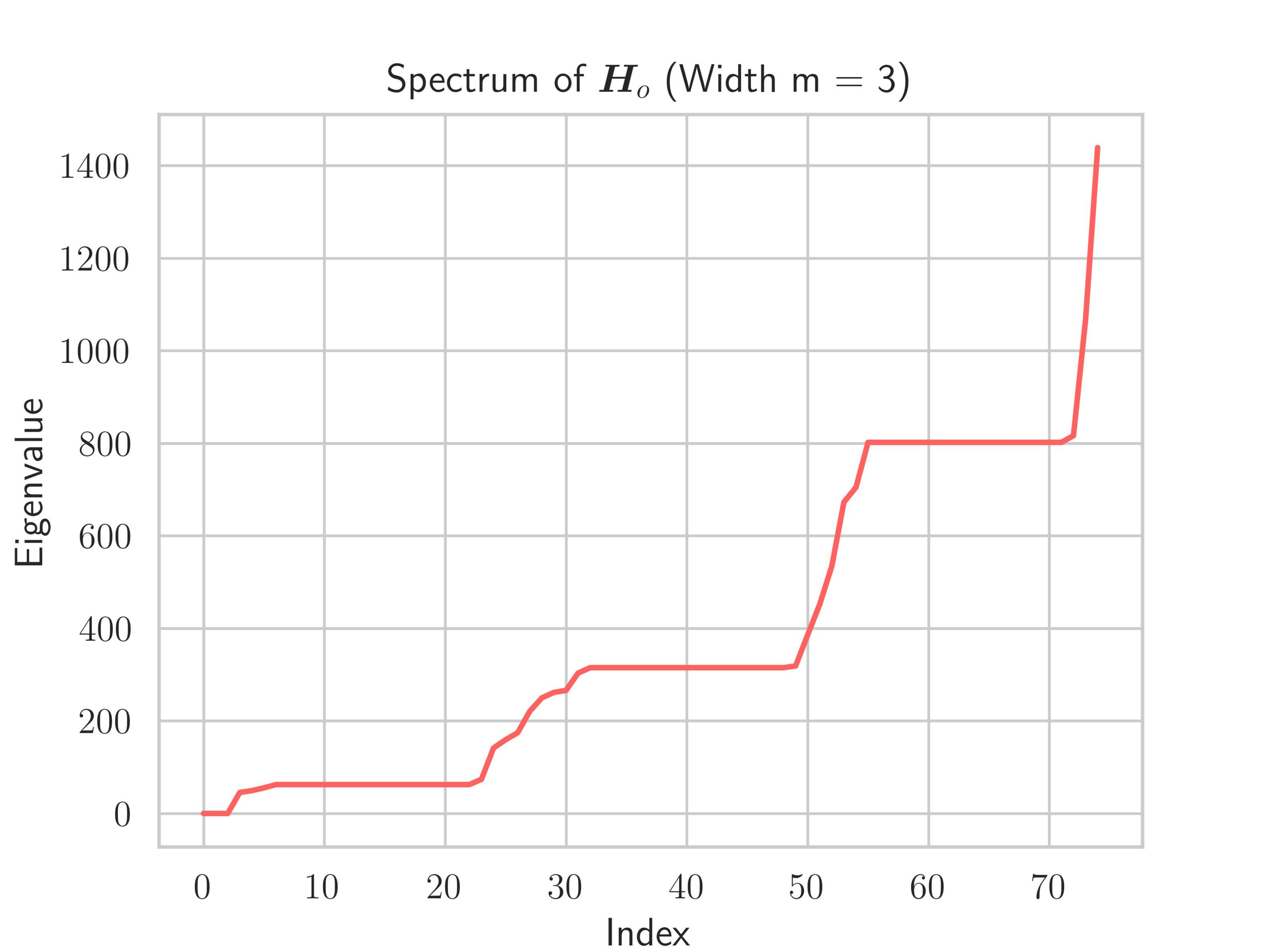}}
	{\includegraphics[trim={0.75cm 0.3cm 0 2.75cm},clip,width=0.35\textwidth]{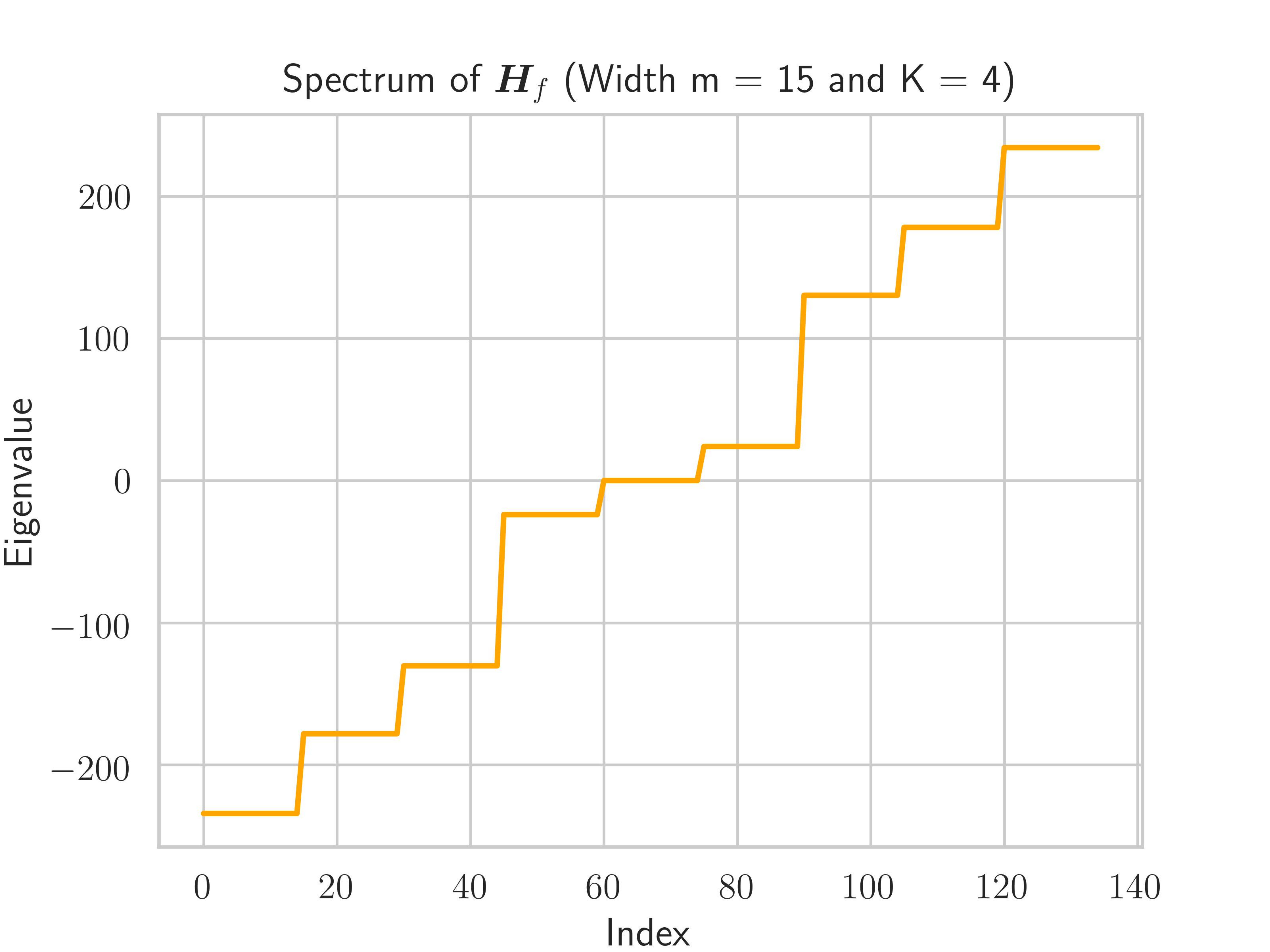}} 
    \caption{\small{Other kinds of redundancies in the Hessian structure due to the occurrence of repeated eigenvalues. In Appendix~\ref{supp:spectra}, we discuss the conditions when they arise and the reasons behind it. Additional such results can also be found therein.}}
    \label{fig:func-spectra}
\end{figure}

\paragraph{Related work.}
The study of the Hessian, in recent times, has been re-invigorated by the empirical observations of \citet{sagun2016eigenvalues,sagun2017empirical} who noted a high degree of degeneracy experimentally, and characterized the spectrum as being composed of a bulk around zero and few outlier eigenvalues. Since then, works such as ~\cite{ghorbani2019investigation,papyan2019measurements} have scaled the empirical analysis to bigger networks via efficient spectral density calculations and better explained the observations in \cite{sagun2017empirical}. However, a drawback is that due to their empirical nature, it remains hard to uncover the exact level of degeneracy via these methods~\cite{granziol2020towards}, and neither are the exact factors that affect the rank deficiency discerned. %Plus,
Another approach has been to fit models from random matrix theory to match the observed Hessian spectrum in neural networks~\cite{pmlr-v70-pennington17a,granziol2020towards}, however, such methods do not properly capture the sharp peak at zero observed in empirical spectra~\cite{Granziol2020BeyondRM}. Yet others have investigated the properties of the spectrum in the asymptotic regime ~\cite{karakida2019universal,jacot2019asymptotic}, but the bounds, if any, are quite coarse (see Fig.~\ref{fig:relu_reconstruction}). Moreover, a common issue underlying most of these approaches is that they 
are blind to the layerwise compositional structure of neural networks, and are often motivated by reference to a black-box decomposition of the Hessian. 

In regards to indicating the parameter redundancy in neural networks, we are far from being the first work. Prior empirical studies have long reported similar observations, like in the form of, post-training pruning~\cite{lecun1990optimal,hassibiStork,han2016deep,dong2017learning,singh2020woodfisher} or inherently contained sub-networks (Lottery Ticket Hypothesis~\cite{Frankle2019TheLT}). Recently,~\cite{parametercounting} have experimentally argued for the effective dimensionality from~\cite{Gull1989,NIPS1991_c3c59e5f_mackay} to be a good predictor of double descent~\cite{Belkin2019ReconcilingMM}. Nevertheless, determining the precise extent of redundancy and the structural reasons behind it, have remained illusory in such approaches.

\section{Setup and Formalism}

\paragraph{General notation.} We use the shorthand, $\wM{k:l}$, to refer to the matrix product chain $\wM k \cdots \wM l$, when $k>l$. When $k < l$, $\wM{k:l}$ will stand for the transposed product chain $\wMt{k} \cdots \wMt{l}$. 
Besides, $\kro$ denotes the Kronecker product of two matrices, $\vect_r$ indicates the row-wise vectorization of matrices, $\rank$ refers to the rank of a matrix, and $\Im_{k}$ is the identity matrix of size $k\times k$.

\paragraph{Deep Neural Networks (DNNs).} A feedforward DNN is a composition of maps, i.e., $F = \Fmap L \circ \dots \circ \Fmap 1$, where the $l$-th layer map  $\Fmap l: \Re^{M_{l-1}} \to \Re^{M_l}$, with  input dimension $d:=M_0$, output dimension $K:=M_L$, total number of hidden neurons $M:=\sum_{l=1}^{L-1} M_l$. Each layer map is parameterized 
by a weight matrix $\wM{l}$ and applies an elementwise activation function $\sigma^{l}$. So we have, \[\,\Fmap l = \sigma^{l} \circ \wM{l}\,\,\text{with}\,\, \wM l \in \Re^{M_{l} \times M_{l-1}}\,\,\text{and}\,\,\sigma^{l}: \Re \to \Re\,. \]
For the sake of tractability, we will often investigate linear DNNs where $\sigma^l = \text{id}$, and so $F(\x) = \wM{L:1}\x$. 
For compactness, we also represent the entire set of parameters by $\btheta:=\{\wM{1}, \cdots, \wM{L}\}$, and where emphasis requires, we will subscript the DNN map $F$ with it and write $\Fn$.

Next, assume that we are given a dataset $S=\{(\x_i,\y_i)\}_{i=1}^{N}$ of $N$ input-output pairs, drawn i.i.d from an underlying distribution $p_{\x, \y}$. 
Our focus will be on the squared loss (MSE),  $\ell_{\x,\,\y}(\btheta) = \frac 12 \|\y - \hat\y\|^2$ and its residual $ \pmb \delta_{\x,\,\y} :=\hat\y-\y = \jacobi{\ell_{\x,\,\y}}{\hat\y}$,
where $\hat\y=F_{\btheta}(\x)$ is the DNN prediction.
The population loss, $\mathcal{L}$ is:  ${\mathcal{L}}(\btheta) = \E_{p_{\x, \y}}\left[\ell_{\x,\,\y}(\btheta)\right]\,$.
Finally, we will analyze the Hessian matrix $\HL
= \dfrac{\partial^2 \mathcal{L}}{\partial  \btheta \, \partial  \btheta}$.
\paragraph{Backpropagation in matrix derivatives.} As all parameters are collected into matrices, we often work with matrix-matrix derivatives by vectorizing \textit{row-wise} in the numerator (Jacobian) layout, i.e., $\dfrac{\partial \mb Y}{\partial \mb X} := \dfrac{\partial \vect_r(\mb Y)}{\partial \vect_r(\mb X)^\top}\,$, see~\citep{magnus2019matrix} and  Appendix~\ref{supp:matrix-derivative}. Alongside this, we use the following  rule:
\begin{align}\label{eq:matrix-derivative}
 \frac{\partial \mb A \mb W \mb B}{\partial \mb W} = \mb A \otimes \; \mb B^\top, \quad \text{e.g.} \quad
\jacobi{F}{\wM \idx} &= 
\wM{L:\idx+1} \otimes \x^\top \wM{1:\idx-1} \,\, \in \Re^{K \times M_{\idx} M_{\idx-1}} \,.
\end{align}

 By the usual chain rule (backpropagation) one has for a linear DNN, at a sample $(\x, \y)$:
\begin{align}
\jacobi{\ell_{\x, \y}}{\wM \idx} 
& = \underbrace{\Big[ \wM{\idx+1:L}\pmb \delta_{\x, \y} \Big]}_{\text{backward } \in \Re^{M_{\idx}}} \,\cdot \, \underbrace{\Big[  \wM{{\idx}-1:1} \x\Big]^\top}_{\text{forward } \in \Re^{M_{{\idx}-1}}}
= \wM{\idx+1:L} \Big[  \wM{L:1}\,\x\,\x^\top - \y\,\x^\top \Big] \wM{1:\idx-1}\,.
\label{eq:dnn-grad-linear}
\end{align}
The above gradient with respect to $\wM \idx$ is of first order in $\wM \idx$ itself and second order in the other weight matrices. Lastly, let us setup the following shorthand, $\Om:=\E[\pmb \delta_{\x, \y} \, \x^\top]=\E[\wM{L:1} \x \x^\top - \y \x^\top]\,$, $\covx:=\E[\x\x^\top]\,$, and $\covyx:=\E[\y\x^\top]$, which we will use throughout the paper.

\section{Hessian Maps of Linear DNNs}\label{sec:hessianstruct}

\subsection{Hessian structure}
 The Hessian map has a natural block structure defined by the layers and their dimensionality. In order to leverage this structure, we directly take the derivative of the loss gradient in a matrix-by-matrix fashion. 
First, consider the $\idx$-th diagonal block of the Hessian, which is independent of $\y$ and is given by,  
\[
\HLsup{\idx\idx} := \frac{\partial^2 \mathcal{L}}{\partial \wM \idx \partial \wM \idx} \, \,
= \,\,\, \wM{\idx+1:L} \wM{L:\idx+1} \; \otimes \;
\wM{\idx-1:1} \,\covx\, \wM{1:\idx-1} \,.
\]

This follows from the matrix-derivative rule in Eq.~\eqref{eq:matrix-derivative} along with the Eq.~\eqref{eq:dnn-grad-linear} and taking expectation.
The calculation of the off-diagonal Hessian blocks ($kl$-th block of size $M_k \,M_{k-1}\times M_l \,M_{l-1}\,$) involves the product rule. Note that the two occurrences of a weight matrix $\wM{l}$, in Eq.~\eqref{eq:dnn-grad-linear}, are once non-transposed (in $\wM{L:1}\x \x^\top$) and once transposed (in $\wM{k+1:L}$ or $\wM{1:k-1}$ respectively for $k<l$ or $k>l$).  For simplicity, let us express these parts without adding them and directly write the other Hessian contribution with respect to the \textit{transposed matrix}, giving:

\begin{align}
\forall \,k \neq l, \quad \HLtildeSup{kl} := \frac{\partial^2 \Loss}{\partial \wM l \partial \wM k}  &= \wM{k+1:L}\wM{L:l+1} \; \otimes \; \wM{k-1:1} \,\covx\, \wM{1:l-1}\,.\label{eq:outer-kl}\\
\forall \,k<l \quad \widehat{\mb H}^{kl}_{\mathcal{L}} := \frac{\partial^2 \Loss}{\partial \wMt l {\partial \wM k}} & = \wM{k+1:l-1} \;\otimes\; \wM{k-1:1} \,{\Om}^\top \,\wM{L:l+1} \,.\label{eq:func-hess-eqn1}\\
\forall \,k>l \quad \widehat{\mb H}^{kl}_{\mathcal{L}} := \frac{\partial^2 \Loss}{\partial \wMt l {\partial \wM k}} & = \wM{k+1:L} \,\Om\, \wM{1:l-1} \;\otimes\;  \wM{k-1:l+1} \,.\label{eq:func-hess-eqn2}
\end{align}

\textbf{Equivalence to Gauss-Newton Decomposition.$\quad$} A common approach is to look at the Hessian map from the perspective of the Hessian chain rule, where we have that, $\HL=\HO + \HF$, with
\[
\HO = \E_{p_{\x, \y}}\left[\nabla_{\btheta} {F(\x)}^\top \; [\partial^2 \ell_{\x,\y} ]\; \nabla_{\btheta} F(\x)\right]\,, \text{\,and} \quad \HF=\E_{p_{\x, \y}}\left[\sum_{c=1}^{K} \,[\partial \ell_{\x, \y}]_{c}\; \nabla^2_{\btheta}\, F_c(\x)\right]\,.\]

For the MSE loss, the Gauss-Newton decomposition is in fact equivalent to what we discussed before (see details in Appendix~\ref{supp:backprop}), where $\HO$ contains the blocks $\HLsup{kk}$ and $\HLtildeSup{kl}$, while $\HF$ consists of $\HLhatSup{kl}$ (although with the non-transposed matrix). Henceforth, we will refer to the first term $\HO$ as the \textit{outer-product Hessian}, while we coin the second-term $\HF$ as the \textit{functional Hessian}.

\subsection{Range of the Hessian map}
Assuming a local Taylor-series approximation of the loss' gradient, we have that for $\|\Delta \btheta\|<\epsilon\,$, $\nabla_{\btheta + \Delta \btheta} \,\mathcal{L} \approx \nabla_{\btheta} \mathcal{L} + \HL \, \Delta \btheta$. This indicates how the gradients will change over any local perturbation $\Delta \btheta$. As a result, this also holds over successive iterations of an optimization algorithm such as gradient descent, and serves to show the significance of the Hessian range. Let us multiply the two parts of the Hessian with a vector $\Delta \btheta$, which we decompose as $\Delta \btheta = \left[\vect_r(\Delta \wM 1)^\top \,\, \dots \,\, \vect_r(\Delta \wM L)^\top\right]^\top \,$ to best reflect the layerwise structure.

\paragraph{Range of the outer-product $\HO$} The product of the $k$-th row-block $\HOsup{k\bullet}$, corresponding to the $k^{\text{th}}$ layer, with $\Delta \btheta$ can be written succinctly as, 
\[
\HOsup{k\bullet} \,\cdot\, \Delta \btheta
 = \vect_r\left(\wM{k+1:L} \cdot \bar \Delta \cdot \covx \wM{1:k-1}\right)\,, \,\,\text{where}\,\,
\bar\Delta  := \sum_{l=1}^L \wM{L:l+1} \,\,\Delta \wM  l\,\,
\wM{l-1:1} \,\in \Re^{K \times d}.\]

This essentially follows from the identity, $\vect_r \Am \Xm \Bm = (\Am \kro \Bm^\top) \vect_r\Xm$ (see proof in Appendix~\ref{supp:matrix-derivative}). 
Note that $\bar \Delta$ represents the net change on the prediction map induced by changes to the weight matrices in the forward pass. We see that: (1) $\Delta \btheta$ is linearly compressed into a highly interpretable matrix $\bar\Delta$ with $Kd$ entries. (2) The same compressed $\bar\Delta$ is shared across all result blocks as it is independent of the row block index $k$. Overall, this already hints that there is a  significant intrinsic structure in the Hessian that constrains its range.

\paragraph{Range of the functional Hessian $\HF$} 
Here we multiply with $\widehat{\Delta \btheta}$, which is similar to $\Delta \btheta$ except that we consider ${\Delta \W^l}^\top$ instead of $\Delta \W^l\,$.
Then the product corresponding to $k^{\text{th}}$ layer is,
\[
\HFsup{k\bullet} \cdot \widehat{\Delta \btheta}
= \vect_r \left(\wM{k+1:L} \,\Om\, [\Delta^{<k}]^\top
\,+\, [\Delta^{>k}]^\top\,\Om\,\wM{1:k-1}\right)\,,\]

where $\Delta^{<k}:=\sum_{l=1}^{k-1}
\wM{k-1:l+1}\,\,\Delta \wM{l}\,\, \wM{l-1:1} $ and $\Delta^{>k}:=\sum_{l=k+1}^L \wM{L:l+1} \,\,\Delta \wM{l}\,\, \wM{l-1:k+1} $.

Notice that the range of $\HF$ is also inherently constrained like in the case of $\HO$, but there are two important differences. First, the data dependent part is now the covariance $\Om=\E[\pmb \delta_{\x,\y}\, \x^\top]$. Clearly, if this matrix is low rank (say $s$), this will directly impact the rank of $\HF$. The other significant difference is that the weight matrix product is split at layer $k$. This reflects the fact that upstream and downstream layers have a different effect. We will see ahead that these factors will neatly gives rise to a dependence on the sum of hidden-layer widths.

\section{Main result: Analysis of the Hessian rank}\label{sec:rank}

\paragraph{Preliminaries.}
Let us denote the rank of the uncentered covariance $\covx = \E[\x \x^\top]$ by $r$. If $r<d$, then without loss of generality, consider $\covx := (\covx)_{\,r\times r}$,
which is always possible by pre-processing the input. Thus, $\covx \succ 0$, always. Also, in such a case, we 
take $\wM{1}\in\mathbb{R}^{M_1 \times r}$. Further, the only assumption we make in our analysis is \ref{assump:2}, which is in fact guaranteed at typical initialization with high probability (c.f. Appendix~\ref{supp:full_rank_weight}). In Section~\ref{sec:rank-evolution}, we see what happens while training. \looseness=-1

\vspace{1mm}
\begin{assump}{Maximal Rank:}\label{assump:2} $\,\forall \, l \in [L]$, $\wM{l} \in \mathbb{R}^{M_l \times M_{l-1}}$ has rank equal to $\min(M_l, M_{l-1})$. 
\end{assump}

Lastly, the omitted proofs in the coming subsections are located in Appendices~\ref{supp:tools},~\ref{supp:outer-product},~\ref{supp:functional} respectively.

\subsection{Analytical tool} The key idea of our analysis technique is to reduce the rank of involved matrices to the rank of a certain special kind of matrix (or its variant), $\Zm$, as shown below: 
\begin{align}
\Zm = \begin{pmatrix}\Im_q \otimes \Cm \\ \Dm \otimes \Im_n\end{pmatrix}\,,\quad \text{with} \,\, \Cm\in\mathbb{R}^{m\times n}\,,\,\,\Dm\in\mathbb{R}^{p\times q}\,. \label{eq:z-structure}
\end{align}
This row-partitioned matrix has a characteristic structure, where an identity matrix alternates between the two sides of the Kronecker product. Such matrices are in fact omnipresent in the Hessian structure, and importantly for our purpose, they possess additional properties on their rank. Inherent to these properties and our analysis, is the use of generalized inverse~\cite{rao1972generalized} and oblique (non-orthogonal) projector matrices. The following  Lemma~\ref{lemma:kronecker-block}  from~\cite{CHUAI2004129} details such a result:

\begin{lemma}\label{lemma:kronecker-block}
Let $\Zm$ be a matrix as in Eq.~\eqref{eq:z-structure}. Then,
$
\rank(\Zm)=q\, \rank(\Cm)+n\, \rank(\Dm)-\rank(\Cm)\, \rank(\Dm)\,.$
\end{lemma}

\subsection{Rank of the outer-product Hessian}
Consider the following decomposition of $\HO$, which reveals its `outer-product' nature:

\begin{proposition}\label{eq:outer-decomp}
For a deep linear network, $\,\HO = \Am_o  \Bm_o {\Am_o}^\top\,$, where $\,\Bm_o = \Im_K \kro \covx\, \in \mathbb{R}^{Kd\times Kd}$,
\[\text{and}\,\,\,\Am_o^\top=\begin{pmatrix}
\wM{L:2} \kro \Im_d \quad \cdots \quad 
\wM{L:l+1} \kro \wM{1:l-1}
\quad \cdots \quad 
\Im_K \kro \wM{1:L-1}
\end{pmatrix}\,\, \in \mathbb{R}^{Kd\times p}\,,
\]
\end{proposition}

where $p$ is the number of parameters. A straightforward consequence is that if there is no bottleneck in between, i.e., no hidden-layer with width $M_i<\min(K, d)$, then the matrix $\Bm_o$ will control the rank of $\HO$. 
Hence, as a first upper bound we get,
$\,\rank(\HO)\leq \min\left(\rank({\Am_o}), \rank(\Bm_o)\right) = \rank(\Bm_o) = \rank(\Im_K) \, \rank(\covx) = K r\,$. \looseness=-1

Such a decomposition is however not new (see~\cite{NEURIPS2018_7f018eb7}), but this only forms an initial step of our analysis and the current bound can be loose in the bottleneck case (e.g., an auto-encoder). Let us define the minimum dimension to be $q:=\min(r, M_1, \cdots, M_{L-1}, K)$. Our main theorem can then be stated as: \looseness=-1
\vspace{1mm}

\begin{mdframed}[leftmargin=1mm,
    skipabove=1mm, 
    skipbelow=-1mm, 
    backgroundcolor=gray!10,
    linewidth=0pt,
    leftmargin=-1mm,
    rightmargin=-1mm,
    innerleftmargin=2mm,
    innerrightmargin=2mm,
    innertopmargin=1mm,
innerbottommargin=1mm]
\begin{theorem}\label{theorem:ub-outer}
Consider the matrix $\Am_o$ mentioned in Proposition~\ref{eq:outer-decomp}. 
Under the assumption \ref{assump:2},
$
\rank(\Am_o) = 
r \rank(\wM{2:L}) + K \rank(\wM{L-1:1}) - \rank(\wM{2:L})\rank(\wM{L-1:1}) = q \,( r + K  - q)\,.
$
\end{theorem}
\end{mdframed}

Now, from Theorem~\ref{theorem:ub-outer}, it is evident that we can simply upper bound the rank of $\HO$, by the rank of $\Am_o$. But actually, it is possible to show an equality (using Lemma~\ref{lemma:outer-equality}), as described below:

\begin{corollary} \label{corollary:outer-rank}
Under the setup of Theorem~\ref{theorem:ub-outer}, the rank of $\HO$ is given by $\rank(\HO) = q \,( r + K  - q)\,.$
\end{corollary}

\textbf{Remark.$\,$} It should be emphasized that this analysis not only delivers the rank of the outer-product Hessian but also of Neural Tangent Kernel~\cite{jacot2018neural} (see Appendix~\ref{supp:ntk}), Fisher information matrix, network Jacobian, due to their underlying intimate relationship, unaltered through the lens of rank.

\vspace{-2mm}
\subsection{Rank of the functional Hessian}

For analyzing the rank of the functional Hessian $\HF$, we will continue operating by having one of the derivatives with respect to a transposed weight matrix (here, $\wMt{l}$), since rank does not change with column or row permutations. We denote this modification of the functional Hessian  by $\HFhat$. Then from Eqs. (\ref{eq:func-hess-eqn1},~\ref{eq:func-hess-eqn2}), we can observe that there is a common structure within the blocks contained in the column $l$. Namely, all the weight matrix-chains have either $l-1$ or $l+1$ as their right indices. So our approach will be to bound the rank of the individual column blocks, as formalized below:
\vspace{1mm}

\begin{mdframed}[leftmargin=1mm,
    skipabove=1mm, 
    skipbelow=-1mm, 
    backgroundcolor=gray!10,
    linewidth=0pt,
    leftmargin=-1mm,
    rightmargin=-1mm,
    innerleftmargin=2mm,
    innerrightmargin=2mm,
    innertopmargin=1mm,
innerbottommargin=1mm]
\begin{theorem} \label{theorem:func-hess-cols}
For a deep linear network, the rank of $l$-th column-block, $\HFhatsup{\bullet l}$, of the matrix $\HFhat$, under the assumption \ref{assump:2} is given as
 $\rank(\HFhatsup{\bullet l}) = q\, M_{l -1} + q\,M_{l}  - q^2\,,\,$ for $l \in [2, \cdots , L-1]$.
When $l=1$, we have $
\rank(\HFhatsup{\bullet 1}) =q\, M_{1} + q \, s - q^2\,.$ And, when $l=L$, we have $
\rank(\HFhatsup{\bullet L}) =q\, M_{L-1} + q \, s - q^2\,.$
Here, $q := \min(r, M_1, \cdots, M_{L-1}, K, s)$ and $s:=\rank(\Om)=\rank(\E[\pmb \delta_{\x,\y}\,\x^\top])$.
\end{theorem}
\end{mdframed}

The upper bound on the rank of $\HF$ follows by combining the above result over all the columns:
\begin{corollary} \label{corollary:functional-rank}
Under the setup of Theorem~\ref{theorem:func-hess-cols}, the rank of $\HF$ can be upper bounded as,
$\qquad 
\qquad\qquad\qquad\qquad\rank(\HF) \leq 2\, q\, M  + 2 \,q\, s - L\, q^2\,, \quad \text{where} \quad M=\sum_{i=1}^{L-1} M_i\,.$

\end{corollary}

Although the above corollary states an upper bound, empirically we find that this is precisely the formula at initialization, and thus we have the \textit{tightest upper bound on the rank of the functional Hessian}. Also, we find that usually at initialization $s=\min(r, K)$, although this is not needed for the proof. \looseness=-1

\textbf{Block-column independence.$\quad$} A surprising element of the above analysis is that just adding the ranks of the block-columns, corresponding to the respective layers, gives the rank of the entire $\HF$ which is tight. 
This phenomenon is quite straightforward to see in a 2-layer network and there Corollary~\ref{corollary:functional-rank} is an equality. However, the more interesting observation is that this holds even for arbitrary depth and also extends to networks \textit{with non-linearities} such as ReLU and Tanh. This implies that the column spaces associated with the layerwise block-columns do not overlap. Thus, \textit{$\HF$ is similar to a block diagonal matrix}, and it should be possible to uncover the similarity transformation. But this is beyond the current scope, and we leave it as an open question. 

\vspace{-2mm}
\subsection{Overall bound on the Hessian Rank}
Finally, in order to get an upper bound on the rank of the entire Hessian, we simply use $\rank(\Am + \Bm) \leq \rank(\Am) + \rank(\Bm)$, along with the Corollary~\ref{corollary:outer-rank} and Corollary~\ref{corollary:functional-rank}, to obtain:
\[
    \rank(\HL) \leq \rank(\HO) + \rank(\HF) \,\leq\, 2\, q\, M  - L\, q^2  + 2 \,q\, s + q\,( r + K - q)\,.
\]

\begin{faact}\label{eq:rank-formula}
The following equality holds empirically:
    $\, \colorboxed{gray!50}{\rank(\HL) = 2q \, M - L\, q^2  + q(r + K)}\,.$
\end{faact}
This implies that our upper bound is off by just a constant additive factor of $2\,q\,s - q^2$, which is another startling finding. Thus, revealing that the \textit{intersection of the $\HO,\,\HF$ column spaces has a rather small dimension}. E.g., take the typical case of no bottlenecks, $q=s$, then our upper bound exceeds the true rank by a small constant $q^2$, i.e., minimum dimension squared. This suggests that the direct sum of their column spaces is not too far-fetched as an approximation to the column space of $\HL$.
Previously,~\cite{pmlr-v70-pennington17a} empirically noted a similar observation for 1-hidden layer networks and ~\cite{jacot2019asymptotic} showed a high degree of mutual orthogonality in the asymptotic regime. Our result shows a similar consequence in the finite case.\looseness=-1

\textbf{Alternative interpretation.$\quad$} Besides the above result, that the rank of the Hessian is proportional to the sum of widths, there is an alternate way of viewing this. Let us calculate the rank deficiency in the network when (uncentered) input-covariance has rank, i.e., $r=d$. Using Fact.~\eqref{eq:rank-formula} this comes out to 
\[\rankdef(\HL) = \sum\limits_{i=0}^{L-1} (M_i \,- \,q)(M_{i+1}\,-\,q)\,,\]
whereas the number of parameters is equal to $p=\sum_{i=0}^{L-1} M_i \, M_{i+1}$. Hence this lends an elegant interpretation to our formula, whereby the amount of rank deficiency is equal to the number of parameters of a hypothetical network whose all layer-widths have been subtracted by the minimum layer-width of the original network.

\section{Empirical Results}
\subsection{Verification of Rank formulas and their behaviour}\label{sec:verifyBehaviour}
\begin{figure}[!h]
    \centering
    \begin{subfigure}[t]{0.32\textwidth}
        \includegraphics[width=\textwidth]{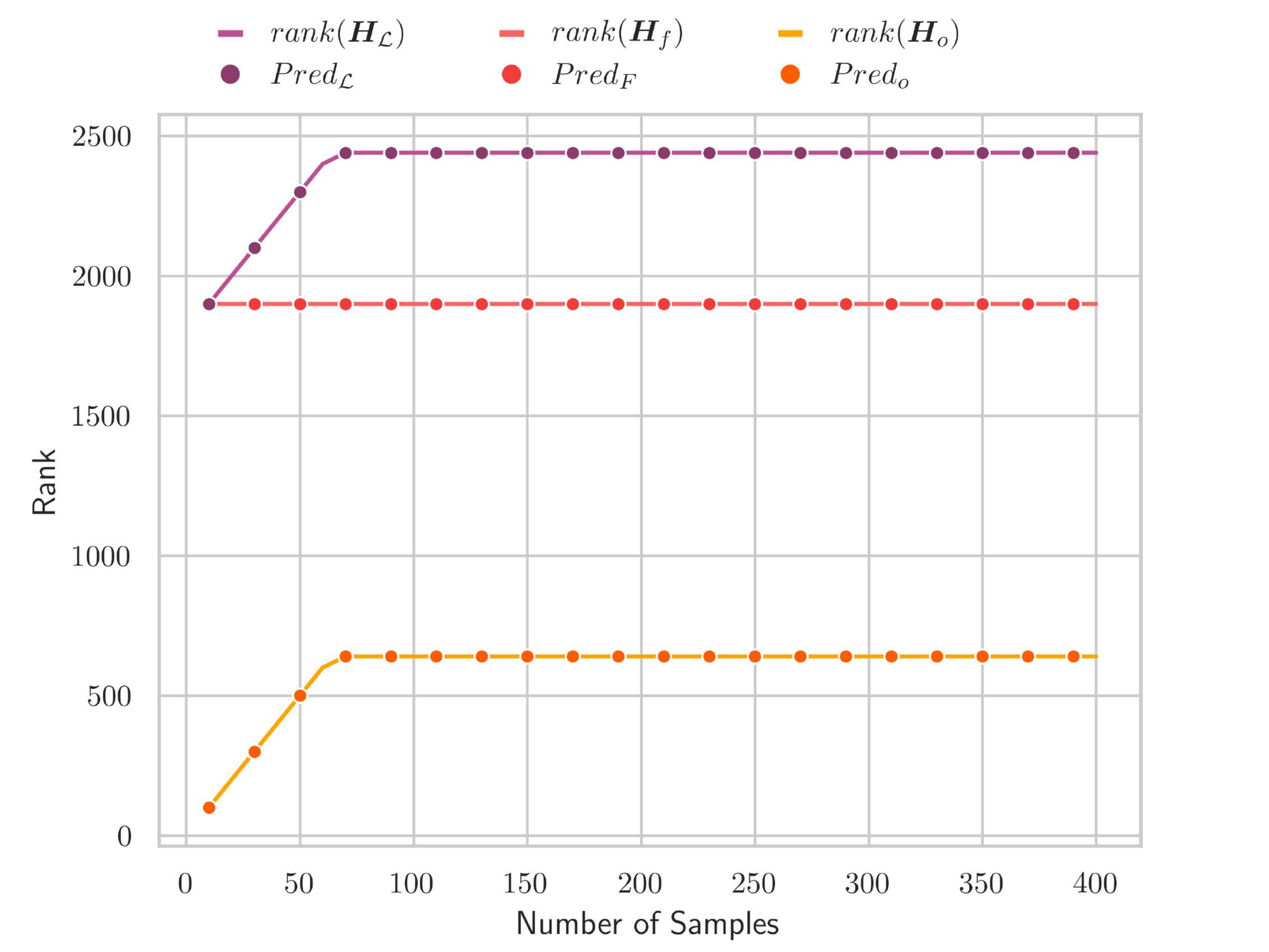}
        \caption{Rank vs \# samples $N$}
        \label{fig:cifar10-a1}
    \end{subfigure}
    \begin{subfigure}[t]{0.32\textwidth}
        \includegraphics[width=\textwidth]{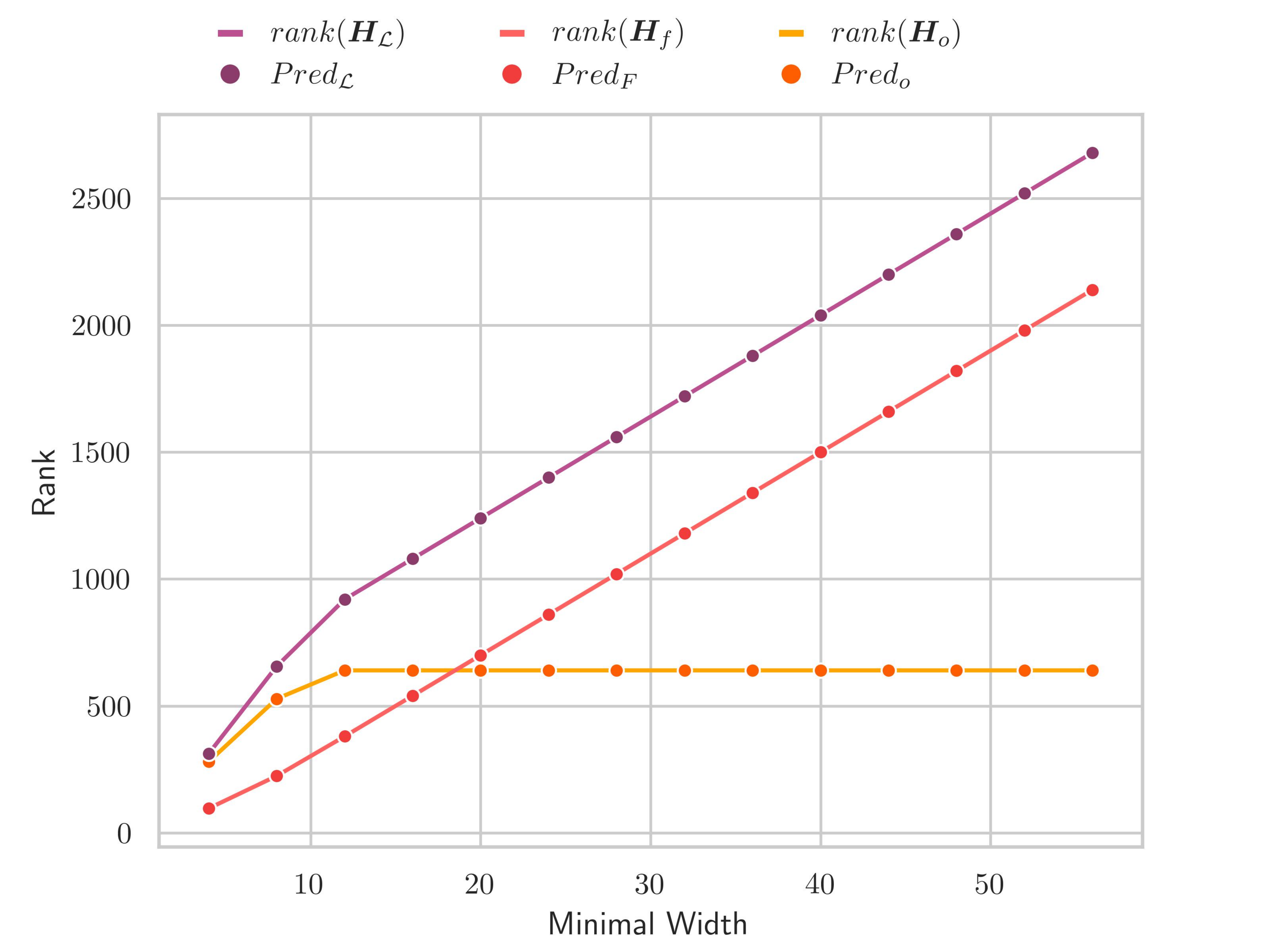}
        \caption{Rank vs minimal width $M_{*}$}
        \label{fig:cifar10-a2}
    \end{subfigure}
        \begin{subfigure}[t]{0.32\textwidth}
        \includegraphics[width=\textwidth]{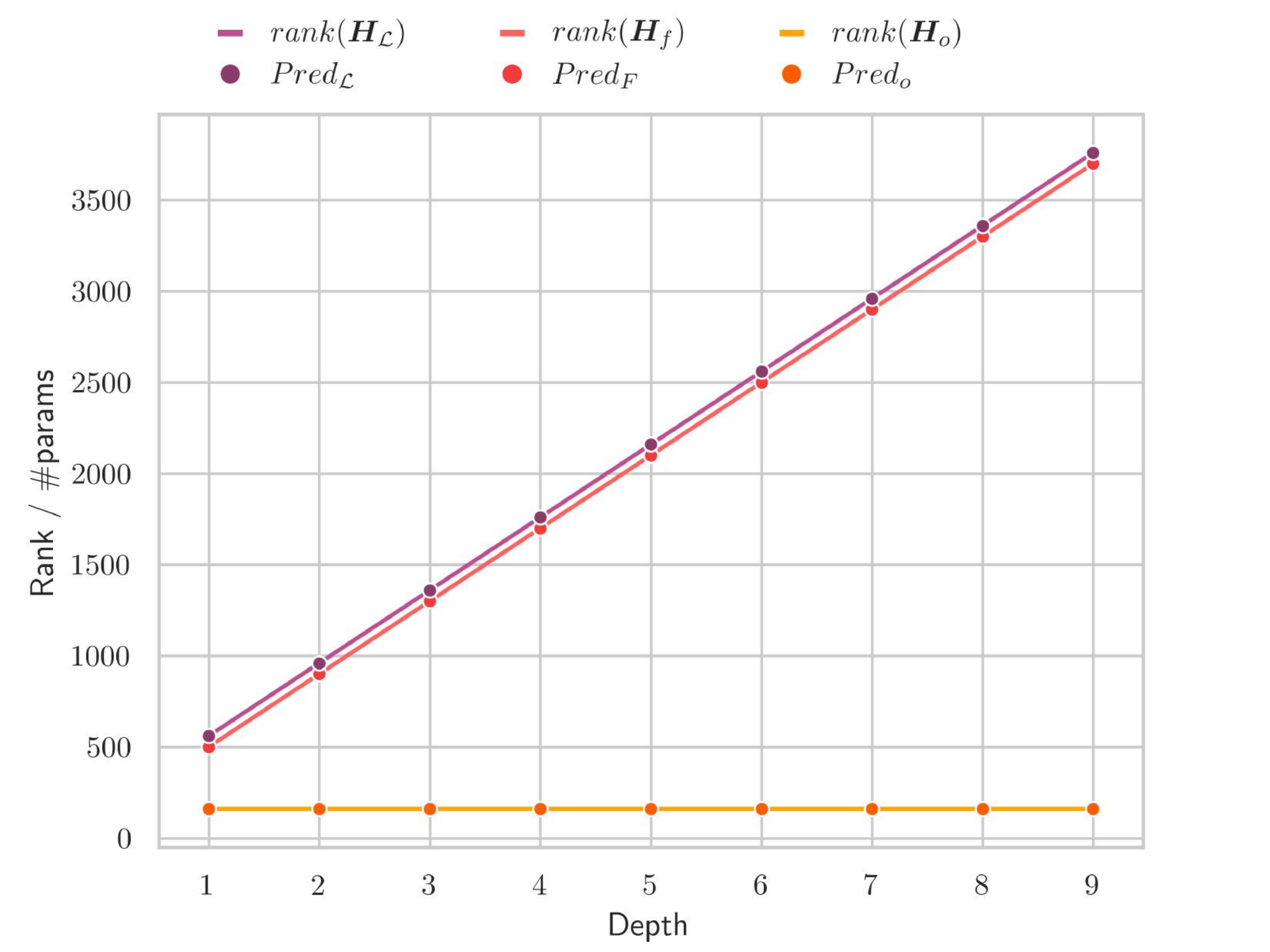}
        \caption{Rank vs depth $L$}
        \label{fig:cifar10-a3}
    \end{subfigure}
    \caption{\small{Behaviour of \textbf{rank} on CIFAR10 using MSE, with hidden layers: $50, 20, 20, 20$ (Fig.~\ref{fig:cifar10-a1}), $M_{*}, M_{*}$ (Fig.~\ref{fig:cifar10-a2}) and $L$ layers of width $M=25$ (Fig.~\ref{fig:cifar10-a3}). The lines indicate the true value and circles denote our formula predictions. }}
    \label{fig:cifar10_rank_behaviour_absolute}
\end{figure}

\textbf{Setup.$\quad$} We test our results on a \textit{variety of datasets}: MNIST \citep{lecun-mnisthandwrittendigit-2010}, FashionMNIST \citep{xiao2017fashionmnist}, CIFAR10 \citep{cifar10}; for \textit{various loss types}: MSE, cross-entropy, cosh; across \textit{several initialization schemes}: Glorot \citep{pmlr-v9-glorot10a}, uniform, orthogonal \citep{saxe2014exact}. Our theory extends to all these settings. However, due to space constraints we only show a subset of experiments here, but the rest can be found in the Appendix~\ref{supp:empirical}.

\textbf{Procedure.$\quad$} To verify the prediction of our theoretical results, we perform an exact calculation of the rank by computing the full Hessian and the corresponding singular value decomposition (SVD). The available iterative schemes for rank approximation (c.f. \citep{fastrank}), although more memory-efficient, are only effective for well-separated spectra and thus cannot provide useful approximations in the case of neural network Hessians. 
Besides the exact Hessian computation, we also utilize \textsc{Float-64} precision to ensure `true' rank calculation, resulting in increased memory costs. Hence, we downscale the image resolution to $d=64$ to test on more realistic networks.

\textbf{Results.$\quad$}We study how the rank varies as a function of the sample size $N$ and the network architecture (for varying widths). Fig.~\ref{fig:cifar10_rank_behaviour_absolute} shows this for a linear network on CIFAR10 with MSE loss. First, in Fig~\ref{fig:cifar10-a1}, we observe that our predictions match the true rank exactly across all sample sizes as the dependence on $N$ is only exhibited in the rank of the empirical covariance $\samplecovx$, confirming that rank is largely a distribution-independent quantity. So, for the rest of our experiments, we sufficiently subsample to ensure that the empirical and true covariance have the same rank. 
Next, in Fig.~\ref{fig:cifar10-a2} and~\ref{fig:cifar10-a3}, we see that \textit{our rank formulas hold for arbitrary-sized network architectures}, across varying width and depth. 

\begin{figure}[!h]
    \centering
    \begin{subfigure}[t]{0.35\textwidth}
    \includegraphics[width=\textwidth]{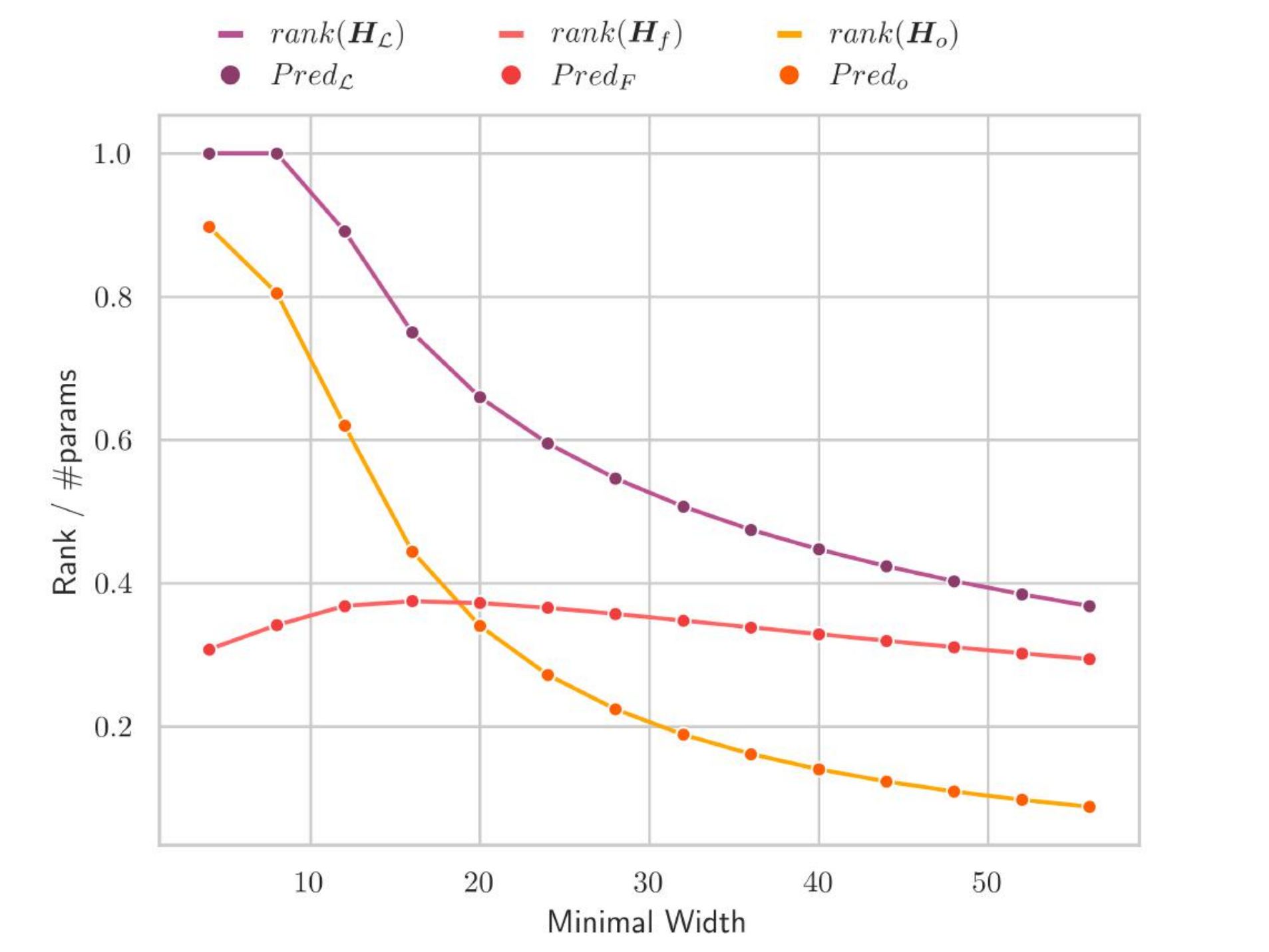}
    \caption{$\effparams$ vs minimal width $M_{*}$}
    \label{fig:cifar10-a4}
    \end{subfigure}
    \begin{subfigure}[t]{0.35\textwidth}
    \includegraphics[width=\textwidth]{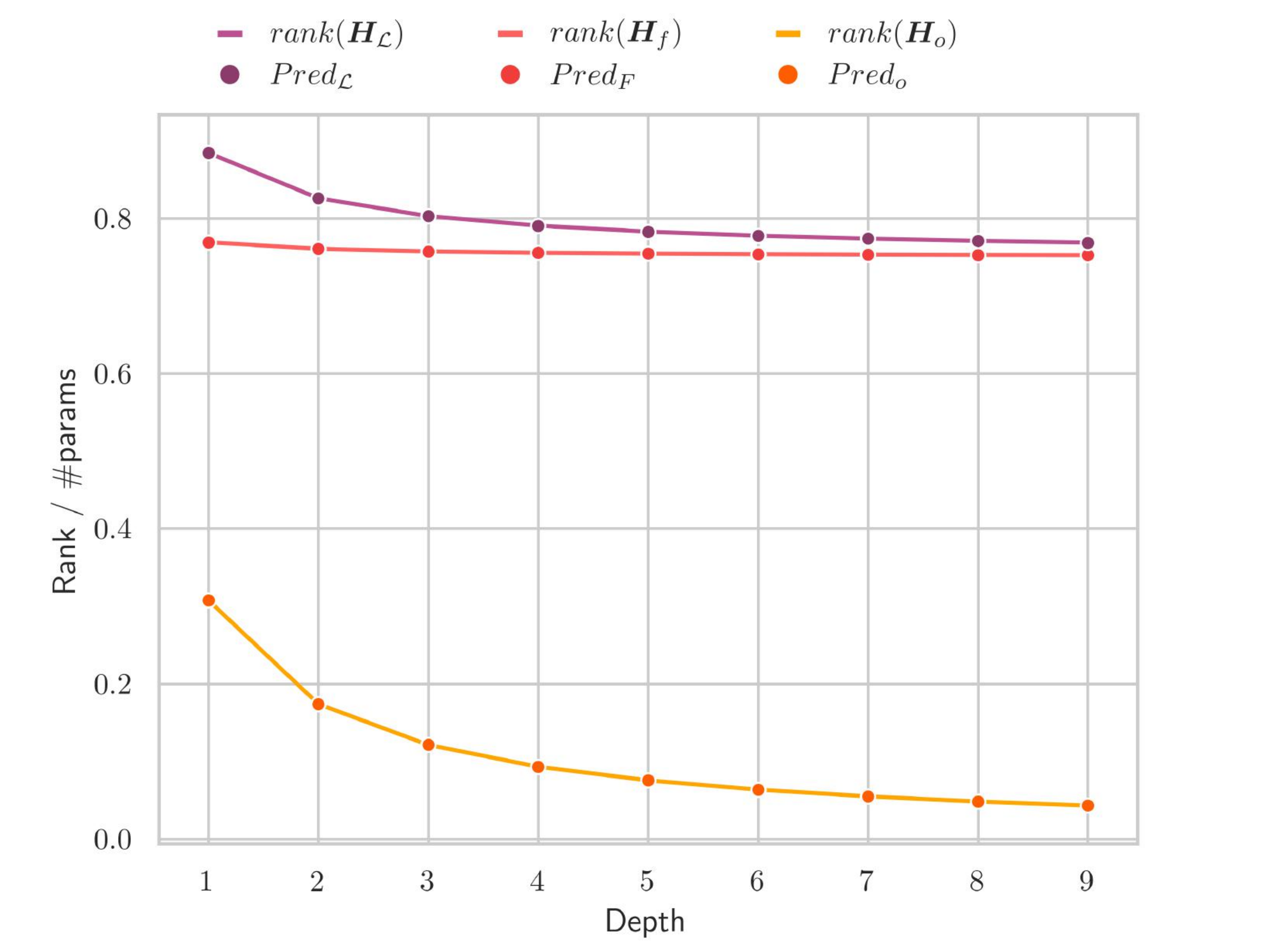}
    \caption{$\effparams$ vs depth $L$}
    \label{fig:cifar10-a5}
    \end{subfigure}
    
    \caption{\small{Behaviour of \textbf{rank/\#params} on CIFAR10 using MSE, with hidden layers: $M_{*}, M_{*}$ (Fig.~\ref{fig:cifar10-a4}) and $L$ layers of width $M=25$ (Fig.~\ref{fig:cifar10-a5}). The lines indicate the true value and circles denote our formula predictions. }}
    \label{fig:cifar10_rank_behaviour_normalized}
\end{figure}

To contextualize the growth of rank with increasing architecture sizes, in Fig.~\ref{fig:cifar10_rank_behaviour_normalized}, we normalize it by \# parameters $p$. We notice that, rank/\# params, which intuitively captures the fraction of effective parameters, saturates down to a small level --- \textit{signalling the extent of redundancy present in the network parameterization}.

\subsection{Simulation of rank formulas for large settings}\label{supp:simulation}
In the previous subsection, we have established how our formulas hold exactly in practice. An added benefit of these formulas is that they allow us to visualize how the rank will growth in relation to the number of parameters for bigger architectures --- \textit{without actually having to do the Hessian computations. }In Fig.~\ref{formulas} we show such a simulation for increasing width and depth.  The simulations make the limiting behaviour of the fraction $\frac{\text{rank}}{\# params}$ even more apparent, as the fraction decreases with more and more overparametrization (in terms of both depth and width), until it reaches a threshold. 

The left subfigure, which is the width-simulation plot, also shows an interesting behaviour. In the early phase, the outer-product Hessian $\HO$ dominates the functional Hessian $\HF$ in terms of $\frac{\text{rank}}{\# params}$, but after a certain width $\HF$ starts to dominate $\HO$ and continues to do so throughout. It would be of relevance for future work to further investigate the interaction between $\HO$ and $\HF$, and provide an understanding of these two phases.

\begin{figure}[!htb]
    \centering
    \includegraphics[width=0.4\textwidth]{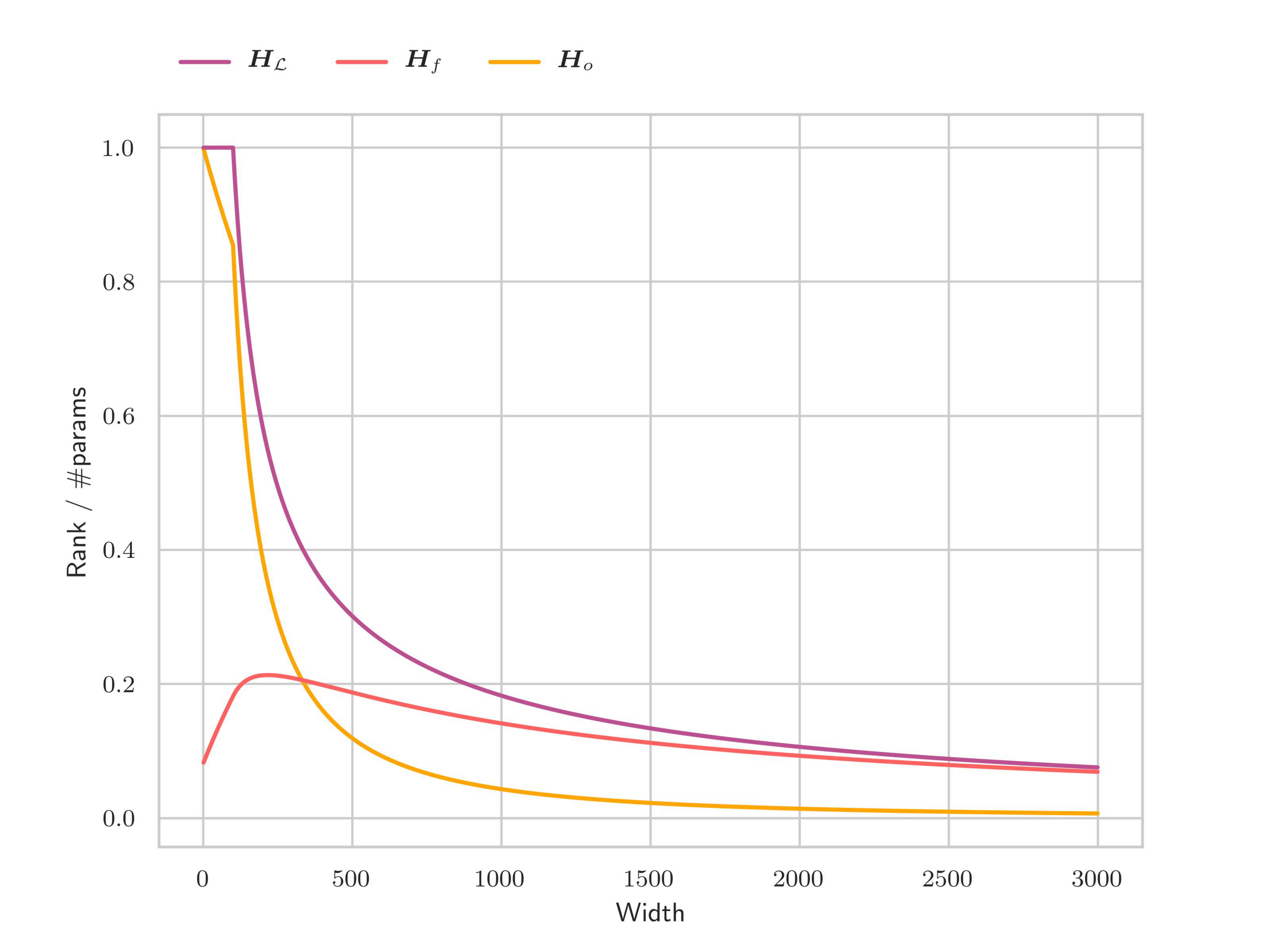}%
    \includegraphics[width=0.4\textwidth]{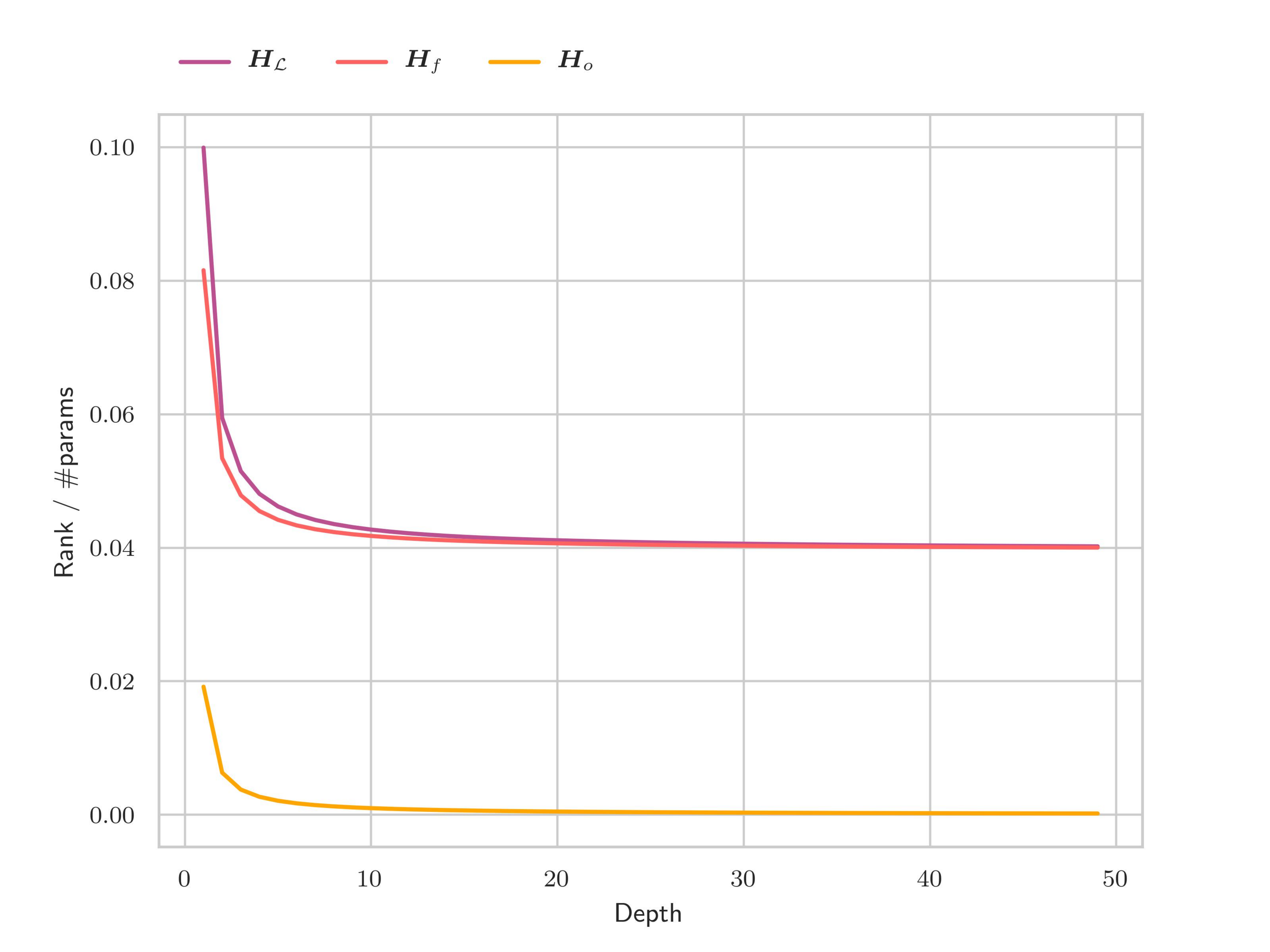}
    
    \caption{\small{Simulating \textbf{rank/\#params} over width (left) and depth (middle). For the width plot: $L=4, r=2352$, and for the depth plot: $r=2352$ and hidden layer size $5000$. }}
    \label{formulas}
\end{figure}

\subsection{The case of non-linearities}\label{sec:empirical-nonlinear}

Although the linear nature of the neural network was crucial to our analysis, in this section we show experimentally that our results also extend to the non-linear setting --- as numerical rank~\cite{numericalrank}. 

\begin{figure}[!htb]
    \centering
    \vspace{-1em}
    \includegraphics[trim={0.1cm 0 0 0.5cm},clip,width=0.3\textwidth]{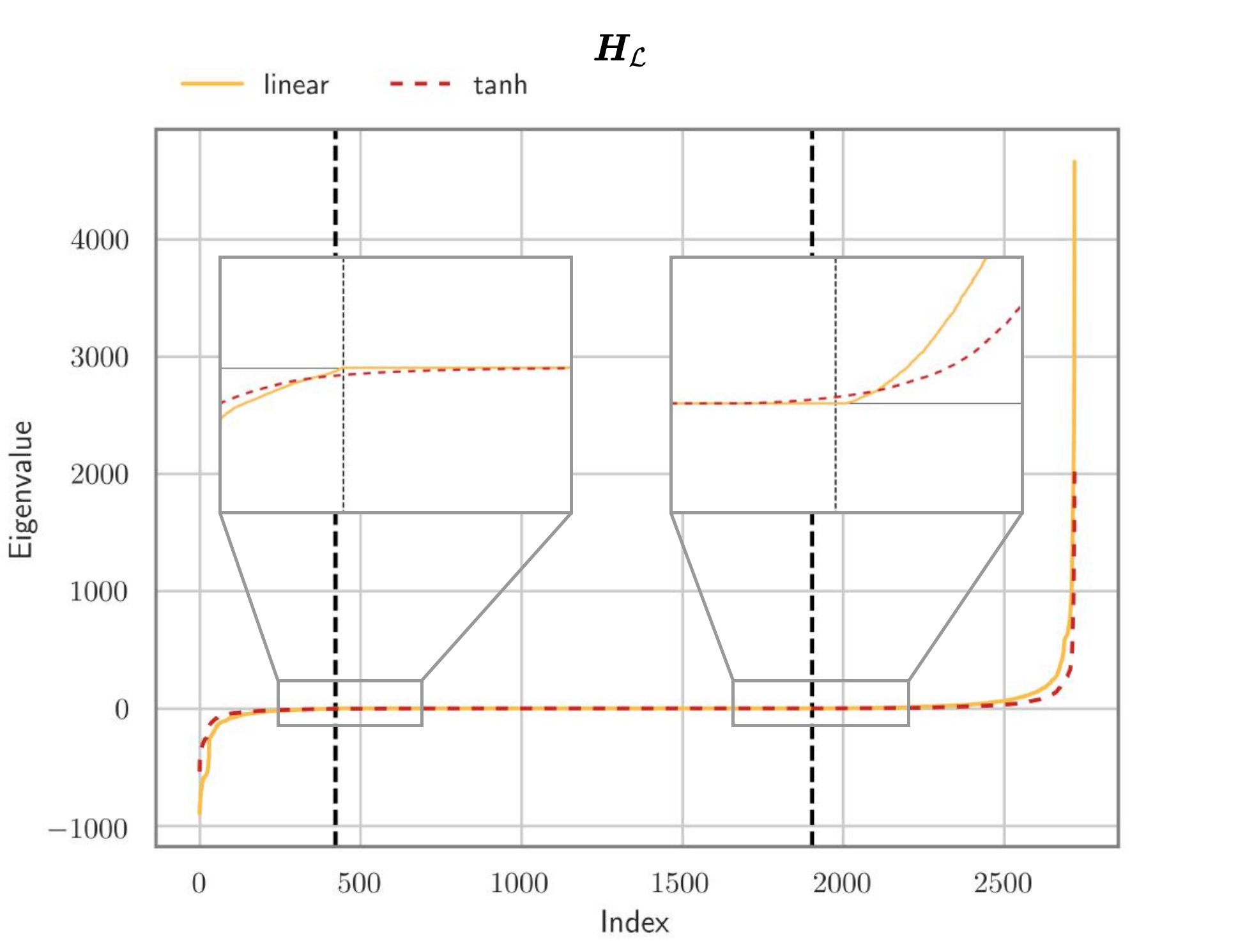}%
    \includegraphics[trim={0.1cm 0 0 0.5cm},clip,width=0.3\textwidth]{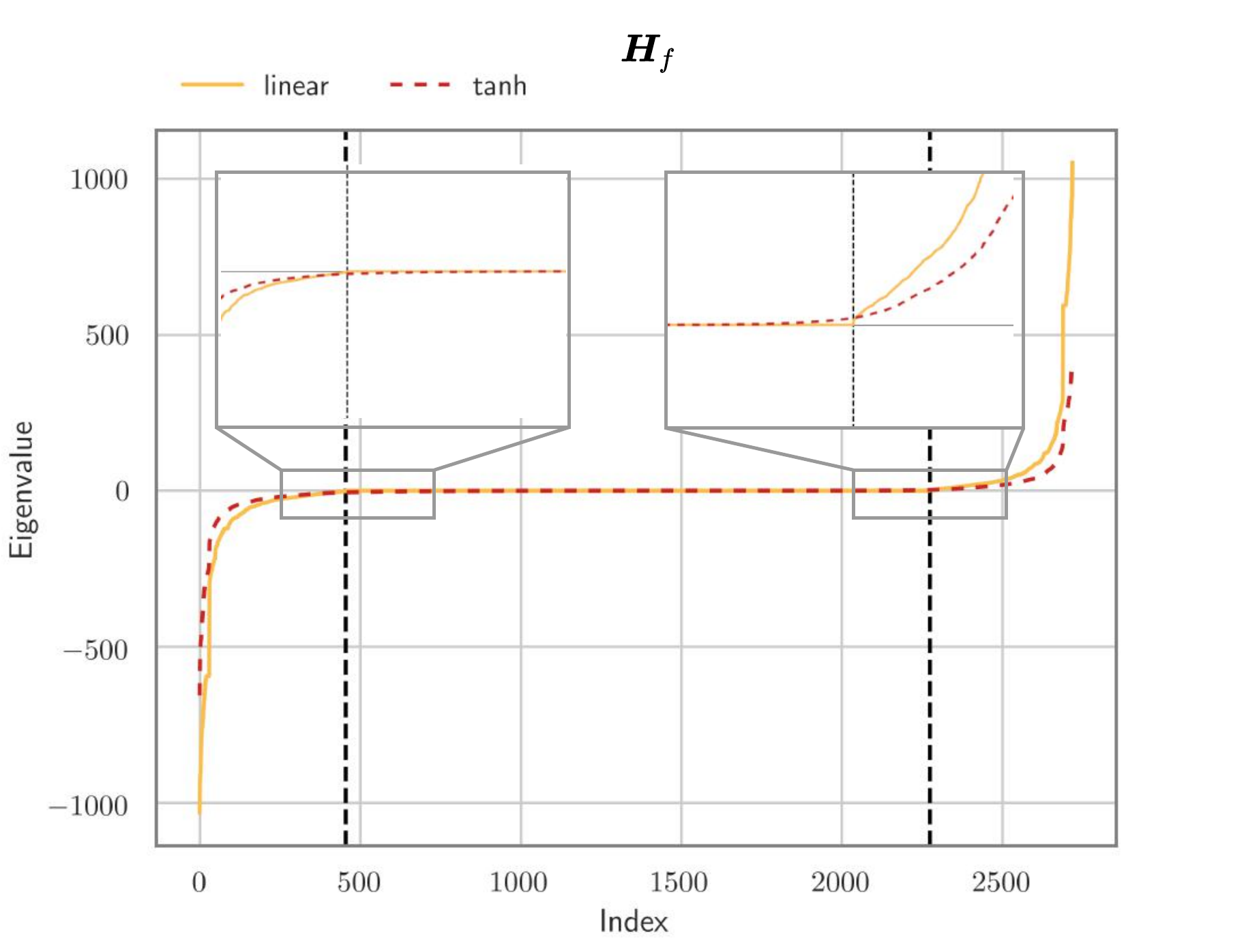}%
    \includegraphics[trim={0.25cm 0 0 0.5cm},clip,width=0.3\textwidth]{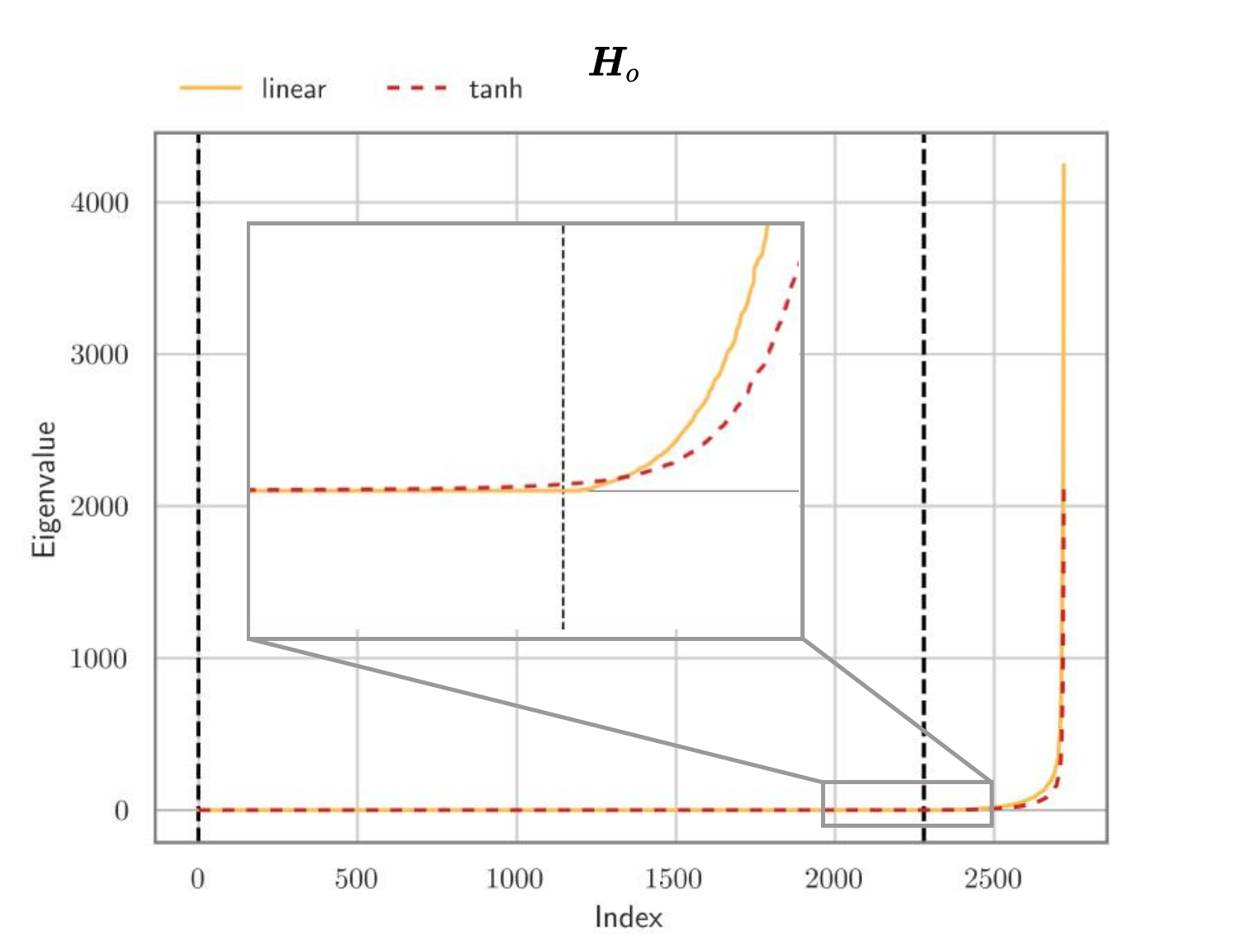} 
    \caption{\small{Spectrum of the loss Hessian $\HL$ (left), functional Hessian $\HF$ (middle)  and outer product $\HO$ (right), for \textcolor{orange}{\textbf{linear}} and \textcolor{red}{\textbf{non-linear}} networks. Black dashed lines are the predictions of the bulk size via our rank formulas. We use 2 hidden layers of size $30,20$ with Tanh activation on MNIST.}}
    \label{bulk_size}
\end{figure}

\textbf{Visual comparison of the Hessian spectra.$\quad$}
Let us first understand how non-linearities affect the Hessian spectrum relative to the spectrum of linear networks. Fig.~\ref{bulk_size} compares the spectra of $\HL$, $\HF$, $\HO$ in these two scenarios (linear vs Tanh), with a zoomed-in inset near the cut-off obtained from rank formulas corresponding to the linear case. We can observe the presence of numerous tiny, but not exactly zero, eigenvalues in the non-linear case. So, if we were to measure the rank with a threshold up to machine precision --- as we did in the linear case --- this would result in an inflation of the rank measurement for the non-linear scenario. From a practical point of view, tiny but non-zero eigenvalues hold little significance, so a more relevant quantity is the numerical rank~\cite{numericalrank} that uses a reasonable threshold to weed out such extraneous eigenvalues. The numerical rank, or alternatively the size of bulk around zero, for the non-linear case, indeed seems to be captured by our (linear) rank formulas to high fidelity, as evident from Fig.~\ref{bulk_size}. Similar results for other non-linearities, loss functions,  datasets can be found in the Appendix~\ref{spectral_app}.

\begin{figure}[!htb]
    \centering
    \includegraphics[width=0.3\textwidth]{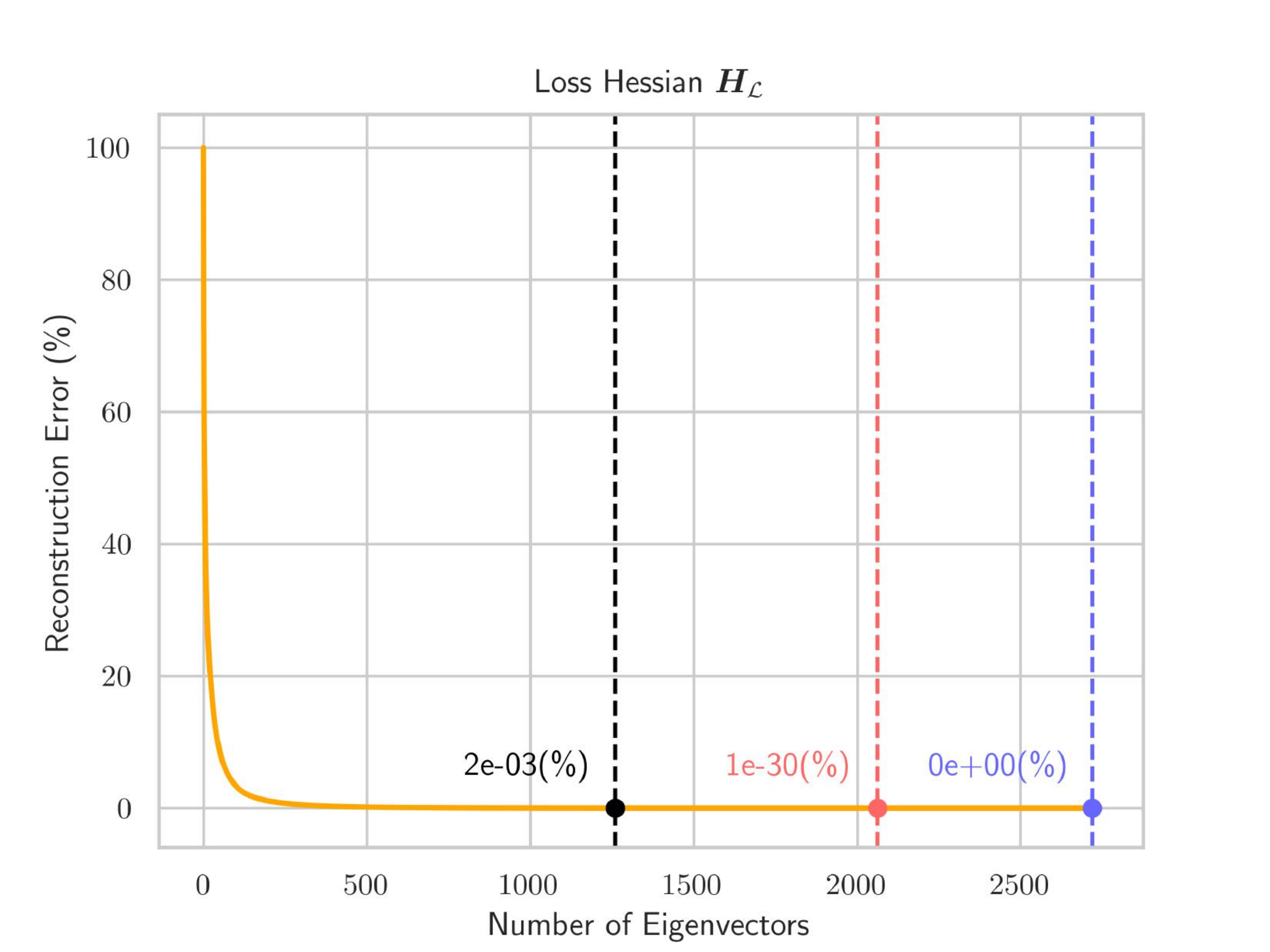}%
    \includegraphics[width=0.3\textwidth]{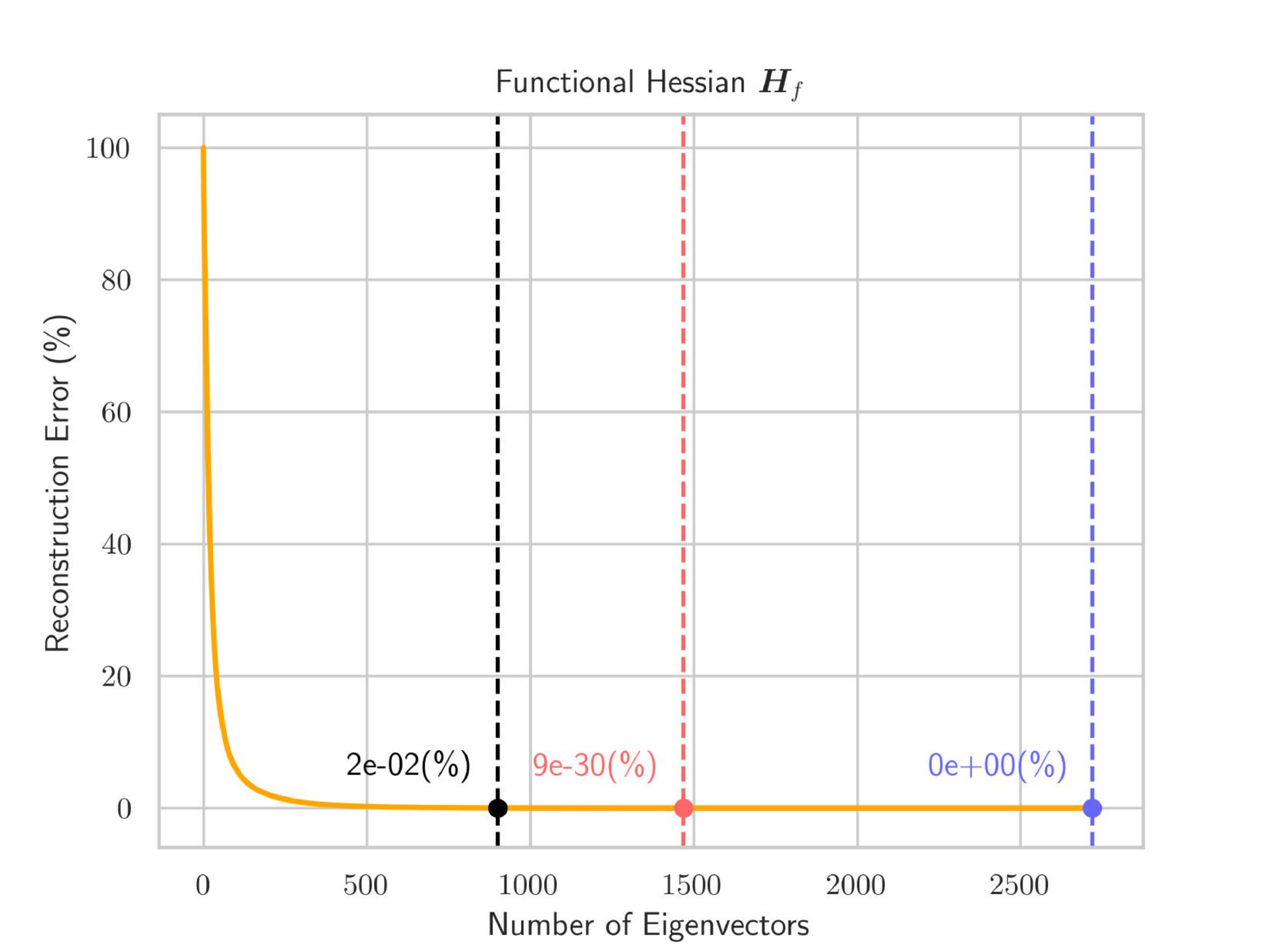}%
    \includegraphics[width=0.3\textwidth]{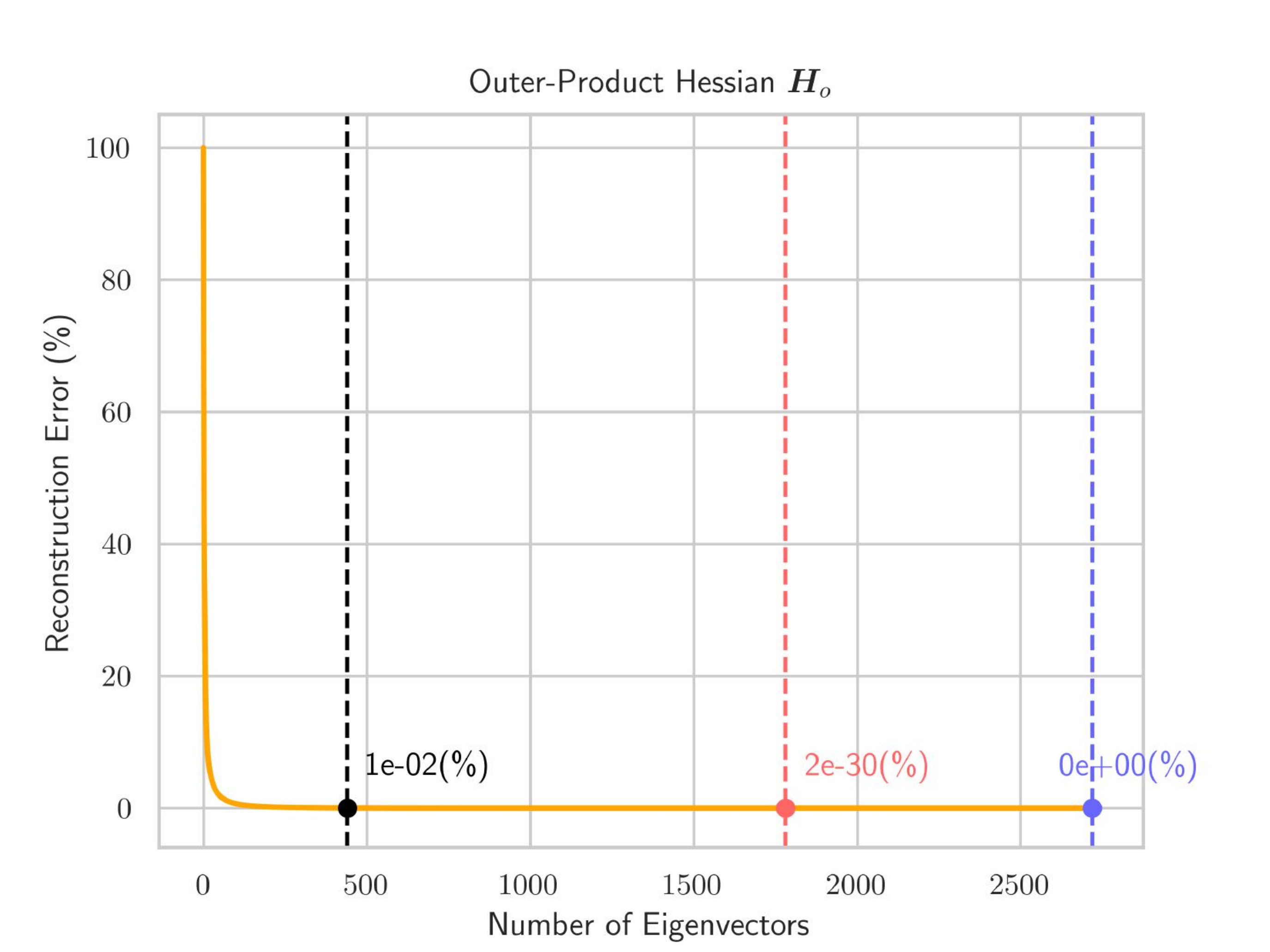}
    \caption{Hessian reconstruction error for a \textbf{ReLU} network of hidden layer sizes $30, 20$ as the rank of the approximation is increased. As before, the three sub-figures reflect this for the loss Hessian $\HL$ (left), functional Hessian $\HF$ (middle),  and outer product $\HO$ (right) respectively. The x-axis represents the number of top eigenvectors that form the low-rank approximation. 
    The dashed vertical lines indicate the cut-off at various values of the rank: \textbf{first line} at the prediction based on the linear model, \textcolor{red}{\textbf{second line}} at the empirical measurement of rank, and \textcolor{blue}{\textbf{third line}} based on upper bounds from \citep{jacot2019asymptotic}, which become too coarse to be of any use (actually even $>$ \# parameters). 
    }
    \label{fig:relu_reconstruction}
\end{figure}

\textbf{Quantitative measure of the fidelity of Rank formulas.$\quad$}
To complement the visual grounds presented above, we now quantitatively measure the (in)significance of such spuriously tiny eigenvalues, in terms of the reconstruction error incurred by excluding them. 
Hence, we perform a low-rank approximation via the SVD and measure the (relative) reconstruction error at the value of rank from our formulas of the linear case. Fig.~\ref{fig:relu_reconstruction} displays the reconstruction error as a function of the \# of top eigenvectors employed, for a ReLU network. As a reference, we consider the empirical rank measurement obtained at machine precision in this case.

\textit{We find that using the linear rank value provides an excellent reconstruction ($0.002\%$ error in case of $\HL$), hence demonstrating the fidelity of our rank formulas to serve as a measure of numerical rank in the non-linear case.} The same observation extends to other non-linearities and losses, which we highlight in the Appendix~\ref{recon_error_app}. %\\[3mm]
Consequently, iterative Hessian estimation procedures, e.g. in second-order optimization methods~\cite{botev2017practical,lacotte2021adaptive}, could benefit from the linear rank as a guiding criterion for their design of Hessian approximation. Besides, these experiments also indicate that previous bounds, such as those by \citet{jacot2019asymptotic}, on the rank of outer-product Hessian and functional Hessian are quite coarse to be of much use. This is because these bounds have a linear dependence on the product of: \# of samples $N$ and \# of classes $K$. %, for e.g., $KN$ for $\HO$ and $2K\,MN$ for $\HF$. 
 For the same reason, other works~\cite{shen2016mathematical,botev2017practical} that bound the rank of the outer-product Hessian $\HO$, trivially, based on the \# of samples $N$ are of little use.

\vspace{-3mm}
\section{Evolution of Rank during training}\label{sec:rank-evolution}

The upper bounds on the Hessian rank detailed before inherently depend on the rank of the weight matrices. While initialization guarantees them to be of maximal rank, the rank of weight matrices might possibly decrease during training, thus bringing about a decrease in the Hessian rank. Under some additional assumptions, Lemma~\ref{lemma:evolve} shows that this does not happen and the rank remains constant.\looseness=-1

\begin{lemma}\label{lemma:evolve}
For a deep linear network, consider the gradient flow dynamics $\dot{\Wm}^{l}_t = -\eta \nabla_{\Wm^{l}}\mathcal{L}_S(\bm{\btheta})\big{|}_{\bm{\btheta} \,=\, \bm{\btheta}_t}$. Assume: (a) Centered classes: $ \frac{1}{N} \sum_{i:y_{ic} = 1}^{N}\x_i= \bm{0}, \hspace{2mm} \forall\, c \in [1,\dots,K]$. (b) Balancedness at initialization: $\Wm^{{k+1}^\top}_0\Wm_0^{k+1} = \Wm^{k}_0\Wm^{k^\top}_0$. (c) Square weight-matrices: $\Wm^{l} \in \mathbb{R}^{M \times M}\,,$ $\,\forall \,l\,$ and $\,K=d=M$. Then for all layers $l$,
$\,\, \rank(\Wm^{l}_t) = \rank(\Wm^{l}_0), \hspace{3mm} \,\forall \, t<\infty\,.$
\end{lemma}
\label{dynamics_of_rank}

Balancedness is a common assumption that has been used in many previous works, like \citet{arora2019convergence}. Centered classes can easily be enforced via a pre-processing step, however empirically this is not required. While the proof (see Appendix~\ref{supp:evol-proof}) holds for square case, empirically we also find this to be true for non-square matrices and non-linearities as shown in Fig.~\ref{fig:weight-ranks} (and more in Appendix~\ref{supp:evol-proof}). 

\begin{figure}[h]
    \centering
    \includegraphics[trim={0.5cm 0 0 0.5cm},clip,width=0.37\textwidth]{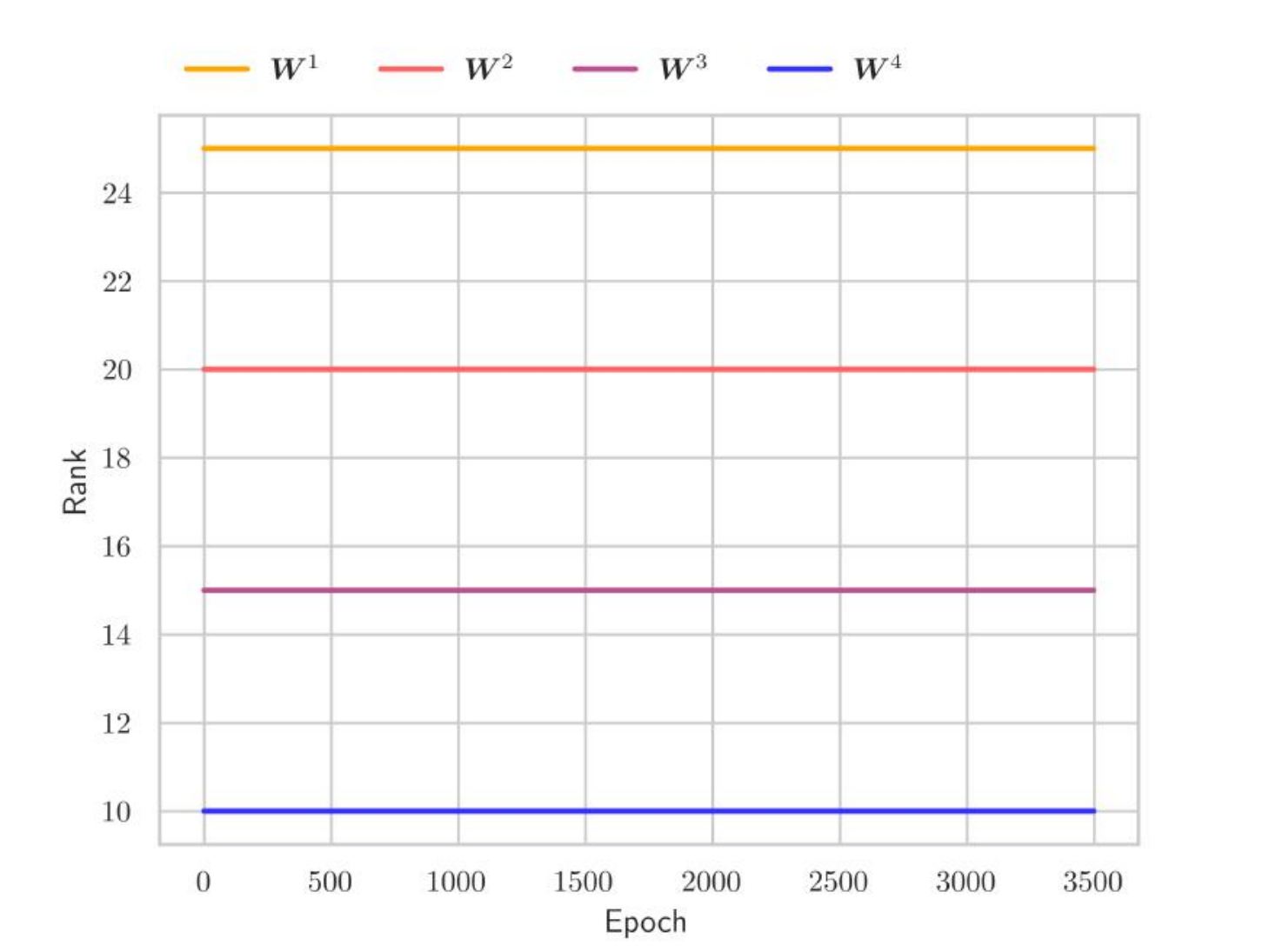}%
    \includegraphics[trim={0.5cm 0 0 0.5cm},clip,width=0.37\textwidth]{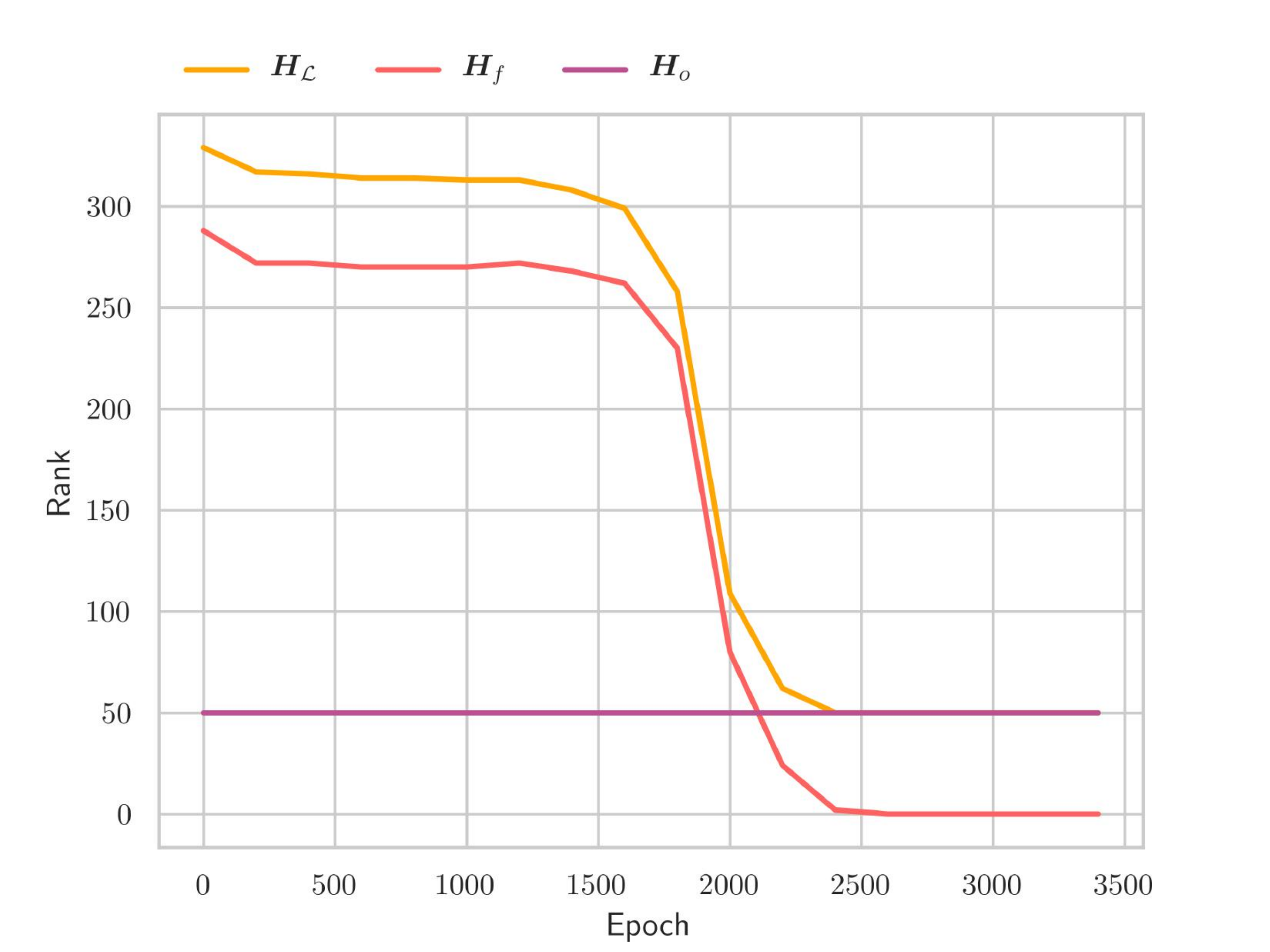}
    \caption{\small{Rank dynamics for a ReLU network with hidden layer sizes $25, 20, 15$, trained on a subset of MNIST. We show the evolution of the rank of the weights (left) with usual Glorot initialization (guaranteeing maximal rank) and the rank of the Hessians as a function of training time (right).}}
    \label{fig:weight-ranks}
\end{figure}

\paragraph{Consequence.} An implication of this result is that \textit{our upper bounds on the Hessian rank remain valid throughout the training.} Even if the rank of the weight matrices were to decrease, say in the rectangular case, the Hessian rank would only decrease and our bounds would still be valid. The other avenue for a decrease in rank is learning-driven, and as $\Om=\E[\pmb \delta_{\x, \y} \, \x^\top] \to 0$, the functional Hessian $\HF\to \mb 0$ (as it has an explicit dependence on the residual $\pmb \delta_{\x, \y}\,$) . So, if the learning architecture is powerful enough, the rank of $\HF$ will approach $0$, as well as leading to a decrease in the rank of $\HL$, and so $\rank(\HL) \to \rank(\HO)$. In other words, the loss Hessian is completely captured through the outer product Hessian at convergence. These observations also extend to the non-linear case, as shown in Fig.~\ref{fig:weight-ranks}.\looseness=-1

\section{Further results}

\subsection{Provable Hessian degeneracy with non-linearities}\label{sec:relu-bound} 

From Section~\ref{sec:empirical-nonlinear}, it is clear that our upper bounds faithfully estimate the numerical rank in the non-linear case, despite the numerous tiny eigenvalues which cause inflated rank measurements. This might raise the question if it is possible to establish, in this non-linear case, that the Hessian has provable degeneracy and not just approximate? The challenge is that here the data distribution manifests additionally via the activations at each layer, thus a theoretical analysis seems intractable without imposing strong assumptions. For a 1-hidden layer network, $F(\x) = \wM{2}\sigma(\wM{1}\,\x)$, we show it is still possible to get a pessimistic yet non-trivial upper bound 
with the following (mild) assumption:\looseness=-1 

\begin{assump}\label{assump:data}
For each active hidden neuron $i$, the weighted input covariance has the same rank as the overall input covariance, i.e.,  $\rank(\E[\alpha_{\x}\,\x \x^\top]) = \rank(\covx) = r$, with $\,\alpha_{\x} = {\sigma^\prime(\x^\top \, \wM{1}_{i\,\bullet})}^2$. \looseness=-1
\end{assump}
\begin{mdframed}[leftmargin=1mm,
    skipabove=1mm, 
    skipbelow=-1mm, 
    backgroundcolor=gray!10,
    linewidth=0pt,
    leftmargin=-1mm,
    rightmargin=-1mm,
    innerleftmargin=2mm,
    innerrightmargin=2mm,
    innertopmargin=1mm,
innerbottommargin=1mm]
\begin{theorem}\label{thm:relu-rank}
Consider a 1-hidden layer network with non-linearity $\sigma$ such that $\sigma(z) = \sigma^{\prime}(z) z$ and let $\widetilde{M}$ be the \# of active hidden neurons (i.e., probability of activation $>0$). Then, under assumption \ref{assump:2} and \ref{assump:data},  rank of $\HO$ is given as,
$\rank(\HO) \leq r\widetilde{M} \,+\, \widetilde{M} K \,-\, \widetilde{M}\,$.
\end{theorem}
\end{mdframed}

\textit{Assumption~\ref{assump:data} is rather mild}, i.e., in the finite-sample case, it holds as soon as the \# of samples $N> 2 d$, for typical initialization of parameter weights. Besides, the class of non-linearities which satisfy the above-mentioned condition includes e.g., \textit{ReLU, Leaky-ReLU.} 
Further, this result extends to $\HL$: \looseness=-1
\begin{corollary}\label{cor:relu-rank}
At convergence to the minimum, the rank of the loss Hessian $\HL$, for the same setup as Theorem~\ref{thm:relu-rank}, is upper bounded by: $\rank(\HL) \leq r\widetilde{M} \,+\, \widetilde{M} K \,-\, \widetilde{M}\,$.
\vspace{-0.5mm}
\end{corollary}

Contrast this with the view from~\cite{NEURIPS2018_18bb68e2}, who claim the spectrum to be generically non-degenerate. Or, unlike~\cite{poggio2018theory}, we establish this without any assumptions on a particular kind of overparameterization.

\begin{faact}
For multiple hidden-layers, the following generalization of Theorem~\ref{thm:relu-rank} holds empirically, $\rank(\HO) \leq p - M_1 (d-r) - \sum_{i=1}^{L-1} M_i\,$, where $p$ is the \# of parameters and assuming no dead neurons. \looseness=-1
\end{faact}

While these bounds are likely to be quite loose as noticeable from Section~\ref{sec:empirical-nonlinear}, \textit{but more importantly they help establish provable degeneracy of the Hessian at the minimum, with the number of `absolutely-flat' directions (i.e., those in the Hessian null space) in proportion to the sum of hidden-layer sizes.}\looseness=-1

\vspace{-2mm}
\subsection{Effect of bias on the rank of Hessian}\label{sec:bias} 
Now, we see how the Hessian rank changes when bias is enabled throughout a deep linear network.
We make the following simplifying assumption, which is actually a standard convention in practice. \looseness=-1 
\begin{assump}\label{assump:zero-mean}
The input data has zero mean, i.e., $\x \sim p_{\x}$ is such that $\E[\x]=0$. 
\end{assump}

\begin{mdframed}[leftmargin=1mm,
    skipabove=1mm, 
    skipbelow=-1mm, 
    backgroundcolor=gray!10,
    linewidth=0pt,
    leftmargin=-1mm,
    rightmargin=-1mm,
    innerleftmargin=2mm,
    innerrightmargin=2mm,
    innertopmargin=1mm,
innerbottommargin=1mm]
\begin{theorem}\label{theorem:bias}
Under the assumption \ref{assump:2} and \ref{assump:zero-mean},  for a deep linear network with bias, the rank of $\HO$ is upper bounded as,
$\rank(\HO) \leq q(r+K-q) + K\,$, where $q:= \min(r, M_1, \cdots, M_{L-1}, K)$.

\end{theorem}
\end{mdframed}

The proof can be found in the Appendix~\ref{supp:bias}. Empirically, we do not require the input to be mean zero and our upper bound actually holds with equality. Also, we list rank formulas for functional Hessian and the overall loss Hessian in Appendix~\ref{supp:bias}, in the non-bottleneck case. E.g., 

\begin{faact}
$
\rank\left(\HL\right) =
     2q^\prime M \,+\, {q^\prime}(r + K) \,-\, L {q^\prime}^2 \,+\, L{q^\prime} 
\,$, where $q^\prime := \min(r+1, M_1, \cdots, M_{L-1}, K)$.
\end{faact}

Here as well, rank deficiency has a cleaner interpretation of being equal to the \# of parameters in a hypothetical network with bias enabled, albeit with the minimum dimension $q^\prime$ (that reflects the homogeneous coordinate at input) subtracted:
\[\quad\rankdef(\HL) = \sum\limits_{i=0}^{L-1} (M_i + 1 \,- \,q^\prime)(M_{i+1}\,-\,q^\prime)\,.\]

\paragraph{Empirical verification.} We briefly discuss some empirical experiments that verify the accuracy of our Hessian rank formulas in case of bias. Fig.~\ref{fig:bias-a1} illustrates this for network architectures of arbitrary depth using the \textsc{CIFAR10} dataset and MSE loss. Thus, it demonstrates that our rank formulas --- also in the case of bias --- exactly predict the empirically observed Hessian rank. Further results, e.g., across arbitrary width as well as number of samples are located in the Appendix~\ref{sec:app-bias-exp}.

\begin{figure}[!htb]
    \centering
    \begin{subfigure}[t]{0.4\textwidth}
    \includegraphics[width=\textwidth]{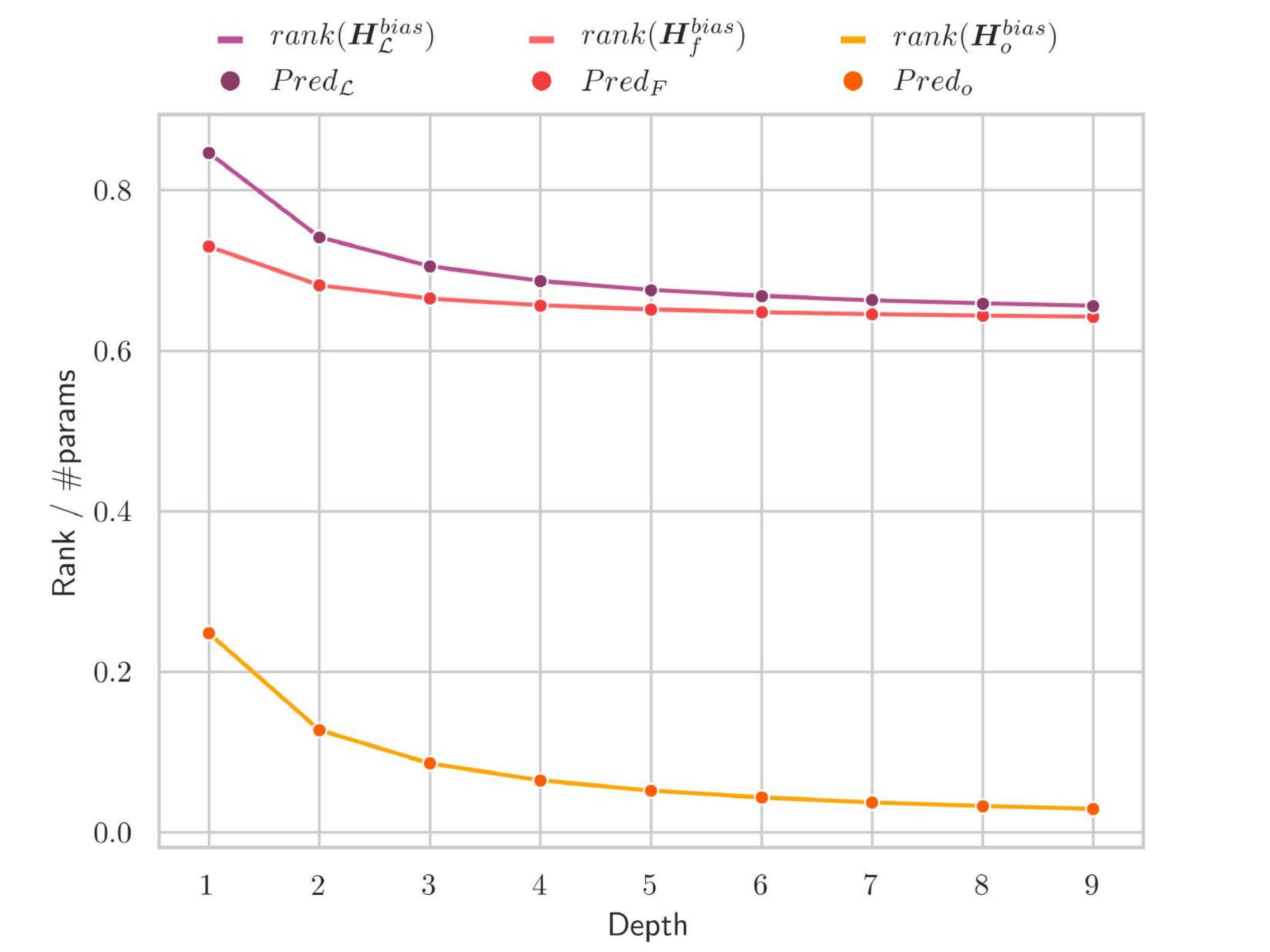}
    \caption{$\effparams$ vs depth $L$}
    \label{fig:bias-a1}
    \end{subfigure}
    \begin{subfigure}[t]{0.4\textwidth}
    \includegraphics[width=\textwidth]{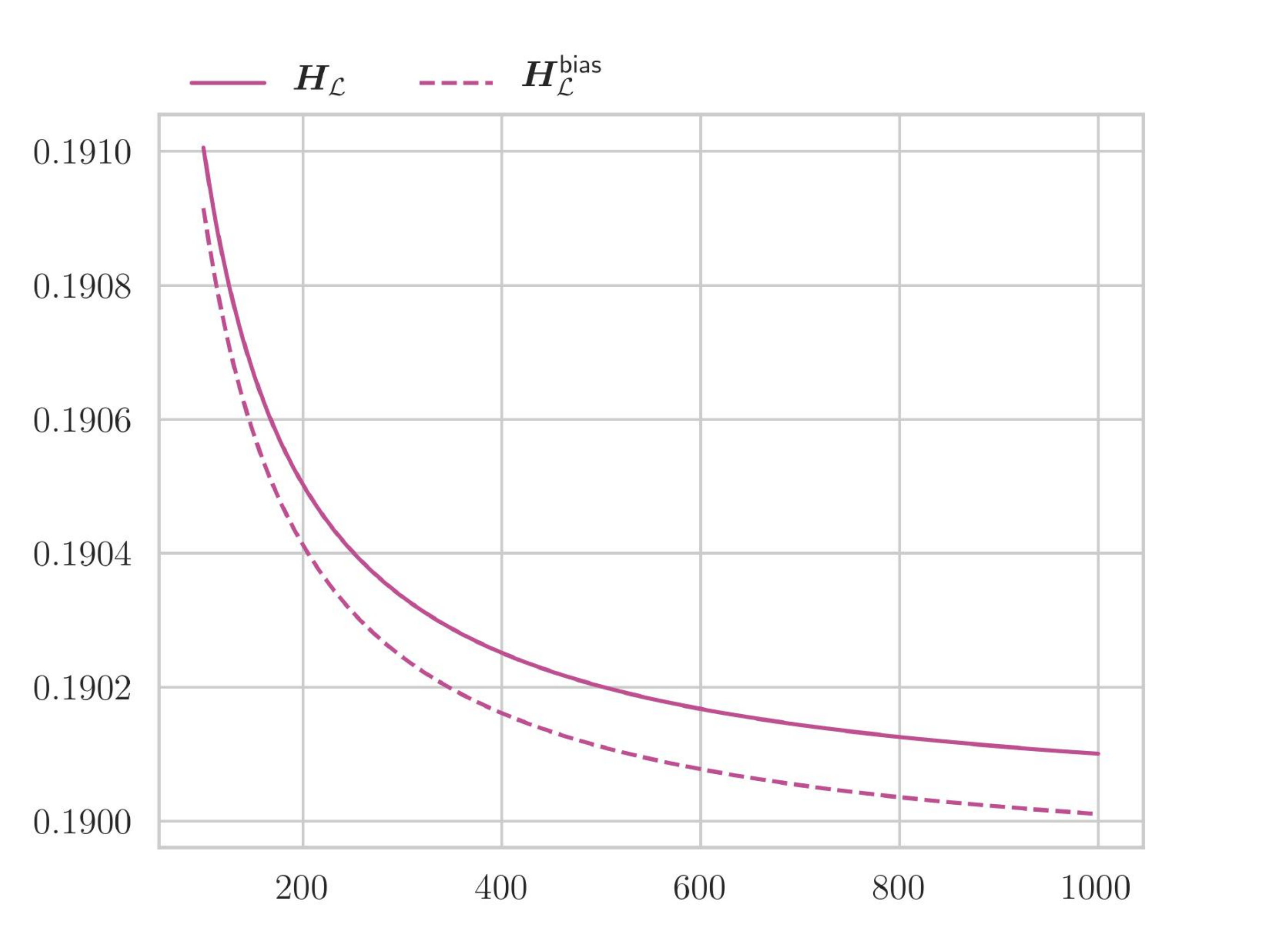}
    \caption{Effect of bias on $\effparams$ for varying depth}
    \label{fig:bias-a2}
    \end{subfigure}
    \caption{The fraction of effective parameters, i.e., \textbf{rank/\#params} as a function of depth $L$ for MSE loss on down-scaled \textsc{CIFAR10}, \textit{using bias}. For the left Fig.~\ref{fig:bias-a1}, we use network architectures with $L$ hidden layers each of width $M=25$. In the case of Fig.~\ref{fig:bias-a2} on the right, we simulating this ratio over depth for $\HL, \HLsup{\text{bias}}$ --- in other words, with bias disabled and bias enabled. For this, we use $M=1000$, $d=r=784$ and $K=100$.}
    \label{fig:bias_formulas}
\end{figure}

Next, in Fig.~\ref{fig:bias-a2} we showcase the resulting effect of enabling bias in the network on the Hessian rank, by simulating the ratio $\frac{\text{rank}}{\# params}$ across increasing depth, for the loss Hessian $\HL$ with and without bias. We find that in both cases the $\frac{\text{rank}}{\# params}$ curve saturates to a small threshold. But interestingly, we see that enabling bias further results in a decrease in this ratio.

\section{Conclusion}

\textbf{Summary.$\quad$} Our paper provides a precise understanding of how the neural network structure constrains the Hessian range and the resulting rank deficiency. In contrast to the number of parameters which are proportional to layer-widths squared, we obtain that rank is proportional to layer-width. The proof strategy relies on bounding the rank of the two parts of the Hessian separately, i.e., the outer-product Hessian $\HO$ and the functional Hessian $\HF$, both of which are replete with the special $\Zm$-like structure. The analysis also reveals several striking properties of the Hessian, such as surprisingly small overlap in the column spaces of $\HO$ and $\HF$, and independence of the layer-wise column blocks in $\HF$.  While our results were derived assuming linear activations, we demonstrate that, even with non-linearities, our formulas faithfully capture the numerical rank.  All in all, our work delivers important insights into the nature and degree of parameterization contained in a neural network as a result of its compositional structure.

\textbf{Discussion.$\quad$} Our results merit discussion on some of the fundamental aspects of deep learning:

\textit{(i) Overparameterization:} Modern DNNs, with billions of parameters, are in stark contrast to the traditional statistical viewpoint of having \# of parameters approximately equal to the \# of samples. While several works have argued for measuring model complexity instead through weight norms ~\cite{pmlr-v40-Neyshabur15}, margins~\cite{bartlett2017spectrallynormalized}, compressibility~\cite{arora2018stronger}, yet it remains hard to get an interpretable ballpark on the model complexity of neural networks. Since rank intuitively captures the notion of effective parameters, it could be a \textit{possible alternative to benchmark overparameterization}, e.g. for double descent~\cite{Belkin2019ReconcilingMM}. 

\textit{(ii) Flatness:}
A growing number of works~\cite{keskar2016large,jastrzebski2018finding,baldassi2020shaping} correlate the choice of regularizers, optimizers, or hyperparameters, with the additional flatness brought about by them at the minimum. However, the significant rank degeneracy of the Hessian, which we have provably established, also points to \textit{another source of flatness} --- that exists as a virtue of the compositional model structure ---from the initialization itself. Thus, a prospective avenue of future work would be to compare different architectures based on this inherent kind of flatness.

\textit{(iii) Generalization:} An interesting observation available from our work is that factors such as width, depth, enabling bias --- commonly observed to improve generalization ---\textit{ also result in decreasing the rank/\# parameters ratio}, see Fig.~\ref{fig:cifar10-a4},~\ref{fig:cifar10-a5},~\ref{fig:bias-a2}. In a similar vein, recent work of~\cite{seroussi2021lower} has provided a lower bound to the generalization error of statistical estimators in terms of the rank of the Fisher (which is intimately related to the Hessian) divided by \# of parameters. Practically, one could use a further relaxation of rank as nuclear norm normalized by the spectral norm, in scenarios with spurious rank inflation. Overall, this suggests the relevance of studying the link between  rank and generalization. \looseness=-1

Besides the above aspects, there are many other interesting questions that naturally arise. These include, to list a few: (a) finding the Hessian rank for convolutional networks, (b) using the rank formulae to build informed models of the Hessian spectrum~\cite{pmlr-v70-pennington17a,granziol2020towards}, (c) obtaining an equality on the bounds --- which is what we observe empirically at initialization, (d) better understanding the interaction between $\HO$ and $\HF$. To conclude, by providing fundamental insights into the range of the Hessian map, our work paves the way to exciting avenues for future research.

\begin{ack}
We would like to thank Nicol\`{o} Ruggeri for reviewing a first draft of the paper. Further, we would like to thank the members of DA lab and Bernhard Sch\"{o}lkopf for useful comments. Sidak Pal Singh would also like to acknowledge the financial support from Max Planck ETH Center for Learning Systems.
\end{ack}

\begin{small}
\bibliographystyle{unsrtnat}
\bibliography{references}
\end{small}

%%%%%%%%%%%%%%%%%%%%%%%%%%%%%%%%%%%%%%%%%%%%%%%%%%%%%%%%%%%%

\input{supplementary}

\end{document}

%% file: header.tex
\usepackage[utf8]{inputenc} % allow utf-8 input
\usepackage[T1]{fontenc}    % use 8-bit T1 fonts
\usepackage{xr-hyper} 
\usepackage{url}            % simple URL typesetting
\usepackage{booktabs}       % professional-quality tables
\usepackage{amsfonts}       % blackboard math symbols
\usepackage{nicefrac}       % compact symbols for 1/2, etc.
\usepackage{microtype}      % microtypography
\usepackage{xcolor}         % colors

\usepackage{float}
\usepackage{amsmath}
\usepackage{amsthm}
\usepackage{dsfont}
\usepackage{amssymb}
\usepackage{bm}
\usepackage{nicematrix}
\usepackage{hyperref}
\hypersetup{
    colorlinks=true,
    linkcolor=blue,
    filecolor=magenta,      
    urlcolor=cyan,
    citecolor=blue,
}
\usepackage{url}
\usepackage[section]{placeins}
\usepackage{graphicx}
\usepackage{wrapfig}
\usepackage{mathtools} %for coloneqq
\usepackage{sidecap}
\usepackage{subcaption}
\usepackage{xcolor}
\usepackage{enumitem,lipsum}
\usepackage[most]{tcolorbox}
\tcbset{boxsep=0mm,boxrule=0pt,colframe=white,arc=0mm,left=0.5mm,right=0.5mm}
\usepackage[textwidth=3.5cm]{todonotes}

\usepackage[font=small,labelfont=bf]{caption}

\usepackage[toc,page,header]{appendix}
\usepackage{minitoc}
\usepackage{setspace}

%% file: math_defs.tex
\newcommand{\mb}{\mathbf}
\newcommand{\wM}[1]{\mb W^{#1}}
\newcommand{\wMt}[1]{\mb W^{#1^\top}}

\newcommand{\wMasupp}{\mb V}
\newcommand{\wMbsupp}{\mb W}
\newcommand{\wMa}{\mb W^{1}}
\newcommand{\wMb}{\mb W^{2}}

\newcommand{\bias}[1]{\mb b^{#1}}

\newcommand{\W}{\mathbf W}

\newcommand{\btheta}{\pmb \theta}

\newcommand{\x}{\mathbf x}
\newcommand{\y}{\mathbf y}

\newcommand{\Fn}{F_{\btheta}}
\newcommand{\Fni}[1]{F_{\btheta_{#1}}}
\newcommand{\Loss}{{\mathcal{L}}}
\newcommand{\Om}{{\mb{\Omega}}}
\newcommand{\effparams}{\frac{\textrm{rank}}{\textrm{\#\,params}}}

\newcommand{\covx}{{\bm{\Sigma}_{\x \x}}}
\newcommand{\covyx}{{\bm{\Sigma}_{\y \x}}}

\newcommand{\samplecovx}{{\widehat{\bm{\Sigma}}_{\x \x}}}

\newcommand{\E}{\mathbf E \,}
\newcommand{\jacobi}[2]{\frac{\partial #1}{\partial #2}}

\renewcommand{\Re}{\mathbb R}
\makeatletter
\newtheorem*{rep@theorem}{\rep@title}
\newcommand{\newreptheorem}[2]{%
\newenvironment{rep#1}[1]{%
 \def\rep@title{#2 \ref{##1}}%
 \begin{rep@theorem}}%
 {\end{rep@theorem}}}
\makeatother

\newtheorem{theorem}{Theorem}
\usepackage{mdframed}

\newtheorem{lemma}[theorem]{Lemma}

\newtheorem{proposition}[theorem]{Proposition}
\newtheorem{corollary}[theorem]{Corollary}

\newtheorem{remark}{Remark}
\newtheorem{faact}[theorem]{Fact}

\DeclareMathOperator{\rank}{rk}
\DeclareMathOperator{\rankdef}{rank-deficiency}
\DeclareMathOperator{\diag}{diag}
\DeclareMathOperator{\vect}{vec}

\DeclareMathOperator{\evals}{evals}
\newcommand{\Fmap}[1]{F^{#1}}
\newcommand{\Fmapt}[1]{F^{{#1}^\top}}
\newcommand{\Jmap}[1]{\mb J^{#1}}

\newcommand{\Lmap}[1]{\pmb \Lambda^{#1}}

\newcommand{\HO}{\mb{H}_{o}}
\newcommand{\HF}{{\mb{H}_{f}}}
\newcommand{\HOtilde}{{\tilde{\mb{H}_{o}}}}
\newcommand{\HFhat}{{\widehat{\mb{H}}_{f}}}

\newcommand{\HFcol}[1]{{\mb{H}_{f}^{\bullet #1}}}
\newcommand{\HL}{{\mb{H}_{\mathcal{L}}}}
\newcommand{\HOsup}[1]{{\mb{H}^{#1}_{o}}} % don't use \mb before them in latex
\newcommand{\HFsup}[1]{{\mb{H}_{f}^{#1}}} % don't use \mb before them in latex
\newcommand{\HFhatsup}[1]{{\widehat{\mb{H}}_{f}^{#1}}} % don't use \mb before them in latex
\newcommand{\HLhatSup}[1]{{\widehat{\mb{H}}_{\mathcal{L}}^{#1}}}
\newcommand{\HLtildeSup}[1]{{\widetilde{\mb{H}}_{\mathcal{L}}^{#1}}}

\newcommand{\HLsup}[1]{{\mb{H}_{\mathcal{L}}^{#1}}} % don't use \mb before them in latex

\newcommand{\idx}{k}
\newcommand{\e}{{\mb e}}
\newcommand{\ab}{{\mb a}}
\newcommand{\bb}{{\mb b}}

\newcommand{\kro}{%
  \mathbin{\mathop{\otimes}}%
}

\def\Am{{\bf A}}
\def\Bm{{\bf B}}
\def\Cm{{\bf C}}
\def\Dm{{\bf D}}

\def\Im{{\bf I}}

\def\Mm{{\bf M}}
\def\Nm{{\bf N}}
\def\Pm{{\bf P}}
\def\Qm{{\bf Q}}

\def\Wm{{\bf W}}
\def\Vm{{\bf V}}

\def\Zm{{\bf Z}}
\def\Ym{{\bf Y}}
\def\Xm{{\bf X}}
\usepackage{xcolor}% http://ctan.org/pkg/xcolor
\usepackage{ifmtarg}% http://ctan.org/pkg/ifmtarg
\usepackage{xifthen}% http://ctan.org/pkg/xifthen
\usepackage{environ}% http://ctan.org/pkg/environ
\usepackage{multido}% http://ctan.org/pkg/multido
\usepackage{lipsum}%
\usepackage{mdframed}
\theoremstyle{plain}

\newcommand*{\QEDB}{\null\nobreak\hfill\ensuremath{\square}}%

\usepackage{thmtools}
\usepackage{cleveref}

\declaretheoremstyle[%
  spaceabove=-2pt,%
  spacebelow=6pt,%
  headfont=\normalfont\itshape,%
  postheadspace=1em,%
  qed=\qedsymbol%
]{mystyle}



\newtheorem{assump}{Assumption}
\newcommand*{\colorboxed}{}
\def\colorboxed#1#{%
  \colorboxedAux{#1}%
}
\newcommand*{\colorboxedAux}[3]{%
  \begingroup
    \colorlet{cb@saved}{.}%
    \color#1{#2}%
    \boxed{%
      \color{cb@saved}%
      #3%
    }%
  \endgroup
}

\newreptheorem{proposition}{Proposition}
\newreptheorem{lemma}{Lemma}
\newreptheorem{theorem}{Theorem}
\newreptheorem{corollary}{Corollary}
\newreptheorem{assump}{Assumption}

\makeatletter
\newcommand{\neutralize}[1]{\expandafter\let\csname c@#1\endcsname\count@}
\makeatother

\newenvironment{assumpprime}[1]
  {\renewcommand{\theassump}{\ref{#1}$'$}%
   \addtocounter{assump}{-1}%
   \begin{assump}}
  {\end{assump}}

%% file: supplementary.tex
\clearpage
\appendix
% \appendixpagename
% {\centering \huge{Supplementary Material} \vskip .11in \hrule height1pt \vskip .3in}
% \usepackage{float}
\renewcommand{\thesection}{S\arabic{section}}
\renewcommand{\thetable}{S\arabic{table}}
\renewcommand{\thefigure}{S\arabic{figure}}
\renewcommand{\thefootnote}{S\arabic{footnote}}
\setcounter{figure}{0}
\setcounter{table}{0}
\setcounter{footnote}{0}

\hypersetup{linkcolor=black}

\addcontentsline{toc}{section}{Appendix} % Add the appendix text to the document TOC
\vspace{-2em}

\part{Supplementary Material} % Start the appendix part
\vspace{2em}
\parttoc % Insert the appendix TOC

% \setlength\cftparskip{2pt}
% \setlength\cftbeforesecskip{2pt}
% \setlength\cftaftertoctitleskip{3pt}
% \addtocontents{toc}{\protect\setcounter{tocdepth}{2}}
% \setcounter{tocdepth}{1} % include only top level sections
% \tableofcontents

\clearpage

\hypersetup{linkcolor=red}

\section{Backpropagation in matrix-derivatives for the general case}\label{supp:backprop}
In the general case, we can represent the gradient in analogy as to Eq.~\eqref{eq:matrix-derivative}, 
\begin{align}
\jacobi{F}{\wM{k}}= \Jmap{L:k+1} \Lmap{k}\kro \Big[ \Fmap{k-1:1}(\x) \Big]^\top,\,\, \text{where}\,\, \Fmap{k-1:1} = \Fmap{k-1} \circ \cdots \circ \Fmap{1}\,.
\end{align}
And, further $\mb J$ denotes the Jacobian map across the indexed layers, which is itself a composition of elementary Jacobians.
\begin{align}
\Jmap{k} & = \Lmap{k} \wM{k}, \quad \Lmap{k} = \diag(\dot \sigma^{(k)}), \quad  \Jmap{k+1:L} = {\Jmap{L:k+1}}^\top \,.
\end{align}
The Jacobian maps depend on functions of the input as each $\Lmap{k}$ depends on the pre-activation $\wM{k}\, \x_{k-1}$. By the usual chain rule (backpropagation) one has for a linear DNN, at a sample $(\x, \y)$:
\begin{align}
\jacobi{\ell}{\wM k} := \jacobi{\ell_{\x, \y}}{\wM k} 
& = \underbrace{\Big[ \wM{k+1:L}\pmb \delta_{\x, \y} \Big]}_{\text{backward } \in \Re^{M_k}} \;\cdot \; \underbrace{\Big[  \wM{k-1:1} \x\Big]^\top}_{\text{forward } \in \Re^{M_{k-1}}} %\\[2mm]
= \wM{k+1:L} [\wM{L:1}\x \x^\top - \y \x^\top] \wM{1:k-1}\,.
\label{eq:supp-dnn-grad-linear}
\end{align}
The gradient with regard to $\wM k$ is first order in $\wM k$ and second order in the other matrices. In the general case, we get 
\begin{align}
\jacobi{\ell}{\wM k} 
& = \Big[\Jmap{k+1:L}\pmb \delta_{\x, \y}\Big] \cdot \Big[ \Fmap{k-1:1}(\x) \Big]^\top
\label{eq:dnn-grad}
\end{align}
Clearly, the partial forward maps are non-linear, whereas the backward maps are linearized at an argument determined by the current input.  

\subsection{Equivalence with Gauss-Newton decomposition}

Remember, $\btheta \in \mathbb{R}^p$ denotes the (vectorized) parameters of the neural network map $F$, which then feeds into the loss $\ell$. Then, the Hessian of the composition of $\ell$ and $F$ with respect to $\btheta$ (computed over a sample $(\x, \y)$, but we omit specifying it for brevity) is, 
\vspace{-0.5em}
\[
\small
\nabla^2_{\btheta} (\ell \circ F) = \nabla_{\btheta} F^\top \; [\partial^2 \ell ]\; \nabla_{\btheta} F \;+\; \sum_{c=1}^{K} \, [\partial \ell]_{c}\; \nabla^2_{\btheta}\, F_c  \,.
\]
where $\partial \ell$ and $\partial^2 \ell$ are respectively the gradient and Hessian  of the loss $\ell$ with respect to the network function, $F$. Also, $\nabla_{\btheta} F \in \mathbb{R}^{K\times p}$ is the Jacobian map of the network function $F(\x)$ with respect to the parameters $\btheta$. Let us employ the shorthand $\nabla_{k} F$ to denote the Jacobian of the network function with respect to the weight matrix $\wM{k}$ in the numerator-layout style as mentioned earlier. Similarly, let $\nabla^2_{kl} \,F_c$ be the Hessian of $c$-th component of the network function with respect to weight matrices $\wM{k}, \wM{l}$. Hence, we obtain, \[\partial^2 \ell = \Im_K\,, \quad \nabla_{\idx} F =  \wM{L:\idx+1} \,\otimes \, \x^\top \wM{1:\idx-1}\,,\partial^2 \ell = \Im_K\,,\]
\begin{align*}
\nabla_{\idx} F &=  \wM{L:\idx+1} \,\otimes \, \x^\top \wM{1:\idx-1}\,,\\[1mm]
\nabla_{k} F^\top  \mb I_{K}
\nabla_{l} F & =\wM{k+1:L} \wM{L:l+1} \, \otimes \, \wM{k-1:1} \x\x^\top \wM{1:l-1}\,.
\end{align*}

Now, the other term has a reduction (or contraction) with the components of the residual, i.e., $\partial \ell = \pmb \delta_{\x, \y}$. Importantly, we notice that it has a block-hollow structure since, 
% \begin{align}
$\nabla^2_{\idx \idx}\, F_c  = \mb 0, \,\, \forall c \in [1\cdots K]$.
% \end{align}
Thus all the diagonal blocks come from the first term in the Gauss-Newton decomposition. Now, comparing the above expressions with the Eqns. (\ref{eq:outer-kl},~\ref{eq:func-hess-eqn1},~\ref{eq:func-hess-eqn2}), it is evident that the two approaches yield the same structure of the Hessian.

As a side-remark, note that in Eqns. (\ref{eq:outer-kl},~\ref{eq:func-hess-eqn1},~\ref{eq:func-hess-eqn2}) we express the $kl$-th block as $\dfrac{\partial^2 \Loss}{\partial \wM l \partial \wM k}$ 
instead of $\dfrac{\partial^2 \Loss}{\partial \wM k \partial \wM l}$,
only to ensure consistent shape as per matrix derivative convention (but corresponding entries are ofcourse equal). 

\clearpage

\section{Tools for the analysis}\label{supp:tools}

\subsection{General notation}
We employ the shorthand notation, $\wM{k:l}$, to refer to the matrix product chain $\wM k \cdots \wM l$, when $k>l$. When $k < l$, $\wM{k:l}$ will stand for the transposed product chain $\wMt{k} \cdots \wMt{l}$. 
In the edge case $k=l$, this will imply either $\wM{k}$ or $\wMt{k}$ depending on the context. Although, we will make the notation explicit on occasions where it might not be evident. 
Besides, we use the $\kro$ to denote the Kronecker product of two matrices, $\vect_c$ and $\vect_r$ to denote the column-wise and row-wise vectorization of matrices respectively. $\Im_{k}$ denotes the identity matrix of size $k$, while $\mathds{1}_{k}$ denotes the all ones vector of length $k$, $\mb 0_{k}$ denotes an all-zeros matrix of size $k$. The generalized inverse~\citep{rao1972generalized} \footnote{This is like a general version of pseudoinverse which only satisfies the first Moore-Penrose condition.} of a matrix $\Am$ is given by a matrix $\Am^{-}$ which obeys $\Am \Am^{-} \Am = \Am$. The notation $\Am^{\bullet i}$ or $\Am_{\bullet i}$ denotes the $i$-th column of the matrix $\Am$, while $\Am^{i \bullet}$ or $\Am_{i \bullet}$ refers to its $i$-th row. We place the column and row indices into the subscript or superscript depending on the context they are used.

\subsection{Helper Lemmas}\label{supp:helper-lemmas}

\begin{lemma}
\label{lemma:supp-kronecker-block}

Let $\Am\in\mathbb{R}^{m\times n} $ and $\Bm\in\mathbb{R}^{p\times q}$. Then the row-partitioned matrix $\left[\begin{array}{c}\Im_q \otimes \Am \\ \Bm \otimes \Im_n\end{array}\right]$ has the rank, 
\[
\rank\left[\begin{array}{c}\Im_q \otimes \Am \\ \Bm \otimes \Im_n\end{array}\right]=q \rank(\Am)+n \rank(\Bm)-\rank(\Am) \rank(\Bm)
\]
\begin{proof}\label{supp:proof-kronecker-block}
The proof relies on the following rank formula due to~\cite{doi:10.1080/03081087408817070}, and are based on using the generalized inverse of a matrix. Besides, the following proof closely follows~\citet{CHUAI2004129}.

$$\rank\left[\begin{array}{c}\Am \\ \Cm\end{array}\right]=\rank(\Am)+\rank\left(\Cm-\Cm \Am^{-} \Am\right)=\rank(\Cm)+\rank\left(\Am-\Am \Cm^{-} \Cm\right)$$

Here, $ \Cm^{-}$ denotes the weak (generalized) inverse of $\Cm$, i.e., any solution such that $\Cm  \Cm^{-} \Cm = \Cm$. Then, we have that, 

\begin{align*}
\rank\left[\begin{array}{c}\Im_q \otimes \Am \\ \Bm \otimes \Im_n\end{array}\right]&= \rank(\Im_q \otimes \Am)+ \rank\big(\left(\Bm \otimes \Im_n\right) - \left(\Bm \otimes \Im_n\right){(\Im_q \otimes \Am)}^{-} \left(\Im_q \otimes \Am\right) \big) \\
&= \rank(\Im_q) \rank(\Am)+ \rank\big(\left(\Bm \otimes \Im_n\right) - \left(\Bm \otimes \Im_n\right){(\Im_q \otimes \Am^{-})} \left(\Im_q \otimes \Am\right) \big) \\
&= q \rank(\Am)+ \rank\big(\left(\Bm \otimes \Im_n\right) - \left(\Bm \otimes  A^{-} A\right) \big)\\
&= q \rank(\Am)+ \rank\big(\Bm \otimes \left(\Im_n -  A^{-} A\right) \big) \\
&= q \rank(\Am)+ \rank(\Bm)\rank\left(\Im_n -  A^{-} A\right) \\
&\overset{(a)}{=} q \rank(\Am)+ \rank(\Bm)(n-\rank(\Am)) \\
&= q \rank(\Am)+ n\rank(\Bm) - \rank(\Am)\rank(\Bm)
\end{align*}

where, in (a) we have used that $\rank\left(\Im_n -  A^{-} A\right)= n - \rank(\Am)$, which follows from the fact that column space of the matrix $\Im_n -  A^{-} A$, satisfies $\mathcal{C}\left(\Im_n -  A^{-} A\right)  \subset \mathcal{N}(A)$ (because the null space $\mathcal{N}(A)$ is the set of column vectors $\alpha$ for which $A\alpha=0$ and since any $x = \left(\Im_n -  A^{-} A\right) y \, \implies Ax=0 $) which means $\rank\left(\Im_n- A^{-} A\right)\leq \operatorname{dim} \mathcal{N}(A)$ and $\rank\left(\Im_n -  A^{-} A\right) \geq \rank\left(\Im_n\right)- \rank\left(A^{-} A\right) = n - \rank\left( A\right) = \operatorname{dim}\mathcal{N}(A)$.
\end{proof}
\end{lemma}

\begin{lemma}\label{lemma:supp-kronecker-block-general}
Let $\Am_1\in\mathbb{R}^{m_1\times n_1} $, $\Am_2\in\mathbb{R}^{m_2\times n_2}$, $\Bm\in\mathbb{R}^{p\times n_1}$. Then the column block matrix $\left[\begin{array}{c} \Am_1\otimes \Am_2 \\ \Bm \otimes \Im_{n_2}\end{array}\right]$ has the rank, 

$$
\rank\left[\begin{array}{c}\Am_1\otimes \Am_2 \\ \Bm \otimes \Im_{n_2}\end{array}\right]= \rank(\Am_2) 
\left(
    \rank\left[\begin{array}{c}\Am_1 \\ \Bm\end{array}\right] - \rank(\Bm) 
\right) 
+ n_2 \, \rank(\Bm)
$$
\begin{proof}

\begin{align*}
    \rank\left[\begin{array}{c}\Am_1\otimes \Am_2 \\ \Bm \otimes I_{n_2}\end{array}\right] &\overset{(a)}{=} \rank(\Bm \kro \Im_{n_2}) + \rank\left((\Am_1\otimes \Am_2) - (\Am_1\otimes \Am_2) (\Bm \kro \Im_{n_2})^{-}(\Bm \kro \Im_{n_2})\right) \\
    &= n_2 \, \rank(\Bm) + \rank\left((\Am_1\otimes \Am_2) - (\Am_1\otimes \Am_2) (\Bm^{-}\Bm \kro \Im_{n_2})\right)\\
    &= n_2 \, \rank(\Bm) + \rank\left((\Am_1 - \Am_1 \Bm^{-}\Bm)\otimes \Am_2)\right)\\
    &= n_2 \, \rank(\Bm) + \rank(\Am_1 - \Am_1 \Bm^{-}\Bm)\rank(\Am_2) \\
    &\overset{(b)}{=} n_2 \, \rank(\Bm) +\left( \rank\left[\begin{array}{c}\Am_1 \\ \Bm \end{array}\right] - \rank(\Bm)\right)\rank(\Am_2) \\[1mm]
    &=\rank(\Am_2) 
\left(
    \rank\left[\begin{array}{c}\Am_1 \\ \Bm\end{array}\right] - \rank(\Bm) 
\right) 
+ n_2 \, \rank(\Bm)
\end{align*}
Step (a) and (b) are due to the rank formula for the column block matrix~\cite{doi:10.1080/03081087408817070}.

$$\rank\left[\begin{array}{c}\Xm \\ \Ym\end{array}\right]=\rank(\Ym)+\rank\left(\Xm- \Xm\Ym^{-} \Ym\right)$$

And, the other steps follow from the basic properties of Kronecker product.

\end{proof}
\end{lemma}

\begin{lemma}\label{lemma:block-shared}
Let $\Am \in\mathbb{R}^{p\times q} $, $\Bm\in\mathbb{R}^{m\times p}$, then we have that:
$$
\rank\left[\begin{array}{c} \Am\\ \Bm\Am\end{array}\right]=\rank(\Am)
$$
\begin{proof}

This follows simply from the definition of generalized inverse ($\Am \Am^{-} \Am = \Am$) and the rank formula for column-block matrix. 
\begin{align*}
    \rank\left[\begin{array}{c}\Am \\ \Bm\Am\end{array}\right]&=\rank(\Am)+\rank\left((\Bm\Am)- (\Bm\Am) \Am^{-} \Am\right)\\
    &=\rank(\Am)+\rank(\Bm\Am- \Bm\Am) = \rank(\Am)
\end{align*}
\end{proof}
\end{lemma}
\begin{corollary}
Let $\Bm\in\mathbb{R}^{m\times p}$, then we get the elementary identity:
$$
\rank\left[\begin{array}{c} \Im_p\\ \Bm\end{array}\right]=\rank(\Im_p) = p\,.
$$
\end{corollary}

\vspace{1em}
\begin{lemma}\label{lemma:outer-equality}
Given a matrix $\mb M:= \mb A \mb B \mb A^T$, with $\Bm \succ 0$ symmetric, then $\rank(\Mm) = \rank(\Am)$. 
\end{lemma}

\begin{proof}
Since $\Bm\succ 0$, we can write $\Mm = (\Am \Bm^{\frac{1}{2}}) (\Am \Bm^{\frac{1}{2}})^\top$, where $\Bm^{\frac{1}{2}}$ denotes the matrix square root of $\Bm$. This implies $\rank(\Mm) = \rank(\Am\Bm^{\frac{1}{2}})$, since the null space of $\Xm^\top \Xm$ is the same as the null space of any arbitrary matrix $\Xm$, i.e.,  $\mathcal{N}(\Xm^\top\Xm) =\mathcal{N}(\Xm)$, and additionally using $\rank(\Xm^\top)=\rank(\Xm)$. Next, as $\Bm \succ 0$, we have $\Bm^{\frac{1}{2}}\succ 0$, which further implies $\Bm^{\frac{1}{2}}$ is full rank. Hence $\rank(\Am\Bm^{\frac{1}{2}}) = \rank(\Am)$, which at last gives, $\rank(\Mm) = \rank(\Am)$.
\end{proof}

\subsection{\{Left, Right, Pseudo\}- Inverses}
For a matrix $\Am$ with full column rank, the left inverse is defined to be a matrix $\Am^{-L}$ such that $\Am^{-L} \Am = \Im$. Likewise, when the matrix $\Am$ has full rank, we can define a right inverse which is a matrix $\Am^{-R}$ such that $\Am \Am^{-R} = \Im$. The left and right inverses need not be unique. But often a nice or convenient choice for the left inverse is $\Am^{-L} = {(\Am^\top \Am)}^{-1} \Am^\top$, while that for the right inverse is $\Am^{-R} = \Am^\top {(\Am \Am^\top)}^{-1}$. 

The (Moore-Penrose) pseudoinverse $\Am^{\dagger}$ of a matrix $\Am \in \mathbb{R}^{m\times n}$ is a unique matrix that satisfies the following properties:

\begin{align*}
    \Am \Am^\dagger \Am &= \Am \\
    \Am^\dagger \Am \Am^\dagger &= \Am^\dagger \\
    {(\Am \Am^\dagger)}^\top &= \Am \Am^\dagger \\
    {(\Am^\dagger \Am)}^\top &= \Am^\dagger \Am  \\
\end{align*}

When the matrix has full column rank or full row rank, then the pseudoinverse agrees with the particular choice of left and right inverse we mentioned above. In such scenarios of full column or row rank, when we want to refer to this choice of left or right inverse, we will simply denote them by $\Am^{\dagger}$.

\subsection{Block row and column operations}\label{supp:row-col-ops}

In our proofs, we make use of row and column operations, taken altogether on blocks of matrices rather than just individual rows or columns. So, here we clarify what we actually mean by such block row and column operations, and how their usage does not affect the rank of matrix to which these are applied. Essentially, we will look at the corresponding ``elementary matrices'' that get formed and argue that they multiplying with them does not change rank.

Let us consider that we have with us the following matrix, $\Pm$, with its row blocks labeled as $R_1, \cdots, R_L$.
\begin{equation}
\Pm = 
\begin{pNiceArray}[first-row,first-col,nullify-dots]{c}
        \\[2mm]
R_1& \Om^1\kro\Nm^1   \\[1mm]
\Vdots & \vdots \\[1mm]
R_{i} & \Om^i\kro\Nm^i  \\[2mm]
\Vdots & \vdots \\[1mm]
R_{L}  &   \Om^L\kro\Nm^L  \\
\end{pNiceArray}.
\end{equation}\\

\subsubsection{Factoring-out block operations}
Assume that the matrix $\Om^i$ has full column rank and then $\Om^i\kro \Im$ also has full column rank and is left-invertible. That means we can write the following decomposition $\Pm = \Pm_1 \Pm_2$. 

\begin{equation}
\Pm = \underbrace{
\begin{pNiceArray}[first-row,nullify-dots]{ccccc}
        & \\[2mm]
\Im & \cdots & \bm 0 & \cdots& \bm 0   \\[1mm]
\vdots &\ddots &\vdots & \ddots& \vdots\\[1mm]
\bm 0 & \cdots& \Om^i \kro \Im& \cdots & \bm 0 \\[2mm]
\vdots & \ddots & \vdots& \ddots & \vdots\\[1mm]
\bm 0 &\cdots & \bm 0 &\cdots & \Im \\
\end{pNiceArray}}_{\Pm_1}
\underbrace{
\begin{pNiceArray}[first-row,nullify-dots]{c}
     \\[2mm]
\Om^1\kro\Nm^1  \\[1mm]
\vdots \\[1mm]
\Im \kro\Nm^i \\[2mm]
\vdots \\[1mm]
\Om^L\kro\Nm^L   \\
\end{pNiceArray}}_{\Pm_2}\,.
\label{eq:blockop-example}
\end{equation}

Since even $\Pm_1$ is left-invertible (identity matrix in all diagonals except for $\Om^i\kro\Im$ but which is left-invertible), we have that, 
\[\rank(\Pm) = \rank(\Pm_1 \Pm_2) = \rank(\Pm_2)\,.\]

This is what we actually mean when applying the row operation, \[R_i \leftarrow {(\Om^i\kro \Im)}^{\,\dagger} R_i\,.\] 

Subsequently, we start working with the matrix $\Pm_2$, although we may not explicitly update the name of the matrix.

\paragraph{Remark.} Likewise, we could have instead assumed $N^i$ to be left invertible, and factored out a corresponding $\Pm_1$ matrix with $\Im\kro N^i$ as one of its diagonal blocks. Going further, in a similar manner, we can define the column operation analogue of this by factoring out from the right a matrix which is right-invertible, and will thus preserve rank.

\subsubsection{Deletion block operations}
The block operations in this section are nothing but the analogue of the usual row operations, except they are carried out at the level of blocks.
Say that we have done the above factoring-out operation in Eq.~\eqref{eq:blockop-example}. Now, we are given that $N^1=N^i$, i.e, we have the matrix:

\begin{equation}
\Pm =
\begin{pNiceArray}[first-row,nullify-dots]{c}
     \\[2mm]
\Om^1\kro\Nm^1  \\[1mm]
\vdots \\[1mm]
\Im \kro\Nm^1 \\[2mm]
\vdots \\[1mm]
\Om^L\kro\Nm^L   \\
\end{pNiceArray}\,.
\end{equation}

Now, consider we multiply from the left with the matrix $\Qm$ to yield $\Pm^\prime$:

\begin{equation}
\Pm^\prime = \underbrace{
\begin{pNiceArray}[first-row,nullify-dots]{ccccc}
        & \\[2mm]
\Im & \cdots & - \Om^1 \kro \Im & \cdots& \bm 0   \\[1mm]
\vdots &\ddots &\vdots & \ddots& \vdots\\[1mm]
\bm 0 & \cdots&  \Im& \cdots & \bm 0 \\[2mm]
\vdots & \ddots & \vdots& \ddots & \vdots\\[1mm]
\bm 0 &\cdots & \bm 0 &\cdots & \Im \\
\end{pNiceArray}}_{\Qm}
\begin{pNiceArray}[first-row,nullify-dots]{c}
     \\[2mm]
\Om^1\kro\Nm^1  \\[1mm]
\vdots \\[1mm]
\Im \kro\Nm^1 \\[2mm]
\vdots \\[1mm]
\Om^L\kro\Nm^L   \\
\end{pNiceArray} \quad = \quad 
\begin{pNiceArray}[first-row,nullify-dots]{c}
     \\[2mm]
\bm 0 \\[1mm]
\vdots \\[1mm]
\Im \kro\Nm^1 \\[2mm]
\vdots \\[1mm]
\Om^L\kro\Nm^L   \\
\end{pNiceArray}\,.
\label{eq:blockop-example}
\end{equation}

Since $\Qm$ is a upper-triangular matrix with ones on the diagonal, it is invertible. As a result, rank of $\Pm$ does not change when multiplied by $\Qm$ from the left. 

This is what we actually mean when applying the row operation, \[R_1 \leftarrow R_1 - {(\Om^1\kro \Im)} R_i\,.\] 

Subsequently, we start working with the matrix $\Pm^\prime$, although we may not explicitly update the name of the matrix.

\paragraph{Remark.} Notice, we could have done a similar thing on other side of Kronecker factors. Going further, in a similar manner, we can define the column operation analogue of this by multiplying on the right such a matrix which is invertible, and which will thus preserve rank.

\paragraph{Aliter.} One can perhaps intuit these block operations from the point of view of inclusion of subspaces, but here we are being a bit pedantic.

\subsection{Matrix derivatives}\label{supp:matrix-derivative}
Let us start by discussing some simple facts on vectorization. Consider a matrix $\Am \in \mathbb{R}^{m\times n}$, and recall that $\vect_r$ and $\vect_c$ denote rowwise and columnwise vectorization respectively. Then firstly we have the following simple relation between them:
\begin{equation}\label{eq:vects}
    \vect_c(\Am) = \vect_r(\Am^{\top})\,.
\end{equation}

Now, we give the proof of the commonly-used identity, $\vect_c(\Am\Xm\Bm) = \left(\Bm^{\top} \kro \Am\right) \vect_c(\Xm)$, where $\Am \in \mathbb{R}^{m\times n}\,,\, \Xm \in \mathbb{R}^{n\times p}\,,\, \Bm \in \mathbb{R}^{p\times q}$. For more details on this, refer to ~\cite{magnus2019matrix,10.5555/59921}.

The main idea is to write $\Xm$ in the form of canonical basis vectors $\e_i$, i.e., $\Xm = \sum_{i=1}^{p} \Xm_{\bullet i} \,\e_{i}^{\top}$, where $\Xm_{\bullet i}$ denotes the $i$-th column of $\Xm$.

\begin{align*}
    \vect_c\left(\Am  \sum_{i=1}^{p} \Xm_{\bullet i} \,\e_{i}^{\top} \Bm \right) &= \vect_c\left(\sum_{i=1}^{p} \left(\Am \Xm_{\bullet i}\right) \,\left(\Bm^{\top} \e_{i}\right) \right)\\
    &\overset{(a)}{=} \sum_{i=1}^{p} \left(\Bm^{\top} \e_{i}\right) \kro \left(\Am \Xm_{\bullet i}\right) \\
    &\overset{(b)}{=} \left(\Bm^{\top} \kro \Am\right) \, \sum_{i=1}^p \e_i \kro \Xm_{\bullet i} = \left(\Bm^{\top} \kro \Am\right) \vect_c(\Xm) \quad \QEDB
\end{align*}

In step (a), we have used that for two vectors $\ab, \bb \,$, the following basic fact $\vect_c(\ab \bb^{\top}) = \bb \kro \ab$ holds. And, in (b) we have employed the mixed-product property of Kronecker products.

Since, we utilize row-wise vectorization in our paper, let use find the equivalent relation in terms of that:
\begin{align}\label{eq:axb-kro}
    \vect_r\left(\Am\Xm \Bm\right) \overset{\textrm{Eq.~\eqref{eq:vects}}}{=} \vect_c\left(\Bm^{\top} \Xm^\top \Am^\top\right) = \left(\Am\kro \Bm^\top\right) \vect_c\left(\Xm^\top\right) \overset{\textrm{Eq.~\eqref{eq:vects}}}{=} \left(\Am\kro \Bm^\top\right) \vect_r(\Xm)\,.
\end{align}

Recall, we use the numerator (Jacobian) layout to express matrix-by-matrix derivatives, i.e., \[\dfrac{\partial \mb Y}{\partial \mb X} := \dfrac{\partial \vect_r(\mb Y)}{\partial \vect_r(\mb X)^\top}\,.\]

Thus, when $\mb Y = \Am \Xm \Bm$, we use the above property along the first identification theorem of vector calculus~\cite{magnus2019matrix} as mentioned below:

\begin{theorem}{(first identification theorem):} 
$$
\mathrm{d} f=A(\x) \,\mathrm{d} \x \Longleftrightarrow \frac{\partial f(\x)}{\partial \x^{\top}}=A(\x)\,,
$$

where, $\mathrm{d}$ denotes the differential. 
\end{theorem}

Finally, this yields that:
\[\dfrac{\partial \Am\Xm\Bm}{\partial \mb X} = \Am \kro \Bm^\top \,.\]

\paragraph{Remark.} If we were using the column vectorization, we would have instead obtained $\Bm^\top \kro \Am$.

\subsection{Rank of weight matrices at initialization}\label{supp:full_rank_weight}
Here we study the rank of random matrices to understand how the weight matrices of a neural network at initialization influence the rank.
\begin{lemma}
\label{full_rank_weight}
Consider a random matrix $\Wm \in \mathbb{R}^{m \times n}$ for $m,n \in \mathbb{N}$ where each entry is sampled i.i.d. w.r.t. to some continuous (i.e. not discrete) probability distribution $p$, i.e. $W_{ij} \sim p $. Then it holds that
\begin{equation*}
    \operatorname{rank}(\Wm) = \operatorname{min}(m, n) \hspace{2mm} \text{a.s.}    
\end{equation*}

\end{lemma}

\begin{proof}

Assume w.l.o.g. that $n\leq m$ (otherwise consider the transposed matrix) and enumerate the columns of 
$\Wm$ as $\Wm_1, \dots, \Wm_{n} \in \mathbb{R}^{m}$.
We need to show the linear independence of $\{\Wm_1, \dots,
\Wm_n\}$ over $\mathbb{R}^{m}$. 
Let us show this inductively.
First $\Wm_1 \not= \bm{0}$ with probability $1$ since $p$ is continuous. Consider now $\Wm_2$. 
Conditioned on the previously sampled column, $\Wm_1$, due to the continuous nature of the distribution, 
the probability of $\Wm_2$ being contained in the span of $\Wm_1$ is zero:
\begin{equation*}
    \mathbb{P}\left(\Wm_{2} \in \operatorname{span}(\Wm_1 )\right) = \int_{\mathbb{R}}\underbrace{\mathbb{P}\left(\Wm_{2} \in \operatorname{span}(\Wm_1 )|\Wm_1\right)}_{=0}p(\Wm_1)d\Wm_1  = 0
\end{equation*}

Finally, consider $\Wm_m$ and condition on the previously sampled vectors $\Wm_1, \dots, \Wm_{m-1}$. By the induction hypotheses, they span an $m-1$-dimensional space. Again, due to the independence of $\Wm_m$ from the previous vectors and the fact that an $m-1$-dimensional subspace has Lebesgue measure $0$ in $\mathbb{R}^{m}$, it holds in a similar fashion that
\begin{equation}
    \mathbb{P}\left(\Wm_{m} \in \operatorname{span}(\Wm_1, \dots, \Wm_{m-1})\right) = 0
\end{equation}
and the matrix hence has full rank.
\end{proof}
It turns out that we can apply a similar argument for the case of the product of two random matrices:
\begin{lemma}
\label{two_matrix}
Consider random matrices $\Vm \in \mathbb{R}^{m \times n}$ and $\Wm \in \mathbb{R}^{n \times k}$, both drawn with i.i.d. entries according to some continuous probability distribution $p$. Define $\bm{Z} = V W \in \mathbb{R}^{m \times k}$. Then it holds that
\begin{equation*}
    \operatorname{rank}(\bm{Z}) = \operatorname{min}(m, n, k)
\end{equation*}

\end{lemma}
\begin{proof}

First, notice that by Lemma \ref{full_rank_weight}, both matrices have full rank, i.e. $\operatorname{rank}(\Vm) = \operatorname{min}(m,n)$ and $\operatorname{rank}(\Wm) = \operatorname{min}(n, k)$. 
By standard linear algebra results (not involving the fact that we have random matrices), we get that for $n \leq m$, $\operatorname{rank}(\bm{Z}) = \operatorname{rank}(\Wm) = \operatorname{min}(m, k)$ and for $n \leq k$, $\operatorname{rank}(\bm{Z}) = \operatorname{rank}(\Vm) = \operatorname{min}(m, n)$. 

Thus it remains to show the case where $n \geq k, m$, i.e. the contracting dimension is the biggest. Assume w.l.o.g. that $k \leq m$ (otherwise study the transposed matrix). The columns of $\bm{Z}$ are given by $\bm{z}_i = \Vm\Wm_i$ for $i=1, \dots, k$. It thus suffices to show the linear independence of  $\{\bm{z}_1, \dots, \bm{z}_k\}$. 

Notice that $\operatorname{rank}(\Vm) = m$ from the assumptions, thus $\{\Vm\x: \x \in \mathbb{R}^{n}\}$ is a $m$-dimensional subspace of $\mathbb{R}^{n}$. We will apply  a similar argument as in Lemma \ref{full_rank_weight}. Consider $\bm{z}_1 = \Vm\Wm_1$. Due to the independence and the fact that $\operatorname{rank}(\Vm)=m>0$, $\bm{z}_1 \not= 0$ a.s. Assume that $\bm{z}_1, \dots, \bm{z}_{k-1}$ are linearly independent, they thus span a $k-1$ dimensional space. 

Conditioned on $\bm{z}_1, , \dots, \bm{z}_{k-1}$, $\bm{z}_k = \Vm\Wm_k$ is a random vector in $\operatorname{Im}(\Vm)$ (think of $\Vm$ as a fixed linear map since we are conditioning on it). Since $\operatorname{span}(\bm{z}_1, \dots, \bm{z}_{m-1})$ forms a $k-1$ dimensional subspace of $\operatorname{Im}(\Vm)$, which has dimension $m$, the subspace again has Lebesgue measure zero and we conclude that 
\begin{equation*}
    \mathbb{P}\left(\bm{z}_k \in \operatorname{span}(\bm{z}_1, \dots, \bm{z}_{k-1})\right) = \int_{\mathbb{R}}\underbrace{\mathbb{P}\left(\bm{z}_k \in \operatorname{span}(\bm{z}_1, \dots, \bm{z}_{k-1})|\bm{z}_1, \dots, \bm{z}_{k-1}\right)}_{=0} p(\bm{z}_k)d\bm{z}_k
\end{equation*}

\end{proof}

We can now easily use this result for an arbitrary sized matrix product: 
\begin{corollary}
Consider random matrices $\Wm^{i} \in \mathbb{R}^{m_i \times m_{i+1}}$ for $i=1, \dots, n$ where each entry is initialized i.i.d. w.r.t. a continuous distribution $p$. Define the product matrix $\Wm = \Wm^1\dots \Wm^{n}$. Then it holds that
\begin{equation*}
    \operatorname{rank}(\Wm) = \operatorname{min}(m_1, \dots, m_n)
\end{equation*}

\end{corollary}
\begin{proof}
Apply Lemma~\ref{two_matrix} recursively, i.e. for $\Wm = \Wm^{1}\Wm^{2:n}$, then for $\Wm^{2:n} = \Wm^{2}\Wm^{3:n}$ up until $\Wm^{n}$.
\end{proof}

\clearpage

\section{Rank of the outer-product term}\label{supp:outer-product}
We begin by discussing the proof of Proposition~\ref{eq:outer-decomp}. Then, we briefly discuss the example of two-layer networks to motivate the proof, and after that we discuss the proof of the Theorem~\ref{theorem:ub-outer}. Subsequently, we present the proof of Corollary~\ref{corollary:outer-rank}.

\subsection{Proof of Proposition~\ref{eq:outer-decomp}}

\begin{repproposition}{eq:outer-decomp}
For a deep linear network, $\,\HO = \Am_o  \Bm_o {\Am_o}^\top\,$, where $\,\Bm_o = \Im_K \kro \covx\, \in \mathbb{R}^{Kd\times Kd}$,
\begin{equation*}
\text{and}\,\,\,\Am_o^\top=\begin{pmatrix}
\wM{L:2} \kro \Im_d \quad \cdots \quad 
\wM{L:l+1} \kro \wM{1:l-1}
\quad \cdots \quad 
\Im_K \kro \wM{1:L-1}
\end{pmatrix}\,\, \in \mathbb{R}^{Kd\times p}\,,
\end{equation*}
% \end{align}
\begin{proof}
From Eq.~\eqref{eq:outer-kl} we can notice that any block, say $kl$-th, of $\HO$ can be re-written as, 
\begin{align*}
\HOsup{kl} =\left(\wM{k+1:L} \kro \wM{k-1:1} \right) \; \left(\Im_K \kro \covx\right) \; {\left(\wM{L:l+1} \otimes \wM{1:l-1}  \right)}\,,
\end{align*}

where, we have used the mixed-product property of Kronecker products i.e., $\Am\Bm\kro\Cm\Dm=(\Am\kro\Cm)(\Bm\kro\Dm)$. Now, it is clear from looking at the terms which are on the left and right of $\Im_K \kro \covx$, that we get the required decomposition.

\end{proof}
\end{repproposition}

\begin{remark}\label{remark:cov}
As mentioned in the preliminaries, we consider, without loss of generality, that when the (uncentered) input covariance $\covx$ has rank $r<d$, then we take it to be 

\[\covx = \begin{pmatrix}
     (\covx)_{\,r\times r}& \mb{0}_{\,r\times (d-r)} \\[2mm]
     \mb{0}_{\,(d-r)\times r}& \mb{0}_{\,(d-r)\times (d-r)}
\end{pmatrix}\,.\]

which is always possible by pre-processing the input (although this is not needed in practice). Thus, in such a scenario we can equivalently work with $\covx := (\covx)_{\,r\times r}$ and just consider the first $r$ columns of $\wM{1}$ (whose shape will then be $\mathbb{R}^{M_1\times r}$). Otherwise, if $r=d$, then we just continue with $\covx$ and $\wM{1}\in \mathbb{R}^{M_1\times d}$ as usual. To simplify our discussion ahead, we will always write $r$ in place of $d$, however the meaning of it should be clear from this remark.

\end{remark} 

Thus, in our presentation of the proof of Theorem~\ref{theorem:ub-outer}, we will make a similar adaptation to Proposition~\ref{eq:outer-decomp}, and so
\begin{equation*}
\Am_o^\top=\begin{pmatrix}
\wM{L:2} \kro \Im_r & \cdots \quad 
\wM{L:l+1} \kro \wM{1:l-1}
\quad \cdots &
\Im_K \kro \wM{1:L-1}
\end{pmatrix}\,\, \in \mathbb{R}^{Kr\,\times \, p^\prime}\,,
\end{equation*}

where $p^{\prime} = p+ K(r-d)$.

\subsection{Example for two-layer networks}
Let us first illustrate Theorem~\ref{theorem:ub-outer} via the example of $L=2$, i.e., we have a 2-layer network $\Fn(\x)= \wMb \wMa \x$, with weight matrices $\wMb \in \mathbb{R}^{K\times M_1}$ and $\wMa \in \mathbb{R}^{M_1\times r}$. Applying Proposition~\eqref{eq:outer-decomp}, we obtain $\Am_o$ with the familiar structure:
\[\Am_o = \begin{pNiceArray}{c}
\mb W^{2^\top}  \kro \Im_r \\[1mm]
\Im_K \,\kro \wMa \\[1mm]
\end{pNiceArray}\]

Applying Lemma~\ref{lemma:kronecker-block} on $\Am_o\,$ thus yields:
$\rank(\Am_o) = r \, \rank(\mb W^{2^\top} ) \,+\, K \, \rank({\wMa})\, - \,\rank(\mb W^{2^\top} )  \rank({\wMa})$. If we assume that the hidden layer is the bottleneck, i.e., $q:=\min(r, M_1, K)=M_1$, then we get $\rank(\Am_o) = r \, M_1 + K \, M_1 - {M_1}^2$, keeping in mind the Assumption \ref{assump:2}. 

While here the special $\Zm$-like structure is apparent at the outset, the general case of $L$-layers is more involved and requires additional work to reduce to this structure, as illustrated in our proof ahead.

\subsection{Proof of Theorem~\ref{theorem:ub-outer}}
Let us restate the Theorem~\ref{theorem:ub-outer} from the main text,
\begin{reptheorem}{theorem:ub-outer}
Consider the matrix $\Am_o$ mentioned in Proposition~\ref{eq:outer-decomp}. 
Under the assumption \ref{assump:2},
\[
\rank(\Am_o) = 
r \rank(\wM{2:L}) + K \rank(\wM{L-1:1}) - \rank(\wM{2:L})\rank(\wM{L-1:1}) = q \,( r + K  - q)\,.
\]
\end{reptheorem}

\begin{proof}

The proof is divided into two parts: (1) Bottleneck case and (2) Non-bottleneck case. For more details about the block-row operations that we employ here, please refer to the Section~\ref{supp:row-col-ops}.

\paragraph{Part 1: Bottleneck case.}
We assume without loss of generality that the layer $\ell-1$ has the minimum layer width out of all hidden layers, besides what is known that $M_{\ell-1} < \min(d, K)$. We will therefore have that $\Wm^{\ell}\in \mathbb{R}^{M_{\ell}\times M_{\ell-1}}$ will have full column rank and will be left-invertible. Due to random initialization of weight matrices (see Section~\ref{supp:full_rank_weight}) we will also have $\wM{k:\ell}, \, k\geq \ell$ to also have a left inverse. 

Then let us write the block matrix $\Am_o$ in the column manner and label the row blocks corresponding to layer $\ell$ as $R_\ell$:
\NiceMatrixOptions{code-for-first-row = \color{blue},
                   code-for-first-col = \color{blue},}
\begin{equation}\label{eq:startA}
\Am_o = 
\begin{pNiceArray}[first-row,first-col,nullify-dots]{c}
        & \\[2mm]
R_1& \wM{2:L}\kro  \Im_r   \\[1mm]
\Vdots & \vdots \\[1mm]
R_{\ell-1} & \wM{\ell:L} \kro  \wM{\ell-2:1} \\[2mm]
R_{\ell} & \wM{\ell+1:L} \kro  \wM{\ell-1:1} \\[2mm]
R_{\ell+1} & \wM{\ell+2:L} \kro  \wM{\ell:1} \\[2mm]
\Vdots & \vdots \\[1mm]
R_{L}  &   \mathbf{I}_K \kro \wM{L-1:1}  \\
\end{pNiceArray}.
\end{equation}

Consider the following (block) row operations: 
$$
R_{k} \leftarrow \left(\bm I_{M_{k}} \kro {\wM{k-1:\ell}}\right)^{\,\dagger} R_{k}, \quad \forall k \in {\ell +1, \cdots, L}
$$

These row-operations are valid as the pre-factor has full-column rank, and are rank preserving, as discussed in Section~\ref{supp:row-col-ops}. In this way, we have 
\begin{equation}
\Am_o = 
\begin{pNiceArray}[first-row,first-col,nullify-dots]{c}
        & \\[2mm]
R_1& \wM{2:L}\kro  \Im_r   \\[1mm]
\Vdots & \vdots \\[1mm]
R_{\ell-2} & \wM{\ell-1:L} \kro  \wM{\ell-3:1} \\[2mm]
R_{\ell-1} & \wM{\ell:L} \kro  \wM{\ell-2:1} \\[2mm]
R_{\ell} & \wM{\ell+1:L} \kro  \wM{\ell-1:1} \\[2mm]
R_{\ell+1} & \wM{\ell+2:L} \kro  \wM{\ell-1:1} \\[2mm]
\Vdots & \vdots \\[1mm]
R_{L}  &   \mathbf{I}_K \kro \wM{\ell-1:1}  \\
\end{pNiceArray}.
\end{equation}

Similarly, $\wM{{\ell-1}^\top} \in \mathbb{R}^{M_{\ell-2} \times M_{\ell-1}}$ as well as $\wM{k\,:\,\ell-1}$ for  $k\leq \ell-1$ is also full-column rank and thus left-invertible. Then apply the following row operations, 

$$
R_{k} \leftarrow \left(\wM{k +1\,:\, \ell-1} \kro I_{M_{k-1}}\right)^{\,\dagger} R_{k} , \quad \forall k \in {\ell -2, \cdots, 1}
$$

And we get, 
\begin{equation}
\Am_o = 
\begin{pNiceArray}[first-row,first-col,nullify-dots]{c}[code-before =\rowcolor{blue!15}{1,9}]
        & \\[2mm]
R_1& \wM{\ell:L}\kro  \Im_r   \\[1mm]
\Vdots & \vdots \\[1mm]
R_{\ell-2} & \wM{\ell:L} \kro  \wM{\ell-3:1} \\[2mm]
R_{\ell-1} & \wM{\ell:L} \kro  \wM{\ell-2:1} \\[2mm]
\hline\\
R_{\ell} & \wM{\ell+1:L} \kro  \wM{\ell-1:1} \\[2mm]
R_{\ell+1} & \wM{\ell+2:L} \kro  \wM{\ell-1:1} \\[2mm]
\Vdots & \vdots \\[1mm]
R_{L}  &   \mathbf{I}_K \kro \wM{\ell-1:1}  \\
\end{pNiceArray}.
\end{equation}

Now using $R_1$, we can apply the deletion block operations to remove $\{R_2, \cdots, R_{\ell-1}\}$ as the left term in the Kronecker is identical. Next, via $R_{L}$ we can apply the deletion block operations to get rid of $\{R_\ell, \cdots, R_{L}\}$ as now the right term in the Kronecker is identical.  We are left with:
\begin{equation}
\Am_o = 
\begin{pNiceArray}[first-row,first-col,nullify-dots]{c}
        & \\[2mm]
R_1& \wM{\ell:L}\kro  \Im_r   \\[2mm]
R_{L}  &   \mathbf{I}_K \kro \wM{\ell-1:1}  \\
\end{pNiceArray}.
\end{equation}

Finally, we can apply Lemma~\ref{lemma:kronecker-block} to obtain that, when $\Wm^{\ell}$ has full column rank or ${M_{\ell-1}}$ is the minimum layer-width (i.e., the bottleneck dimension):

\begin{align*}
\rank(\Am_o) 
&= r \rank(\wM{\ell:L}) + K \rank(\wM{\ell-1:1}) - \rank(\wM{\ell:L})\rank(\wM{\ell-1:1}) \\
&= r {M_{\ell-1}} + K {M_{\ell-1}} - {M_{\ell-1}}^2 \\
&=  q \,( r + K  - q)\,,
\end{align*}

where, in the last step, we have used the definition of $q:=\min(r, M_1, \cdots, M_{L-1}, K)=M_{\ell-1}$ which gives rise to the equivalent expression.

\paragraph{Part 2: Non-bottleneck case.} 

This part is very similar and we will see that it uses just one set of row operations (like for the layers $> \ell$ and $< \ell$). In the non-bottleneck case, there are further two possibilities: 

\textbf{When $K$ is the minimum:} This means that $\wMt{L}$ has full column-rank as $M_{L} = K$ is the minimum of all layer widths and input dimensionality. We start from the same $\Am_o$ matrix as in Eq.~\eqref{eq:startA}. Now consider the following factoring-out operations:

$$
R_{k} \leftarrow \left({\wM{k+1:L}} \kro \bm I_{M_{k-1}} \right)^{\,\dagger} R_{k}  \quad \forall k \in {1, \cdots, L-1}
$$

These row-operations are valid (i.e., rank-preserving) as the pre-factor of (block-)row $R_k$ has full column rank, as discussed in Section~\ref{supp:row-col-ops}. This results in,
\begin{equation}
\Am_o = 
\begin{pNiceArray}[first-row,first-col,nullify-dots]{c}[code-before =\rowcolor{blue!15}{1}]
        & \\[2mm]
R_1&\mathbf{I}_K\kro  \Im_r   \\[1mm]
\Vdots & \vdots \\[1mm]
R_{\ell-1} & \mathbf{I}_K\kro  \wM{\ell-2:1} \\[2mm]
R_{\ell} & \mathbf{I}_K \kro  \wM{\ell-1:1} \\[2mm]
R_{\ell+1} & \mathbf{I}_K \kro  \wM{\ell:1} \\[2mm]
\Vdots & \vdots \\[1mm]
R_{L}  &   \mathbf{I}_K \kro \wM{L:1}  \\
\end{pNiceArray}.
\end{equation}\\

This is then followed by,
$$
R_{k} \leftarrow R_{k} - \left( \bm I_{K}  \kro {\wM{k-1:1}} \right) R_{1}, \quad \forall \, k \in {2, \cdots, L}\,.
$$
This results in,
\begin{equation}
\Am_o = 
\begin{pNiceArray}[first-row,first-col,nullify-dots]{c}
        & \\
R_1& \Im_K \kro  \Im_r   \\[1mm]
\end{pNiceArray}.
\end{equation}

\textbf{When $r$ is the minimum:} This means that $\wM{1}$ has full column-rank as $M_{0} = r$ is the minimum of all layer widths, input and output dimensionality.  We start from the same $\Am_o$ matrix as in Eq.~\eqref{eq:startA}. Now consider the following factoring-out operations:

$$
R_{k} \leftarrow \left(\bm I_{M_{k}} \kro {\wM{k-1:1}}\right)^{\,\dagger} R_{k} \quad \forall \, k \in {2, \cdots, L}
$$

As before, these row-operations are valid (i.e., rank-preserving) as the pre-factor of (block-)row $R_k$ has full column rank, as discussed in Section~\ref{supp:row-col-ops}. This results in,

\begin{equation}
\Am_o = 
\begin{pNiceArray}[first-row,first-col,nullify-dots]{c}[code-before =\rowcolor{blue!15}{7}]
        & \\[2mm]
R_1& \wM{2:L}\kro  \Im_r   \\[1mm]
\Vdots & \vdots \\[1mm]
R_{\ell-1} & \wM{\ell:L} \kro   \Im_r \\[2mm]
R_{\ell} & \wM{\ell+1:L} \kro   \Im_r \\[2mm]
R_{\ell+1} & \wM{\ell+2:L} \kro   \Im_r \\[2mm]
\Vdots & \vdots \\[1mm]
R_{L}  &   \mathbf{I}_K \kro  \Im_r  \\
\end{pNiceArray}.
\end{equation}

This is then followed by,
$$
R_{k} \leftarrow  R_{k} - \left({\wM{k+1:L}}  \kro \bm I_{r} \right) R_{L}, \quad \forall k \in {1, \cdots, L-1}\,.
$$

This results in,
\begin{equation}
\Am_o = 
\begin{pNiceArray}[first-row,first-col,nullify-dots]{c}
        & \\
R_L& \Im_K \kro  \Im_r   \\[1mm]
\end{pNiceArray}.
\end{equation}

\textbf{Resulting rank: } Thus, for either scenario of $K$ or $r$ being the minimum we get, $$\rank(\Am_o) = \rank(\bm I_{K} \kro  \Im_r) = K r. $$

\paragraph{Final note.} It is easy to check that for both the first and second case,  we can summarize the obtained rank in the form of the following equality:

$$\rank(\Am_o) = r \rank(\wM{L:2}) + K \rank(\wM{L-1:1}) - \rank(\wM{L:2})\rank(\wM{L-1:1}) =  q \,( r + K  - q).$$

\end{proof}

\subsection{Proof of Corollary~\ref{corollary:outer-rank}}
Let us first recall the Corollary,

\begin{repcorollary}{corollary:outer-rank}
Under the setup of Theorem~\ref{theorem:ub-outer}, the rank of $\HO$ is given by \[\rank(\HO) = q \,( r + K  - q)\,.\]
\end{repcorollary}
\begin{proof}
It is quite evident that we can use the rank of $\Am_o$ to bound the rank of $\HO$ due to the decomposition from Proposition~\ref{eq:outer-decomp}. However, as mentioned in the main text, we can show an equality using the Lemma~\ref{lemma:outer-equality} since $\HO$ is also of the form $\Am \Bm \Am^{\top}$ with $\Bm = (\Im_K \kro \covx) \succ \mb{0}$.

\begin{align*}
    \rank(\HO) \overset{\textrm{Prop.~\ref{eq:outer-decomp}}}{=} \rank(\Am_o\Bm \Am_o^{\top}) = &\overset{\textrm{Lemma~\ref{lemma:outer-equality}}}{=} \rank(\Am_o)\\ &\overset{\textrm{Thm.~\ref{theorem:ub-outer}}}{=} 
    q \,( r + K  - q)\,.
\end{align*}
\end{proof}

\clearpage
\section{Rank of the functional Hessian term}\label{supp:functional}

\subsection{Proof of Theorem~\ref{theorem:func-hess-cols}}
Similar to the outer-product case, here as well the proof is divided into two cases. However, there is a subtle difference in that the input-residual covariance matrix ($\Om=\E[\pmb \delta_{\x, \y} \, \x^\top]$) can also dictate the rank, besides the weight matrix with minimum dimension. Especially since during training, as the residual approaches zero, the matrix $\Om\rightarrow0$. To abstract this, we update our definition of $q$ to include $s=\rank(\Om)$, and thus we get
\[q=\min(r, M_1, \cdots, M_{L-1}, K, s)\,.\]

Notice the additional $s$ as the last argument in the minimum. Further, since $\Om\in \mathbb{R}^{K\times r}$, we have that $s\leq\min(K, r)$. 

Given these considerations, we split our analysis to the case where $q=s$ (referred to as the non-bottleneck case) and where $q=\min(M_1, \cdots, M_{L-1})$ (referred to as the bottleneck case). Note, if $q=r$ or $q=K$, then these cases are already subsumed by $q=s$ case, since $s\leq\min(K,r)$. 

Further, in each of the two cases, we will analyse the rank of the block-columns of the functional Hessian formed with respect to a transposed weight matrix, i.e., $\HFhatsup{\bullet \ell}$. This makes it easier to deal with underlying structure of the actual block-columns $\HFcol{\ell}$, and does not affect the rank as it is invariant to row or column permutations. Finally, before we proceed into the details of the two parts, let us recollect the theorem statement,

\begin{reptheorem}{theorem:func-hess-cols}
For a deep linear network, the rank of $\ell$-th column-block, $\HFhatsup{\bullet \ell}$, of the matrix $\HFhat$, under the assumption \ref{assump:2} is given as
 $\rank(\HFhatsup{\bullet \ell}) = q\, M_{\ell -1} + q\,M_{\ell}  - q^2\,,\,$ for $\ell \in [2, \cdots , L-1]$.
When $\ell=1$, we have $
\rank(\HFhatsup{\bullet 1}) =q\, M_{1} + q \, s - q^2\,.$ And, when $\ell=L$, we have $
\rank(\HFhatsup{\bullet L}) =q\, M_{L-1} + q \, s - q^2\,.$
Here, $q := \min(r, M_1, \cdots, M_{L-1}, K, s)$ and $s:=\rank(\Om)=\rank(\E[\pmb \delta_{\x,\y}\,\x^\top])$.
\end{reptheorem}

\subsubsection{Non-bottleneck case ($q = s$)}

In this case, the Theorem boils down to showing the following:
\begin{enumerate}
\item The rank of the $\ell^\text{th}$ column ($\ell \in [2, \cdots , L-1]$) of the functional Hessian, i.e. $\HFhatsup{\bullet \ell}$, is given by:
\[
\rank(\HFhatsup{\bullet \ell}) = s\, M_{\ell-1} + s\,M_{\ell}  - s^2\]

\item The rank of the first column of the functional Hessian, i.e. $\HFhatsup{\bullet 1}$, is given by:
\[
\rank(\HFhatsup{\bullet 1}) = s\, M_{1} 
\]

\item The rank of the last column of the functional Hessian, i.e. $\HFhatsup{\bullet L}$, is given by:
\[
\rank(\HFhatsup{\bullet L}) = s\, M_{L-1}
\]
\end{enumerate}
\begin{proof}
In this case, $\Om$, which shows up in every single block will dictate the rank. As before, for more details about the factoring-out and deletion block-row operations that we employ here, please refer to the Section~\ref{supp:row-col-ops}. Let us now look at each of the parts:

\paragraph{Part 1.} Let us first consider the case of the inner-columns, with $\ell>1$ and $\ell<L$. Now, the expression for $\HFhatsup{\bullet \ell}$ is given by,

\NiceMatrixOptions{code-for-first-row = \color{blue},
                   code-for-first-col = \color{blue},}
$$
\small
\HFhatsup{\bullet \ell}= \begin{pNiceArray}[first-row,first-col,nullify-dots]{c}
       & \vect_{r}({\wMt{\ell}}) \\[2mm]
\vect_{r}(\wM{1}) & \wM{2:\ell-1} \kro \Om^\top  \wM{L:\ell+1} \\[1mm]
\Vdots & \vdots \\[1mm]
\vect_{r}(\wM{j}) & \wM{j +1:\ell-1} \kro \wM{j-1:1} \Om^\top \wM{L:\ell+1}  \\[1mm]
\Vdots & \vdots \\[1mm]
\vect_{r}(\wM{\ell-1}) & \Im_{M_{\ell-1}}\kro \wM{\ell-2:1} \Om^\top  \wM{L:\ell+1}  \\[2mm]
\hline\\
\vect_{r}(\wM{\ell}) & \bm{0}  \\[2mm]
\hline\\
\vect_{r}(\wM{\ell+1}) &  \wM{\ell+2:L} \Om \wM{1:\ell-1} \kro \Im_{M_{\ell}} \\[1mm]
\Vdots & \vdots \\[1mm]
\vect_{r}(\wM{k}) & \wM{k+1:L} \Om \wM{1:\ell-1} \kro \wM{k-1:\ell+1} \\[1mm]
\Vdots & \vdots \\[1mm]
\vect_{r}(\wM{L}) & \Om \wM{1:\ell-1} \kro \wM{L-1:\ell +1} \\[1mm]

\end{pNiceArray}
$$

Now, since $\rank(\Om) = s$, we can express it as $\Om = \Cm \Dm$, where $\Cm \in \mathbb{R}^{K\times s}$ and $\Dm \in \mathbb{R}^{s \times d}$. Then consider the following factoring-out block operations: 

$$
R_{j} \leftarrow  \left(\Im_{m_j} \kro \wM{j-1:1}\Dm^\top \right)^\dagger R_{j}, \quad \forall j \in {1, \cdots, \ell-1}
$$

where, $^\dagger$ denotes the left-inverse. When $j=1$, $\wM{j-1:1}:=\Im_r$. These row-operations are valid (i.e., rank-preserving) as $\wM{j-1}\Dm^\top$ has full column rank and is left-invertible. Similarly, consider the following factoring-out block operations: 

$$
R_{k} \leftarrow  \left(\wM{k+1:L} \Cm \kro \Im_{M_{k-1}} \right)^\dagger R_{k}, \quad \forall k \in {\ell+1, \cdots, L}
$$

Note, when $k=L$, $\wM{k+1:L}:=\Im_K$. These row-operations are also valid (i.e., rank-preserving) as $\wM{k+1:L}\Cm$ has full column rank and is left-invertible. Then we can express the $\HFhatsup{\bullet \ell}$ as follows:
$$
\small
\HFhatsup{\bullet \ell}= \begin{pNiceArray}[first-row,first-col,nullify-dots]{c}
       & \vect_{r}({\wMt{\ell}}) \\[2mm]
\vect_{r}(\wM{1}) & \wM{2:\ell-1} \kro \Cm^\top  \wM{L:\ell+1} \\[1mm]
\Vdots & \vdots \\[1mm]
\vect_{r}(\wM{j}) & \wM{j +1:\ell-1} \kro \Cm^\top \wM{L:\ell+1}  \\[1mm]
\Vdots & \vdots \\[1mm]
\vect_{r}(\wM{\ell-1}) & \Im_{M_{\ell-1}}\kro \Cm^\top  \wM{L:\ell+1}  \\[2mm]
\hline\\
\vect_{r}(\wM{\ell}) & \bm{0}  \\[2mm]
\hline\\
\vect_{r}(\wM{\ell+1}) & \Dm \wM{1:\ell-1} \kro \Im_{M_{\ell}} \\[1mm]
\Vdots & \vdots \\[1mm]
\vect_{r}(\wM{k}) & \Dm \wM{1:\ell-1} \kro \wM{k-1:\ell+1} \\[1mm]
\Vdots & \vdots \\[1mm]
\vect_{r}(\wM{L}) & \Dm \wM{1:\ell-1} \kro \wM{L-1:\ell +1} \\[1mm]

\end{pNiceArray}
$$

Now, it is clear that we can use row block $\ell-1$ to eliminate all row blocks prior to it and row block $\ell+1$ to eliminate all row blocks after it by the following deletion block operations:

$$
R_{j} \leftarrow R_{j} - \left(\wM{j +1:\ell-1} \kro \Im_{q}\right) R_{\ell-1}, \quad \forall j \in {1, \cdots, \ell-1}
$$
$$
R_{k} \leftarrow R_{k} - \left(\Im_{q} \kro \wM{k-1:\ell+1}\right) R_{\ell+1}, \quad \forall k \in {\ell+1, \cdots, L}
$$

Hence, we have that \begin{align*}
\rank(\HFhatsup{\bullet \ell}) = \rank\left(\begin{array}{c}
      \Im_{M_{\ell-1}}\kro \Cm^\top  \wM{L:\ell+1} \\[1mm]
      \Dm \wM{1:\ell-1} \kro \Im_{M_{\ell}}
\end{array}\right)  =  s\, M_{\ell-1} + s\,M_{\ell}  - s^2 \end{align*}
where, we used the Lemma from~\cite{CHUAI2004129} in the last step. 

\paragraph{Part 2.} The procedure for this part will follow Part 1 procedure for blocks \textbf{after} the zero block.
\begin{align*}
\rank(\HFhatsup{\bullet \ell}) = \min(s\, M_1,  d\, M_1) = s\, M_1
\end{align*}
\paragraph{Part 3.} The procedure for this part will follow Part 1 procedure for blocks \textbf{before} the zero block.

\begin{align*}
\rank(\HFhatsup{\bullet \ell}) 
=\min(s\, M_{L-1},  K\, M_{L-1}) = s\, M_{L-1}
\end{align*}

\end{proof}

\subsubsection{Bottleneck case ($q \neq s$)}
Here, we need to prove the following:

\begin{enumerate}
\item The rank of the $\ell^\text{th}$ column ($\ell \in [2, \cdots , L-1]$) of the functional Hessian, i.e. $\HFhatsup{\bullet \ell}$, is given by:

$$
\rank(\HFhatsup{\bullet \ell}) =q\, M_{\ell-1} + q\,M_{\ell}  - q^2$$

\item The rank of the first column of the functional Hessian, i.e. $\HFhatsup{\bullet 1}$, is given by:

$$
\rank(\HFhatsup{\bullet 1}) = q\, M_{1} + q \, s - q^2
$$

\item The rank of the last column of the functional Hessian, i.e. $\HFhatsup{\bullet L}$, is given by:

$$
\rank(\HFhatsup{\bullet L}) =q\, M_{L-1} + q \, s - q^2
$$
\end{enumerate}

\begin{proof}
Let us assume that $M_k$ is the bottleneck width, so it will ``dictate" the rank now. Like in previous parts, for more details about the block-row operations that we employ here, please refer to the Section~\ref{supp:row-col-ops}. Let us now look at each of the parts:

\paragraph{Part 1.} Let us first consider the case of the inner-columns, with $\ell>1$ and $\ell<L$. Further, let us take $k>\ell$, and the procedure in the other scenario of $k<\ell\,$ is similar. 

We have that $\HFhatsup{\bullet \ell}$ is as follows:

\NiceMatrixOptions{code-for-first-row = \color{blue},
                   code-for-first-col = \color{blue},}
$$
\small
\HFhatsup{\bullet \ell}= \begin{pNiceArray}[first-row,first-col,nullify-dots]{c}
       & \vect_{r}({\wMt{\ell}}) \\[2mm]
\vect_{r}(\wM{1}) & \wM{2:\ell-1} \kro \Om^\top  \wM{L:\ell+1} \\[1mm]
\Vdots & \vdots \\[1mm]
\vect_{r}(\wM{j}) & \wM{j +1:\ell-1} \kro \wM{j-1:1} \Om^\top \wM{L:\ell+1}  \\[1mm]
\Vdots & \vdots \\[1mm]
\vect_{r}(\wM{\ell-1}) & \Im_{M_{\ell-1}}\kro \wM{\ell-2:1} \Om^\top  \wM{L:\ell+1}  \\[2mm]
\hline\\
\vect_{r}(\wM{\ell}) & \bm{0}  \\[2mm]
\hline\\
\vect_{r}(\wM{\ell+1}) &  \wM{\ell+2:L} \Om \wM{1:\ell-1} \kro \Im_{M_{\ell}} \\[1mm]
\Vdots & \vdots \\[1mm]
\vect_{r}(\wM{k}) & \wM{k+1:L} \Om \wM{1:\ell-1} \kro \wM{k-1:\ell+1} \\[1mm]
\Vdots & \vdots \\[1mm]
\vect_{r}(\wM{L}) & \Om \wM{1:\ell-1} \kro \wM{L-1:\ell +1} \\[1mm]

\end{pNiceArray}
$$

Notice, we can write $\wM{i+1:L} = \wM{i+1:k} \wM{k+1:L}, \, \forall i \leq k -1 $. The matrix $\wM{i+1:k}$ is left-invertible as $M_k$ is the bottleneck width. Then, we consider the following factoring-out block operations (see Section~\ref{supp:row-col-ops}) for the below zero part: 

$$
R_{i} \leftarrow  \left(\wM{i+1:k} \kro \Im_{M_{i-1}} \right)^\dagger R_{i}, \quad \forall i \in {\ell +1, \cdots, k-1}
$$

where, $^\dagger$ denotes the left-inverse. Now, for a layer $i$ between $\{k+2, \cdots,L\}$, we notice that $\wM{i-1:\ell+1} = \wM{i-1:k+1} \wM{k:\ell+1}$ with $\wM{i-1:k+1}$ being left-invertible. Hence consider the factoring-out operations, 

$$
R_{i} \leftarrow  \left(\Im_{M_{i}} \kro \wM{i-1:k+1} \right)^\dagger R_{i}, \quad \forall i \in {k +2, \cdots, L}
$$

Now, for the row-blocks in the part above zero, we have that $\wM{j-1:1}\Om^\top\wM{L:\ell+1} = \wM{j-1:1}\Om^\top\wM{L:k+1}\wM{k:\ell+1}, \, \forall j \leq \ell-1$. Notice, that $\wM{j-1:1}\Om^\top\wM{L:k+1}$ is left invertible due to the bottleneck $M_k$. Hence, consider the following block-row operations: 

$$
R_{j} \leftarrow  \left(\Im_{m_j}\kro \wM{j-1:1}\Om^\top\wM{L:k+1} \right)^\dagger R_{j}, \quad \forall j \in {1, \cdots, \ell-1}
$$

Here, it does not matter to us what the actual value of $s$ is, since we know that $M_k$ is the minimum width which guarantees that the above expression containing $\Om^\top$ is left invertible.

\noindent Overall, we can thus express the $\HFhatsup{\bullet \ell}$ as follows:
$$
\small
\HFhatsup{\bullet \ell}= \begin{pNiceArray}[first-row,first-col,nullify-dots]{c}
       & \vect_{r}({\wMt{\ell}}) \\[2mm]
\vect_{r}(\wM{1}) & \wM{2:\ell-1} \kro \wM{k:\ell+1} \\[1mm]
\Vdots & \vdots \\[1mm]
\vect_{r}(\wM{j}) & \wM{j +1:\ell-1} \kro \wM{k:\ell+1}  \\[1mm]
\Vdots & \vdots \\[1mm]
\vect_{r}(\wM{\ell-1}) & \Im_{M_{\ell-1}}\kro \wM{k:\ell+1}  \\[2mm]
\hline\\
\vect_{r}(\wM{\ell}) & \bm{0}  \\[2mm]
\hline\\
\vect_{r}(\wM{\ell+1}) & \wM{k+1:L}\Om \wM{1:\ell-1} \kro \Im_{M_{\ell}} \\[1mm]
\Vdots & \vdots \\[1mm]
\vect_{r}(\wM{k}) & \wM{k+1:L} \Om \wM{1:\ell-1} \kro \wM{k-1:\ell+1} \\[1mm]
\Vdots & \vdots \\[1mm]
\vect_{r}(\wM{L}) & \Om \wM{1:\ell-1} \kro \wM{k:\ell +1} \\[1mm]

\end{pNiceArray}
$$

Now, it is clear that we can use row $\ell-1$ to eliminate all rows prior to it as well as the rows from $k+1$ to $L$. While the row $\ell+1$ can be used to eliminate all rows from $\ell+2$ until $k$. Hence, we have that, 

\begin{align*}
\rank(\HFhatsup{\bullet \ell}) = \rank\left(\begin{array}{c}
      \Im_{M_{\ell-1}}\kro \wM{k:\ell+1} \\[2mm]
      \wM{k+1:L} \Om \wM{1:\ell-1} \kro \Im_{M_{\ell}}
\end{array}\right)  =  q\, M_{\ell-1} + q\,M_{\ell}  - q^2 \end{align*}
where, we used the Lemma from~\cite{CHUAI2004129} in the last step. 

\paragraph{Part 2.} Now we deal with the column block corresponding to first layer, which is:
$$
\small
\HFhatsup{1}= \begin{pNiceArray}[first-row,first-col,nullify-dots]{c}
       & \vect_{r}({\wM{1}}^\top) \\[2mm]
\vect_{r}(\wM{1}) & \bm{0}  \\[2mm]
\vect_{r}(\wM{2}) &  \wM{3:L} \Om  \kro \Im_{M_1}  \\[1mm]
\Vdots & \vdots \\[1mm]
\vect_{r}(\wM{k-1}) &  \wM{k:L} \Om \kro \wM{k-2:2}  \\[2mm]
\vect_{r}(\wM{k}) & \wM{k+1:L} \Om \kro \wM{k-1:2} \\[2mm]
\vect_{r}(\wM{k+1}) &  \wM{k+2:L} \Om \kro \wM{k:2}  \\[1mm]
\Vdots & \vdots \\[1mm]
\vect_{r}(\wM{L}) & \Om \kro \wM{L-1:2} \\[1mm]

\end{pNiceArray}
$$

Basically, we have to follow the same procedure for row blocks before $k$ and for the ones after $k$ as done in the part 1.
In other words, for the row blocks prior to $k$, we can write the $\wM{i+1:L} = \wM{i+1:k}\wM{k+1:L}$. Since $\wM{i+1:k}$ is left-invertible due to the bottleneck, we consider the following factoring-out operations:

$$
R_{i} \leftarrow  \left(\wM{i+1:k} \kro \Im_{M_{i-1}} \right)^\dagger R_{i}, \quad \forall i \in {2, \cdots, k-1}
$$

Now, for a layer $i$ between $\{k+2, \cdots,L\}$, we notice that $\wM{i-1:2} = \wM{i-1:k+1} \wM{k:2}$ with $\wM{i-1:k+1}$ being left-invertible. Hence the factoring-out operations will be, 

$$
R_{i} \leftarrow  \left(\Im_{M_{i}} \kro \wM{i-1:k+1} \right)^\dagger R_{i}, \quad \forall i \in {k +2, \cdots, L}
$$

This resulting $\HFhatsup{1}$ is as follows:

$$
\small
\HFhatsup{1}= \begin{pNiceArray}[first-row,first-col,nullify-dots]{c}
       & \vect_{r}({\wM{1}}^\top) \\[2mm]
\vect_{r}(\wM{1}) & \bm{0}  \\[2mm]
% \hline\\
\vect_{r}(\wM{2}) &  \wM{k+1:L} \Om  \kro \Im_{M_1}  \\[1mm]
\Vdots & \vdots \\[1mm]
\vect_{r}(\wM{k-1}) &  \wM{k+1:L} \Om \kro \wM{k-2:2}  \\[2mm]
\vect_{r}(\wM{k}) & \wM{k+1:L} \Om \kro \wM{k-1:2} \\[2mm]
\vect_{r}(\wM{k+1}) &  \wM{k+2:L} \Om \kro \wM{k:2}  \\[1mm]
\Vdots & \vdots \\[1mm]
\vect_{r}(\wM{L}) & \Om \kro \wM{k:2} \\[1mm]

\end{pNiceArray}
$$

Now, we can use row $2$ to eliminate rows $3$ until $k$, and similarly we can use row $L$ to eliminate rows $k+1$ to $L-1$. Thus, we have that:
\vspace{-1mm}
\begin{align*}
\rank(\HFhatsup{1}) =  \rank\left(\begin{array}{c}
       \wM{k+1:L} \Om  \kro \Im_{M_1} \\ [1mm]
       \Om \kro \wM{k:2} 
\end{array}\right)  = q\, M_1 + q\, s - q^2\end{align*}

Different from the previous analysis, in the above matrix we did not have a ``naked" Identity matrix on the left, so we could not directly use the Lemma~\ref{lemma:kronecker-block}. But, we used its generalized version contained in Lemma~\ref{lemma:supp-kronecker-block-general} along with the Lemma~\ref{lemma:block-shared} to obtain the rank in the final step. 

\paragraph{Part 3.} Finally, we have the last column:

$$
\small
\HFhatsup{\bullet L}= \begin{pNiceArray}[first-row,first-col,nullify-dots]{c}
       & \vect_{r}({\wM{L}}^\top) \\[2mm]
\vect_{r}(\wM{1}) & \wM{2:L-1} \kro \Om^\top \\[1mm]
\Vdots & \vdots \\[1mm]
\vect_{r}(\wM{k-1}) & \wM{k:L-1} \kro \wM{k-2:1} \Om^\top  \\[1mm]
\vect_{r}(\wM{k}) & \wM{k +1:L-1}\kro \wM{k-1:1} \Om^\top   \\[2mm]
\vect_{r}(\wM{k+1}) & \wM{k +2:L-1}\kro \wM{k:1} \Om^\top   \\[1mm]
\Vdots & \vdots \\[1mm]
\vect_{r}(\wM{L-1}) & \Im_{M_{L-1}} \kro \wM{L-2:1} \Om^\top  \\[1mm]
\vect_{r}(\wM{L}) & \bm{0}  \\[2mm]
\end{pNiceArray}
$$

We can follow a similar strategy as carried out in Part 2 to get, 

\begin{align*}
\rank(\HFhatsup{\bullet L}) = \rank\left(\begin{array}{c}
       \wM{k+1:L-1}\kro \Om^\top \\ [1mm]
       \Im_{M_{L-1}} \kro \wM{k:1} \Om^\top
\end{array}\right)  = q\, M_{L-1} + q\, s - q^2\end{align*}

\end{proof}

\subsection{Proof of Corollary~\ref{corollary:functional-rank}}

Let us remember the Corollary from the main text, 
\begin{repcorollary}{corollary:functional-rank}
Under the setup of Theorem~\ref{theorem:func-hess-cols}, the rank of $\HF$ can be upper bounded as,
\[\rank(\HF) \leq 2\, q\, M  + 2 \,q\, s - L\, q^2\,, \quad \text{where} \quad M=\sum_{\ell=1}^{L-1} M_{\ell}\,.\]
\end{repcorollary}

\begin{proof}

By using the above Theorem~\ref{theorem:func-hess-cols} and applying the fact that $\rank(\left[\Am \,\, \Bm\right])\leq \rank(\Am) + \rank(\Bm)$ on the column-blocks of $\HFhat$, we get the desired upper bound on the rank of entire functional Hessian.
, 
\begin{align*}
    \rank(\HF) &= \rank(\HFhat) \,\,\leq \,\, \sum\limits_{\ell=1}^L \rank(\HFhatsup{\bullet \ell}) \\
    &= q\, M_1 + q\, s - q^2 + \sum\limits_{\ell=2}^{L-1} \left(q\, M_{\ell -1} + q\,M_{\ell}  - q^2\right) + q\, M_{L-1} + q\, s - q^2 \\
    &= 2\, q\, \left(\sum\limits_{\ell=1}^{L-1}M_{\ell}\right)  + 2 \,q\, s - L\, q^2 \,\,=\,\, 2\, q\, M  + 2 \,q\, s - L\, q^2
\end{align*}
\end{proof}

\clearpage

\section{Evolution of Rank during training}
\label{supp:evol-proof}
Here we prove Lemma \ref{dynamics_of_rank}, stating that the rank of the individual weights remains invariant under gradient flow dynamics.
For sake of readability, let us restate Lemma \ref{lemma:evolve}:
\begin{replemma}{lemma:evolve}
For a deep linear network, consider the gradient flow dynamics $\dot{\Wm}^{l}_t = -\eta \nabla_{\Wm^{l}}\mathcal{L}_S(\bm{\btheta})\big{|}_{\bm{\btheta} \,=\, \bm{\btheta}_t}$. Assume: (a) Centered classes: $ \frac{1}{N} \sum_{i:y_{ic} = 1}^{N}\x_i= \bm{0}, \hspace{2mm} \forall\, c \in [1,\dots,K]$. (b) Balancedness at initialization: ${\Wm^{{k+1}^{\top} }_0}\Wm_0^{k+1} = \Wm^{k}_0\\Wm^{k^{\top}}_0$. (c) Square weight-matrices: $\Wm^{l} \in \mathbb{R}^{M \times M}\,,$ $\,\forall \,l\,$ and $\,K=d=M$. Then for all layers $l$,
$\,\, \rank(\Wm^{l}_t) = \rank(\Wm^{l}_0), \hspace{3mm} \,\forall \, t<\infty\,.$
\end{replemma}
To prove Lemma \ref{lemma:evolve}, we first need some helper lemmas. We start with a well-known result:
\begin{lemma}
Under assumption b) in Lemma \ref{lemma:evolve} and gradient flow dynamics, it holds that
$$\left(\Wm^{k+1}_t\right)^{T}\Wm_t^{k+1} = \Wm^{k}_t\left(\Wm^{k}_t\right)^{T}$$
\end{lemma}
\begin{proof}
We refer to \citep{arora2018optimization} (Theorem 1) for a proof.
\end{proof}
\noindent
This essentially guarantees that balancedness is preserved throughout training if we guarantee it at initialization. Let us assume that the weights $\Wm^{l}$ are full rank, which for the squared matrix case is equivalent to $\operatorname{det}\left(\Wm^{l}\right) \not = 0$. The Jacobi formula allows us to extend the dynamics of $\Wm^{l}$ to the determinant:
\begin{lemma}
\label{det_dynamics}
Given a dynamic matrix $t \mapsto \Am(t)$, the determinant follows:
$$\frac{d}{dt}\operatorname{det}\left(\Am(t)\right) = \operatorname{tr}\left(\operatorname{adj}\left(\Am(t)\right)\frac{d \Am(t)}{dt}\right)$$
where $\operatorname{adj}\left(\Am(t)\right)$ is the adjugate matrix satisfying $$\Am(t) \operatorname{adj}\left(\Am(t)\right) =  \operatorname{adj}\left(\Am(t)\right)\Am(t)= \operatorname{det}(\Am(t)) \bm{I}$$
\end{lemma}
\begin{proof}
For a proof of this standard result we refer to \citet{magnus99} for instance.
\end{proof}

\vspace{2mm}
\noindent
We are now ready to proof the main claim:\\
\textbf{Proof of Lemma \ref{lemma:evolve}:}
Let us first simplify the right-hand side of the gradient flow equation using the assumptions, before applying Lemma \ref{det_dynamics} :
\begin{equation*}
    \begin{split}
        \nabla_{\Wm^{l}}\mathcal{L}(\Wm)\big{|}_{\Wm = -\eta \Wm_t} &= -\eta\Wm_t^{l+1:L}\bm{\Omega}\Wm_t^{1:l-1} \\&= -\eta\Wm_t^{l+1:L}\Wm_t^{L:l+1}\Wm_t^{l}\Wm_t^{l-1 : 1}\bm{\Sigma}\Wm_t^{1:l-1} \\
        &\stackrel{L1}{=} -\eta\left(\Wm_t^{l}\Wm_t^{l  T}\right)^{L-l}\Wm_t^{l}\Wm_t^{l-1 : 1}\bm{\Sigma}\Wm_t^{1:l-1}
    \end{split}
\end{equation*}
Applying Lemma \ref{det_dynamics}, gives
\begin{equation*}
    \begin{split}
        \frac{d}{dt}\operatorname{det}\left(\Wm^{l}_t\right) &= \operatorname{tr}\left(\operatorname{adj}\left(\Wm^{l}_t\right)\dot{\Wm}^{l}_t\right) \\
        &= -\eta\operatorname{tr}\left(\operatorname{adj}\left(\Wm^{l}_t\right)\left(\Wm_t^{l}\Wm_t^{l  T}\right)^{L-l}\Wm_t^{l}\Wm_t^{l-1 : 1}\bm{\Sigma}\Wm_t^{1:l-1}\right) \\
        &= -\eta\operatorname{tr}\left(\operatorname{det}\left(\Wm^{l}_t\right)\left(\Wm_t^{l T}\Wm_t^{l  }\right)^{L-l}\Wm_t^{l-1 : 1}\bm{\Sigma}\Wm_t^{1:l-1}\right) \\
        &= -\eta\operatorname{det}\left(\Wm^{l }_t\right)\operatorname{tr}\left(\left(\Wm_t^{l T}\Wm_t^{l }\right)^{L-l}\Wm_t^{l-1 : 1}\bm{\Sigma}\Wm_t^{1:l-1}\right) \\
        &= -\eta\operatorname{det}\left(\Wm^{l}_t\right) \operatorname{tr}(\Bm(t))
    \end{split}
\end{equation*}
We can write the solution of this differential equation as 
$$\operatorname{det}\left(\Wm^{l }_t\right) = e^{-\eta \int_{0}^{t} \operatorname{tr}(\Bm(s)) ds}\operatorname{det}\left(\Wm^{l}_0\right)$$
Of course, we have no idea how to solve the integral in the exponential, but the exact solution does not matter as it is always positive and thus not zero. Thus, as long as $\operatorname{det}\left(\Wm^{l}_0\right) \not= 0$, also $\operatorname{det}\left(\Wm^{l}_t\right) \not = 0$ holds, at least for a finite time horizon $t < \infty$. \\[3mm]
We illustrated this theoretical finding through non-linear networks. To complete the picture we also show a linear network with its corresponding rank dynamics for weights and Hessians in Figure \ref{mnist_linear}. We use a linear teacher setting, i.e. $\y = \Wm \x$ where $\x \sim \mathcal{N}(\bm{0}, \Im)$ and $\Wm \in \mathbb{R}^{d \times K}$, allowing the model thus to achieve a perfect training error. We train the model for $2000$ epochs with SGD. Again we observe how the weights remain full-rank throughout training and as a consequence, the rank of the Hessians remains constant initially. Due to the linear nature of the teacher, an exact training error of zero is achievable. As soon as the error reaches a certain threshold (around $10^{-30}$), the contribution of the functional Hessian starts to vanish as its eigenvalues, one-by-one become too small to count towards the rank. As a consequence, the rank of the loss Hessian also decreases as it is composed of both the functional and the outer Hessian. Also observe how the outer Hessian remains constant throughout the entire optimization, as it does not depend on error. As expected, at the end of training, the loss Hessian collapses onto the outer Hessian. 
\begin{figure}
    \centering
    \includegraphics[width=0.45\textwidth]{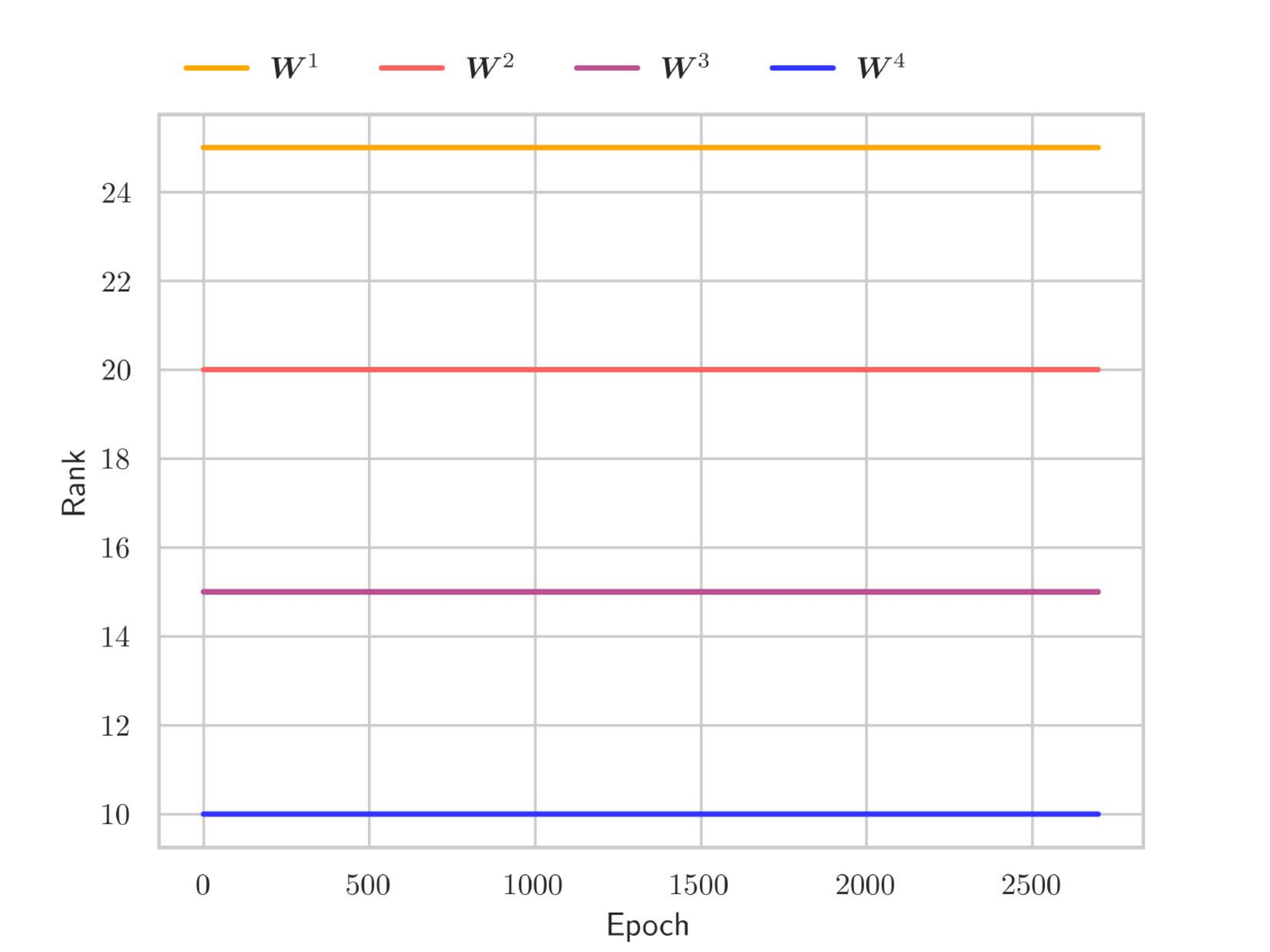}%
    \includegraphics[width=0.45\textwidth]{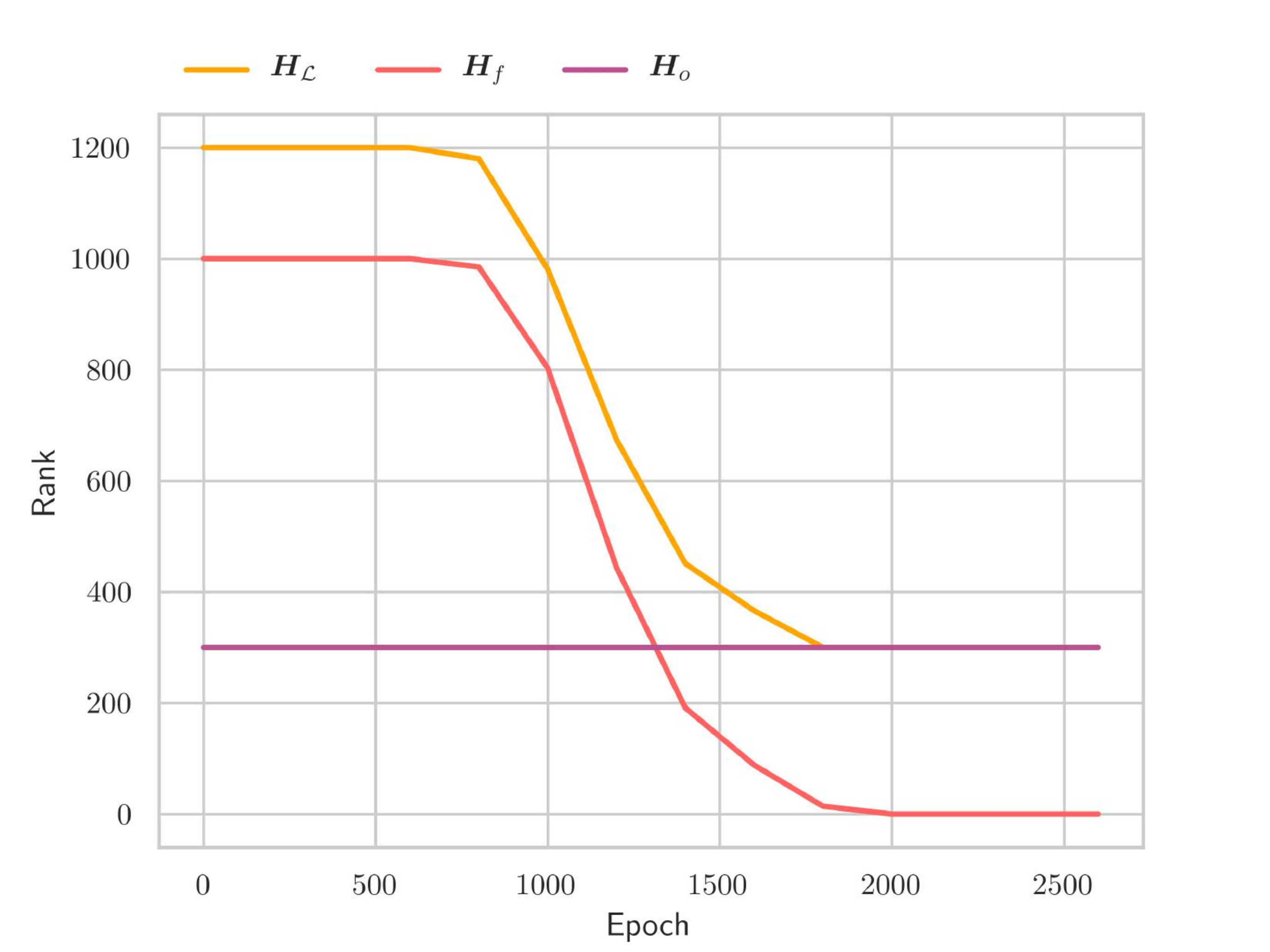}%
    \caption{\textbf{Evolution of Rank}:  We display the rank of the weights $\wM{i}$ (left) and the rank of the Hessians (right) throughout training. We use a $3$ hidden-layer linear network trained on Gaussian data which is labeled by a linear teacher.}
    \label{mnist_linear}
\end{figure}
\clearpage

\section{Pessimistic bounds in the non-linear case}\label{supp:relu-rank}

\subsection{Proof strategy for the $\HO$ upper bound}
The key idea behind our upcoming proof strategy is to do the analysis à la super-position of ``unit-networks'', i.e., networks with one-hidden neuron. This is because we can reformulate the network function $F_{\btheta}$ as sum of network functions of $M$ unit-networks, one per each hidden neuron. Mathematically, 

\begin{equation*}
F_{\btheta}(\x) = \sum\limits_{i=1}^M F_{\theta_i}(\x)\,,\quad \text{where,}\quad F_{\theta_i}(\x) = \wMb_{\bullet \, i} \, \sigma(\mb W^{1^\top}_{i \,\bullet} \, \x). 
\end{equation*}

where, $\wMa_{i \,\bullet} \in \mathbb{R}^{r}$ is the $i$-th row of $\wMa$ and $\wMb_{\bullet \, i}\in \mathbb{R}^{K}$ is the $i$-th column of $\wMb$. 
To ease notation, we will henceforth use the shorthand: $\wMasupp:=\wMa$, and $\wMbsupp:=\wMb$. The column and row indexing follow the convention as before.

Our approach will be to analyze the rank of the outer-product Hessian $\HO$ in the case of empirical loss, or in other words with finitely many samples. Then, we will do a limiting argument to extend it to case of population loss. The benefit of such an approach is that we can bound the rank of $\HO$ via the rank of the Jacobian of the function, $\nabla_{\btheta} F_{\btheta}(\Xm)$ formed over the entire data matrix $\Xm\in \mathbb{R}^{d\times N}$. We will, in turn, bound the rank of such a Jacobian matrix $\nabla_{\btheta} F_{\btheta}(\Xm)$ via the ranks of the Jacobian matrices for the unit-network functions $\nabla_{\btheta} F_{\theta_i}(\Xm)$.

\begin{remark}\label{remark:data}
Similar to Remark~\ref{remark:cov}, when the rank $r$ of the data matrix (or alternatively the empirical (uncentered) input covariance) is less than $d$, we will consider without loss of generality that $\Xm \in \mathbb{R}^{r\times N}$. 
\end{remark}
In other words, for the sake of analysis, we consider that the data matrix has been pre-processed to take into account rank lower than input dimension. 

\subsection{Lemmas on the structure and rank of Jacobian matrix of the unit-network}

As a first step, the following Lemma~\ref{lemma:relu-jacobian-struc} shows the structure of the Jacobian matrix for the unit-network, formed over the entire data matrix.

\begin{lemma}\label{lemma:relu-jacobian-struc}
Consider the unit-network $\Fni{i}(\x) = \wMbsupp_{\bullet \, i} \, \sigma(\wMasupp^\top_{i \,\bullet} \, \x)$ corresponding to i-th neuron, with the non-linearity $\sigma$ such that $\sigma(z) = \sigma^\prime(z) z$. Let the data-matrix be denoted by $\Xm \in \mathbb{R}^{r \times N}$. Further, let us denote the matrix which has on its diagonal --- the activation derivatives $\sigma^\prime(\wMasupp_{i \,\bullet}^\top \, \x)$ over all the samples $\x$, and zero elsewhere, by $\bm \Lambda^i \in \mathbb{R}^{N\times N}$.  Then the Jacobian matrix $\nabla_{\btheta} \Fni{i}(\Xm)$ is given (in transposed form) by, 

\NiceMatrixOptions{code-for-first-row = \color{blue},
                   code-for-first-col = \color{blue},}
$$\nabla_{\btheta} F_{\theta_i}(\Xm)^\top = \hspace{-1cm}\begin{pNiceArray}[first-row,first-col,nullify-dots]{c}
       & \\
\Vdots &        \bm{0} \\[1mm]
\wMasupp_{i \,\bullet}     & \Xm \, \bm \Lambda^i \kro \wMbsupp^\top_{\bullet \,i}  \\[2mm]
\wMbsupp_{\bullet \,i}    &  \wMasupp_{i \,\bullet}^\top  \Xm \, \bm \Lambda^i \kro \Im_{K}  \\[1mm]
\Vdots &        \bm{0} \\
\end{pNiceArray}
$$
\end{lemma}
\begin{proof}

Let us start simple by computing the gradient with respect to $k$-th component of the function, i.e. $F^k$, which comes out as follows:

\[
\frac{\partial F^k}{\partial w_{ij}} = \mathds{1}\lbrace k=i\rbrace \, {\sigma(\Vm\x)}_j
\]
\[
\frac{\partial F^k}{\partial v_{ij}} = w_{ki} \,\, \sigma^\prime(\wMasupp_{i \,\bullet}^\top\,\x)\,\, x_j
\]

 where ${\sigma(\Vm\x)}_j = \sigma( \Vm_{j \,\bullet}^\top \,\x)$, with $\Vm_{j \,\bullet}$ being the $j$-th row of $\Vm$, \textit{in a column vector format} as mentioned above. And, for example in case of ReLU, $\sigma^\prime(z) = \partial\, \sigma(z)/\partial z = \mathds{1}\lbrace z>0\rbrace$.

For a fixed sample $\x$, consider that the activation derivatives for all hidden-neurons are stored in a diagonal matrix $\bm\Lambda^{\x}\in\mathbb{R}^{M\times M}$, i.e., $\bm \Lambda^{\x}_{jj} =  \sigma^\prime( \Vm_{j \,\bullet}^\top \,  \x)$. Then we can rewrite $\sigma(\Vm \x) = \bm \Lambda^{\x} \, \Vm \x$ for all non-linearities that satisfy $\sigma(z) = \sigma^\prime(z) z$. So, we have that the Jacobian of the function with respect to the parameters comes out to be,

\NiceMatrixOptions{code-for-first-row = \color{blue},
                   code-for-first-col = \color{blue},}
\[\nabla_{\bm\theta} F_{\btheta}(\x)^\top= \hspace{-1cm}\begin{pNiceArray}[first-row,first-col,nullify-dots]{ccccc}
       & \bm{F}^{1} & \cdots &  \bm{F}^{k} &  \cdots &  \bm{F}^{K}\\[2mm]
\wMasupp_{1 \bullet}   & w_{11} \, \bm \Lambda^{\x}_{11} \,\, \x & \cdots &  w_{k1}  \, \bm \Lambda^{\x}_{11} \,\, \x & \cdots &  w_{K1}  \, \bm \Lambda^{\x}_{11} \,\, \x\\
\Vdots    & \vdots &  & \vdots &  & \vdots \\[1mm]
\wMasupp_{i \bullet}   & w_{1i} \,  \bm \Lambda^{\x}_{ii}  \,\, \x & \cdots &  w_{ki} \,   \bm \Lambda^{\x}_{ii}  \,\, \x & \cdots &  w_{Ki}  \,  \bm \Lambda^{\x}_{ii}  \,\, \x\\
\Vdots    & \vdots &  & \vdots &  & \vdots \\[1mm]
\wMasupp_{M \bullet}   & w_{1M} \,   \bm \Lambda^{\x}_{MM} \,\, \x & \cdots &  w_{kM} \,   \bm \Lambda^{\x}_{MM} \,\, \x& \cdots &  w_{KM}  \,  \bm \Lambda^{\x}_{MM} \,\, \x\\[2mm]
\hline 
& & & &  \\[-1mm]
\wMbsupp_{1 \bullet}   &  \bm \Lambda^{\x} \, \Vm \x & \cdots & \bm{0} & \cdots &  \bm{0}\\
\Vdots    & \vdots & \ddots &  &  \ddots& \vdots \\[1mm]
\wMbsupp_{k \bullet}    &  \bm{0}&  &  \bm \Lambda^{\x} \, \Vm \x & &  \bm{0}\\
\Vdots    & \vdots &  \ddots &  &  \ddots & \vdots \\[1mm]
\wMbsupp_{K \bullet}   &  \bm{0}& \cdots & \bm{0} & \cdots &  \bm \Lambda^{\x} \, \Vm \x \\
\end{pNiceArray}\]

Hence, from above we can we write the Jacobian of the $i$-th unit-network function $\wMbsupp_{\bullet \, i} \, \sigma(\wMasupp^\top_{i \,\bullet} \, \x)$ with respect to the entire set of parameters and at a given input $\x$, as follows:
\NiceMatrixOptions{code-for-first-row = \color{blue},
                   code-for-first-col = \color{blue},}
\[\nabla_{\btheta} \, F_{\theta_i}(\x)^\top= \hspace{-1cm}\begin{pNiceArray}[first-row,first-col,nullify-dots]{c}
       & \\
\Vdots &        \bm{0} \\[1mm]
\wMasupp_{i \,\bullet}   & \wMbsupp^\top_{\bullet \,i} \kro \bm \Lambda^{\x}_{ii}  \,\, \x\\[2mm]
\wMbsupp_{\bullet \,i}    &  \Im_{K}\kro  \wMasupp^\top_{i \,\bullet}  \big(\bm \Lambda^{\x}_{ii}  \,\x\big) \\[1mm]
\Vdots &        \bm{0} \\
\end{pNiceArray} \,\,\,\,\, \overset{(a)}{=} \hspace{-6mm}\begin{pNiceArray}[first-row,first-col,nullify-dots]{c}
       & \\
\Vdots &        \bm{0} \\[1mm]
\wMasupp_{i \,\bullet}   &  \bm \Lambda^{\x}_{ii}  \,\, \x \kro \wMbsupp^\top_{\bullet \,i}\\[2mm]
\wMbsupp_{\bullet \,i}    & \wMasupp^\top_{i \,\bullet}  \big(\bm \Lambda^{\x}_{ii}  \,\x\big)  \kro  \Im_{K}\\[1mm]
\Vdots &        \bm{0} \\
\end{pNiceArray}\,,\]
where, in (a) we have used the fact that for vectors $\ab,\,\bb$ we have that $\ab^\top \,\kro\, \bb = \bb \,\kro \,\ab^\top = \bb\, \ab^\top$, as well as the fact that $\wMasupp^\top_{i \,\bullet} \big(\bm \Lambda^{\x}_{ii}  \,\x\big)$ is a scalar which allows us to commute the factors in the corresponding Kronecker product.

Finally, we can express the above Jacobian across all the samples in the data matrix, as stated in the lemma:
\NiceMatrixOptions{code-for-first-row = \color{blue},
                   code-for-first-col = \color{blue},}
\[\nabla_{\btheta} F_{\theta_i}(\Xm)^\top = \hspace{-1cm}\begin{pNiceArray}[first-row,first-col,nullify-dots]{c}
       & \\
\Vdots &        \bm{0} \\[1mm]
\wMasupp_{i \,\bullet}     & \Xm \, \bm \Lambda^i \kro \wMbsupp^\top_{\bullet \,i}  \\[2mm]
\wMbsupp_{\bullet \,i}    &  \wMasupp_{i \,\bullet}^\top  \Xm \, \bm \Lambda^i \kro \Im_{K}  \\[1mm]
\Vdots &        \bm{0} \\
\end{pNiceArray}\,.
\]
Here, we utilized that $\Am \kro \Bm = \left[\Am_{\bullet\, 1} \kro \Bm,\, \cdots,\, \Am_{\bullet\, n} \kro \Bm\right]$ for some arbitrary matrix $\Am$ containing $n$ columns. Besides, we have collected the activation derivatives for the $i$-th neuron, i.e., $\bm \Lambda^\x_{ii}\,=\sigma^\prime( \Vm_{i \,\bullet}^\top \,  \x)\,$, across all samples $\x$, into the diagonal matrix $\bm \Lambda^i \in \mathbb{R}^{N\times N}$. 
\end{proof}

From the above Lemma, we can also see that the benefit of analyzing via the unit-networks is that we only have to deal with the activation derivatives of a single neuron at a time. Besides, now that we know the structure of the unit-network Jacobian, we will  analyze its rank. But before, let's recall the assumption~\ref{assump:data} from the main text, in our current notation:
\begin{repassump}{assump:data}
For each active hidden neuron $i$, the weighted input covariance has the same rank as the overall input covariance, i.e.,  $\rank(\E[\alpha_{\x}\,\x \x^\top]) = \rank(\covx) = r$, with $\,\alpha_{\x} = {\sigma^\prime(\x^\top \, \wMasupp_{i\,\bullet})}^2$.
\end{repassump}

This assumption can be translated into finite-sample case as follows. First, note that the (uncentered) input covariance $\covx$ corresponds to $\frac{1}{N} \Xm \Xm^\top$, while the weighted covariance $\E[\alpha_{\x}\,\x \x^\top]$ corresponds to the matrix $\frac{1}{N}\Xm\bm\Lambda^i \bm\Lambda^i \Xm^\top$. This is straightforward to check, and notice $\alpha_\x = (\bm \Lambda^\x_{ii})^2$. Then, the equivalent assumption is to require $\rank\left(\Xm\bm\Lambda^i \bm\Lambda^i \Xm^\top\right) = \rank\left(\Xm\Xm^\top\right) = r$, ignoring the constant $\frac{1}{N}$ which does not affect rank. Further, since for any arbitrary matrix $\Am$, we have that $\rank(\Am\Am^\top) =\rank(\Am)$. Thus, our equivalent assumption can be simplified to as follows:
\begin{assumpprime}{assump:data}{(finite-sample equivalent)}\label{assump:data-equiv}
For each active hidden neuron $i$, assume that $\rank\left(\Xm\bm \Lambda^i\right)=\rank\left(\Xm\right)=r$\,, where $\bm \Lambda^i$, as detailed before, contains the activation derivatives across all samples for this neuron $i$.
\end{assumpprime}

\begin{lemma}\label{lemma:relu-jacobian-1}
Under the same setup as Lemma~\ref{lemma:relu-jacobian-struc} and Assumptions~\ref{assump:2},~\ref{assump:data} (or equivalently~\ref{assump:data-equiv}), the rank of the Jacobian matrix, $\nabla_{\btheta} \Fni{i}(\Xm)$, of the $i$-th unit-network is given by:
\[\rank({\nabla_{\btheta} \Fni{i}(\Xm)}) = r + K - 1\,.\]

\begin{proof}
From Lemma~\ref{lemma:relu-jacobian-struc}, the Jacobian matrix is given by (ignoring the zero blocks which do not matter for the analysis of rank), 

$$\nabla_{\btheta} \Fni{i}(\Xm)^\top = \begin{pmatrix}
\Xm \, \bm \Lambda^i\kro \wMbsupp^\top_{\bullet \,i}  \\[2mm]
\wMasupp^\top_{i \,\bullet}  \Xm \, \bm \Lambda^i \kro \Im_{K} \\[1mm]
\end{pmatrix} = \underbrace{\begin{pmatrix}
\Im_{r} \kro \wMbsupp^\top_{\bullet \,i}\\[2mm]
\wMasupp^\top_{i \,\bullet} \kro \Im_{K}  \\[1mm]
\end{pmatrix}}_{\Am_i \,\in\, \mathbb{R}^{(r+K)\times Kr}} \underbrace{\left(\Xm \, \bm \Lambda^i\kro \Im_{K} \right)}_{\in \, \mathbb{R}^{Kr\times KN}}
$$

Now this factorization reveals the familiar $\Zm$-like structure, and so the matrix labelled $\Am_i$ in the above factorization has rank equal to $r + K - 1$ by Lemma~\ref{lemma:kronecker-block}. And, $\rank(\Xm \, \bm \Lambda^i\kro \Im_{K}) = K \rank(\Xm \, \bm \Lambda^i)= Kr$, by employing assumption~\ref{assump:data-equiv}. Thus, this matrix $\Xm \, \bm \Lambda^i\kro \Im_{K}$ is right invertible. Hence, we have:
\[\rank({\nabla_{\btheta} \Fni{i}(\Xm)}) =\rank({\nabla_{\btheta} \Fni{i}(\Xm)}^\top) = \rank(\Am_i) = r + K-1\,.\]
\end{proof}
\end{lemma}

\subsection{Proof of Theorem~\ref{thm:relu-rank}}
Now, that we are equipped to prove the Theorem, and let us recall its statement from the main text:
\begin{reptheorem}{thm:relu-rank}
Consider a 1-hidden layer network with non-linearity $\sigma$ such that $\sigma(z) = \sigma^{\prime}(z) z$ and let $\widetilde{M}$ be the \# of active hidden neurons (i.e., probability of activation $>0$). Then, under assumption \ref{assump:2} and \ref{assump:data},  rank of $\HO$ is given as,
$\rank(\HO) \leq r\widetilde{M} \,+\, \widetilde{M} K \,-\, \widetilde{M}\,$.
\end{reptheorem}

\begin{proof}

In the case of empirical loss (i.e., finite-sample case), we can express the outer-product Hessian as $\HO^N = \frac{1}{N} \nabla_{\btheta} \Fn(\Xm)^\top \,\, \nabla_{\btheta} \Fn(\Xm)=\frac{1}{N}\sum_{i=1}^N \nabla_{\btheta} \Fn(\x^i)^\top \,\, \nabla_{\btheta} \Fn(\x^i)$. It is clear that rank of $\HO^N$ is the same as the rank of $\nabla_{\btheta} \Fn(\Xm)$ as $\rank(\Am^\top\Am)=\rank(\Am)$ for any arbitrary matrix $\Am$. Thus we have that, 
\begin{align*}
\rank(\HO^N) =\rank(\nabla_{\btheta} \Fn(\Xm)) &= \rank\big(\sum\limits_{i=1}^M \nabla_{\btheta} \Fni{i}(\Xm)\big) 
\leq \sum\limits_{i=1}^M \rank(\nabla_{\btheta} \Fni{i}(\Xm)) \\
&\overset{\small{\text{Lemma~\ref{lemma:relu-jacobian-1}}}}{\leq} \sum\limits_{i=1}^{\widetilde{M}}  r + K -1 
= r\widetilde{M} \,+\, \widetilde{M} K \,-\, \widetilde{M}\,.
\end{align*}

The first inequality is because of subadditivity of rank, i.e., $\rank(\Am+\Bm)\leq \rank(\Am) + \rank(\Bm)$.
Next, here we only sum over the active hidden neurons, whose count is $\widetilde{M}$. Because, for dead neurons $\nabla_{\btheta} \Fn(\Xm)=\bm 0$, and thus rank of Jacobian for dead unit-networks will be zero.

Now, in order to extend this to case of population loss, we essentially have to consider the limit of $N\rightarrow\infty$. As we can see from the analysis so far, the rank of the outer-product Hessian $\HO^N$ is always bounded by $r\widetilde{M} \,+\, \widetilde{M} K \,-\, \widetilde{M}$ for any finite $N\geq N_0$, where $N_0$ is the minimum number of samples that are needed for the assumption~\ref{assump:data} to hold. 

Thus, we have a sequence of matrices $\lbrace\HO^N\rbrace_{N\geq N_0}$, each of which has rank bounded above by $r\widetilde{M} \,+\, \widetilde{M} K \,-\, \widetilde{M}$. Because, matrix rank is a lower semi-continuous function, the above sequence will converge to a matrix, $\HO$ of the population loss, with rank at most $r\widetilde{M} \,+\, \widetilde{M} K \,-\, \widetilde{M}$. Therefore,
\[\rank(\HO) \leq r\widetilde{M} \,+\, \widetilde{M} K \,-\, \widetilde{M}\,.\]

\end{proof}

\subsection{Note on the assumption}
The assumption~\ref{assump:data-equiv} that $\rank(\Xm\bm\Lambda^i)=r$, holds as soon as $\rank(\bm\Lambda^i)\geq r$ in expectation. This is something that depends on the data distribution but only mildly. For instance, one such scenario is when we use the typical form of initialization $v_{ij} \overset{i.i.d.}{\sim} \mathcal{N}(0, 1)$, then conditioned on a fixed example $\x$, we have $\wMasupp^\top_{i \,\bullet} \,\x \sim \mathcal{N}(0, \|\x\|^2)$. To further consolidate this point, let us consider $\sigma(z) = \operatorname{ReLU}(z) = \operatorname{max}(z, 0)$. Then, for instance if the underlying \textit{data distribution is symmetric}, the entries of $\bm \Lambda^i$ --- which are nothing but $\sigma^\prime(\wMasupp^\top_{i \,\bullet} \,\x) = \mathds{1}\lbrace\wMasupp^\top_{i \,\bullet} \,\x > 0\rbrace$ --- will be non-zero with probability $\frac{1}{2}$. The rank of the diagonal matrix $\bm \Lambda^i$ just amounts to the number of non-zero entries. Hence, in expectation, as soon as we have at least $2 r$ examples, or more simply $2d$ examples since $d\geq r$, we should be fine.

\clearpage
\section{Rank of the Hessian with bias}\label{supp:bias}
\subsection{Proof of Theorem~\ref{theorem:bias}}\label{supp:bias-outer}
We consider the case where each layer implements an affine mapping instead of a linear. So now we have additional parameter vectors for these bias terms, $\bias{1}, \cdots, \bias{L}$, and we can write the network function as:
\[
F(\mb x) =  \wM{L}\left(\cdots\left(\wM{2}\left(\wM{1} \mb{x} + \bias{1}\right) + \bias{2}\right)\cdots\right) + \bias{L}
\]

In terms of a recursive expansion, it can also be written in the following manner: 
\begin{align}\label{eq:bias-fn}
    F(\x) := \Fmap{L:1}(\x) = \wM{L} \Fmap{L-1:1}(\x) + \bias{L}\,, \quad \text{where} \quad \Fmap{0}(\x) = \x\,.
\end{align}
We will also use the notation $\Fmap{1:l}$ to mean $\Fmapt{l:1}$. Let us recall the assumption and the theorem stated in the main text:

\begin{repassump}{assump:zero-mean}
The input data has zero mean, i.e., $\x \sim p_{\x}$ is such that $\E[\x]=0$. 
\end{repassump}
In other words, we assume that the input data has zero mean, which is actually a standard practical convention.

\begin{reptheorem}{theorem:bias}
Under the assumption \ref{assump:2} and \ref{assump:zero-mean},  for a deep linear network with bias, the rank of $\HO$ is upper bounded as,
$\rank(\HO) \leq q(r+K-q) + K\,$, where $q:= \min(r, M_1, \cdots, M_{L-1}, K)$.
\end{reptheorem}

\begin{proof}

Since the above function, Eq.~\eqref{eq:bias-fn}, is of a similar form as the one in Eq.~\eqref{eq:matrix-derivative}, we use the matrix-derivative rule in order to obtain the following expression of the \textit{(transposed)} Jacobian at a point $(\x, \y)$:

\NiceMatrixOptions{code-for-first-row = \color{blue},
                   code-for-first-col = \color{blue},}
\begin{align}\label{eq:bias-jacobian}
\small
\nabla F(\x)^\top = \begin{pNiceArray}[first-row,first-col,nullify-dots]{c}
       & \\[2mm]
\vect_{r}(\wM{1}) & \wM{2:L} \kro \Fmap{0}(\x)\\[1mm]
\Vdots & \vdots \\[1mm]
\vect_{r}(\wM{l}) & \wM{l+1:L} \kro \Fmap{l-1:1}(\x)\\[1mm]
\Vdots & \vdots \\[2mm]
\vect_{r}(\wM{L}) & \Im_{K}\kro \Fmap{L-1:1}(\x)\\[4mm]
\bias{1}\hspace{2em} & \wM{2:L} \\[1mm]
\Vdots \hspace{2em}& \vdots \\[1mm]
\bias{l}\hspace{2em} & \wM{l+1:L} \\[1mm]
\Vdots \hspace{2em}& \vdots \\[1mm]
\hspace{2em}\bias{L}\hspace{2em} & \Im_{K}\\[2mm]
\end{pNiceArray}
\end{align}

\textit{Comment about the Hessian indexing: }We will assume that the blocks from $[1,\cdots,L]$ index the weight matrices and those from $[L+1,\cdots, 2L]$ index the bias parameters.

Recall the outer-product Hessian $\HO$ in the case of mean-squared loss is given by \[\HO =\E\left[\nabla F(\x)^\top \nabla F(\x) \right]\,.\]

Let us look at the expression for the $kl$-th block, for $k, l \in [L]$ (i.e., from the sub-matrix corresponding to \textit{weight-weight Hessian}):

\begin{align}
\HOsup{kl} &= \E\left[\wM{k+1:L} \wM{L:l+1} \,\,\kro\,\, \Fmap{k-1:1}(\x) \Fmap{1:l-1}(\x) \right]\\
&=\wM{k+1:L} \wM{L:l+1} \,\,\kro\,\, \E\left[\Fmap{k-1:1}(\x) \Fmap{1:l-1}(\x) \right] \label{eq:bias-outer-kl}
\end{align}

Now, let us make use of the assumption~\ref{assump:zero-mean}. Once we have applied this, the dependence on input is only via the uncentered covariance of input (or the second moment matrix). Alongside we have terms corresponding to $F^{l-1:1}(\mb 0)$, which is the output of the network when $\mb 0$ is passed as the input. Overall, using the zero-mean assumption in Eq.~\eqref{eq:bias-outer-kl} yields:

\begin{align}
    \HOsup{kl} &=\underbrace{\wM{k+1:L} \wM{L:l+1} \,\,\kro\,\,\wM{k-1:1} \covx \wM{1:l-1}}_{\text{Expression in the linear, non-bias, case}} \,+\, \underbrace{\wM{k+1:L} \wM{L:l+1} \,\,\kro \,\, \Fmap{k-1:1}(\mb 0) \Fmap{1:l-1}(\mb 0)}_{\text{New terms containing bias}} \label{eq:bias-outer-kl-final}
\end{align}

We see that first part of the expression is identical to the linear case without bias, and it is only the second part that contains the bias terms. 

Similarly, for the \textit{bias-bias Hessian} blocks $\HOsup{kl}$ such that $k, l \in [L\, \cdots \, 2L]$, there is no dependence on input at all and contains only bias terms. Likewise, the \textit{weight-bias Hessian} blocks has no dependence on the input.

Hence, it seems quite natural to separately analyze the terms without bias and with bias. So, the rank of the first non-bias part comes directly from our previous analysis of Theorem~\ref{theorem:ub-outer} and is equal to $q (r+K-q)$. 

The analysis for the left-over bias part is not too hard either. This can be simply decomposed as the product $\Bm_o\Bm_o^\top$, where $\Bm_o$ is given by:

\NiceMatrixOptions{code-for-first-row = \color{blue},
                   code-for-first-col = \color{blue},}
\begin{align}\label{eq:bias-B}
\small
\Bm_o = \begin{pNiceArray}[first-row,first-col,nullify-dots]{c}
       & \\[2mm]
\vect_{r}(\wM{2}) & \wM{3:L} \kro \Fmap{1}(\mb 0)\\[1mm]
\Vdots & \vdots \\[1mm]
\vect_{r}(\wM{l}) & \wM{l+1:L} \kro \Fmap{l-1:1}(\mb 0)\\[1mm]
\Vdots & \vdots \\[2mm]
\vect_{r}(\wM{L}) & \Im_{K}\kro \Fmap{L-1:1}(\mb 0)\\[4mm]
\bias{1}\hspace{2em} & \wM{2:L} \\[1mm]
\Vdots \hspace{2em}& \vdots \\[1mm]
\bias{l}\hspace{2em} & \wM{l+1:L} \\[1mm]
\Vdots \hspace{2em}& \vdots \\[1mm]
\hspace{2em}\bias{L}\hspace{2em} & \Im_{K}\\[2mm]
\end{pNiceArray}
\end{align}

If we compare this expression to that in Eq.~\ref{eq:bias-jacobian}, we see that there is no block corresponding to the first row there, as $F^0(\mb 0) = \mb 0$. Then, one simply has to notice that the last block in $\Bm_o$, which essentially corresponds to the parameter $\bias{L}$, is the $K\times K$ identity matrix $\Im_K$. Hence, the matrix $\Bm_o$ which itself has $K$ columns, has rank equal to $K$, using Lemma~\ref{lemma:block-shared}. Then, the rank of the bias part is equal to that of $\rank(\Bm_o)=K$, since we know that $\rank(\Am \Am^\top) = \rank(\Am)$.

Finally, we use the subadditivity of rank, i.e., $\rank(\Am+\Bm)\leq \rank(\Am) + \rank(\Bm)$, on this decomposition of the outer-product Hessian into outer-product Hessian for non-bias and the new terms containing the bias parameters. Thus, we obtain that:

\[\rank(\HO) \leq q(r+K-q) + K\,. \]

\end{proof}
\subsection{Formulas for two layer networks }

For two layer (1-hidden layer) networks with linear activation and $M_1$ hidden units, $d$ dimensional input and $K$ classes, empirical evidence seems to suggest the following. Define $s = \min(r, K)$ and $q = \min(r, M_1, K)$. Let us define $s^\prime:= \min(r+1, K)$. Then we find:
\begin{itemize}
    \item $\rank(\HO) = q \, (r + K - q) + K\quad$
    \item $\rank\left(\HF\right) = 2sM_1 + \mathds{1}\{K>r\} \, 2 M_1 = 2 \min(K, r+1) \, M_1 = 2s^\prime M_1$
    \item $\rank\left(\HL\right) = 2 \min(K, r+1) \, (M_1 - q) + q\, (K + r + 1) + K = 2 s^\prime(M_1 - q) \,+\, q\,(K+r+1) \,+\, K$
    
\end{itemize} 

If we compare the upper-bounds for the scenario without bias to the one with bias, we find that change in the rank of $\HF$ is due to changing $r \rightarrow r+1$ in the formula, which makes sense as bias can be understood as adding a homogeneous coordinate in the input. For $\HO$, the rank formula now includes an additive term of $K$. And both these changes together affect the change in rank for $\HL$.

\subsection{Formulas for $L$-layer networks }
\label{formulas_bias}
The upper-bound for $\HO$ that we noted in the previous section also holds for the general case, as evident from our proof in Section~\ref{supp:bias-outer}. Empirically as well, we obtain $ \rank(\HO) = q \, (r + K - q) + K\quad$ as the exact formula.

For the functional and overall Hessian, we list formulas that seem to hold \textit{empirically} for the non-bottleneck case. Here, the input size has to take into account the homogeneous coordinate, so by non-bottleneck it is meant that $M_i \geq \min(r+1, K), \,\forall \, i \in [1,\cdots, L-1]$). 

Define $q^\prime = \min(r+1, M_1, \cdots, M_{L-1}, K) = \min(r+1, K)$, which because of our non-bottleneck assumption comes out to be same as the $s^\prime$ in the previous section.

\begin{itemize}
    \item $\rank\left(\HF\right) = 2q^\prime\left(\sum_{i=1}^{L-1}M_i\right) + 2q^\prime s^\prime - L {q^{\prime}}^2 + (L-2) q^\prime$
    \item $\rank\left(\HL\right) =% 
     2q^\prime\left(\sum_{i=1}^{L-1}M_i\right) + {q^\prime}(r + K) - L {q^\prime}^2 + L{q^\prime} = 2q^\prime M + {q^\prime}(r + K) - L {q^\prime}^2 + L{q^\prime} $ 
\end{itemize}

Let us compare the above bound to the rank of Hessian $\HL$ in the linear case with bias by assuming that the output layer has the smallest size, i.e., $q^\prime=K$.

Then for linear case \textbf{without bias}: 
\[\rank(\HL) = 2K \, M - L\, K^2  + K(r + K)\]

While for linear case \textbf{with bias}: 
\[\rank(\HL) = 2K \, M - L\, K^2  + K(r + K) + L K\]

Basically, we just have an additional term of $LK$ in the rank, whereas the additional number of parameters are,\[\sum_{i=1}^{L} M_i \geq LK\,.\]

Hence under the considered scenario, we find that the ratio of rank to number of parameters also decreases when all layers have bias enabled. For an empirical simulation on the growth of rank relative to number of parameters in these two cases, please refer to Fig.~\ref{fig:bias_formulas}.

\clearpage
\section{Properties of the Hessian Spectrum}\label{supp:spectra}

\subsection{Spectrum of outer-product Hessian}\label{supp:outer-spectra}

\begin{figure}[!h]
    \centering
    \includegraphics[width=0.3\textwidth]{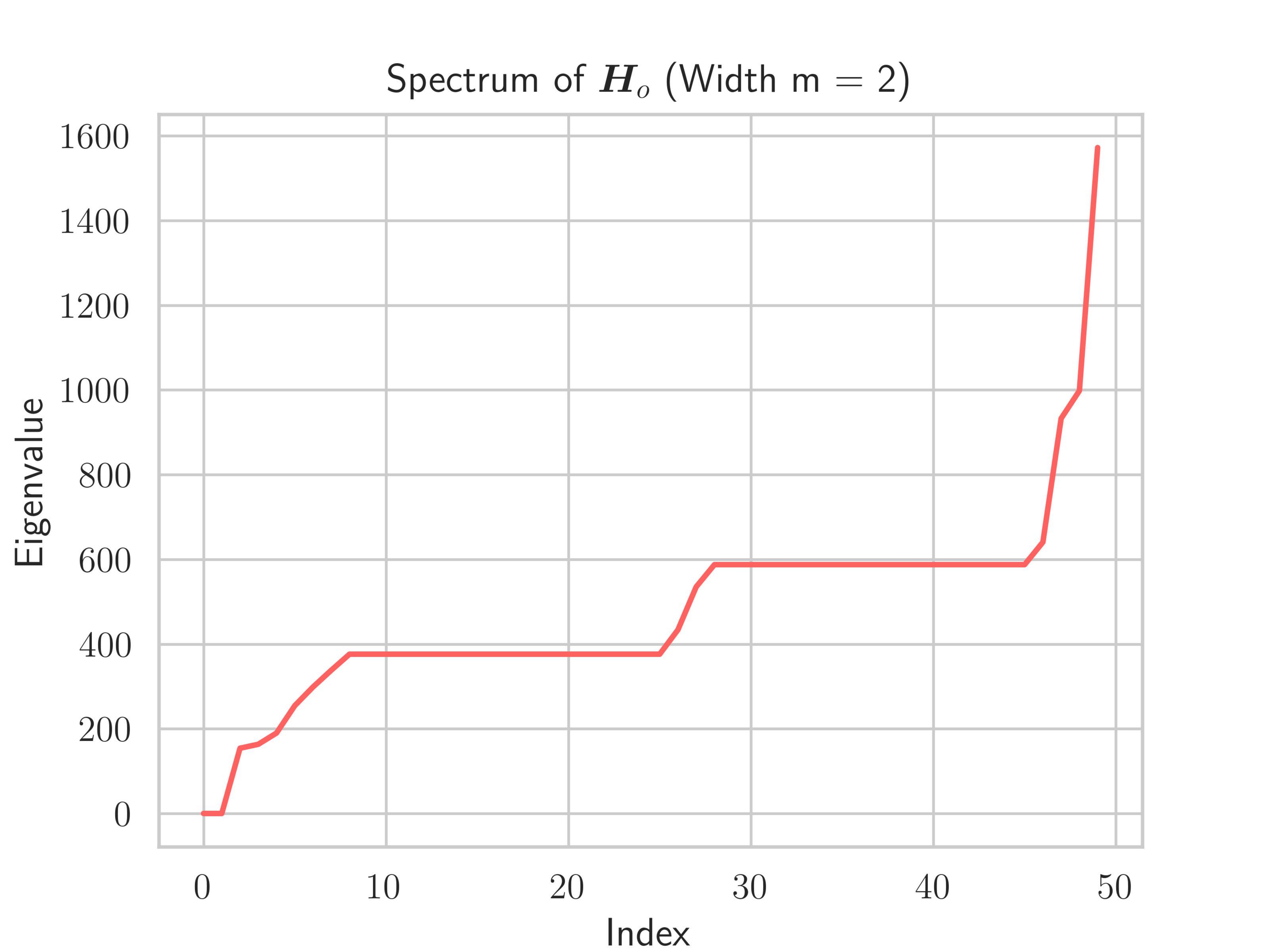}%
    \includegraphics[width=0.3\textwidth]{pdfpics/relu_outer_3.pdf}%
    \includegraphics[width=0.3\textwidth]{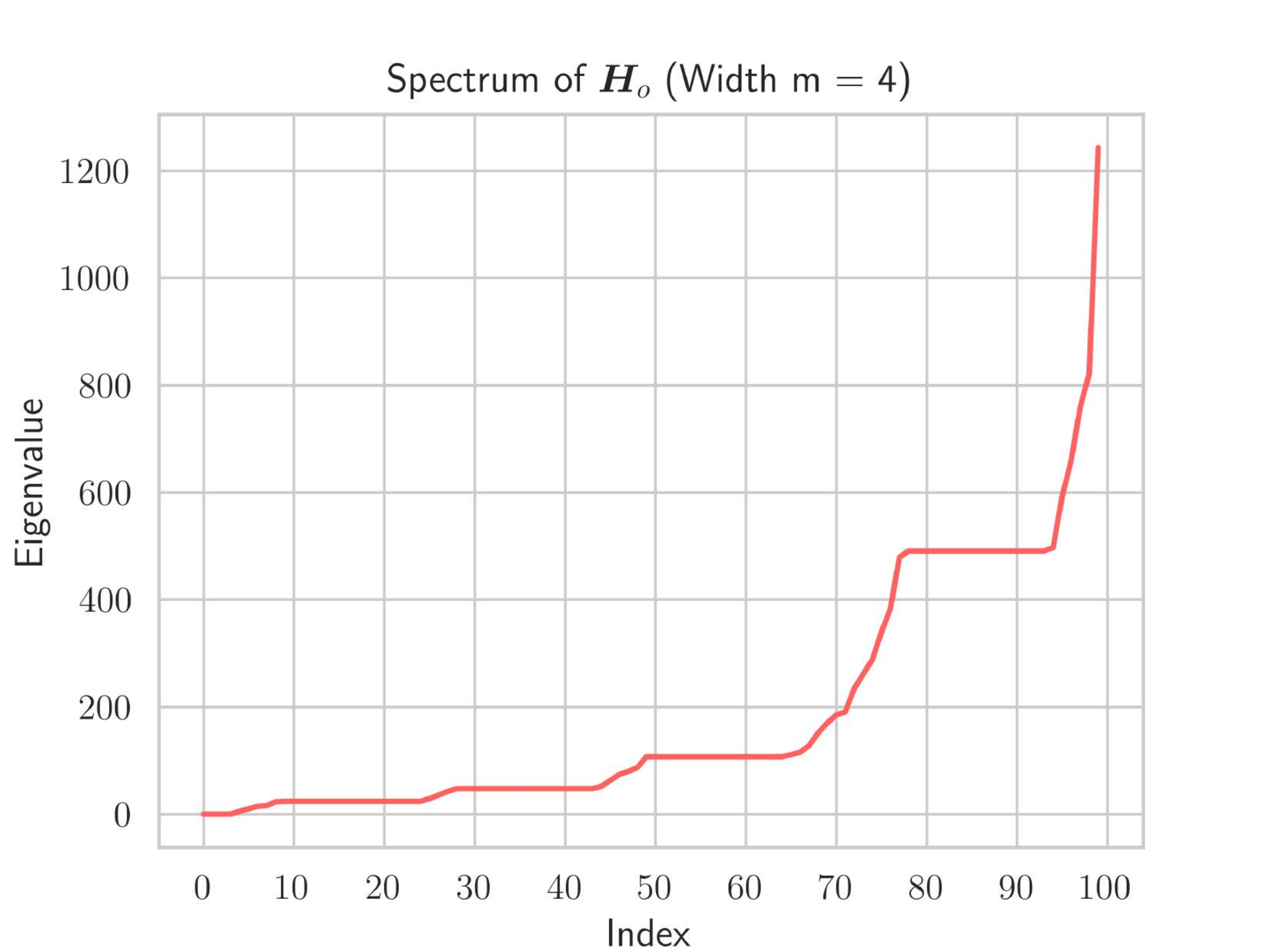}
    \caption{\textit{Spectrum of outer product.} $\HO$ spectrum has $q$ plateaus of size $K-M_{L-1}$ located at the eigenvalues of $\mathbf{E}\left[\Fmap{L-1:1}(\x)\, {\Fmap{1:L-1}}(\x)\right]$, even with non-linearities and for any $L$. Here, $K=20$, and $q=M_{L-1}=2, 3, 4$ in each of the sub-figures respectively. We use Gaussian mixture data of dimension $5$.}
    \label{fig:spectrum_relu_outer}
\end{figure}
\begin{figure}[!h]
    \centering
    \includegraphics[width=0.3\textwidth]{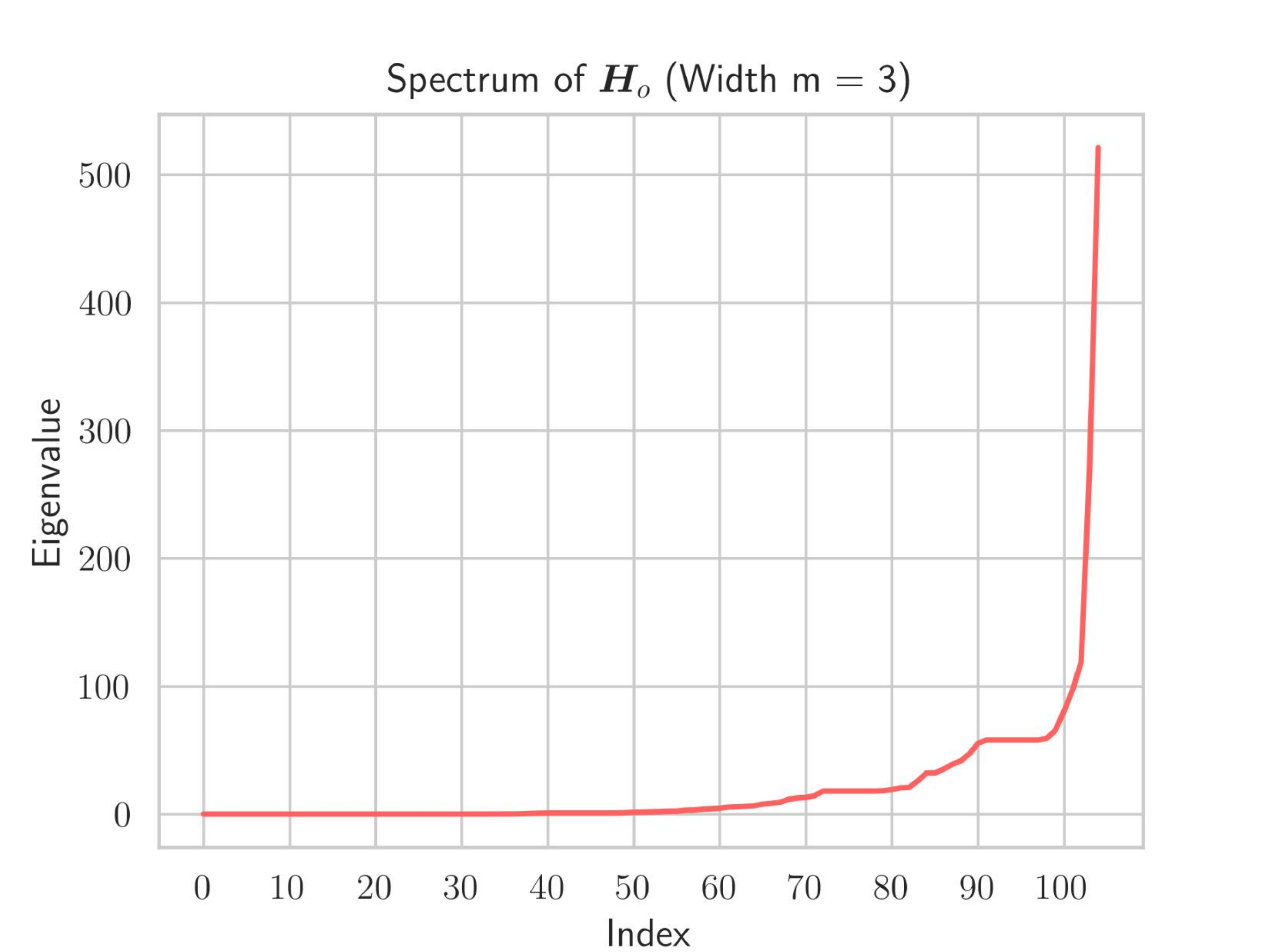}%
    \includegraphics[width=0.3\textwidth]{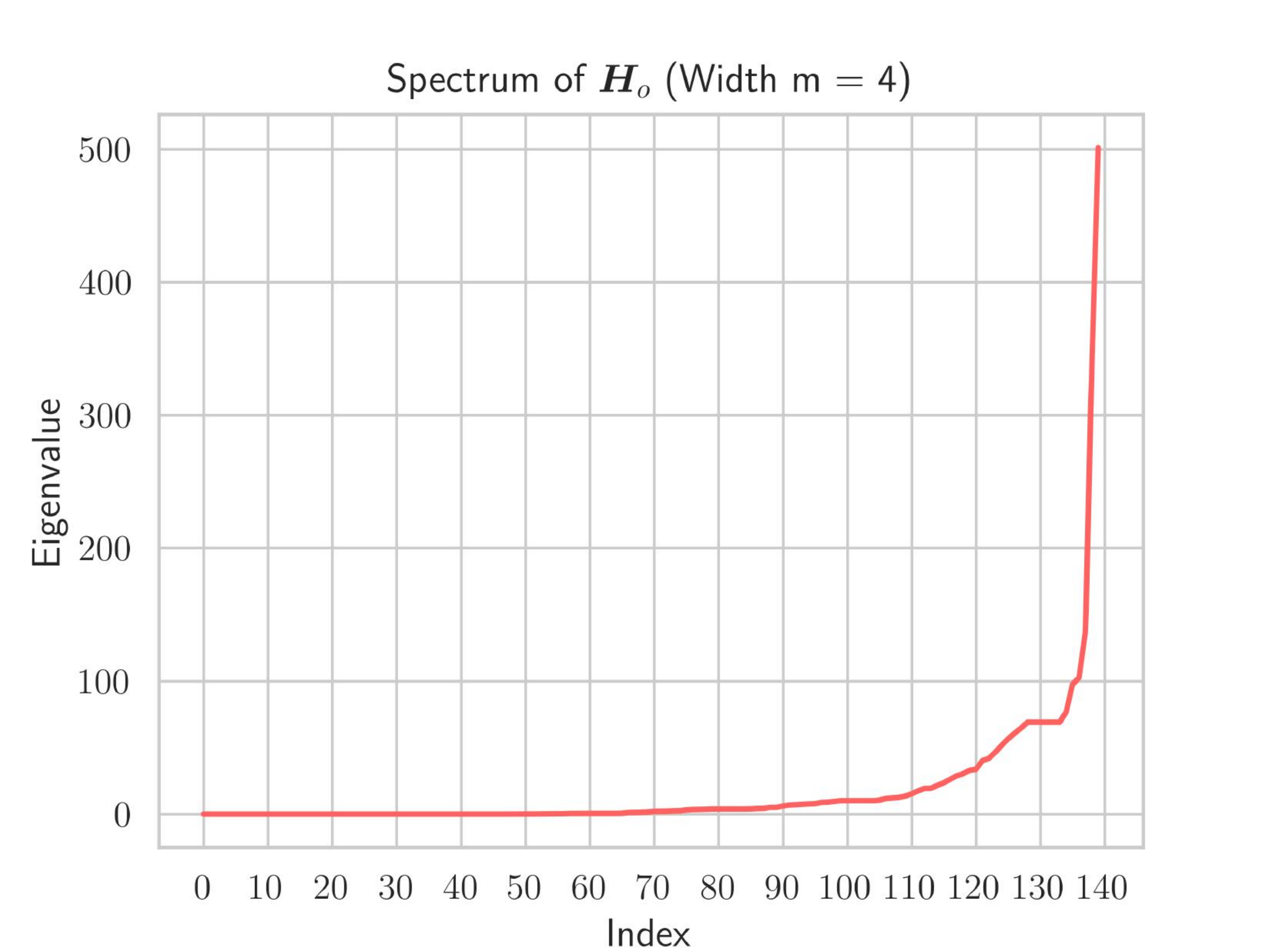}%
    \includegraphics[width=0.3\textwidth]{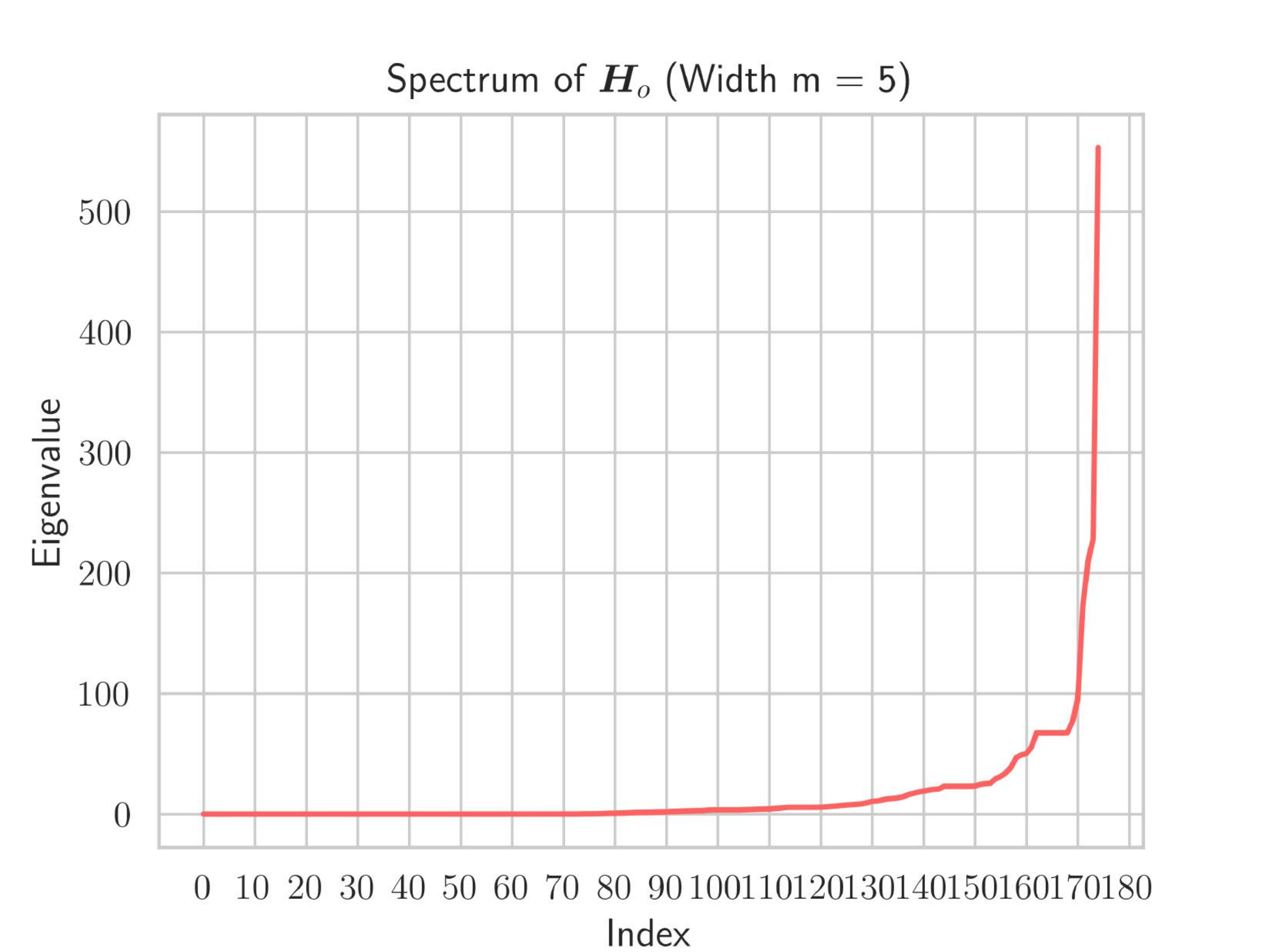}
    \caption{\textit{Spectrum of outer product.} $\HO$ spectrum has $q$ plateaus of size $K-M_{L-1}$ located at the eigenvalues of $\mathbf{E}\left[\Fmap{L-1:1}(\x)\, {\Fmap{1:L-1}}(\x)\right]$, even with non-linearities and for any $L$. Here, $K=10$, and $q=M_{L-1}=3,4,5$  in each of the sub-figures respectively. We use down scaled \textsc{MNIST}  $d=25$.}
    \label{fig:spectrum_relu_outer_mnist}
\end{figure}

The eigenvalues of the outer-product term of the Hessian, which is the one that dominates the spectrum near the end of training, can be written in closed-form for fully-connected neural networks, with linear activations. 

Recall from Proposition.~\eqref{eq:outer-decomp} that the outer-product Hessian can be decomposed as follows, $\HO = \Am_o (\Im_K \kro \covx) \Am_o^\top$, where $\Am_o \in \mathbb{R}^{p\times Kd}$ is as follows:

$$
\Am_o=\begin{pmatrix}
\wM{2:L} \kro \Im_d\\
\vdots \\
\wM{\ell+1:L} \kro \wM{\ell-1:1}\\
\vdots \\
\Im_K \kro \wM{L-1:1}\\
\end{pmatrix}
$$

Since $\Am\Bm$ and $\Bm\Am$ have the same non-zero eigenvalues, we have that eigenvalues of $\HO$ are the same as $\HOtilde = \Am_o^\top \Am_o (\Im_K \kro \covx)$, and notice $\Am_o^\top \Am_o \in \mathbb{R}^{Kd\times Kd}$ and comes out to be, 

$$
\Am_o^\top \Am_o  = \sum\limits_{\ell=1}^{L} \wM{L:\ell+1}\wM{\ell+1:L} \kro \wM{1:\ell-1} \wM{\ell-1:1}
$$

This is nothing but the diagonal-blocks of the Hessian-outer product added in the ``transposed'' fashion. Hence we have the result on the eigenvalues ($\evals$) that, 

$$
\evals(\HO) = \evals(\Am_o^\top\Am_o\left(\Im_K\kro\covx\right)) = \evals \left(\sum\limits_{\ell=1}^{L} \wM{L:\ell+1}\wM{\ell+1:L} \kro \wM{1:\ell-1} \wM{\ell-1:1} \covx \right)
$$

\paragraph{Repeated eigenvalues.} A consequence of this is that a plateau of repeated eigenvalue exists, whenever the last layer is strictly bigger than the penultimate layer. In fact, this plateau phenomenon also holds for non-linear networks, since even for such networks the last layer is not usually followed by non-linearities.

Notice, that when $K>M_{L-1}$, the matrix $\wM{L:\ell+1}\wM{\ell+1:L}$ in the left part of the Kronecker product will have rank at most $M_{L-1}$, except for the case when $\ell=L$. There, for $\ell=L$, we obtain a identity matrix, $\Im_K$, in the left part of the Kronecker product, whose rank is of course $K$. Thus when all terms are added up together, $K-M_{L-1}$ times the eigenvalues of $\wM{1:L-1} \wM{L-1:1}\covx$ will show up for the overall Hessian $\HO$ as well. Obviously, since Kronecker product with identity implies eigenvalues of the other matrix are multiplied by 1. This results in the plateaued behaviour of the eigenvalue spectrum. 
We illustrate this finding in Figure \ref{fig:spectrum_relu_outer} for a ReLU network on Gaussian mixture data. We also show the results for \textsc{MNIST} in Figure \ref{fig:spectrum_relu_outer_mnist}. Due to the spectrum being not as cleanly separated as for the Gaussian case, the results are not as clearly visible but still hold exactly as verified experimentally.

\subsection{Spectrum of Functional Hessian}\label{supp:func-spectra}

\begin{figure}[!h]
    \centering
    \includegraphics[width=0.3\textwidth]{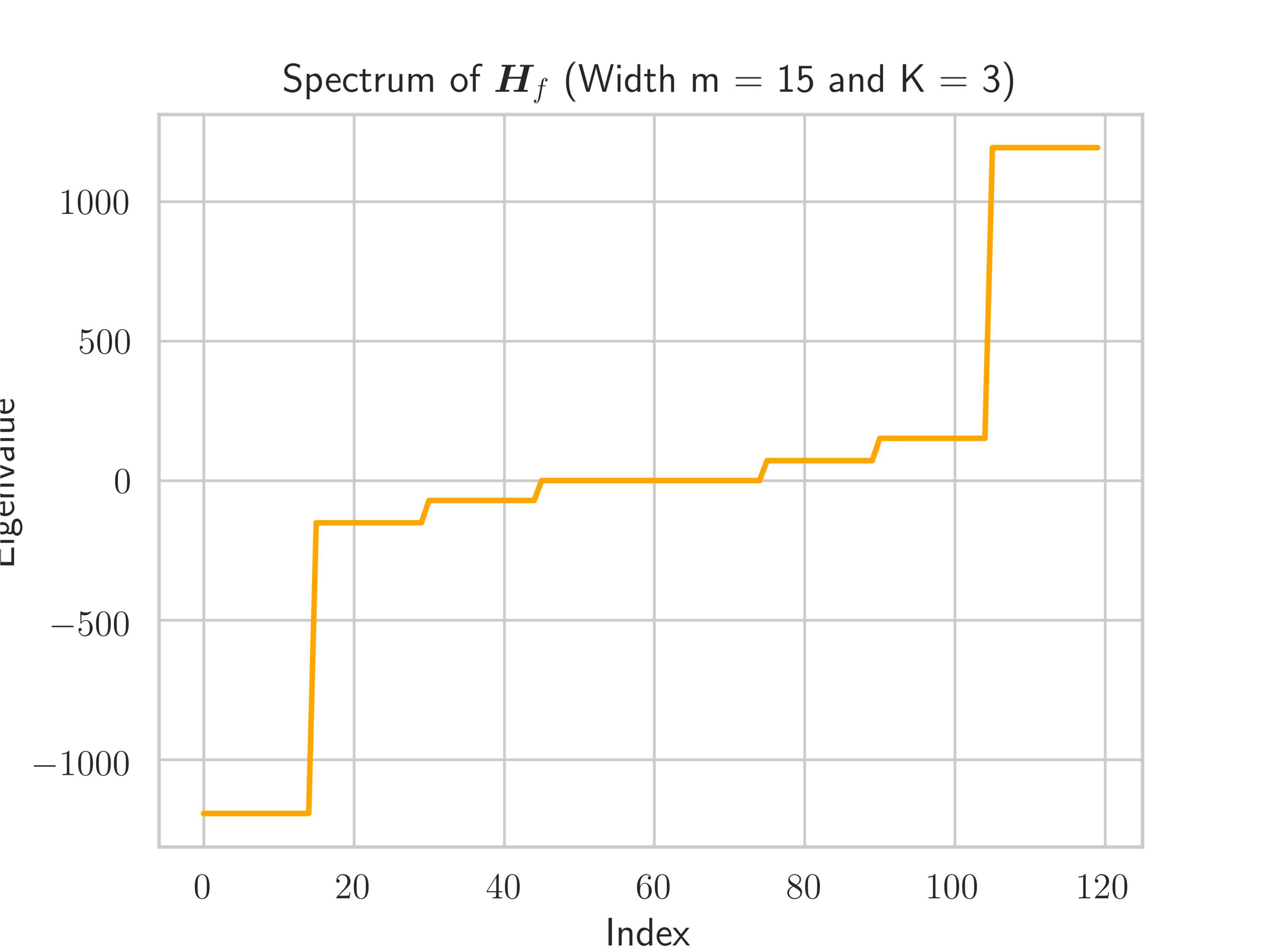}%
    \includegraphics[width=0.3\textwidth]{pdfpics/linear_spectrum_F_20.pdf}%
    \includegraphics[width=0.3\textwidth]{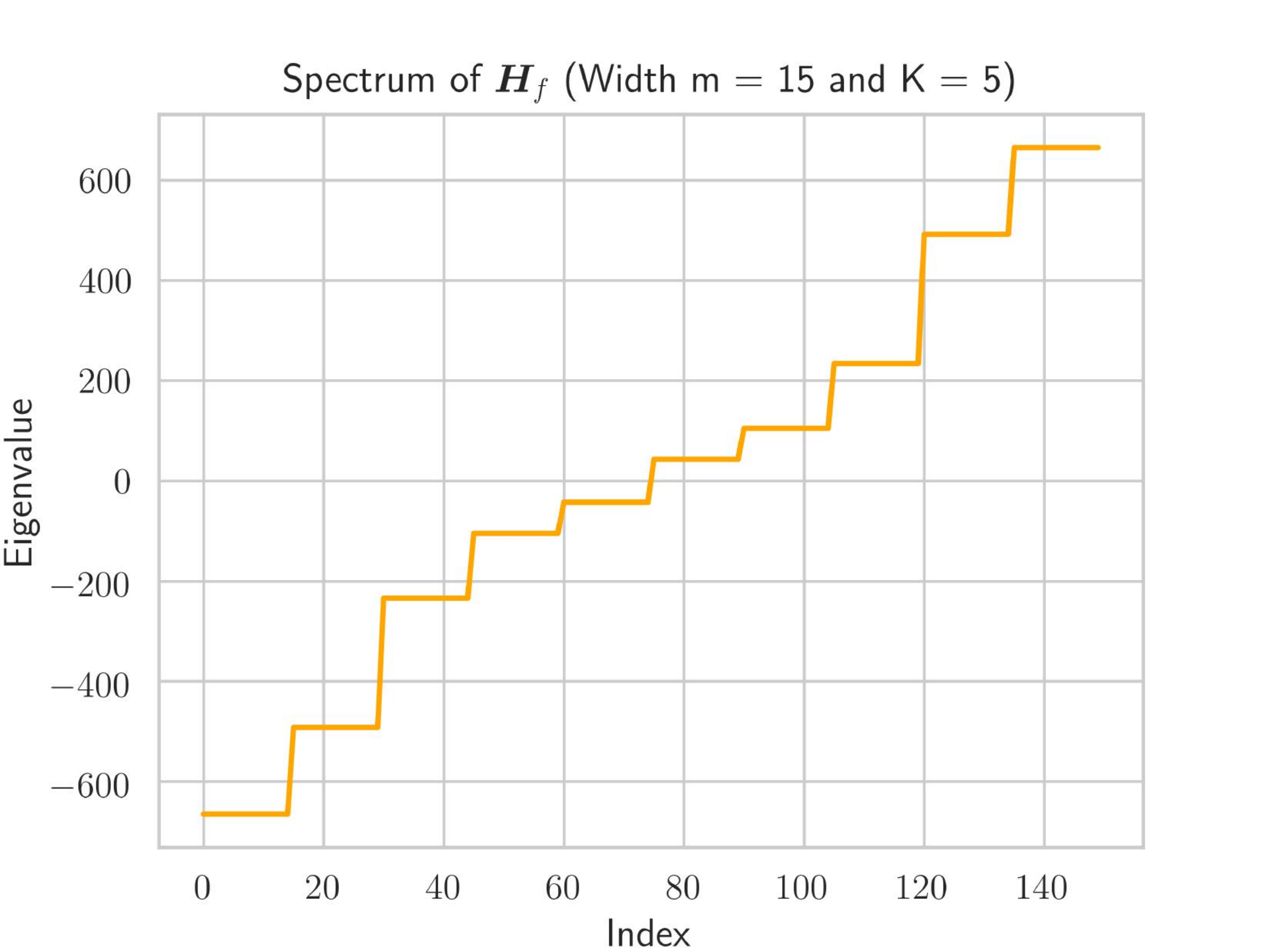}
    \caption{\textit{Spectrum of Functional Hessian.} We use a Gaussian mixture of dimension $d=5$ and a linear model with one hidden layer of size $M=15$. We vary the number of classes $K=3,4,5$ in each of the sub-figures respectively. Notice that we have $2K$ plateaus of width $M=15$.}
    \label{fig:functional_hessian_spectrum_linear}
\end{figure}
\begin{figure}[!h]
   \centering
   \includegraphics[width=0.3\textwidth]{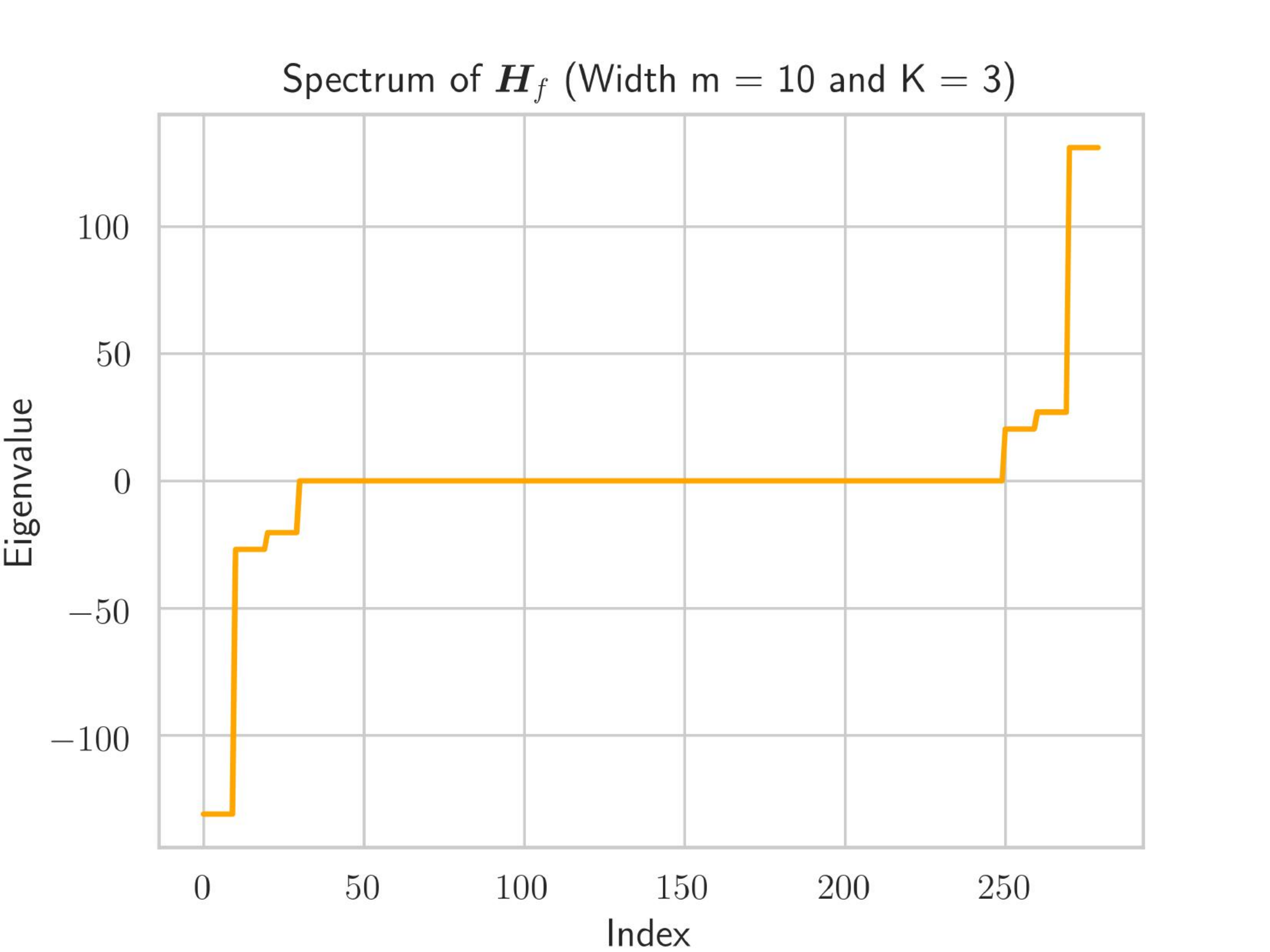}%
   \includegraphics[width=0.3\textwidth]{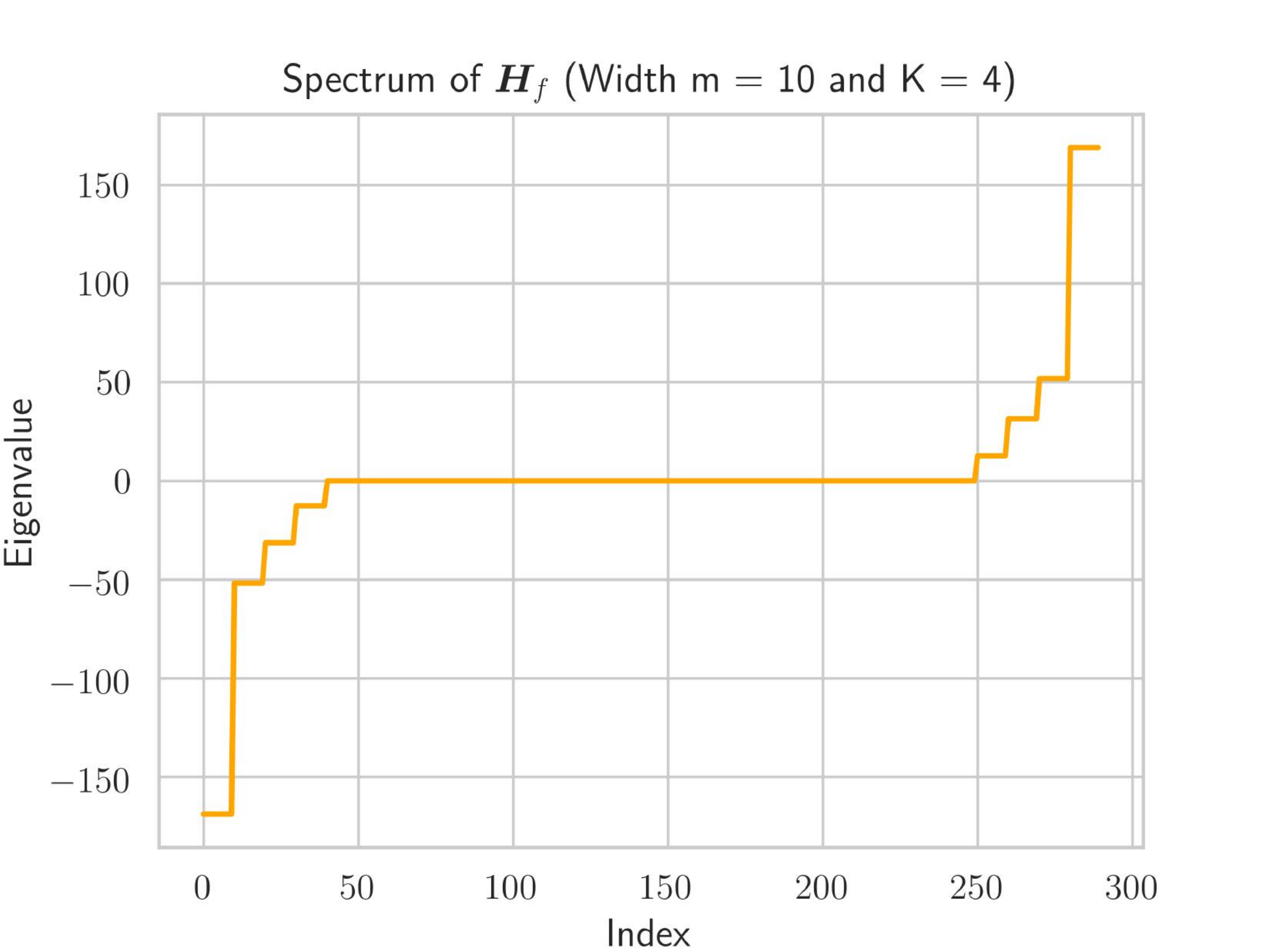}%
   \includegraphics[width=0.3\textwidth]{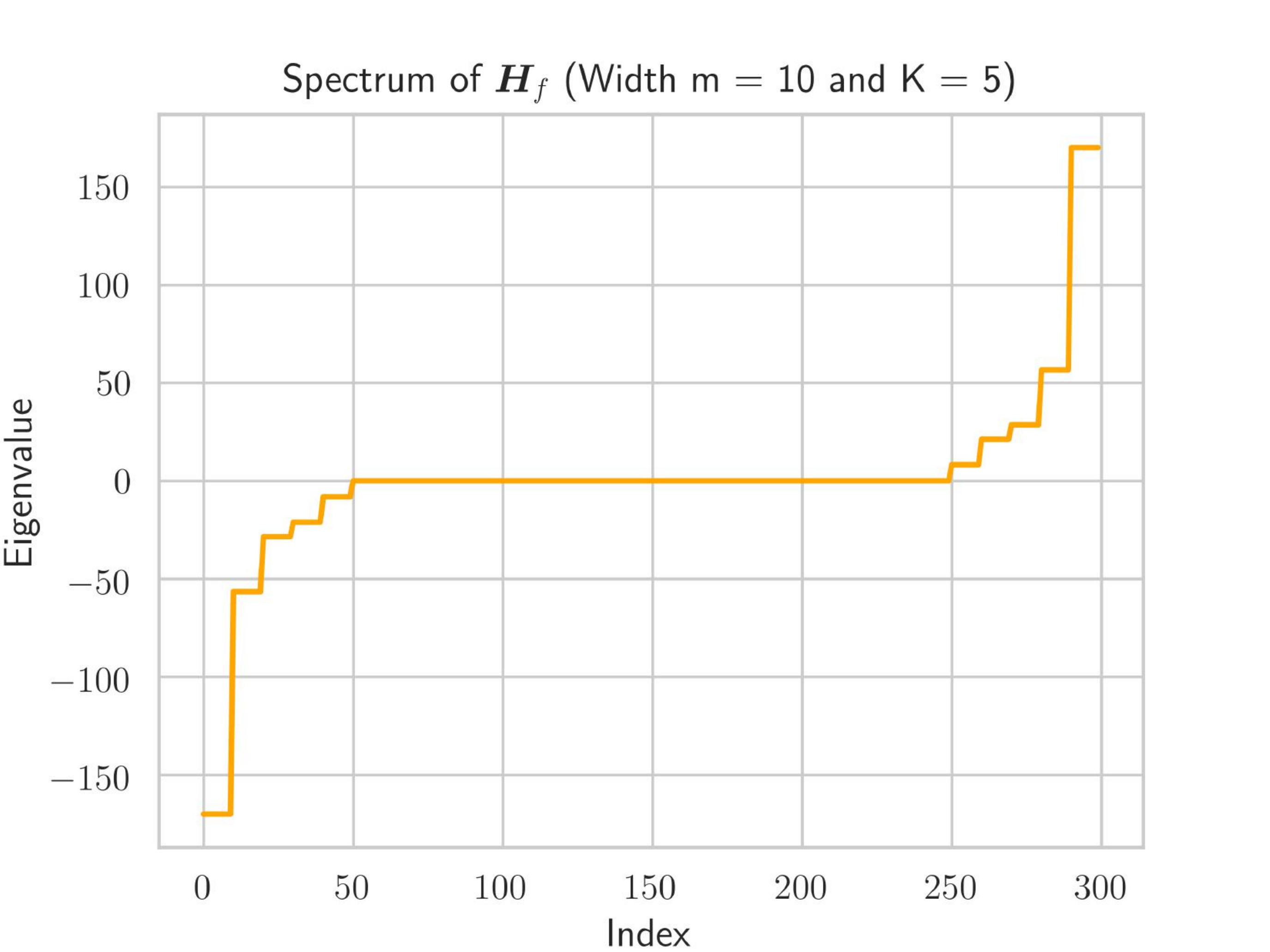}
    \caption{\textit{Spectrum of Functional Hessian.} We use down-sampled \textsc{MNIST} of dimension $d=25$ and a linear model with one hidden layer of size $M=10$. We vary the number of classes $K=3,4,5$  in each of the sub-figures respectively. Notice that we have $2K$ plateaus of width $M=10$.}
   \label{fig:mnist_spec}
\end{figure}

Consider the case of 1-hidden layer network with $M$ hidden neurons. We notice that the functional Hessian in the linear case has a interesting  step-like structure in the spectrum, while in the non-linear case empirically appears to interpolate or pass through it.

Here, the functional Hessian part is given as follows:

$$\HF = \begin{pmatrix}
\bm{0}_{K M} & \bm{\Omega} \otimes \Im_{M} \\[3mm]
\bm{\Omega}^{\top} \otimes \Im_{M} & \bm{0}_{d M}
\end{pmatrix}\,,$$

where, $\Om=\E[\pmb \delta_{\x,\y}\, \x^\top]$ as before.
Now since eigenvalues $\lambda$ are given by the solution to the characteristic polynomial, $\rho(\lambda) = \mathop{det}(\HF - \lambda \Im_{p})$, where $p = d M + K M$ denotes the total number of parameters. We can further write it as, 

$$\HF = \begin{pmatrix}
-\lambda \Im_{K M} & \bm{\Omega} \otimes \Im_{M} \\[3mm]
\bm{\Omega}^{\top} \otimes \Im_{M} & -\lambda \Im_{d M}
\end{pmatrix}$$

Now, we consider the determinant formula through the Schur complement assuming the block matrix $\Am$ is invertible, i.e., 
$$\mathop{det}(\Mm) = \mathop{det}\begin{pmatrix}
\Am & \Bm\\
\Cm  & \Dm
\end{pmatrix}= \mathop{det}(\Am) \mathop{det}(\Dm - \Cm \Am^{-1}\Bm)$$

Hence, in our case we obtain:
$$
\rho(\lambda) = (-\lambda)^{K M} \mathop{det}\left(-\lambda^2 \,\Im_{d M} + (\bm \Omega^\top \bm \Omega \kro \Im_{M})\right)
$$

Where we can see that $\mathop{det}\left(-\lambda^2 \,\Im_{d M} + (\bm \Omega^\top \bm \Omega \kro \Im_{M})\right)$ corresponds to the characteristic polynomial of the matrix $\Zm =(\bm \Omega^\top \bm \Omega) \kro \Im_{M}$ and with each eigenvalue of $\Zm$ occurring with both as positive and negative signs as eigenvalues of $\HF$, repeated $M$ times. See Figure \ref{fig:functional_hessian_spectrum_linear} for Gaussian mixture data and Figure \ref{fig:mnist_spec} for down-sampled \textsc{MNIST}.

\clearpage
\section{Detailed Empirical Results}\label{supp:empirical}

Here we collect the variety of experiments omitted from the main text due to space constraints. We begin by providing further evidence for the validity of our rank predictions for linear networks by varying the dataset and the loss function employed in the calculation of the Hessian. We then present more experiments for the non-linear case, showing more spectral plots and reconstruction errors for more non-linearities. Finally, also show how our rank predictions also extend to the neural tangent kernel.

Experiments were implemented in the JAX framework\footnote{\url{https://github.com/google/jax}} and performed on CPU (AMD EPYC 7H12) with 256 GB memory. 
\let\FloatBarrier\saveFloatBarrier
\subsection{Verification of Rank Predictions for Linear Networks}
\label{verif_exps}

We verify our formulas for \textsc{MNIST}  \citep{lecun-mnisthandwrittendigit-2010}, \textsc{CIFAR10} \citep{cifar10} and \textsc{Fashion-MNIST} \citep{xiao2017fashionmnist}. Moreover we employ diverse losses such as mean-squared error, cross entropy loss and cosh loss. We show the dynamics of rank as a function of sample size, minimal width and depth.

For all the considered settings, we observe exact matches across all datasets and all losses. We structure the experiments as follows. We group by loss functions, starting with MSE, then cross entropy and then cosh loss. For each loss function, we vary the sample size, width and depth of the architecture for the three datasets. Finally, we vary the initialization scheme and study the effect of sample size, width and depth for MSE loss on \textsc{CIFAR10}. Finally, we verify the predictions for architectures that use bias, again using MSE loss on \textsc{CIFAR10}.
\subsubsection{Mean Squared Error (MSE)}
Here we perform more experiments in the spirit of Figure \ref{fig:cifar10-a3}. We also show the rank dynamics with varying depth and only present the normalized plots for both width and depth. Figure \ref{cifar_mse} shows the results for \textsc{CIFAR10}, \ref{fashion_mse} for Fashion\textsc{MNIST} and \ref{mnist_mse} for \textsc{MNIST}. We observe a perfect match for all the datasets. For width and sample size, we down-sample the corresponding dataset to dimensionality $d=64$, while for depth, in order to be able to use deeper models, we down-sample to $d=16$. We use $N=300$ number of samples.

\begin{figure}[!h]
    \centering
    \begin{subfigure}[t]{0.3\textwidth}
    \includegraphics[width=\textwidth]{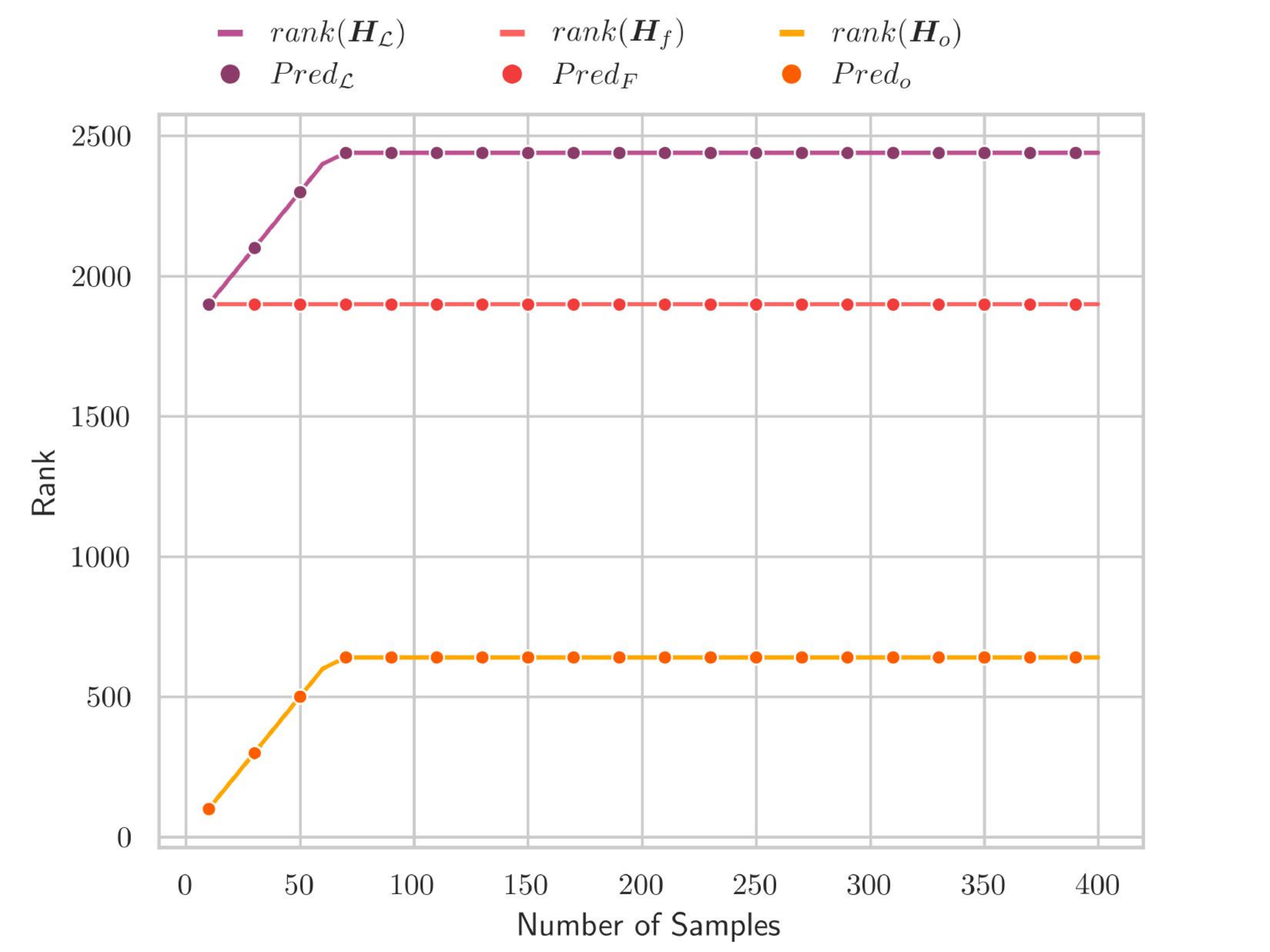}
    \caption{Rank vs sample size $n$}
    \label{fig:cifar10-a1_app}
    \end{subfigure}
    \begin{subfigure}[t]{0.3\textwidth}
    \includegraphics[width=\textwidth]{pdfpics/width_mse_cifar.pdf}
    \caption{$\effparams$ vs minimal width $M_{*}$}
    \label{fig:cifar10-a2_app}
    \end{subfigure}
    \begin{subfigure}[t]{0.3\textwidth}
    \includegraphics[width=\textwidth]{pdfpics/depth_mse_cifar.pdf}
    \caption{$\effparams$ vs depth $L$}
    \label{fig:cifar10-a3_app}
    \end{subfigure}
    \caption{\small{Behaviour of rank and rank/\#params on CIFAR10 using MSE, with hidden layers: $50,20,20,20$ (Fig.~\ref{fig:cifar10-a1_app}), $M_{*}, M_{*}$ (Fig.~\ref{fig:cifar10-a2_app}) and $L$ layers of width $M=25$ (Fig.~\ref{fig:cifar10-a3_app}). The lines indicate the true value and circles denote our formula predictions. }}% and right we use same width $M_l=20$ for every layer.}}
    \label{cifar_mse}
\end{figure}
\FloatBarrier
\begin{figure}[!h]
    \centering
    \begin{subfigure}[t]{0.3\textwidth}
    \includegraphics[width=\textwidth]{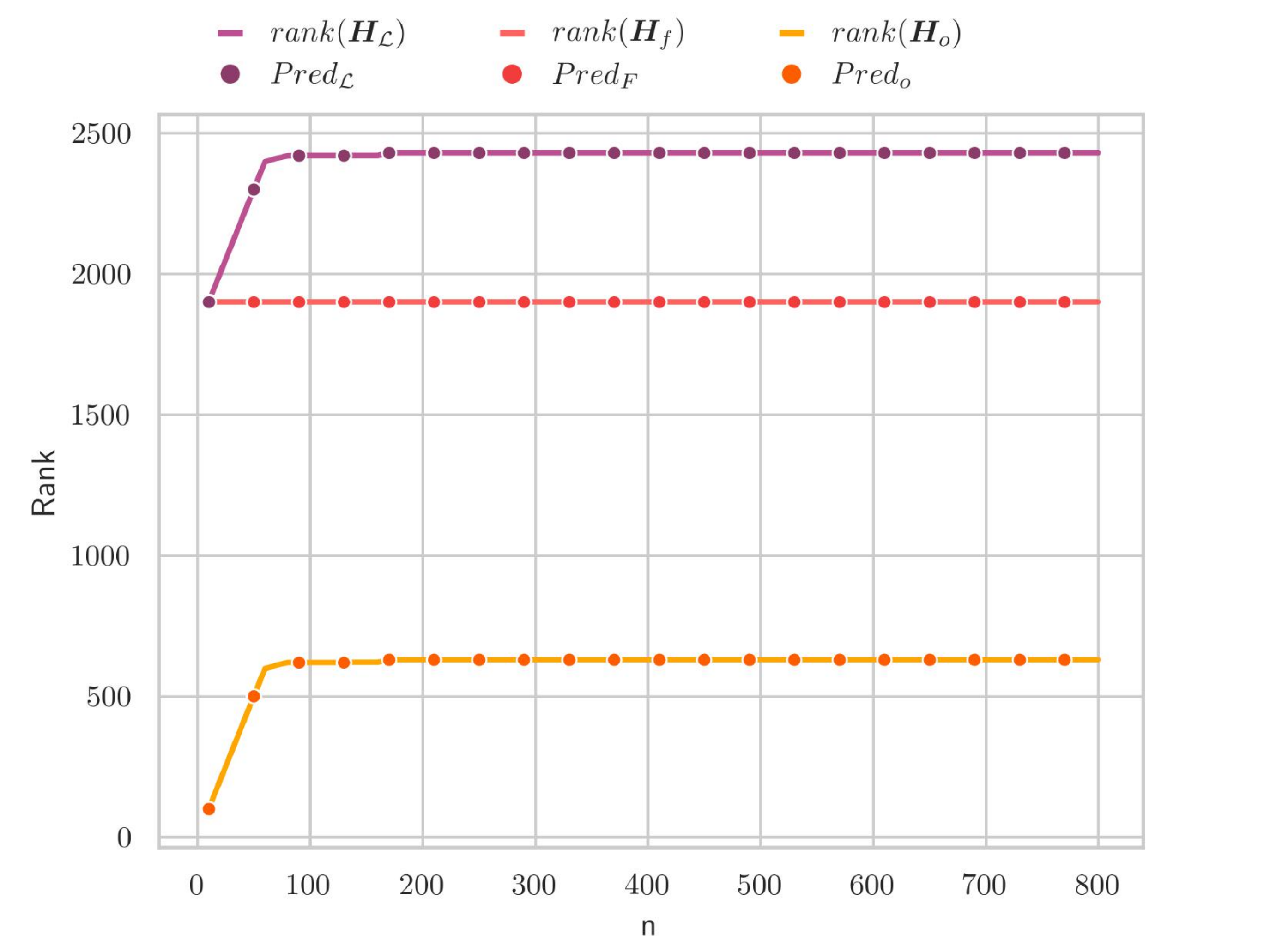}
    \caption{Rank vs sample size $n$}
    \label{fig:fashion-a1_app}
    \end{subfigure}
    \begin{subfigure}[t]{0.3\textwidth}
    \includegraphics[width=\textwidth]{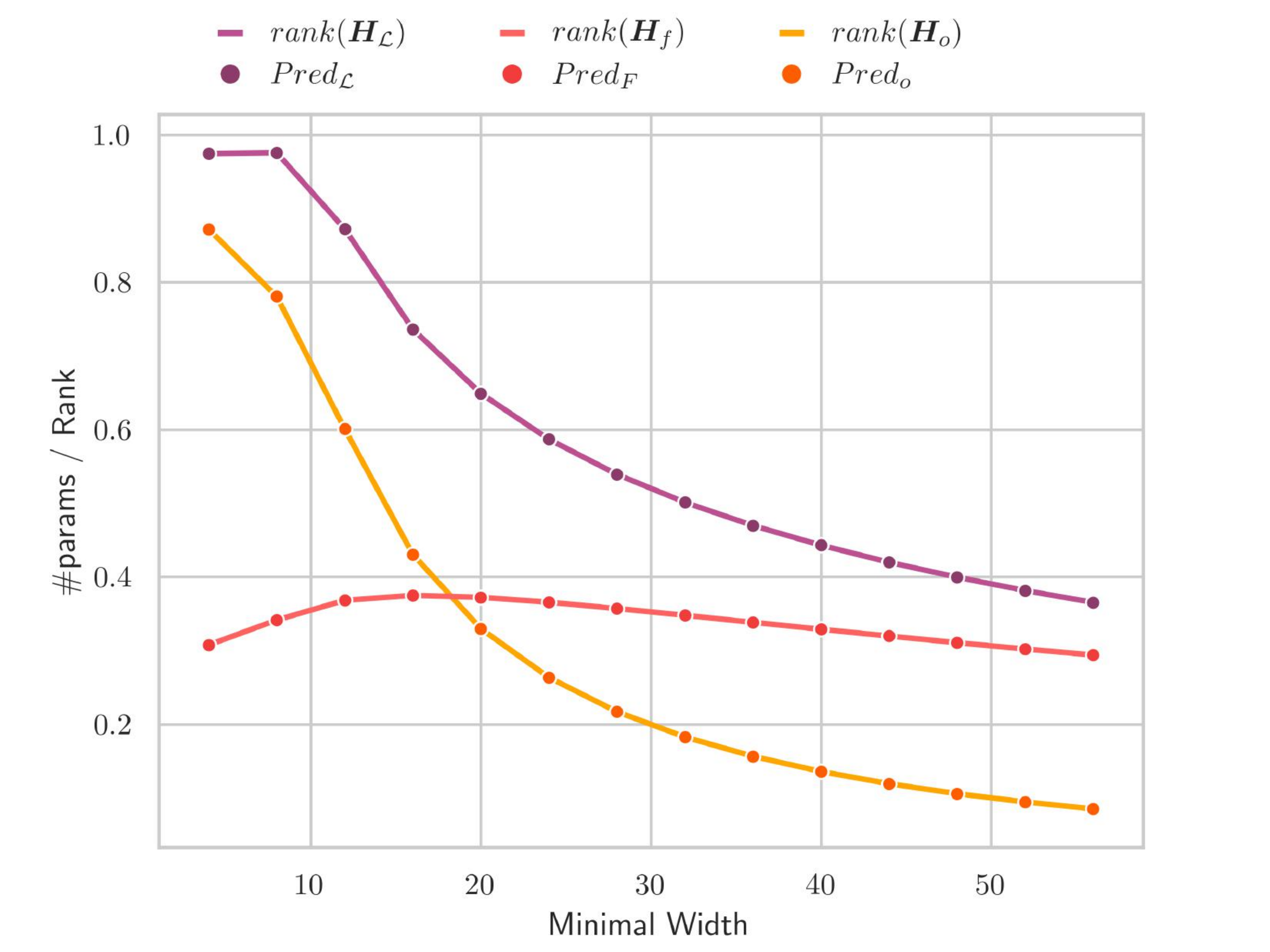}
    \caption{$\effparams$ vs minimal width $M_{*}$}
    \label{fig:fashion-a2_app}
    \end{subfigure}
    \begin{subfigure}[t]{0.3\textwidth}
    \includegraphics[width=\textwidth]{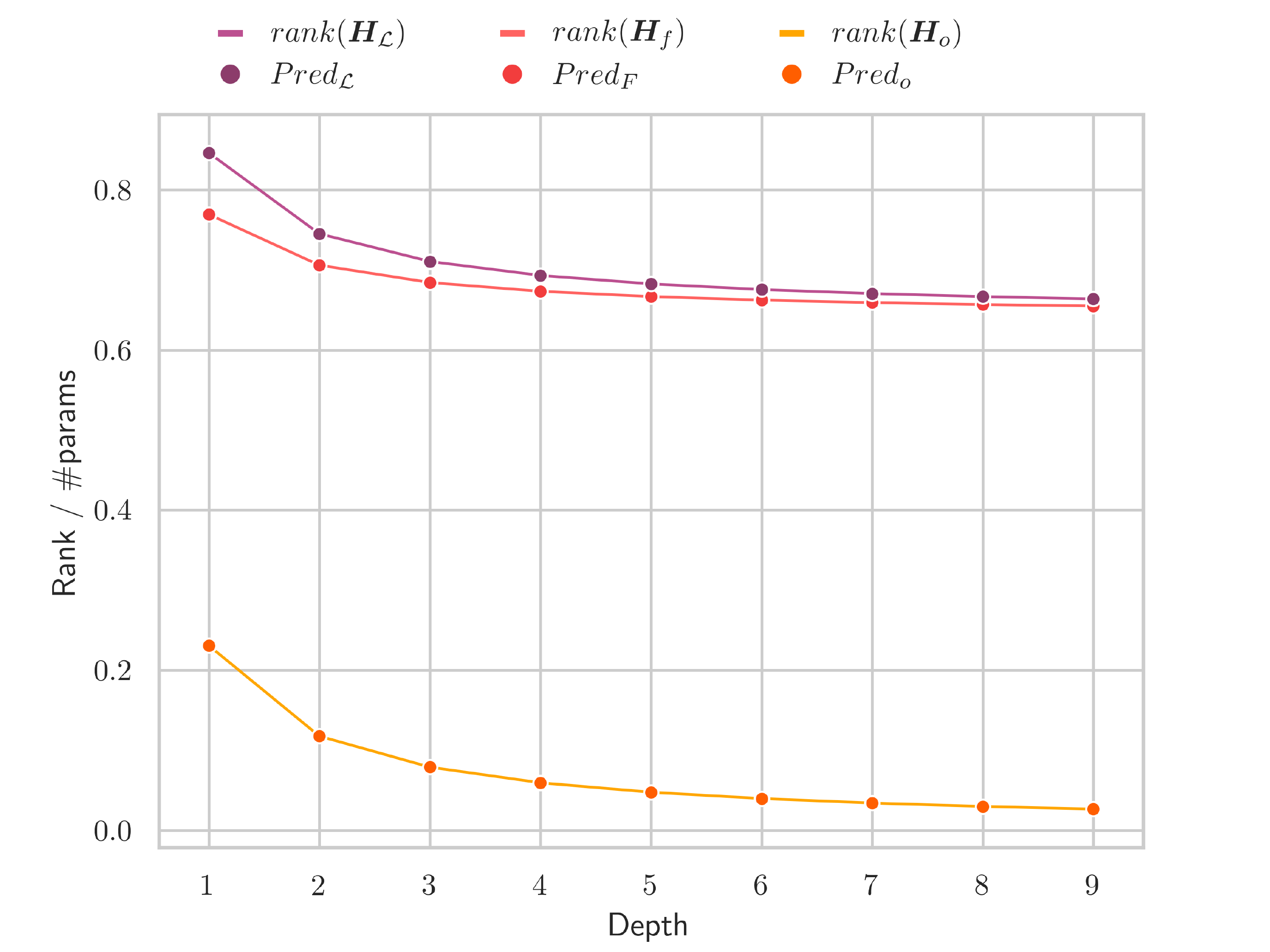}
    \caption{$\effparams$ vs depth $L$}
    \label{fig:fashion-a3_app}
    \end{subfigure}
    \caption{\small{Behaviour of rank and rank/\#params on Fashion\textsc{MNIST} using MSE, with hidden layers: $50,20,20,20$ (Fig.~\ref{fig:fashion-a1_app}), $M_{*}, M_{*}$ (Fig.~\ref{fig:fashion-a2_app}) and $L$ layers of width $M=25$ (Fig.~\ref{fig:fashion-a3_app}). The lines indicate the true value and circles denote our formula predictions. }}
    \label{fashion_mse}
\end{figure}

\FloatBarrier
\begin{figure}[!h]
    \centering
    \begin{subfigure}[t]{0.3\textwidth}
    \includegraphics[width=\textwidth]{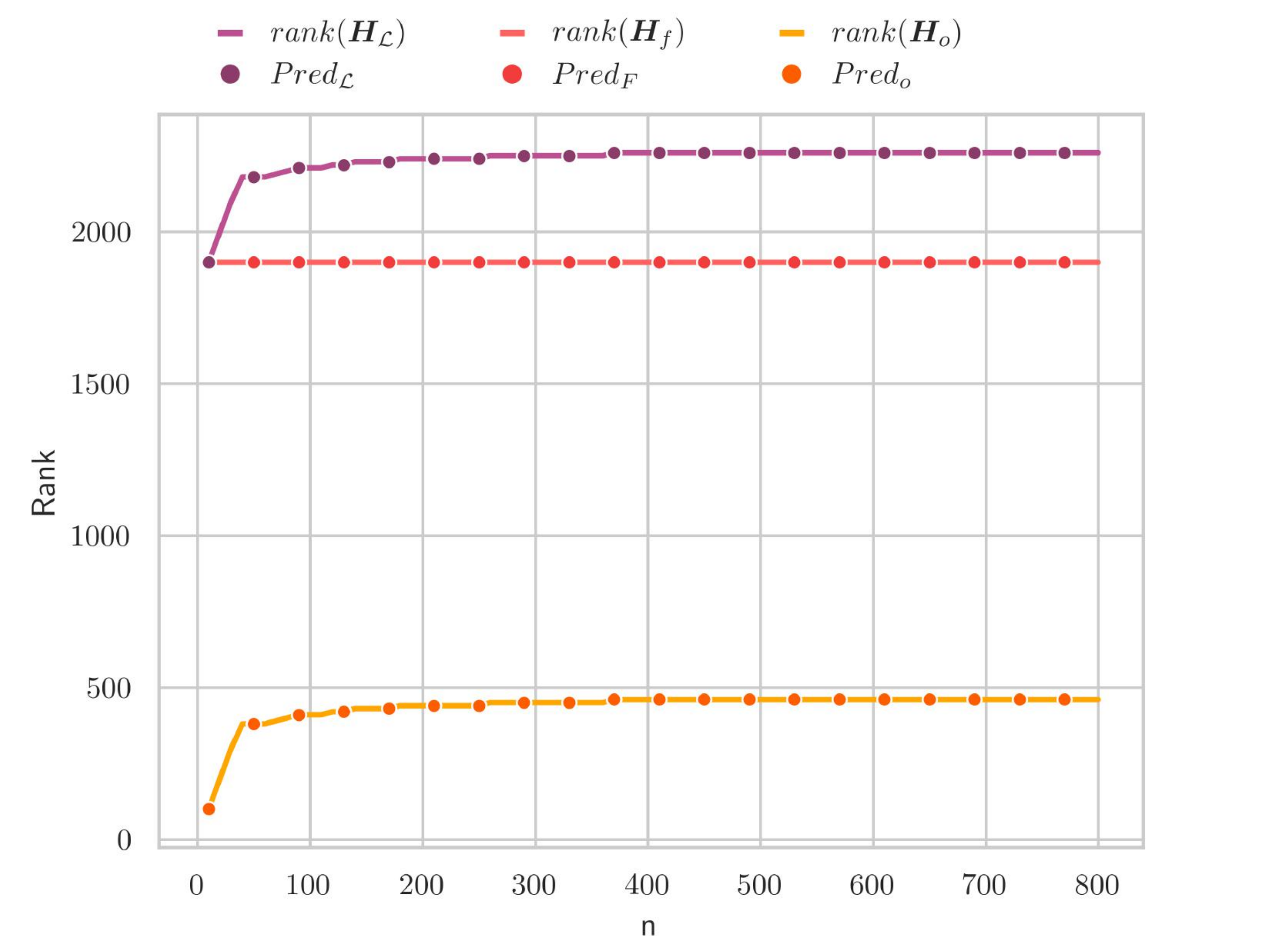}
    \caption{Rank vs sample size $n$}
    \label{fig:mnist-a1_app}
    \end{subfigure}
    \begin{subfigure}[t]{0.3\textwidth}
    \includegraphics[width=\textwidth]{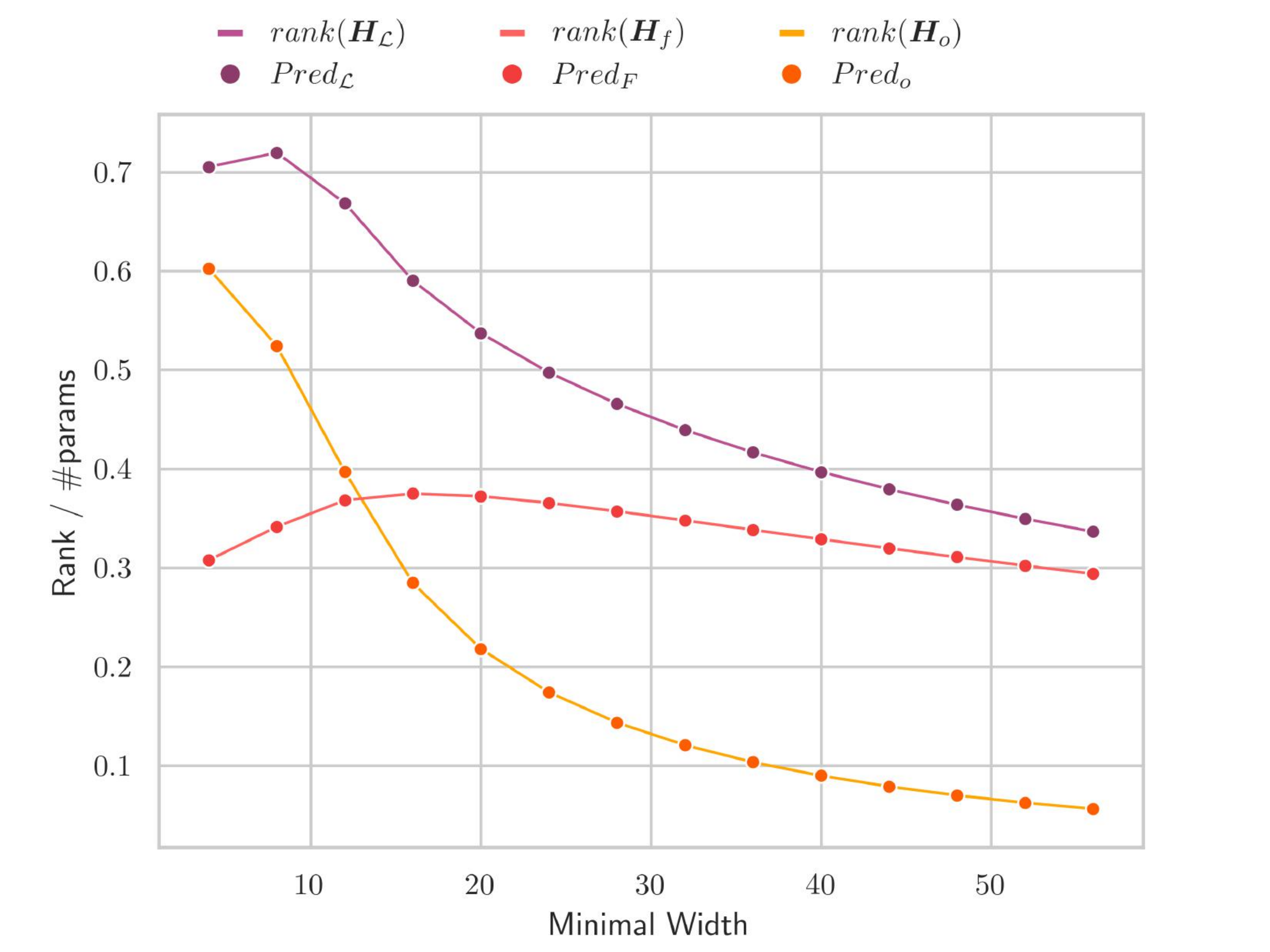}
    \caption{$\effparams$ vs minimal width $M_{*}$}
    \label{fig:mnist-a2_app}
    \end{subfigure}
    \begin{subfigure}[t]{0.3\textwidth}
    \includegraphics[width=\textwidth]{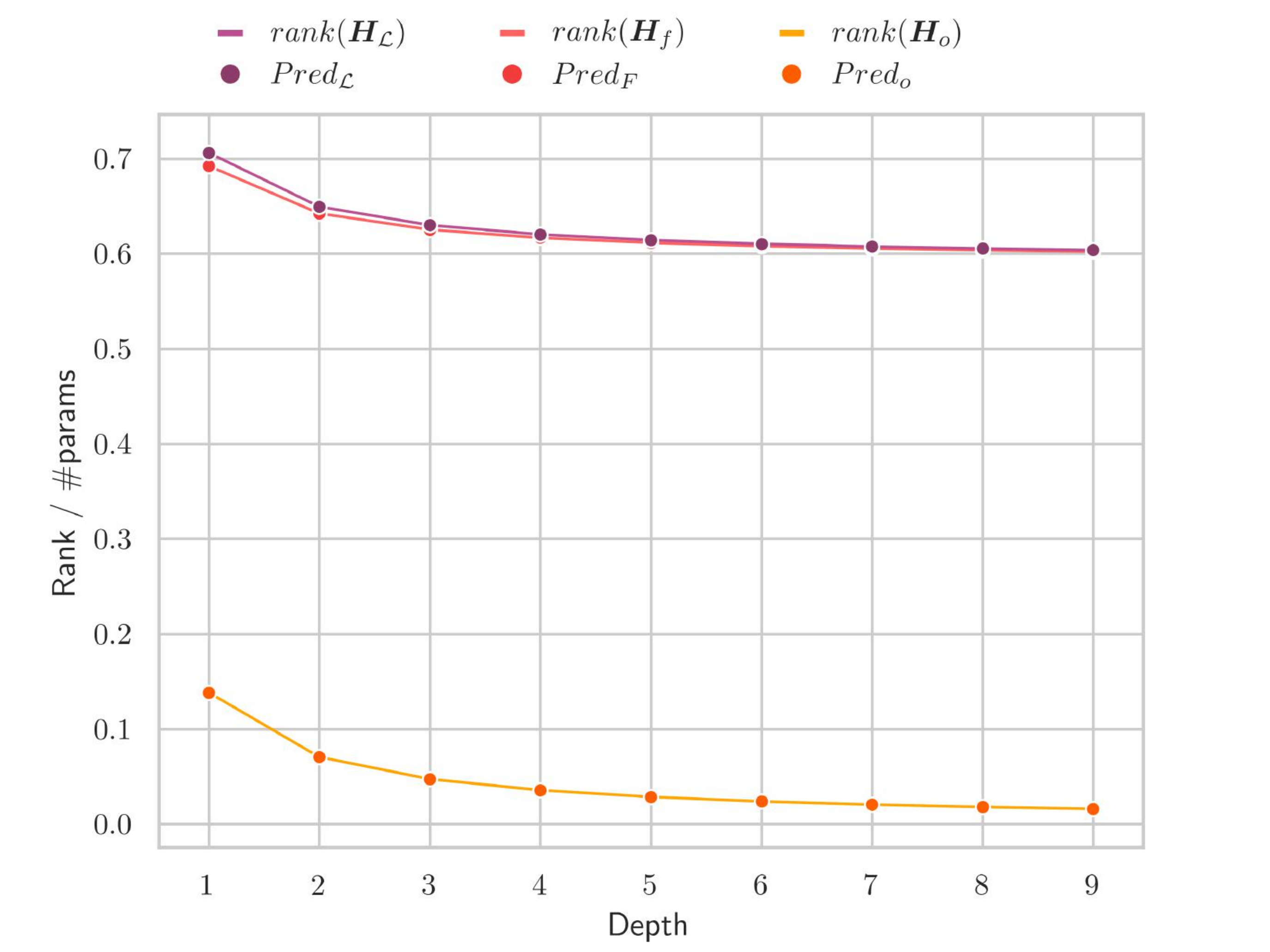}
    \caption{$\effparams$ vs depth $L$}
    \label{fig:mnist-a3_app}
    \end{subfigure}
    \caption{\small{Behaviour of rank and rank/\#params on \textsc{MNIST} using MSE, with hidden layers: $50,20,20,20$ (Fig.~\ref{fig:mnist-a1_app}), $M_{*}, M_{*}$ (Fig.~\ref{fig:mnist-a2_app}) and $L$ layers of width $M=25$ (Fig.~\ref{fig:mnist-a3_app}). The lines indicate the true value and circles denote our formula predictions. }}
    \label{mnist_mse}
\end{figure}

\FloatBarrier
\subsubsection{Cross Entropy}
Here we consider another popular loss function, namely cross entropy, which is defined as 
$$\ell_{\text{cp}}(\btheta)=-\sum_{i=1}^{N}\sum_{k=1}^{K}\operatorname{log}\left(\operatorname{softmax}_{k}(F_{\btheta}(\x_i)\right)y_{ik}$$
where $y_{ik} = \begin{cases} 1 \hspace{3mm} \text{if } ``k" \text{ is the label} \\
0 \hspace{3mm} \text{otherwise}\end{cases}$ and $\operatorname{softmax}_k(\bm{z}) = \frac{e^{z_k}}{\sum_{l=1}^{K}e^{z_l}}$.
Observe that cross entropy is combined with a softmax operation at the output layer, constraining the final vector to sum to $1$, i.e. $\sum_{l=1}^{K}\operatorname{softmax}_k(\bm{z}) = 1$. This induces, by construction a linear dependence at the output, thus instead of having $K$ free outputs, we only have $K-1$ independent outputs. We reflect this in our rank formulas by replacing every occurrence of $K$ by $K-1$. 

Figure \ref{cifar_cross} shows the results for \textsc{CIFAR10}, Figure \ref{fashion_cross} for Fashion\textsc{MNIST} and Figure \ref{mnist_cross} for \textsc{MNIST}. We observe a perfect match for all the datasets.
\begin{figure}[!h]
    \centering
    \begin{subfigure}[t]{0.3\textwidth}
    \includegraphics[width=\textwidth]{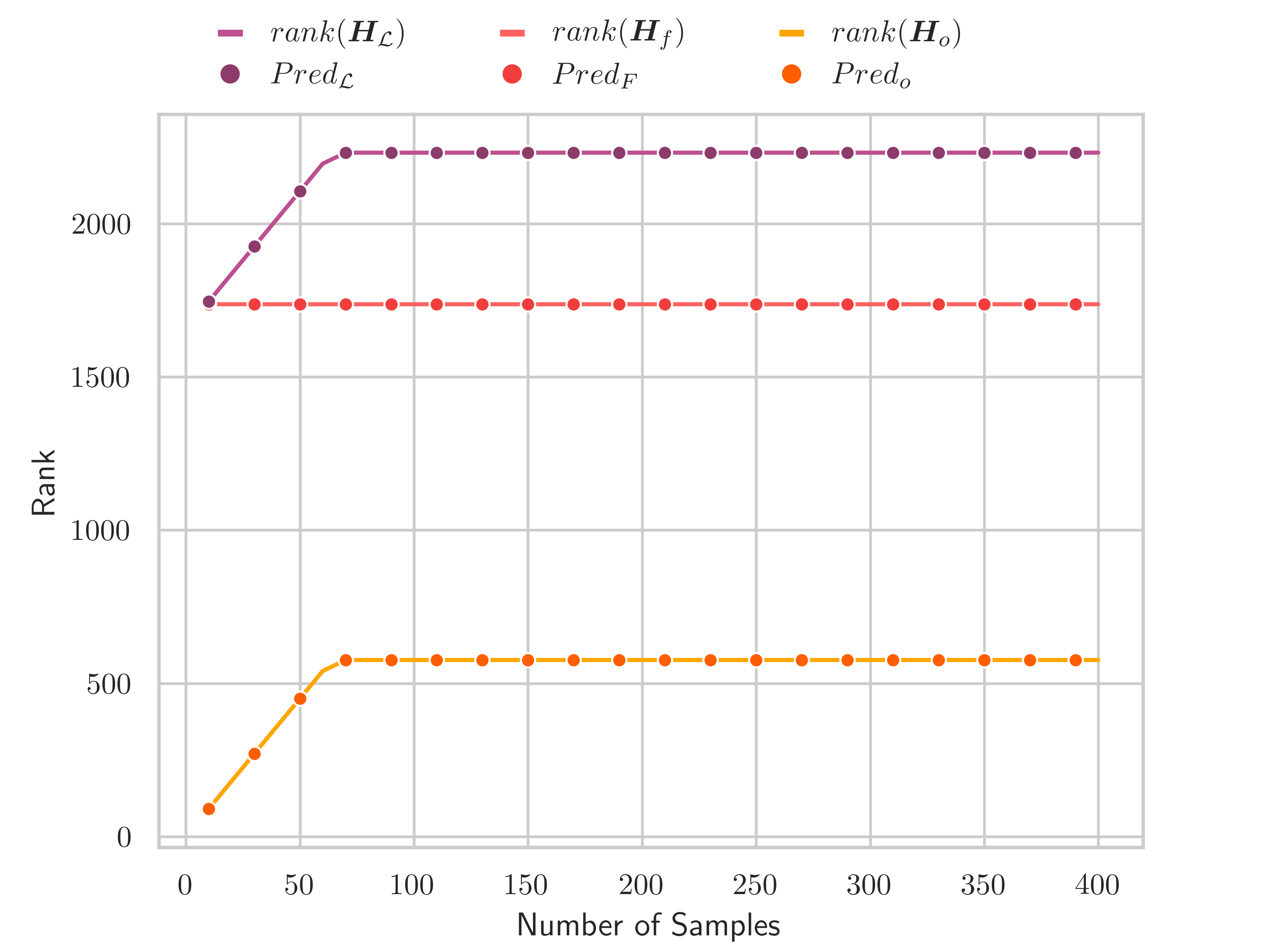}
    \caption{Rank vs sample size $n$}
    \label{fig:cifar10-a1_app_cross}
    \end{subfigure}
    \begin{subfigure}[t]{0.3\textwidth}
    \includegraphics[width=\textwidth]{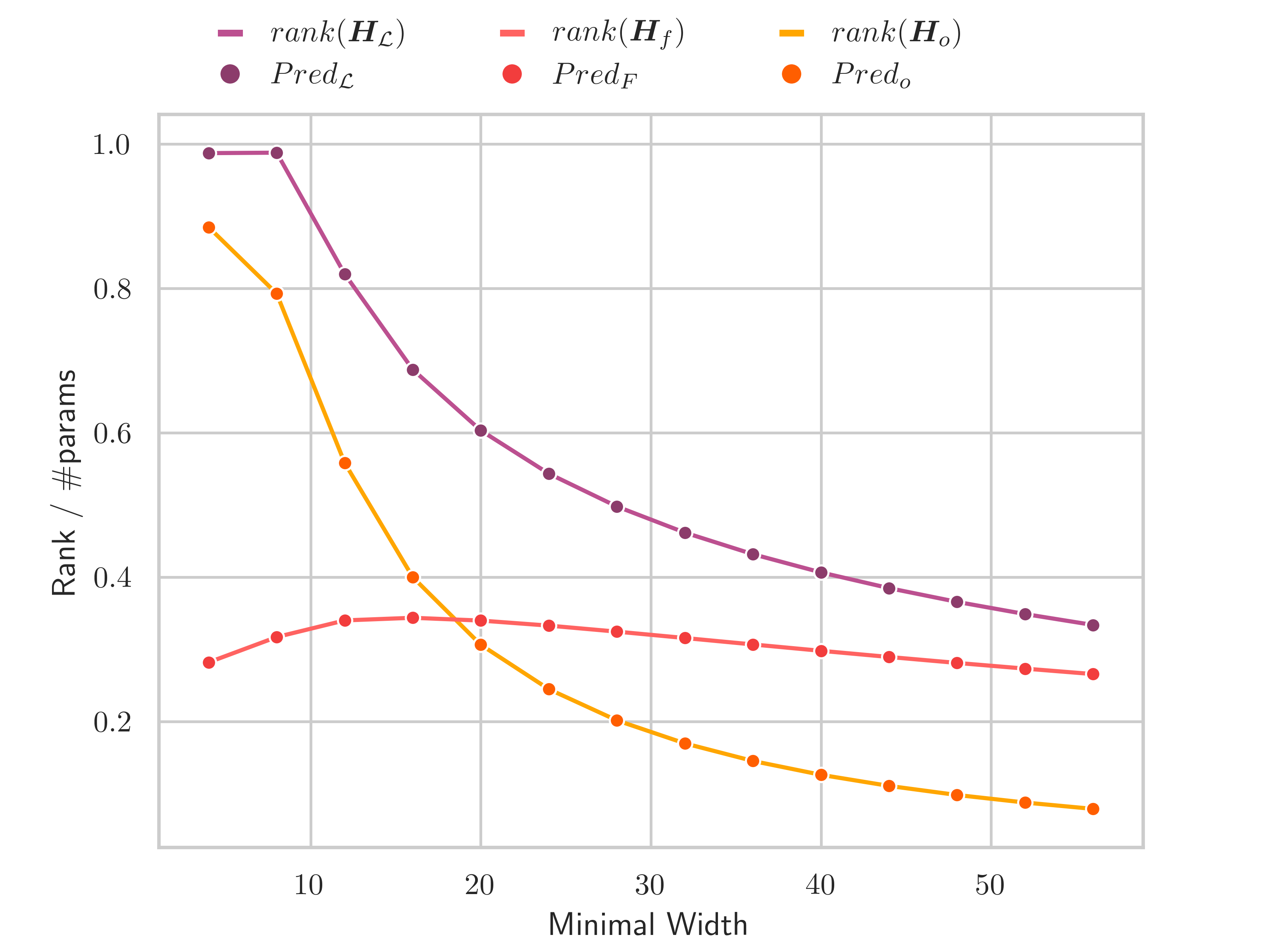}
    \caption{$\effparams$ vs minimal width $M_{*}$}
    \label{fig:cifar10-a2_app_cross}
    \end{subfigure}
    \begin{subfigure}[t]{0.3\textwidth}
    \includegraphics[width=\textwidth]{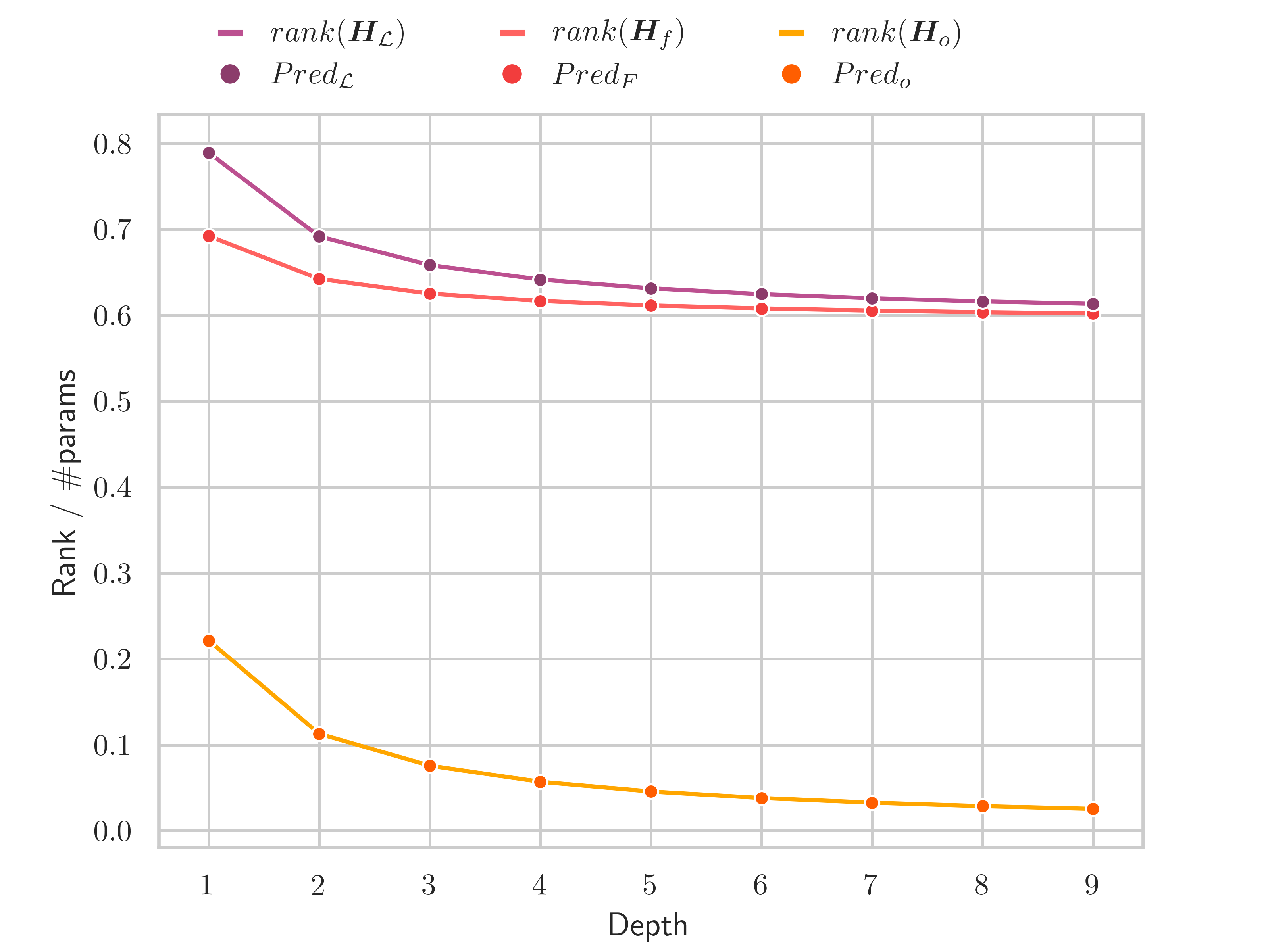}
    \caption{$\effparams$ vs depth $L$}
    \label{fig:cifar10-a3_app_cross}
    \end{subfigure}
    \caption{\small{Behaviour of rank and rank/\#params on CIFAR10 using cross entropy, with hidden layers: $50,20,20,20$ (Fig.~\ref{fig:cifar10-a1_app_cross}), $M_{*}, M_{*}$ (Fig.~\ref{fig:cifar10-a2_app_cross}) and $L$ layers of width $M=25$ (Fig.~\ref{fig:cifar10-a3_app_cross}). The lines indicate the true value and circles denote our formula predictions. }}
    \label{cifar_cross}
\end{figure}
\FloatBarrier
\begin{figure}[!h]
    \centering
    \begin{subfigure}[t]{0.3\textwidth}
    \includegraphics[width=\textwidth]{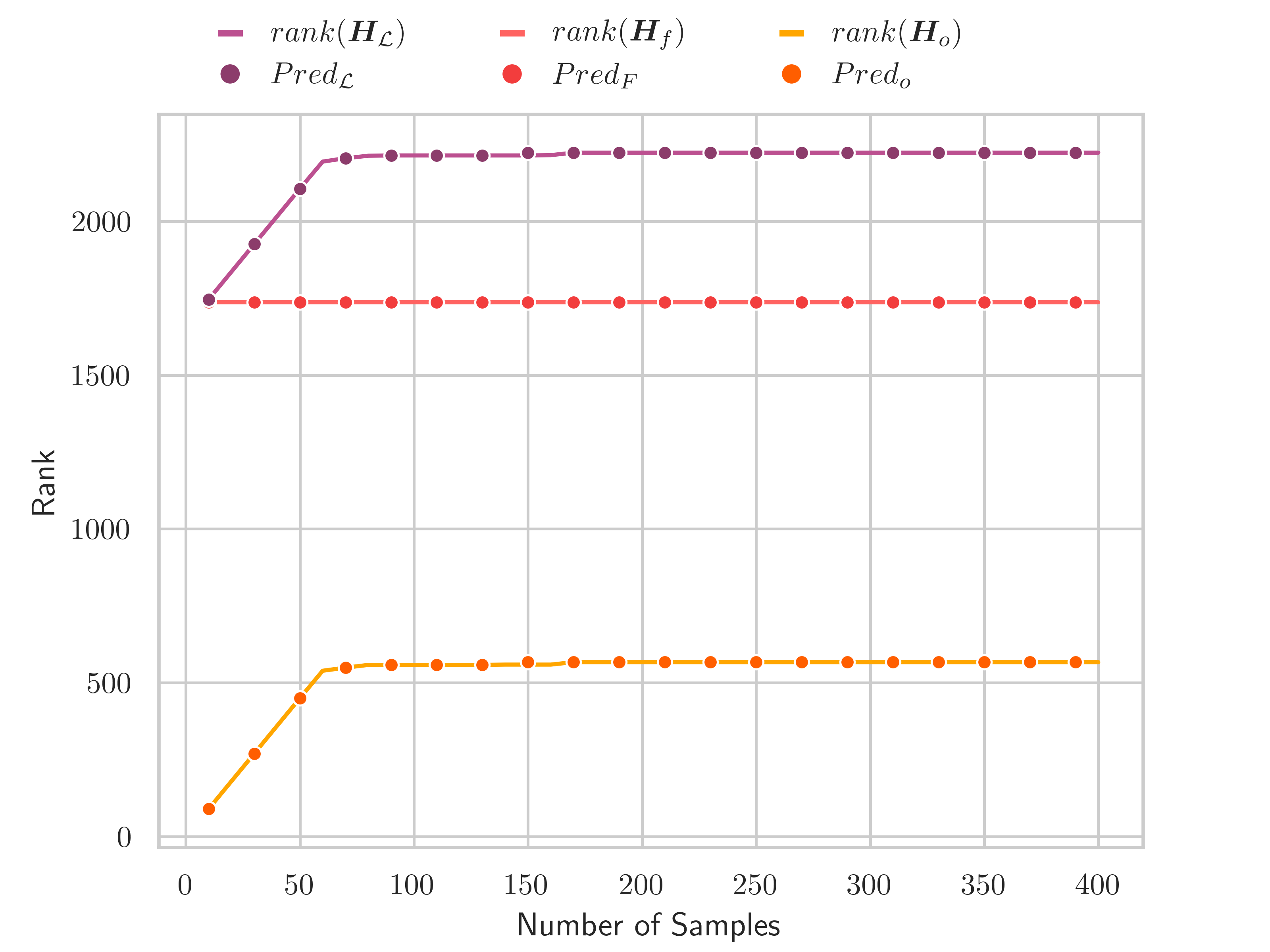}
    \caption{Rank vs sample size $n$}
    \label{fig:fashion-a1_app_cross}
    \end{subfigure}
    \begin{subfigure}[t]{0.3\textwidth}
    \includegraphics[width=\textwidth]{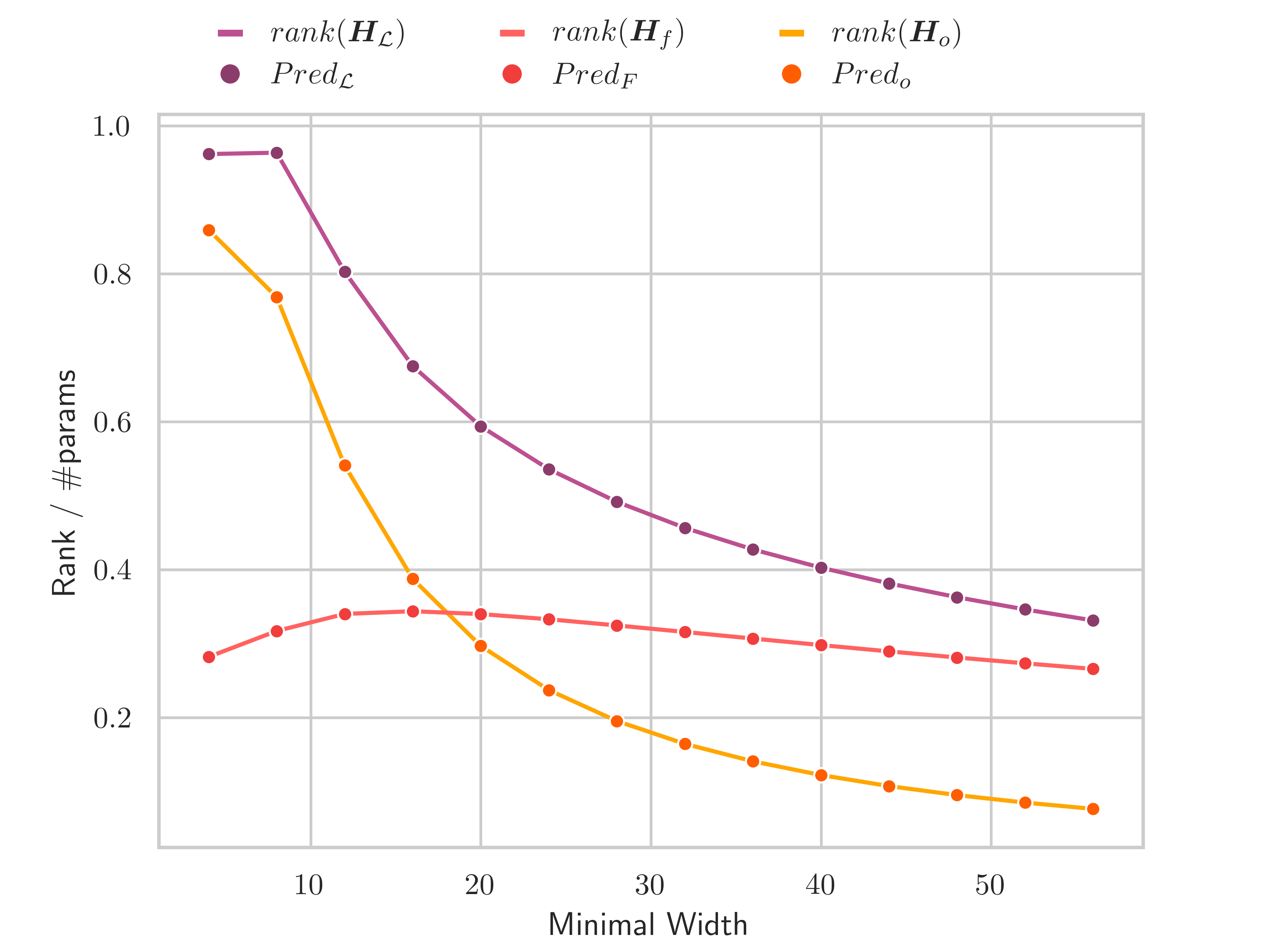}
    \caption{$\effparams$ vs minimal width $M_{*}$}
    \label{fig:fashion-a2_app_cross}
    \end{subfigure}
    \begin{subfigure}[t]{0.3\textwidth}
    \includegraphics[width=\textwidth]{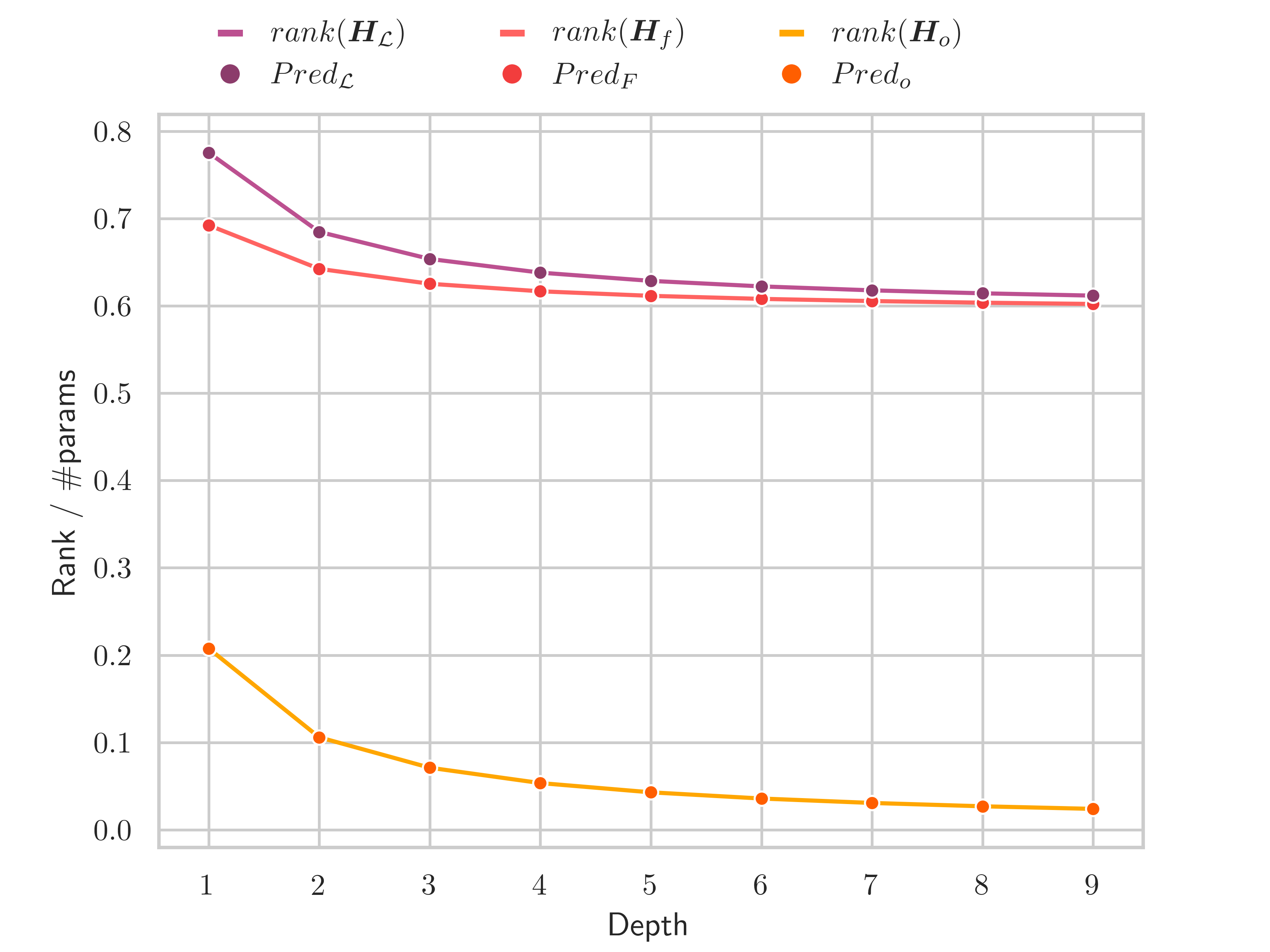}
    \caption{$\effparams$ vs depth $L$}
    \label{fig:fashion-a3_app_cross}
    \end{subfigure}
    \caption{\small{Behaviour of rank and rank/\#params on Fashion\textsc{MNIST} using cross entropy, with hidden layers: $50,20,20,20$ (Fig.~\ref{fig:fashion-a1_app_cross}), $M_{*}, M_{*}$ (Fig.~\ref{fig:fashion-a2_app_cross}) and $L$ layers of width $M=25$ (Fig.~\ref{fig:fashion-a3_app_cross}). The lines indicate the true value and circles denote our formula predictions. }}
    \label{fashion_cross}
\end{figure}
\FloatBarrier

\begin{figure}[!h]
    \centering
    \begin{subfigure}[t]{0.3\textwidth}
    \includegraphics[width=\textwidth]{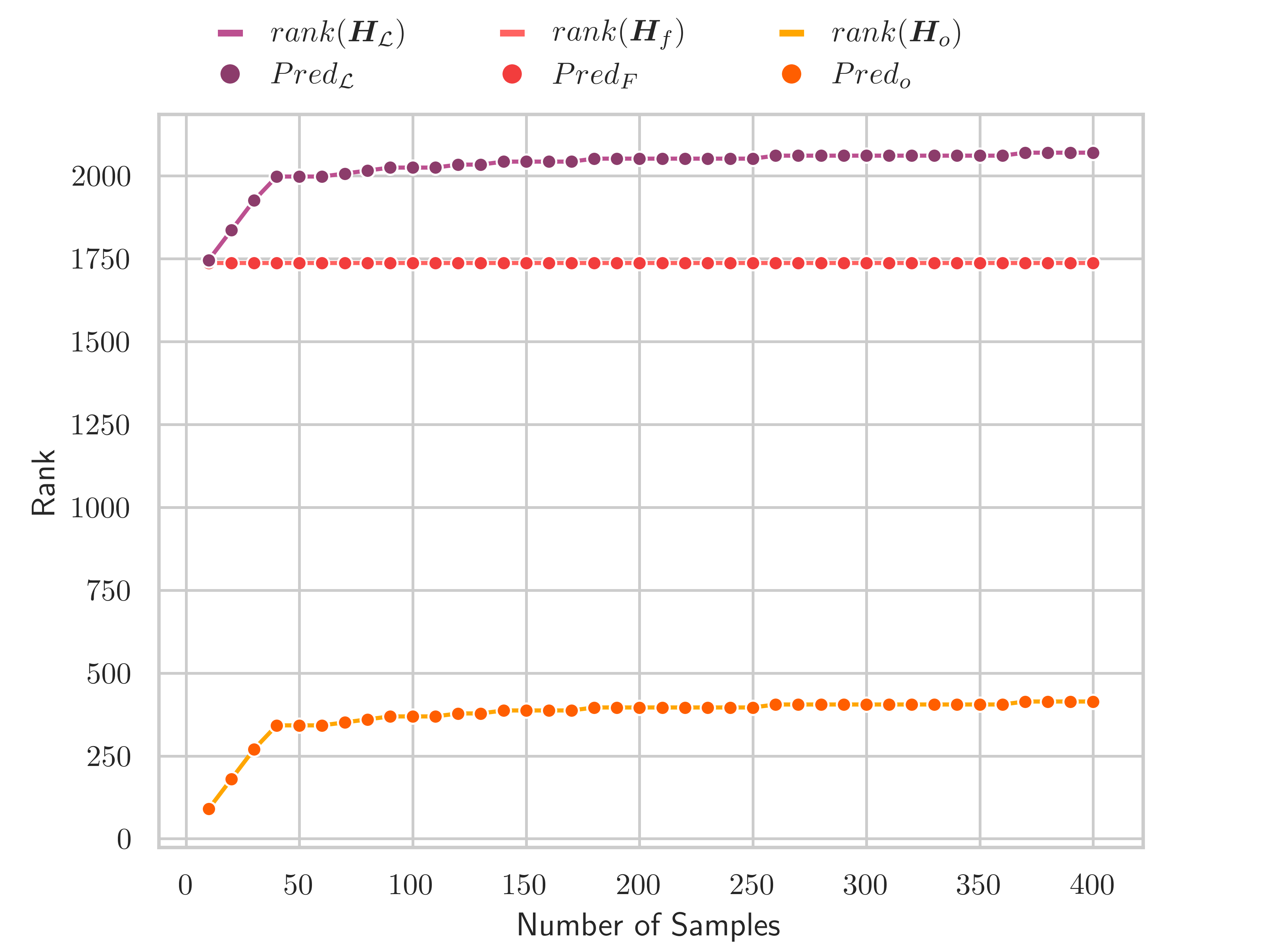}
    \caption{Rank vs sample size $n$}
    \label{fig:mnist-a1_app_cross}
    \end{subfigure}
    \begin{subfigure}[t]{0.3\textwidth}
    \includegraphics[width=\textwidth]{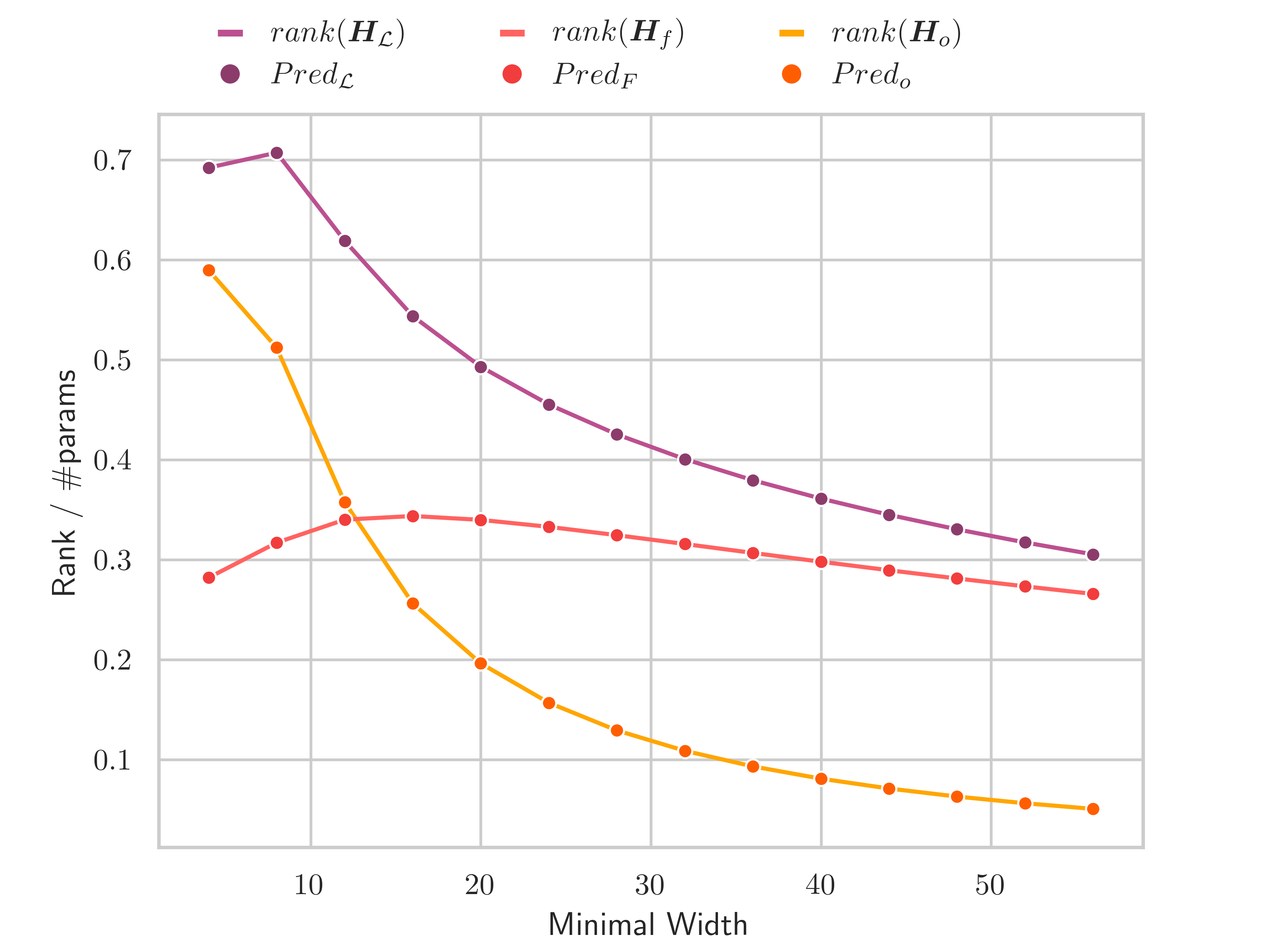}
    \caption{$\effparams$ vs minimal width $M_{*}$}
    \label{fig:mnist-a2_app_cross}
    \end{subfigure}
    \begin{subfigure}[t]{0.3\textwidth}
    \includegraphics[width=\textwidth]{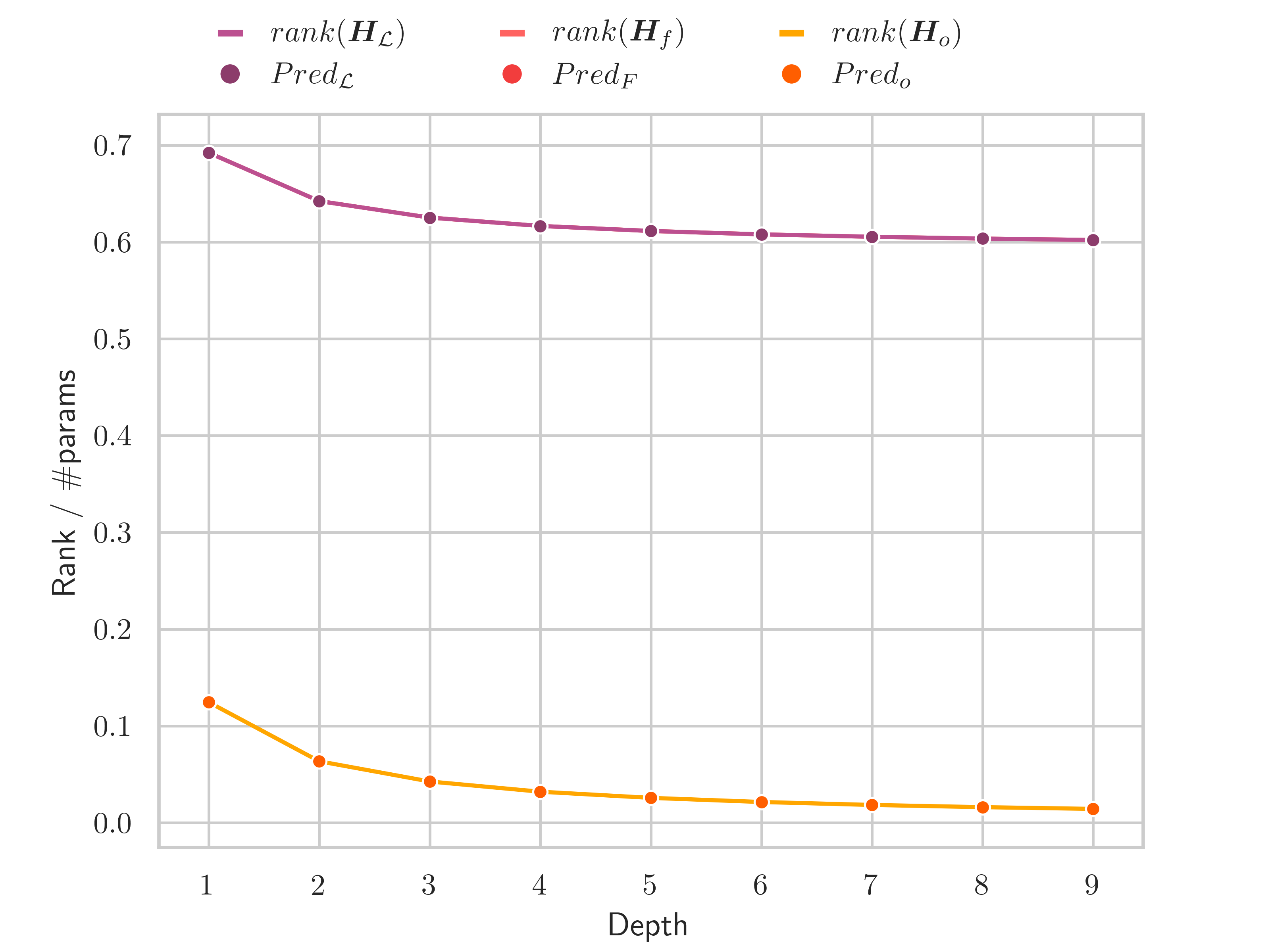}
    \caption{$\effparams$ vs depth $L$}
    \label{fig:mnist-a3_app_cross}
    \end{subfigure}
    \caption{\small{Behaviour of rank and rank/\#params on \textsc{MNIST} using cross entropy, with hidden layers: $50,20,20,20$ (Fig.~\ref{fig:mnist-a1_app_cross}), $M_{*}, M_{*}$ (Fig.~\ref{fig:mnist-a2_app_cross}) and $L$ layers of width $M=25$ (Fig.~\ref{fig:mnist-a3_app_cross}). The lines indicate the true value and circles denote our formula predictions. }}
    \label{mnist_cross}
\end{figure}
\FloatBarrier

\subsubsection{Cosh Loss}
To highlight that our formulas are very robust to even more exotic loss functions, we consider the cosh-loss, defined as
$$\ell_{cosh}(\btheta) = \sum_{i=1}^{N}\sum_{k=1}^{K}\operatorname{log}\left(\cosh\left(\hat{y}_{ik}-y_{ik}\right)\right)$$
Figure \ref{cifar_cosh} shows the results for \textsc{CIFAR10}, Figure \ref{fashion_cosh} for Fashion\textsc{MNIST} and Figure \ref{mnist_cosh} for \textsc{MNIST}. We observe a perfect match for all the datasets.
Also in this case we observe an exact match empirically for all datasets and varying sample size, width and depth. 

\begin{figure}[!h]
    \centering
    \begin{subfigure}[t]{0.3\textwidth}
    \includegraphics[width=\textwidth]{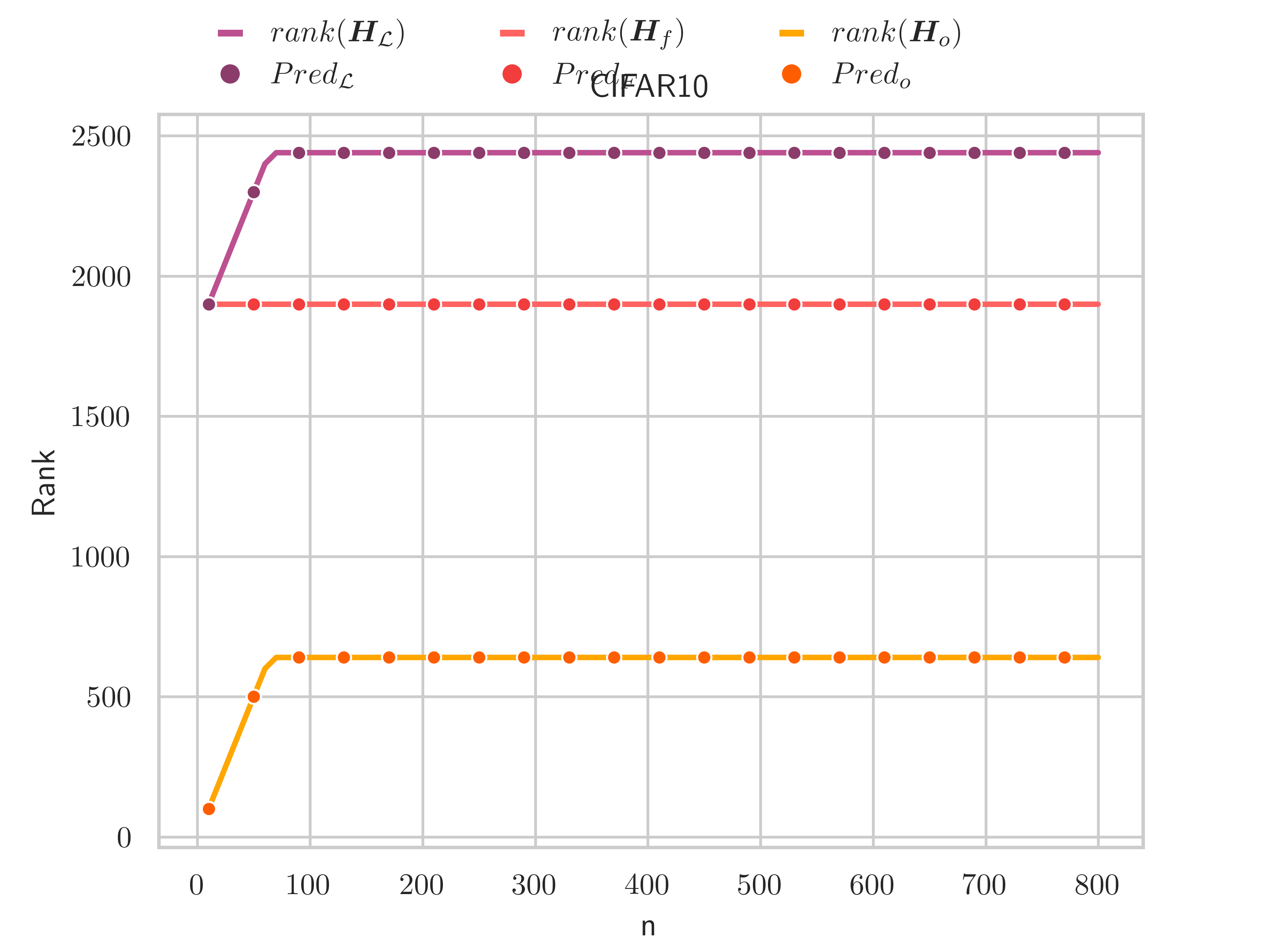}
    \caption{Rank vs sample size $n$}
    \label{fig:cifar10-a1_app_cosh}
    \end{subfigure}
    \begin{subfigure}[t]{0.3\textwidth}
    \includegraphics[width=\textwidth]{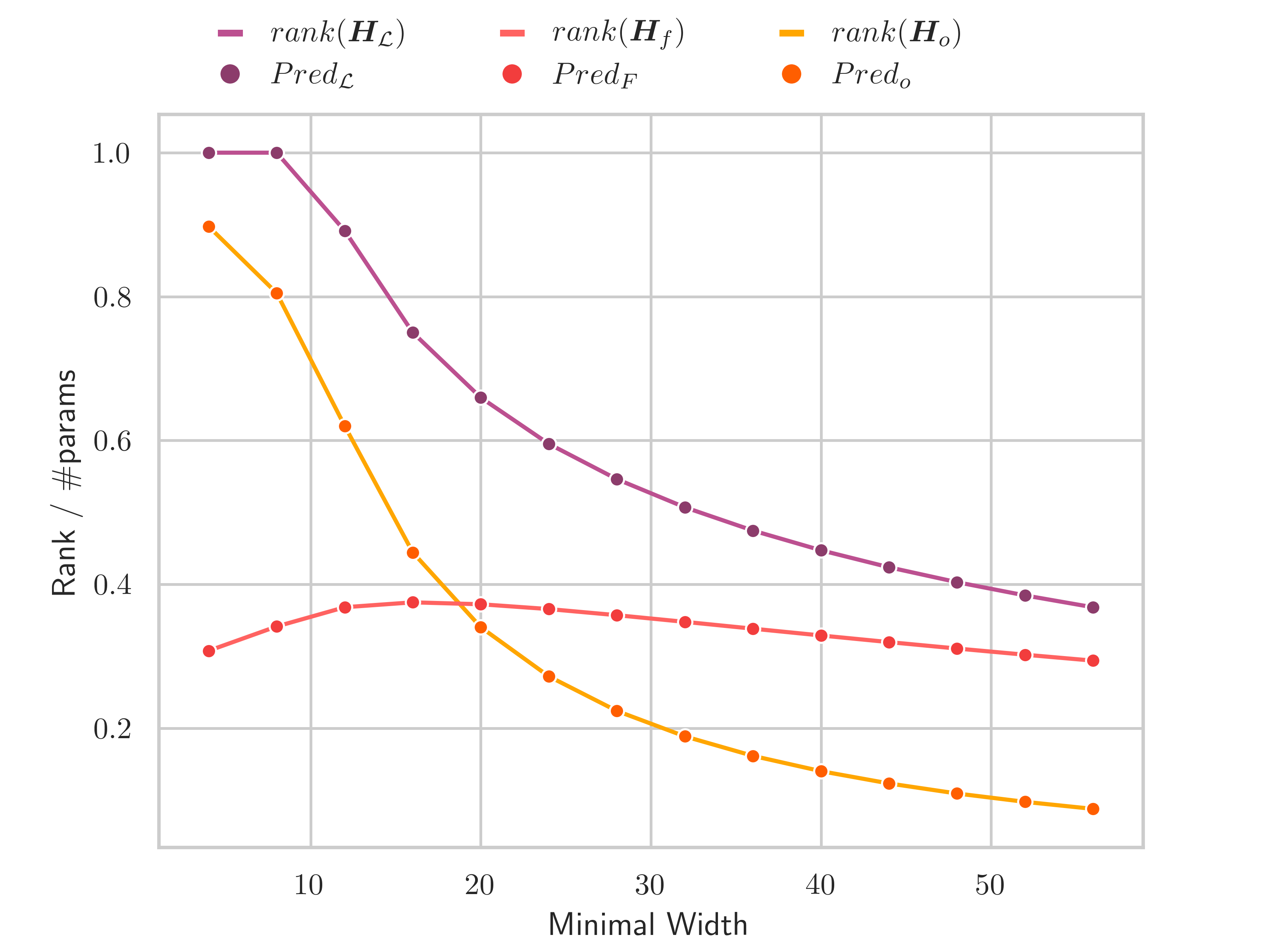}
    \caption{$\effparams$ vs minimal width $M_{*}$}
    \label{fig:cifar10-a2_app_cosh}
    \end{subfigure}
    \begin{subfigure}[t]{0.3\textwidth}
    \includegraphics[width=\textwidth]{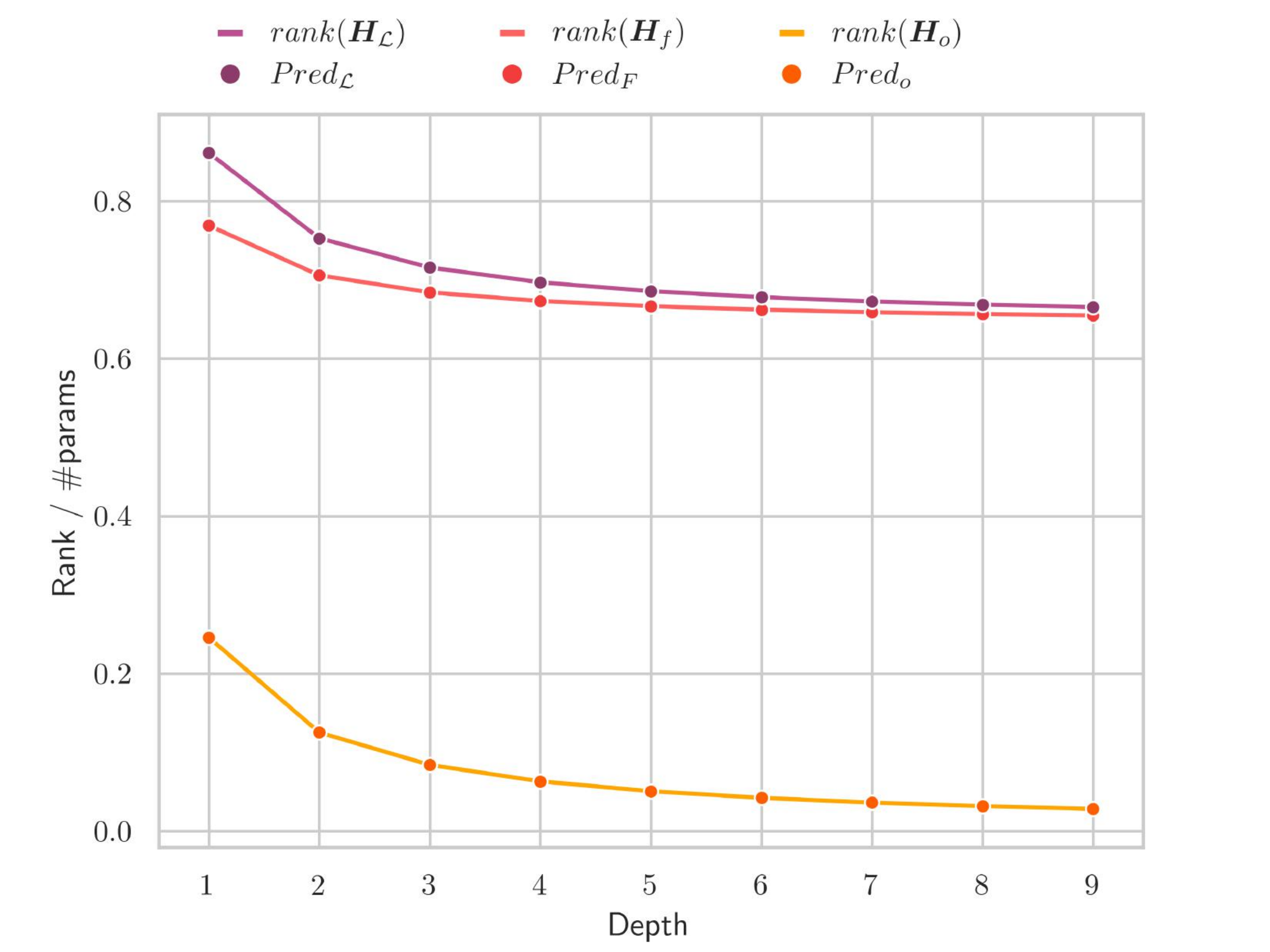}
    \caption{$\effparams$ vs depth $L$}
    \label{fig:cifar10-a3_app_cosh}
    \end{subfigure}
    \caption{\small{Behaviour of rank and rank/\#params on CIFAR10 using cosh loss, with hidden layers: $50,20,20,20$ (Fig.~\ref{fig:cifar10-a1_app_cosh}), $M_{*}, M_{*}$ (Fig.~\ref{fig:cifar10-a2_app_cosh}) and $L$ layers of width $M=25$ (Fig.~\ref{fig:cifar10-a3_app_cosh}). The lines indicate the true value and circles denote our formula predictions. }}
    \label{cifar_cosh}
\end{figure}
\FloatBarrier

\begin{figure}[!h]
    \centering
    \begin{subfigure}[t]{0.3\textwidth}
    \includegraphics[width=\textwidth]{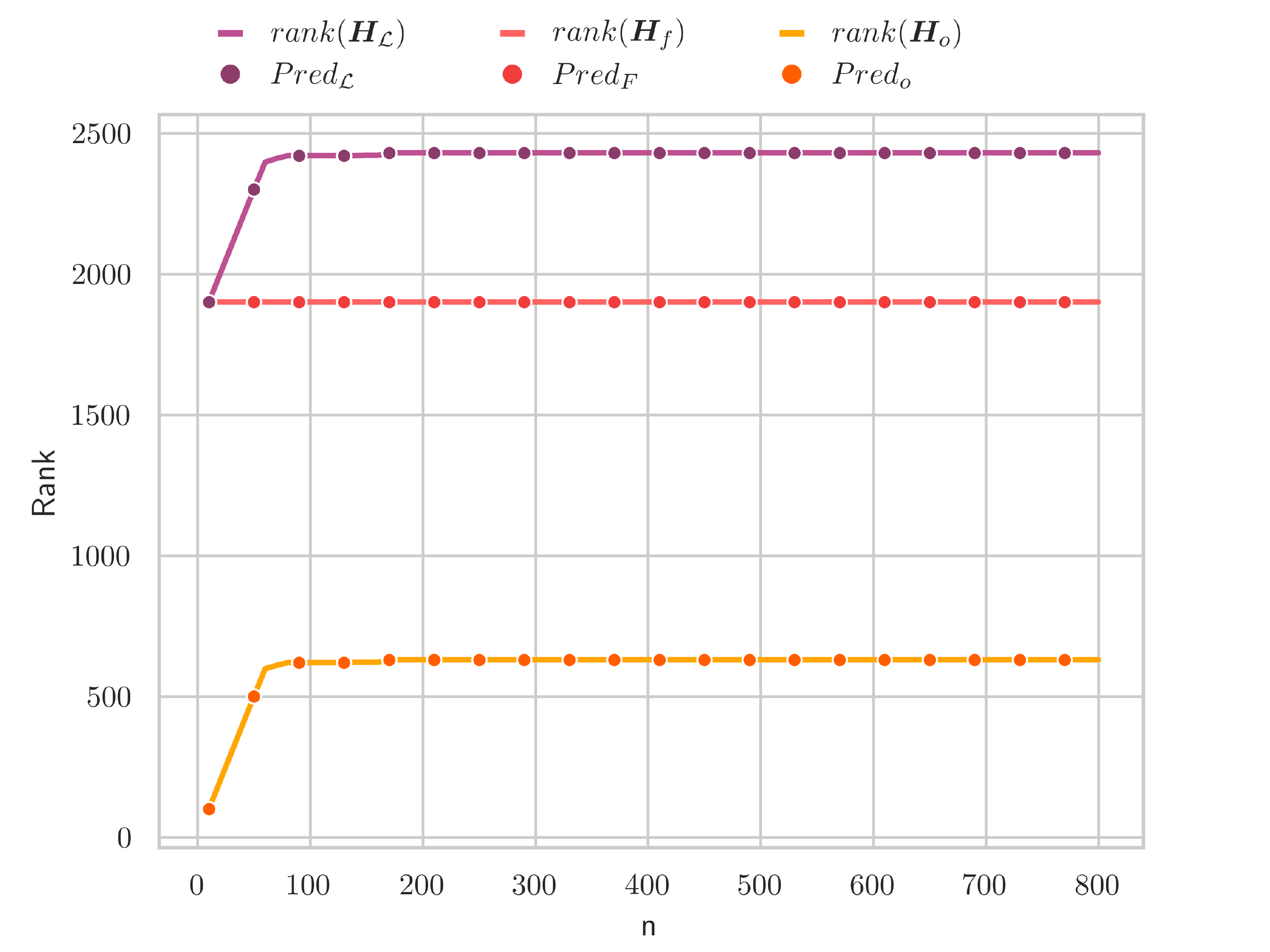}
    \caption{Rank vs sample size $n$}
    \label{fig:fashion-a1_app_cosh}
    \end{subfigure}
    \begin{subfigure}[t]{0.3\textwidth}
    \includegraphics[width=\textwidth]{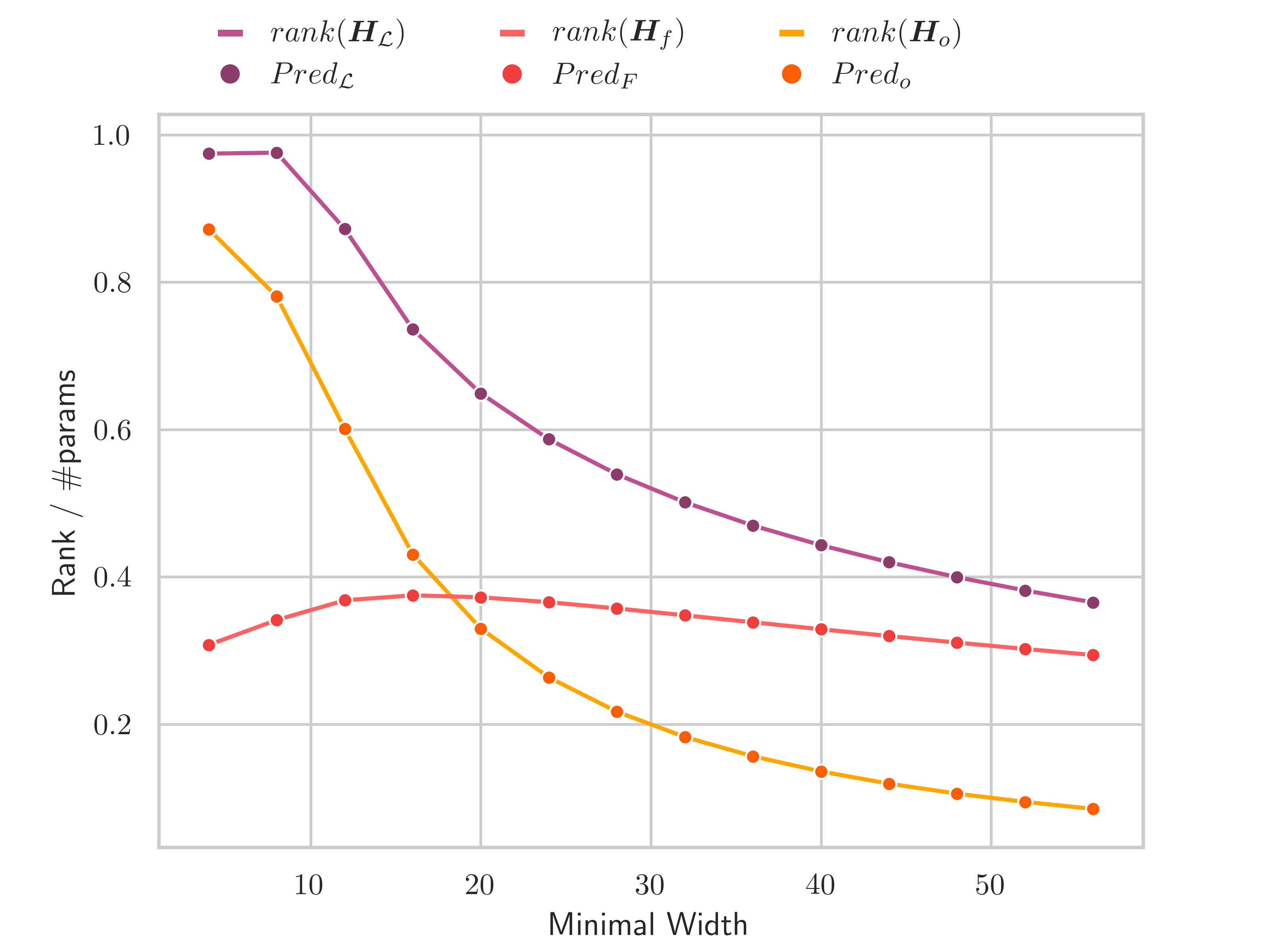}
    \caption{$\effparams$ vs minimal width $M_{*}$}
    \label{fig:fashion-a2_app_cosh}
    \end{subfigure}
    \begin{subfigure}[t]{0.3\textwidth}
    \includegraphics[width=\textwidth]{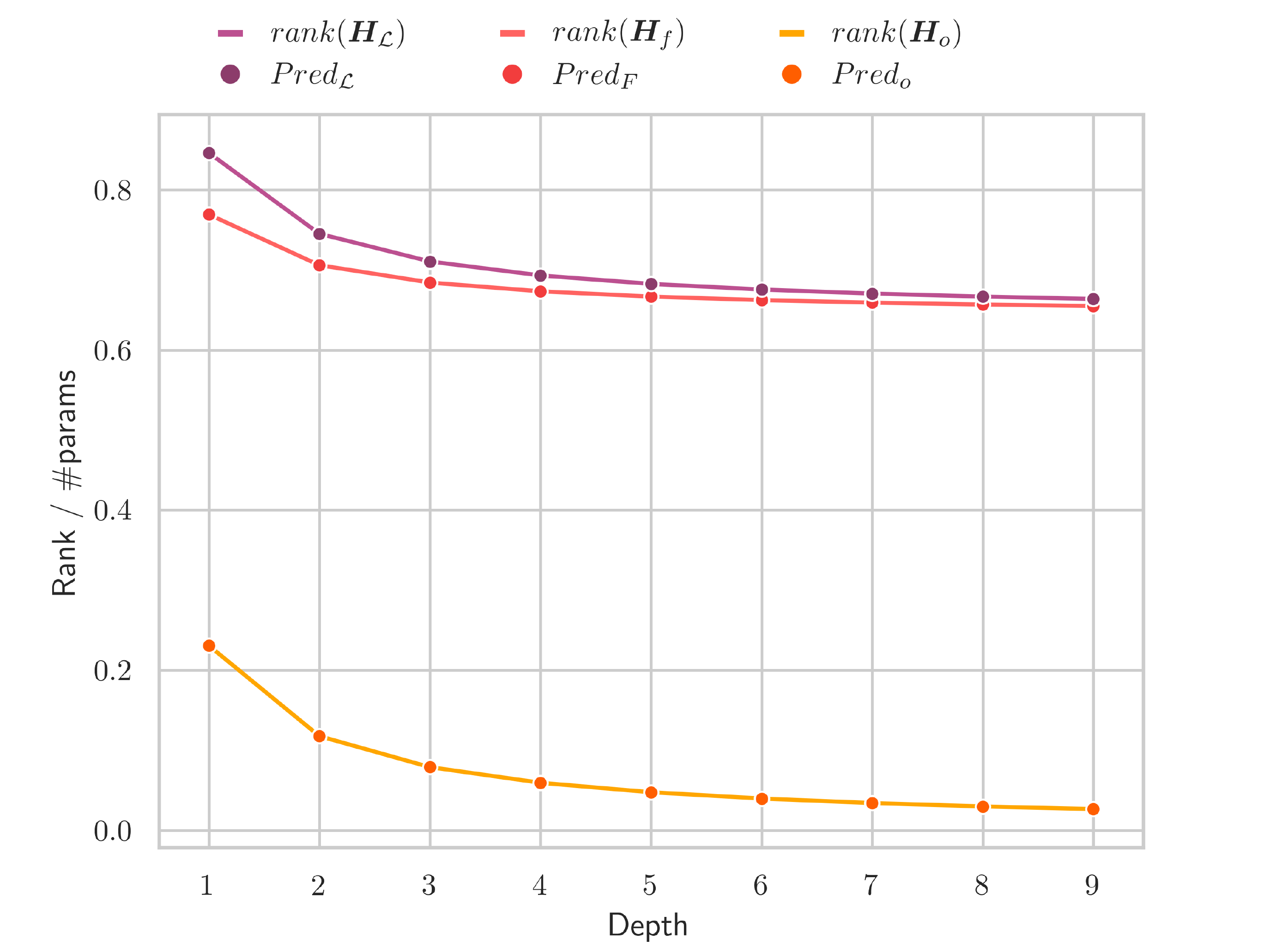}
    \caption{$\effparams$ vs depth $L$}
    \label{fig:fashion-a3_app_cosh}
    \end{subfigure}
    \caption{\small{Behaviour of rank and rank/\#params on Fashion\textsc{MNIST} using cosh loss, with hidden layers: $50,20,20,20$ (Fig.~\ref{fig:fashion-a1_app_cosh}), $M_{*}, M_{*}$ (Fig.~\ref{fig:fashion-a2_app_cosh}) and $L$ layers of width $M=25$ (Fig.~\ref{fig:fashion-a3_app_cosh}). The lines indicate the true value and circles denote our formula predictions. }}
    \label{fashion_cosh}
\end{figure}
\FloatBarrier
\begin{figure}[!h]
    \centering
    \begin{subfigure}[t]{0.3\textwidth}
    \includegraphics[width=\textwidth]{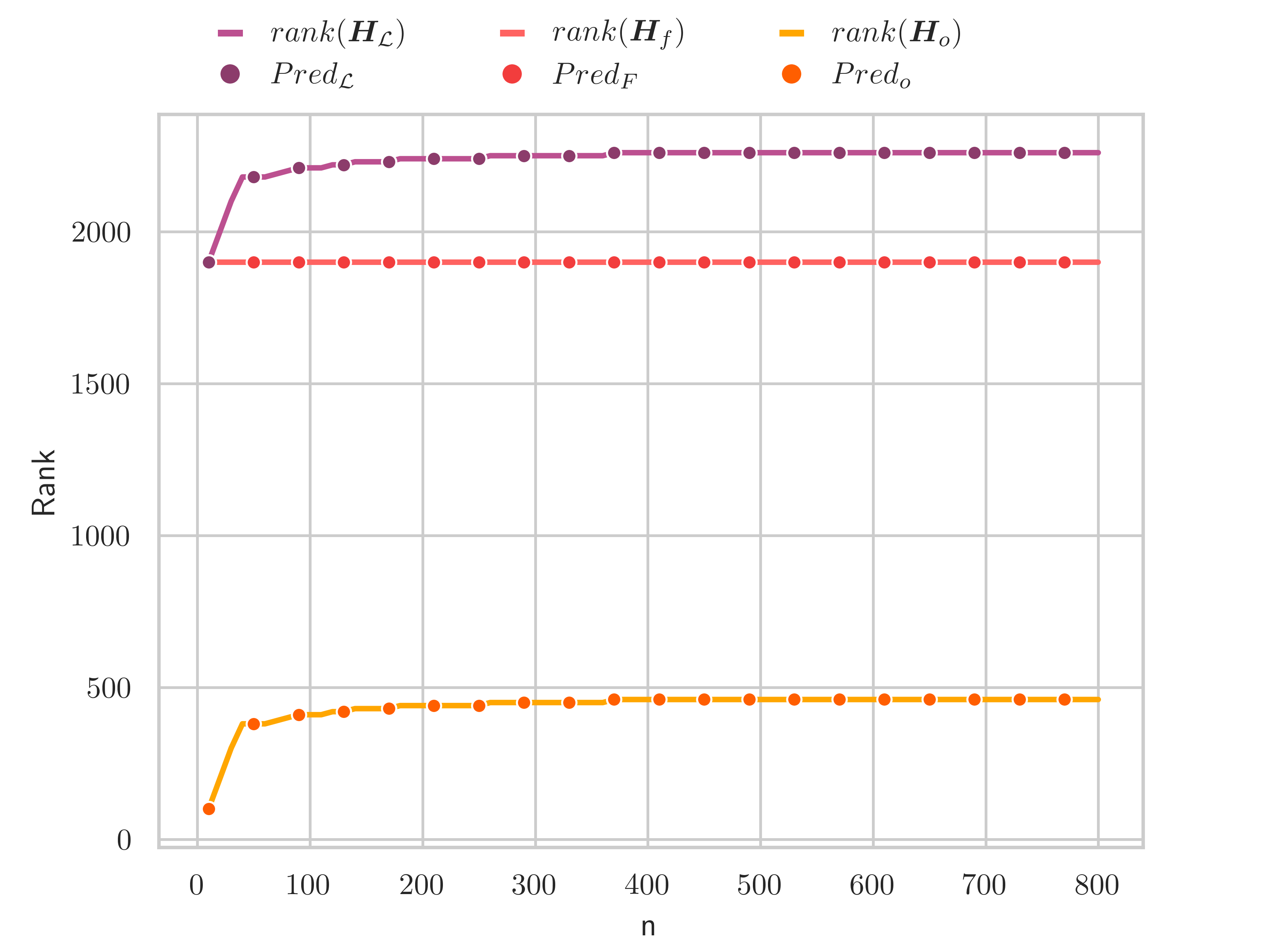}
    \caption{Rank vs sample size $n$}
    \label{fig:mnist-a1_app_cosh}
    \end{subfigure}
    \begin{subfigure}[t]{0.3\textwidth}
    \includegraphics[width=\textwidth]{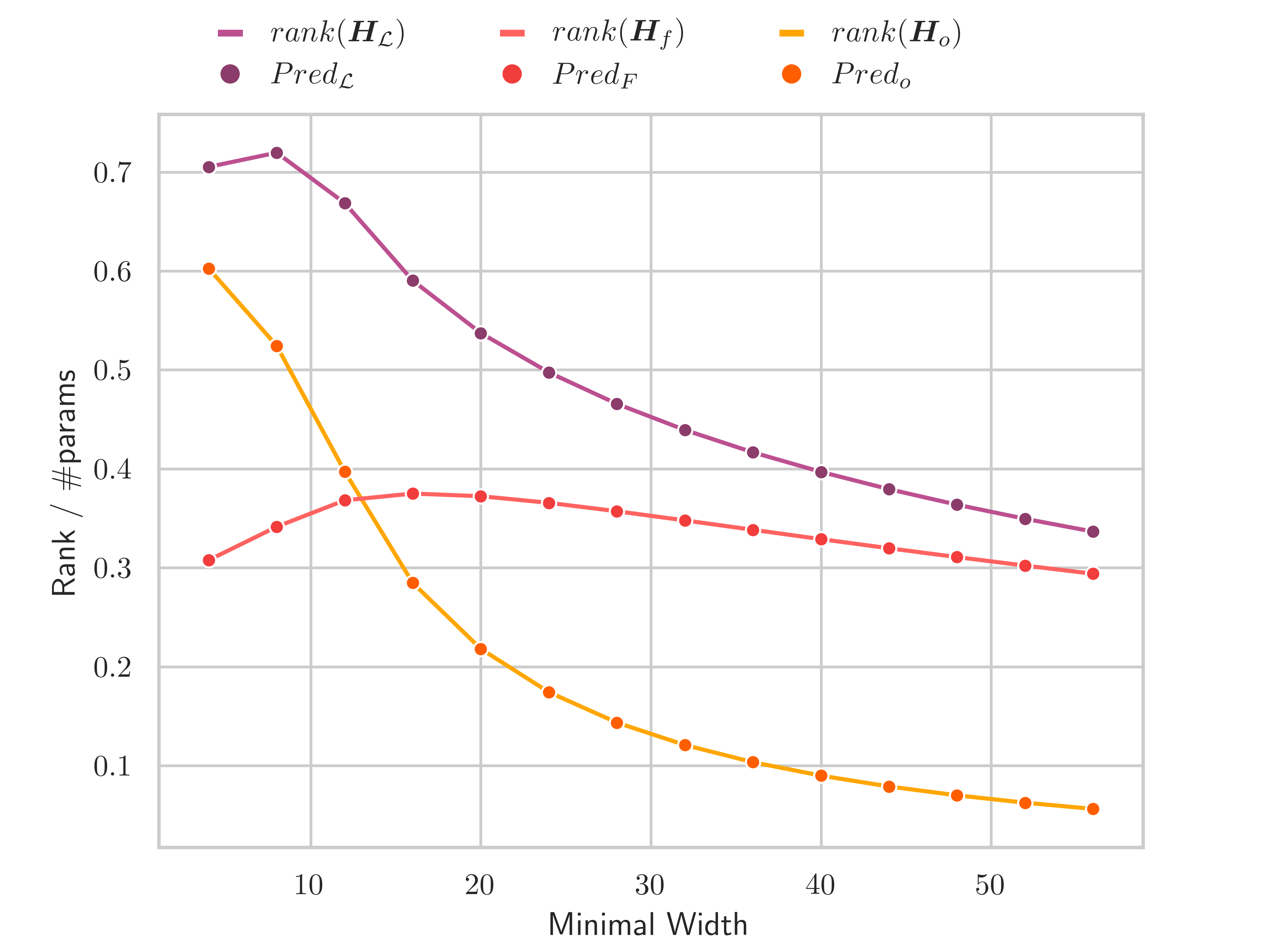}
    \caption{$\effparams$ vs minimal width $M_{*}$}
    \label{fig:mnist-a2_app_cosh}
    \end{subfigure}
    \begin{subfigure}[t]{0.3\textwidth}
    \includegraphics[width=\textwidth]{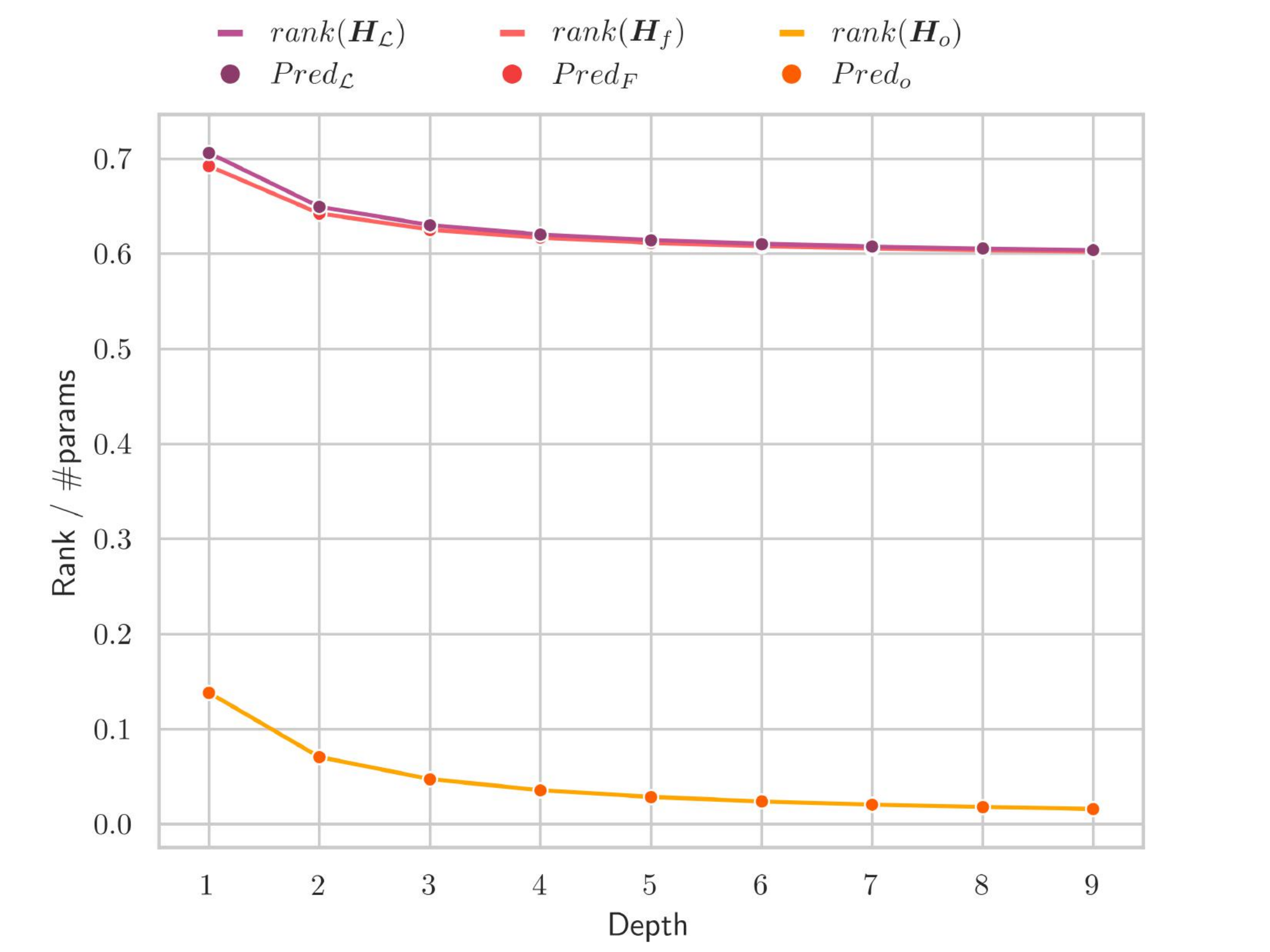}
    \caption{$\effparams$ vs depth $L$}
    \label{fig:mnist-a3_app_cosh}
    \end{subfigure}
    \caption{\small{Behaviour of rank and rank/\#params on \textsc{MNIST} using cosh loss, with hidden layers: $50,20,20,20$ (Fig.~\ref{fig:mnist-a1_app_cosh}), $M_{*}, M_{*}$ (Fig.~\ref{fig:mnist-a2_app_cosh}) and $L$ layers of width $M=25$ (Fig.~\ref{fig:mnist-a3_app_cosh}). The lines indicate the true value and circles denote our formula predictions. }}
     \label{mnist_cosh}
\end{figure}
\FloatBarrier
\subsubsection{Different Initializations}
Here we want to assess whether different initialization schemes can affect our rank predictions. Although our theoretical results suggest that our results hold for any initialization scheme that guarantees full-rank weight matrices, we perform an empirical study on \textsc{CIFAR10} to check this. All the preceding experiments have used Gaussian initialization. Here we also check for uniform initialization, $W^l_{ij} \sim \mathcal{U}(-1, 1)$, and for orthogonal initialization. We display the results for uniform initialization in  Figure \ref{uniform} while Figure \ref{orthogonal} shows the results for orthogonal initialization, for varying sample size, width and depth. As expected from our theoretical insights, we again observe exact matches with our predictions.

\begin{figure}[!h]
    \centering
    \begin{subfigure}[t]{0.3\textwidth}
    \includegraphics[width=\textwidth]{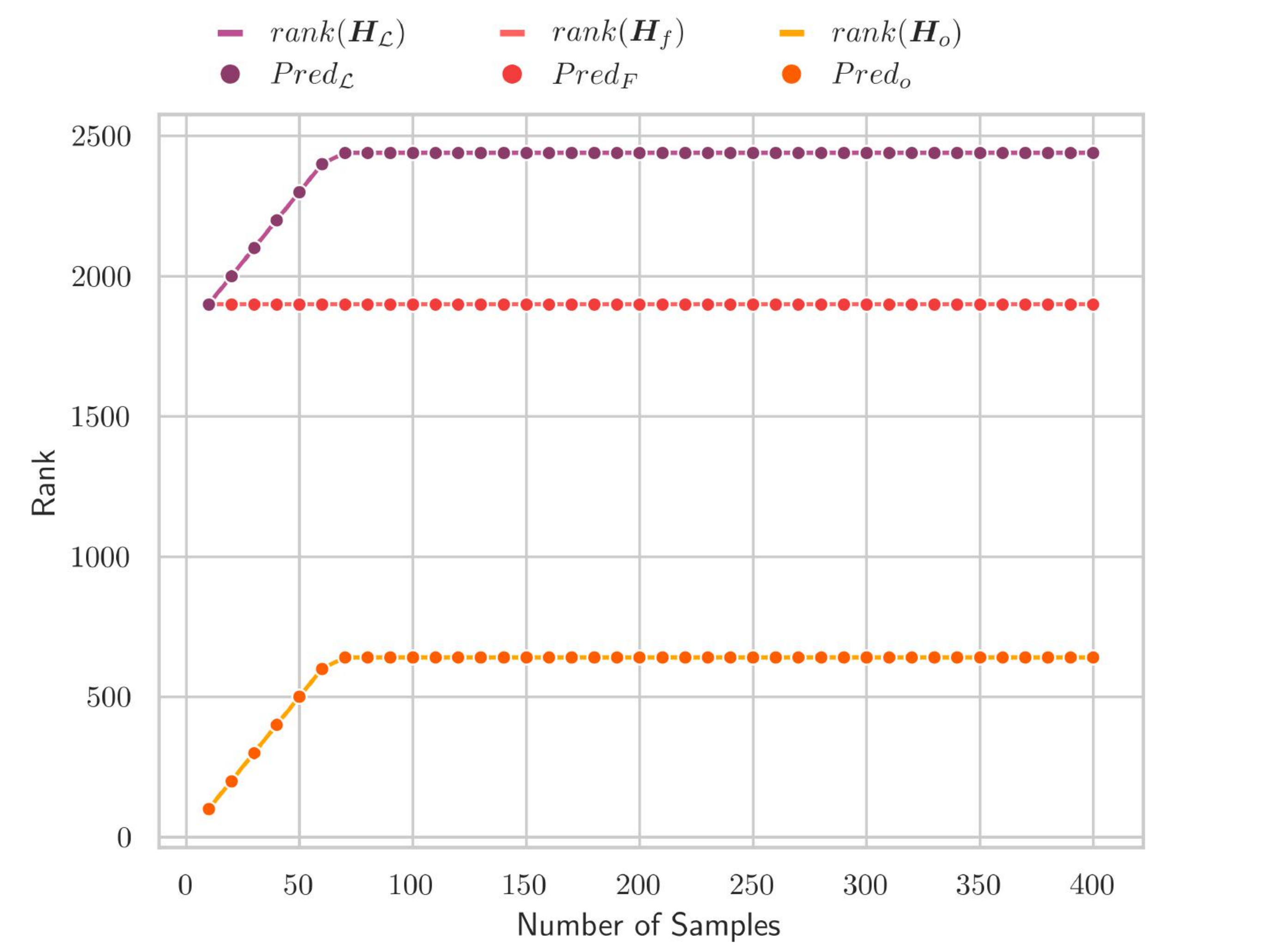}
    \caption{Rank vs sample size $n$}
    \label{fig:mnist-a1_app_uni}
    \end{subfigure}
    \begin{subfigure}[t]{0.3\textwidth}
    \includegraphics[width=\textwidth]{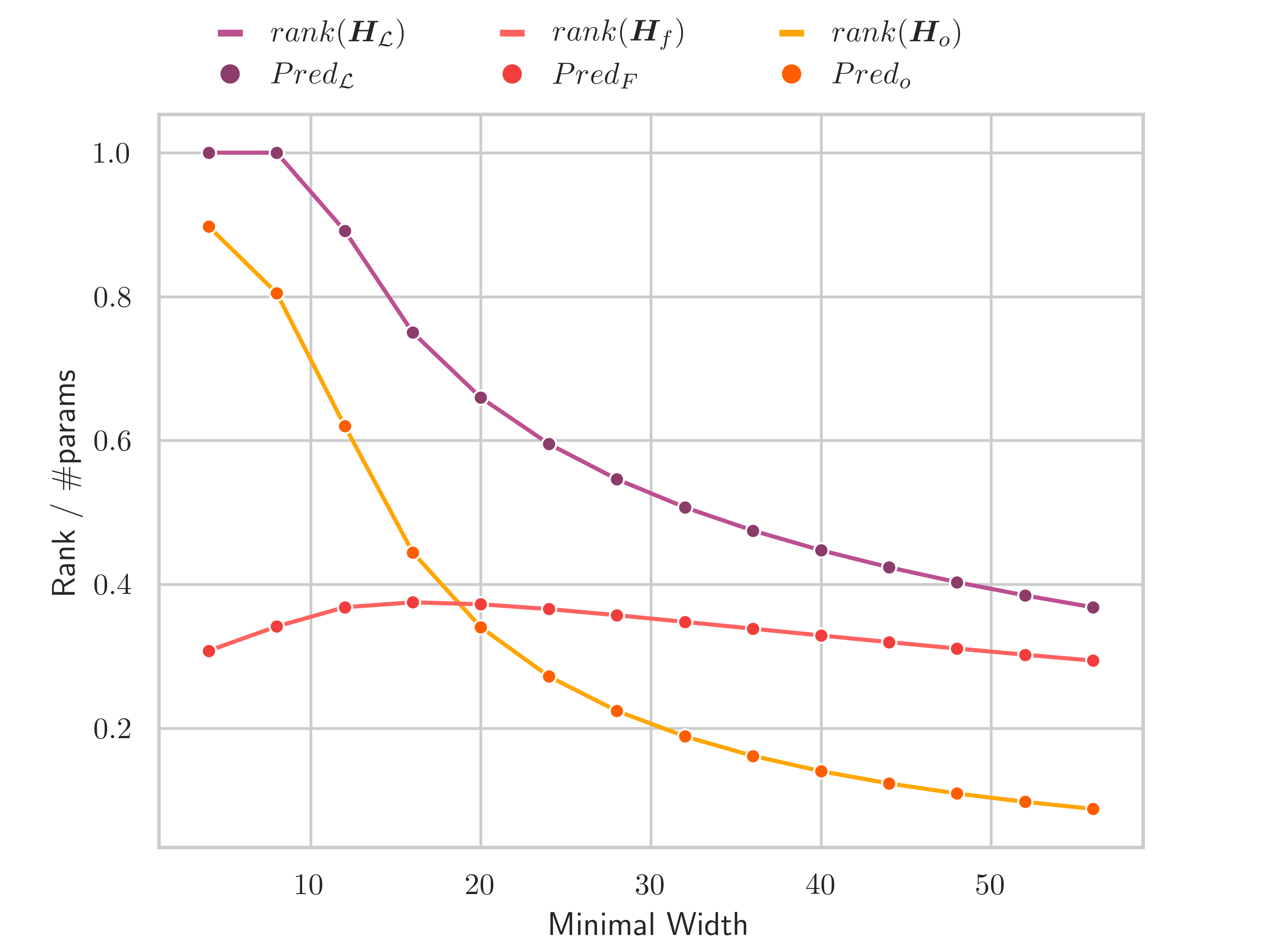}
    \caption{$\effparams$ vs minimal width $M_{*}$}
    \label{fig:mnist-a2_app_uni}
    \end{subfigure}
    \begin{subfigure}[t]{0.3\textwidth}
    \includegraphics[width=\textwidth]{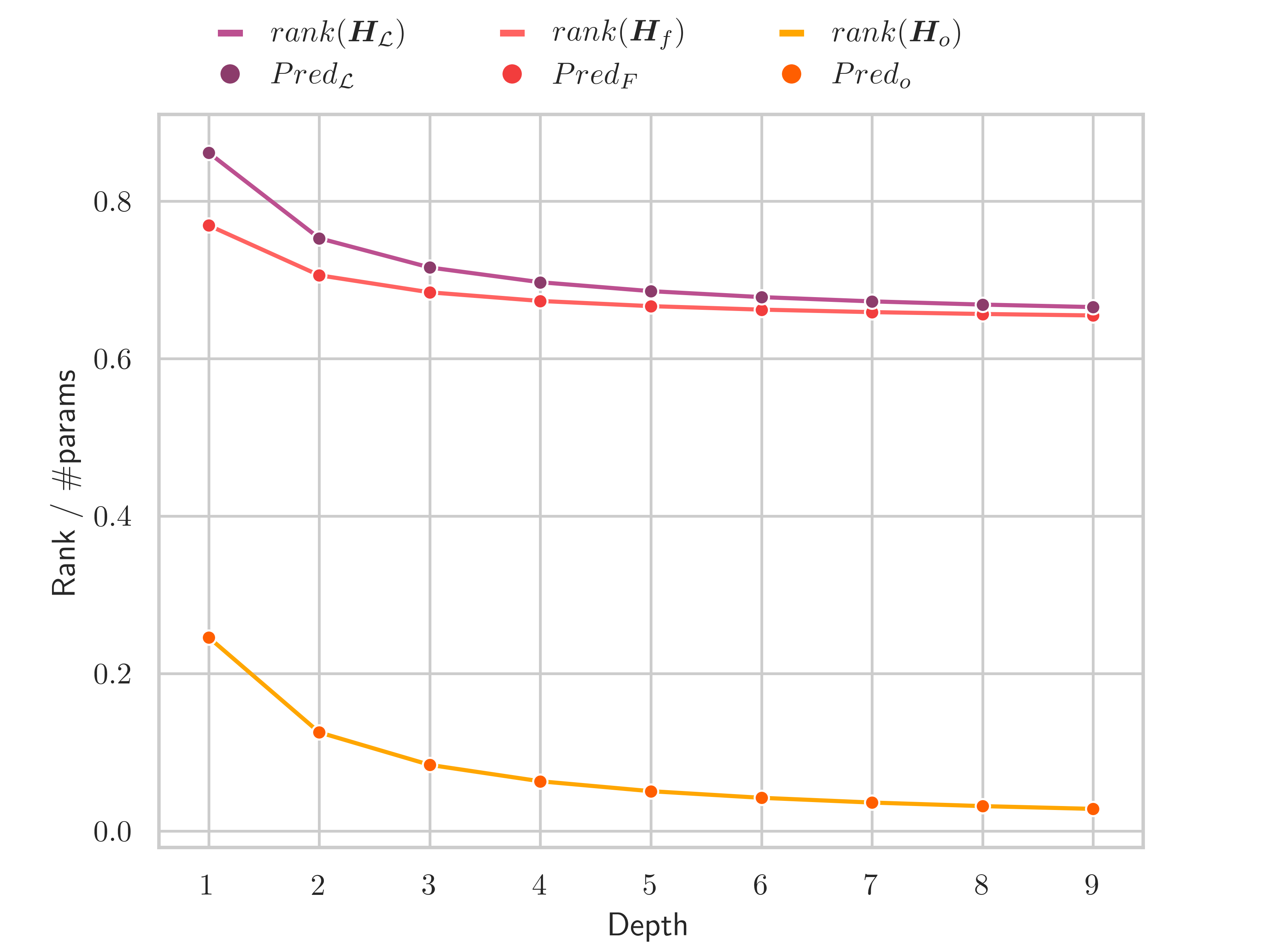}
    \caption{$\effparams$ vs depth $L$}
    \label{fig:mnist-a3_app_uni}
    \end{subfigure}
    \caption{\small{Behaviour of rank and rank/\#params on \textsc{CIFAR10} using MSE and \textbf{uniform initialization} with hidden layers: $50,20,20,20$ (Fig.~\ref{fig:mnist-a1_app_uni}), $M_{*}, M_{*}$ (Fig.~\ref{fig:mnist-a2_app_uni}) and $L$ layers of width $M=25$ (Fig.~\ref{fig:mnist-a3_app_uni}). The lines indicate the true value and circles denote our formula predictions. }}
     \label{uniform}
\end{figure}
\FloatBarrier
\begin{figure}[!h]
    \centering
    \begin{subfigure}[t]{0.3\textwidth}
    \includegraphics[width=\textwidth]{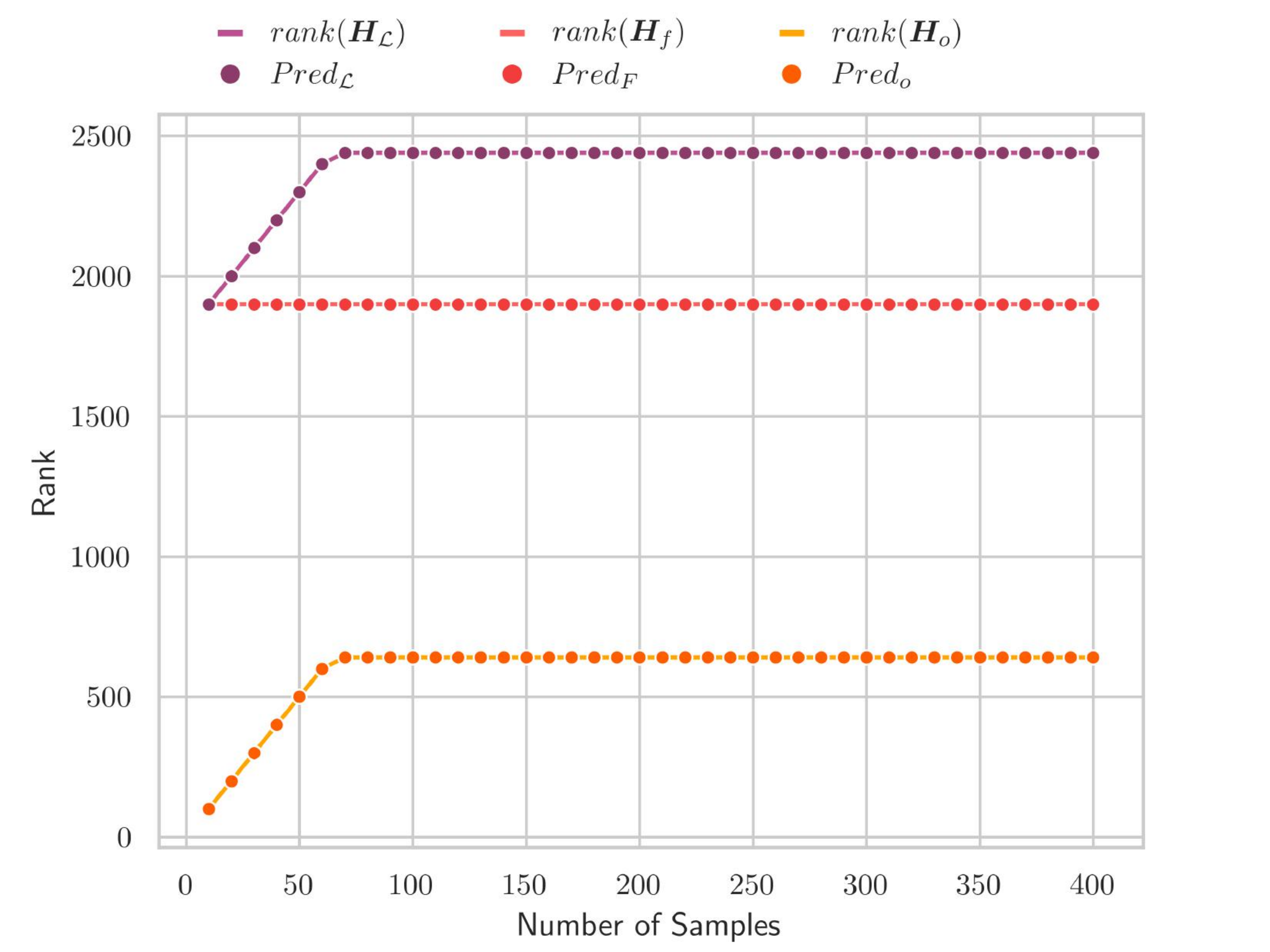}
    \caption{Rank vs sample size $n$}
    \label{fig:mnist-a1_app_orth}
    \end{subfigure}
    \begin{subfigure}[t]{0.3\textwidth}
    \includegraphics[width=\textwidth]{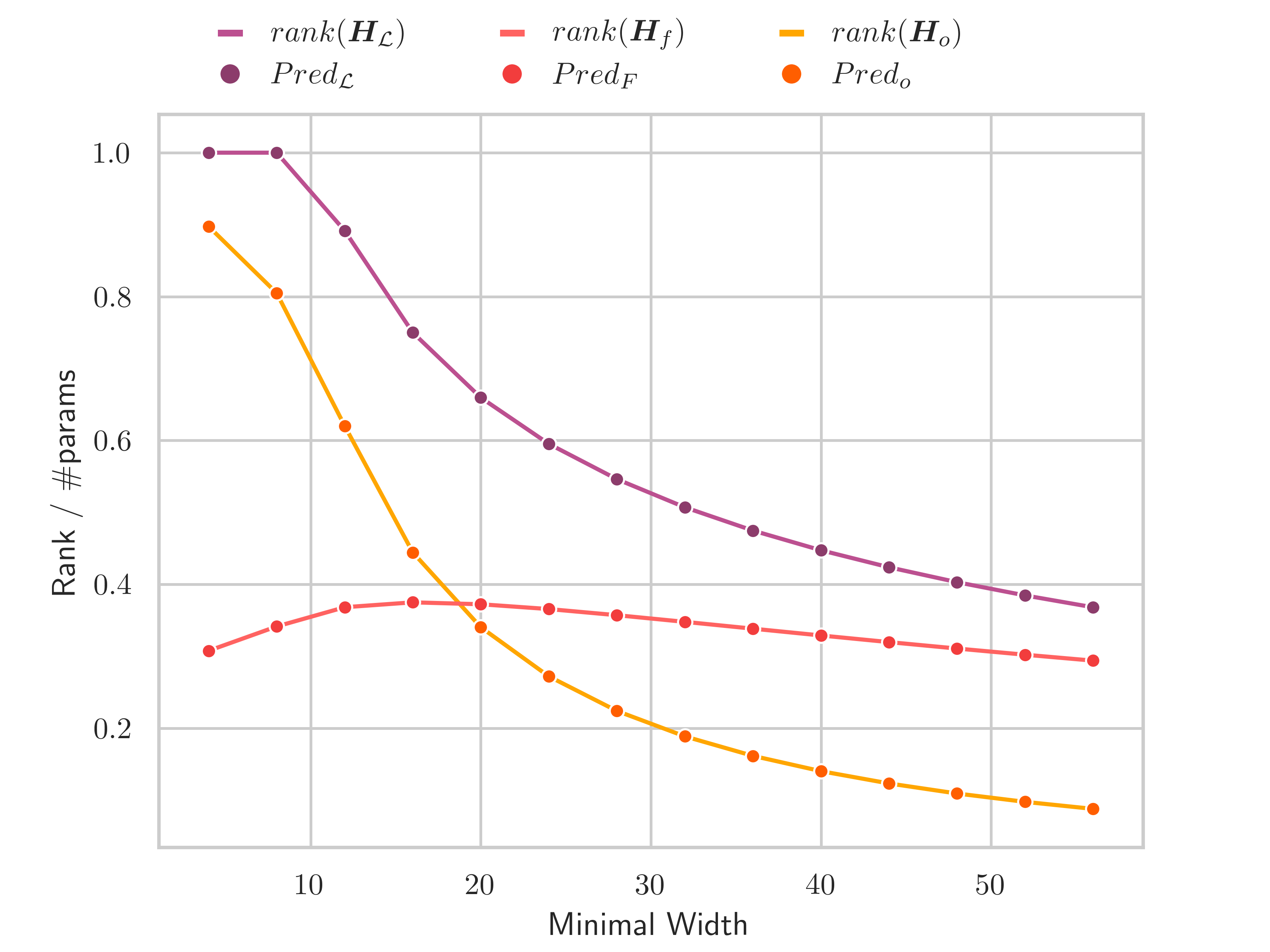}
    \caption{$\effparams$ vs minimal width $M_{*}$}
    \label{fig:mnist-a2_app_orth}
    \end{subfigure}
    \begin{subfigure}[t]{0.3\textwidth}
    \includegraphics[width=\textwidth]{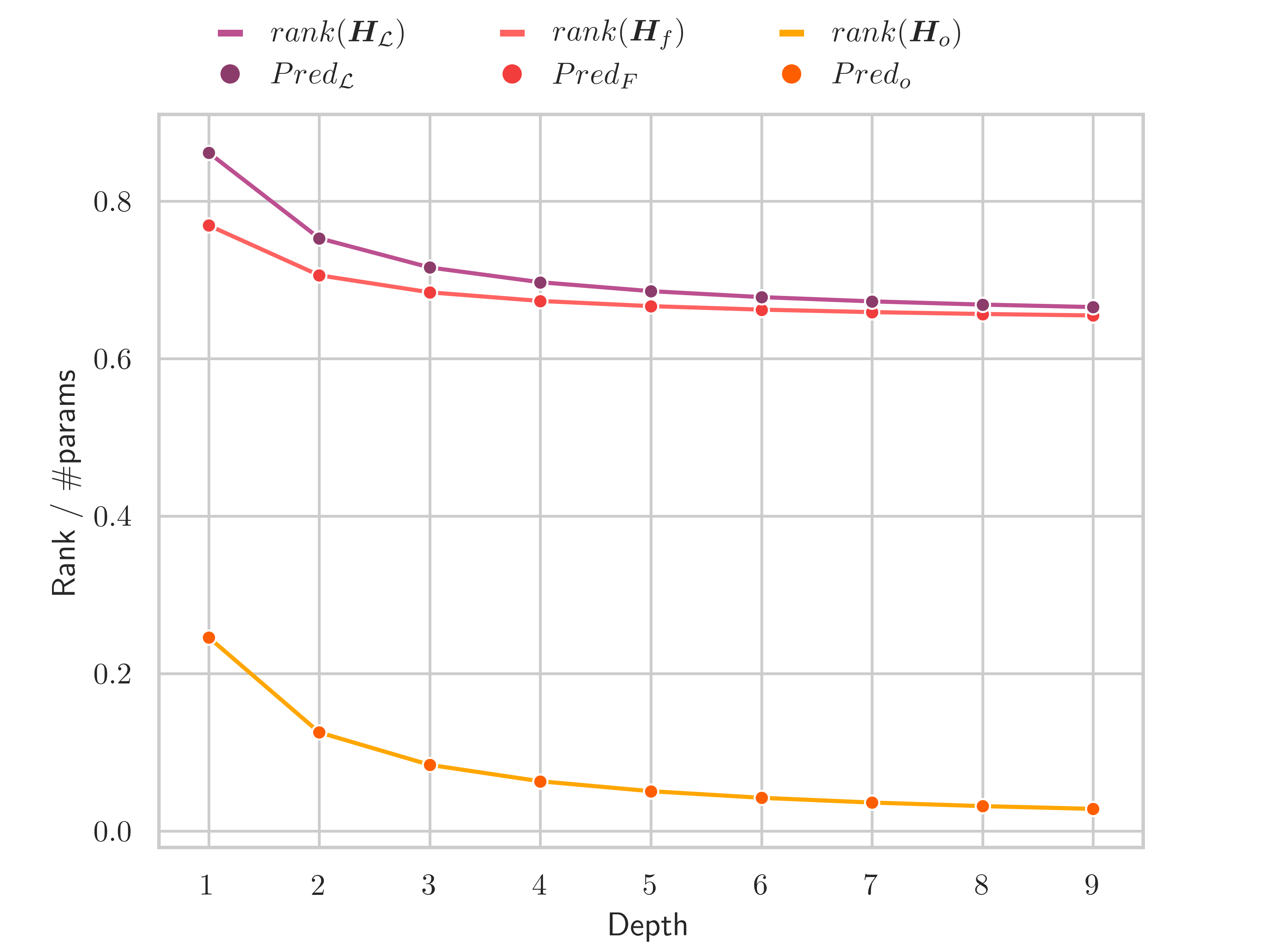}
    \caption{$\effparams$ vs depth $L$}
    \label{fig:mnist-a3_app_orth}
    \end{subfigure}
    \caption{\small{Behaviour of rank and rank/\#params on \textsc{CIFAR10} using MSE and \textbf{orthogonal initialization} with hidden layers: $50,20,20,20$ (Fig.~\ref{fig:mnist-a1_app_uni}), $M_{*}, M_{*}$ (Fig.~\ref{fig:mnist-a2_app_uni}) and $L$ layers of width $M=25$ (Fig.~\ref{fig:mnist-a3_app_uni}). The lines indicate the true value and circles denote our formula predictions. }}
     \label{orthogonal}
\end{figure}
\FloatBarrier

\subsubsection{Rank Formulas With Bias}\label{sec:app-bias-exp}
Here we verify the rank formulas derived for the case with bias in \ref{formulas_bias}. We use MSE loss and \textsc{CIFAR10} as the dataset. Again we see that the the rank predictions from our formulas exactly match the rank observed in practice.

\begin{figure}[!h]
    \centering
    \begin{subfigure}[t]{0.3\textwidth}
    \includegraphics[width=\textwidth]{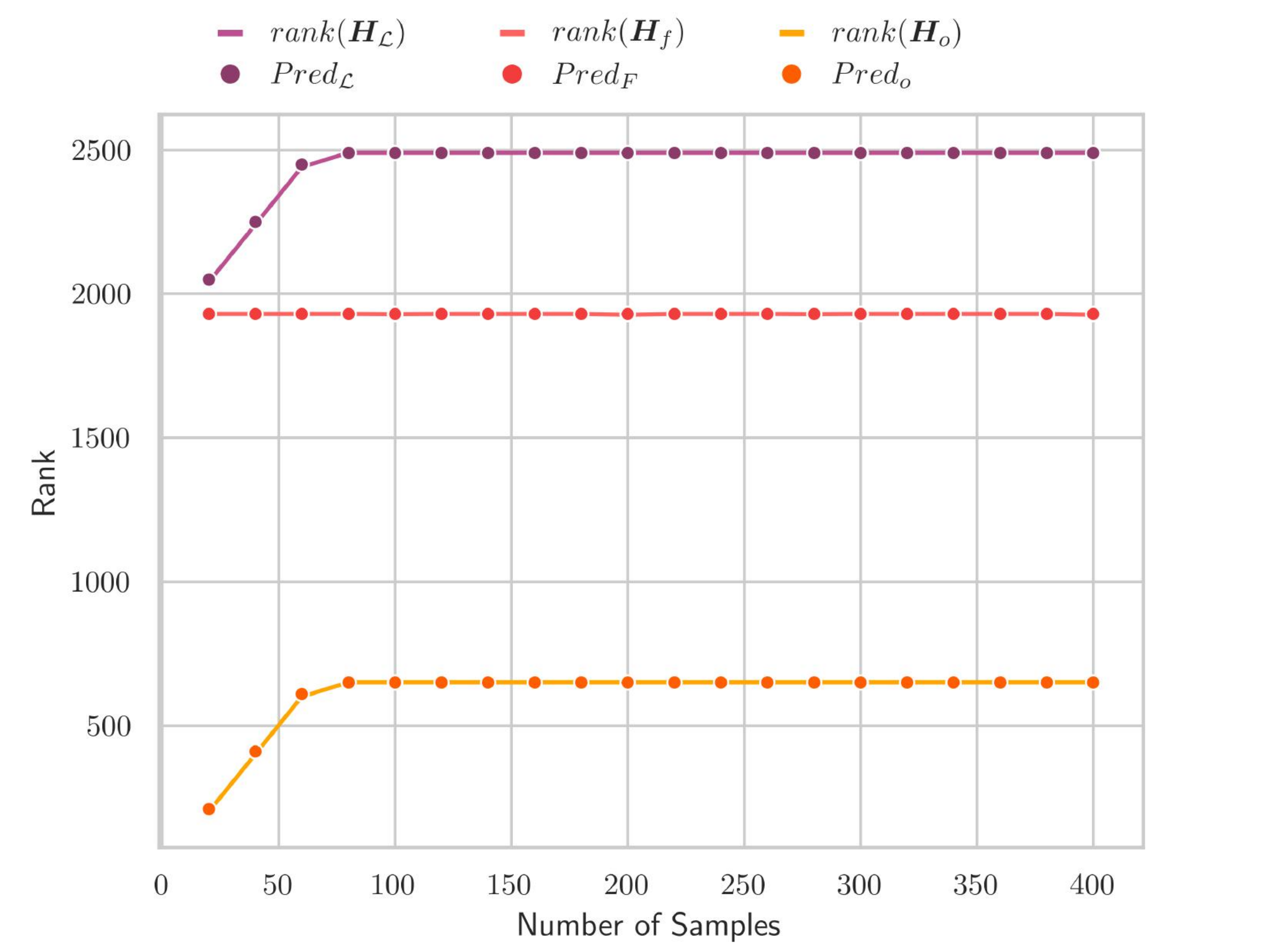}
    \caption{Rank vs sample size $n$}
    \label{fig:mnist-a1_app_bias}
    \end{subfigure}
    \begin{subfigure}[t]{0.3\textwidth}
    \includegraphics[width=\textwidth]{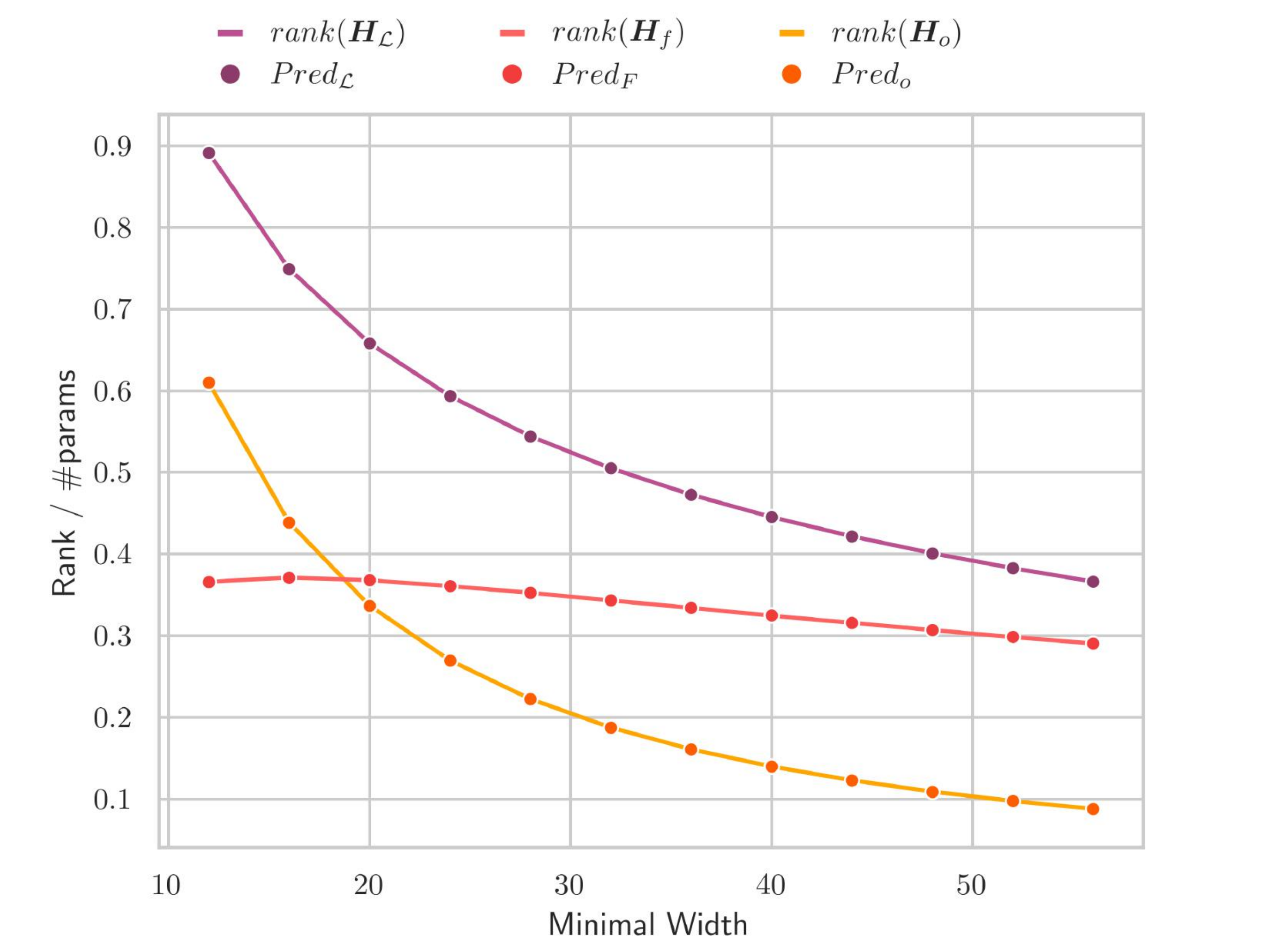}
    \caption{$\effparams$ vs minimal width $M_{*}$}
    \label{fig:mnist-a2_app_bias}
    \end{subfigure}
    \begin{subfigure}[t]{0.3\textwidth}
    \includegraphics[width=\textwidth]{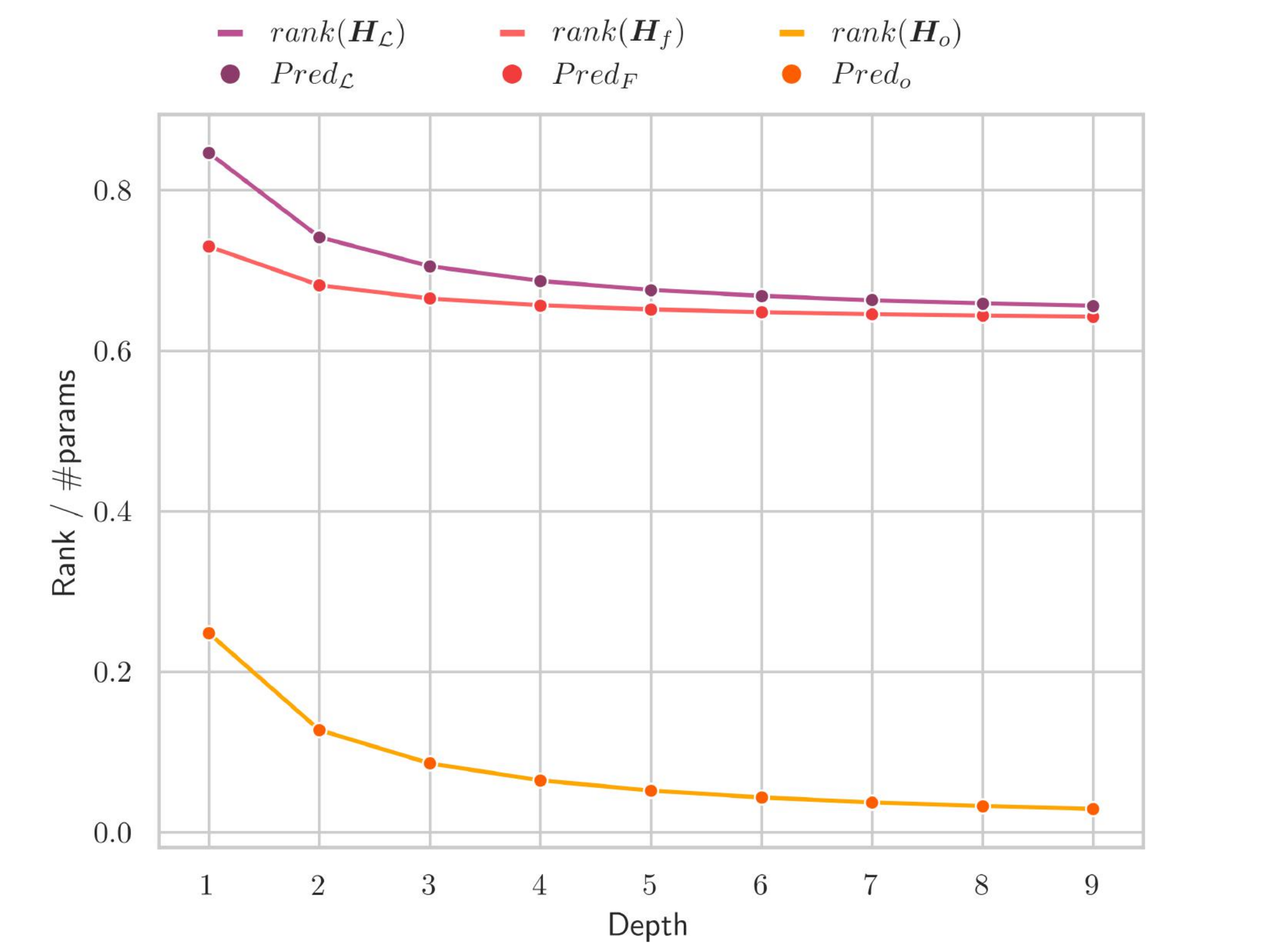}
    \caption{$\effparams$ vs depth $L$}
    \label{fig:mnist-a3_app_bias}
    \end{subfigure}
    \caption{\small{Behaviour of rank and rank/\#params on \textsc{CIFAR10} using MSE and \textbf{bias} with hidden layers: $50,20,20,20$ (Fig.~\ref{fig:mnist-a1_app_bias}), $M_{*}, M_{*}$ (Fig.~\ref{fig:mnist-a2_app_bias}) and $L$ layers of width $M=25$ (Fig.~\ref{fig:mnist-a3_app_bias}). The lines indicate the true value and circles denote our formula predictions. }}
     \label{bias}
\end{figure}
\FloatBarrier

\FloatBarrier

\subsection{Reconstruction Error Plots for More Non-Linearities and Losses}
\label{recon_error_app}
To further illustrate how our predictions extend to different non-linearities, we repeat the reconstruction error experiment for different types of non-linearities and loss functions. In particular we study the activation functions
$$\sigma(x) = \operatorname{ReLU}(x) \hspace{5mm} \sigma(x) = \tanh(x) \hspace{5mm} \sigma(x) = \operatorname{ELU}(x) = \begin{cases}x \hspace{8mm} x > 0 \\
e^{x}-1 \hspace{3mm} x\leq 0\end{cases}$$
As before, we group the experiment by the loss function employed and vary the non-linearity used in each architecture. We test on this down-scaled \textsc{MNIST} with input dimensionality of $d=64$ for the smaller architectures and $d=49$ for the bigger ones. The number of samples $N=200$ across all settings. 
\subsubsection{Mean Squared Error}
Here we expand on the Figure \ref{fig:relu_reconstruction}, using the same setting as presented in the main text but we consider more non-linearities. We display the results for ReLU in Figure \ref{fig:relu_recon}, for ELU in Figure \ref{fig:elu_recon} and for tanh in Figure \ref{fig:tanh_reconstruction}. We also consider slightly bigger architectures in Figures \ref{fig:relu_recon_bigger}, \ref{fig:elu_recon_bigger} and \ref{fig:tanh_reconstruction_bigger}, using the same ordering for the non-linearities as before. Again we observe that our rank prediction offers an excellent cut-off, allowing to preserve almost the entire structure of the Hessian, even for the bigger architectures. This is again strong evidence that our prediction captures the relevant eigenvalues but becomes distorted by smaller, irrelevant ones, inflating the exact rank.

\begin{figure}[!htb]
    \centering
    \includegraphics[width=0.3\textwidth]{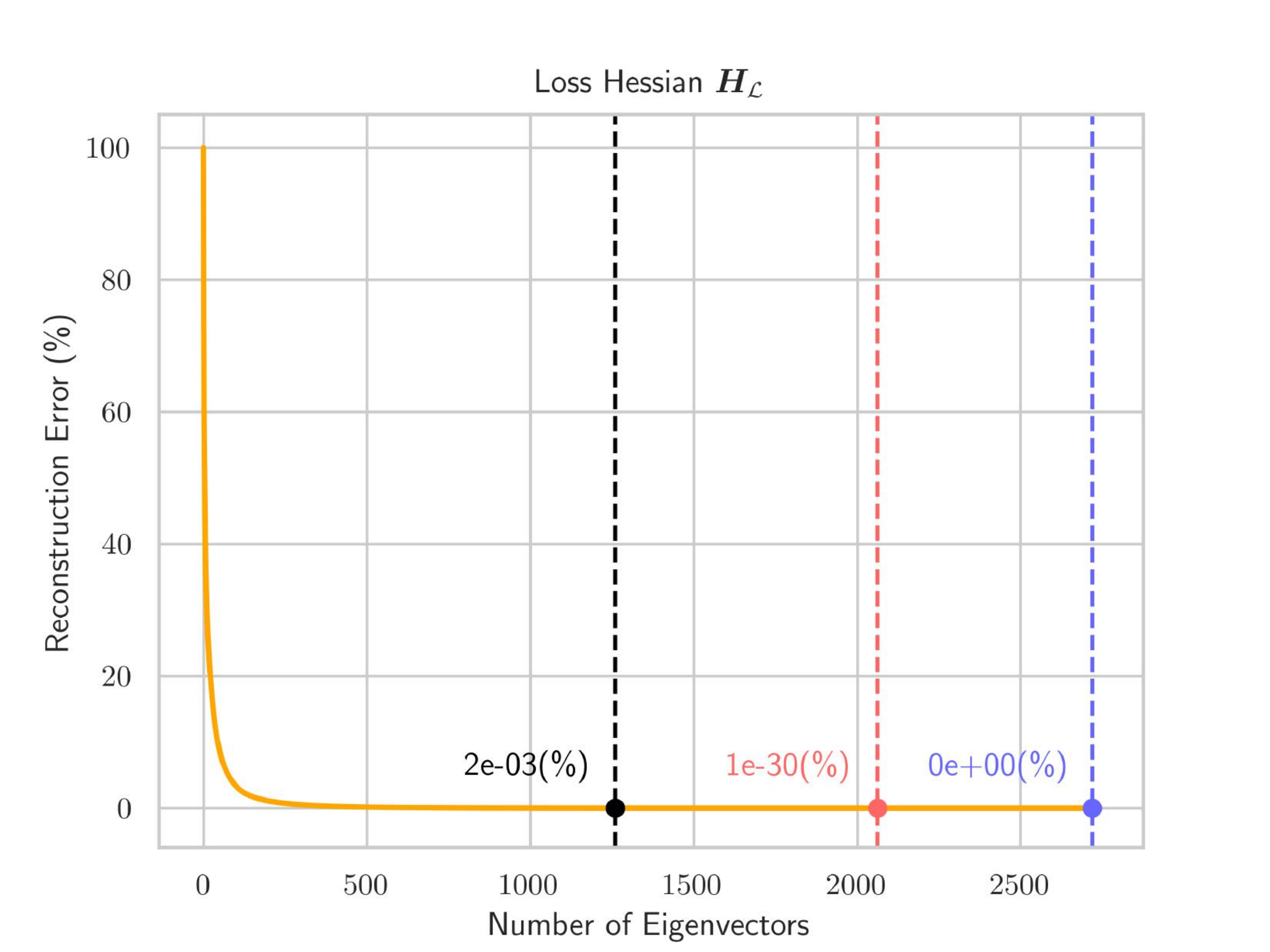}%
    \includegraphics[width=0.3\textwidth]{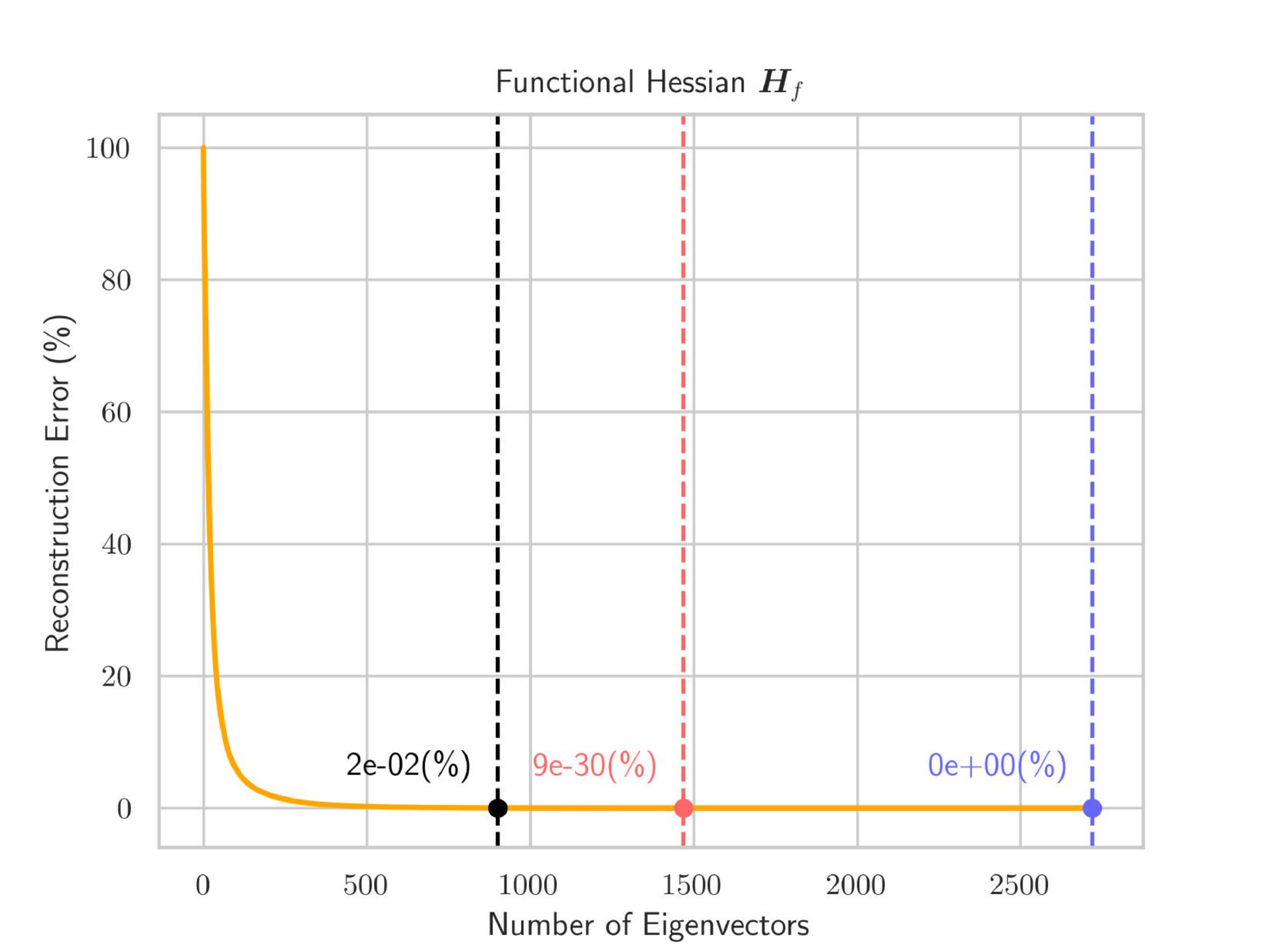}%
    \includegraphics[width=0.3\textwidth]{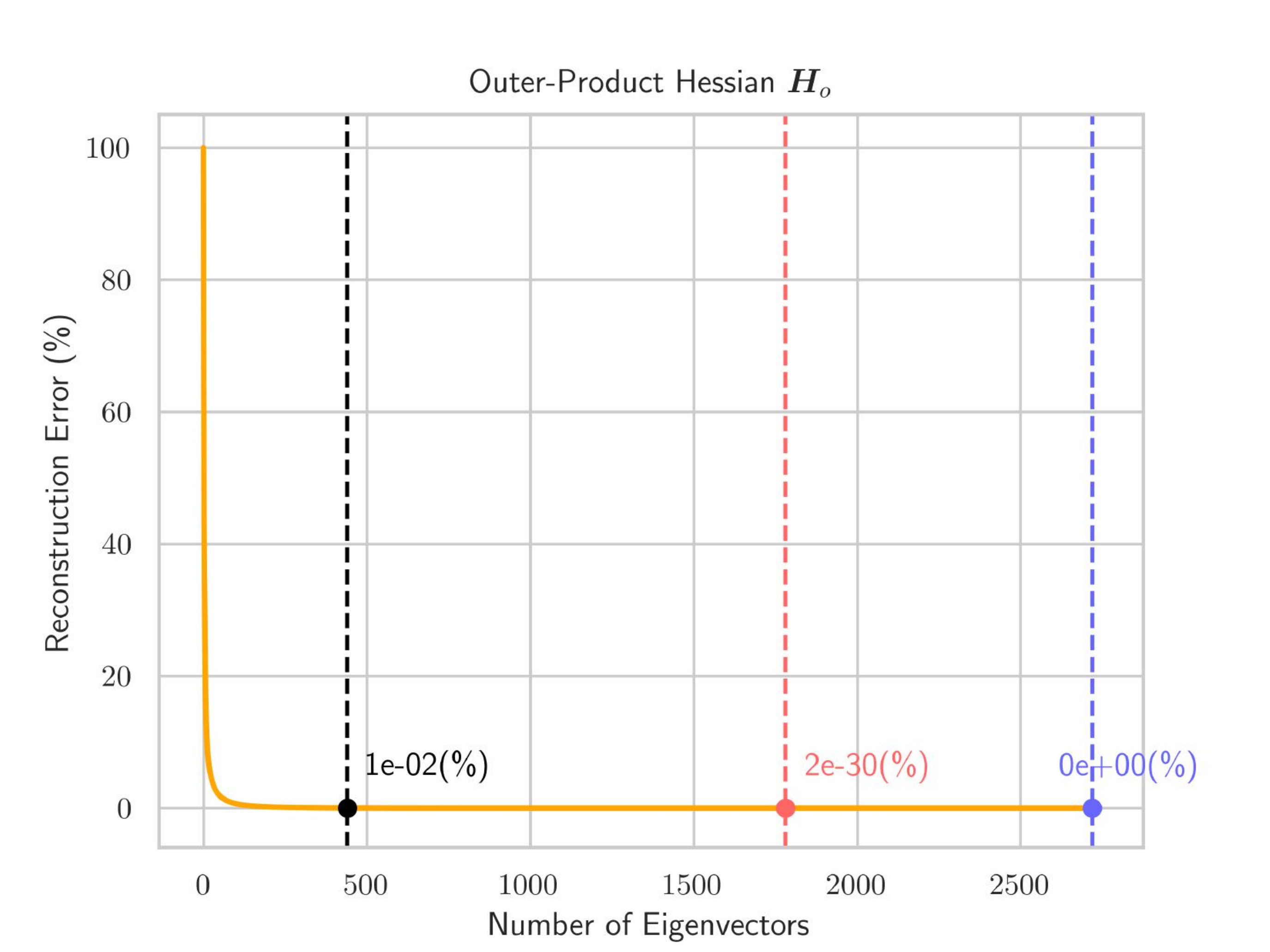}
    \caption{Hessian reconstruction error for \textbf{ReLU} under \textbf{MSE} as the rank of the approximation is increased. The x-axis represents the number of top eigenvectors that form the low-rank approximation. The $y$-axis displays the \textit{reconstruction error in percentage} (100 $\%$ for zero eigenvectors used). The dashed vertical lines indicate the cut-off at various values of the rank: \textbf{first line} at the prediction based on the linear model, \textcolor{red}{\textbf{second line}} at the empirical measurement of rank, and \textcolor{blue}{\textbf{third line}} based on upper bounds from \citep{jacot2019asymptotic}, which become too coarse to be of any use (actually even greater than the \# of parameters but not marked there for visualization purposes). The hidden layer sizes are $30, 20$.}
    \label{fig:relu_recon}
\end{figure}
\FloatBarrier
\begin{figure}[!htb]
    \centering
    \includegraphics[width=0.3\textwidth]{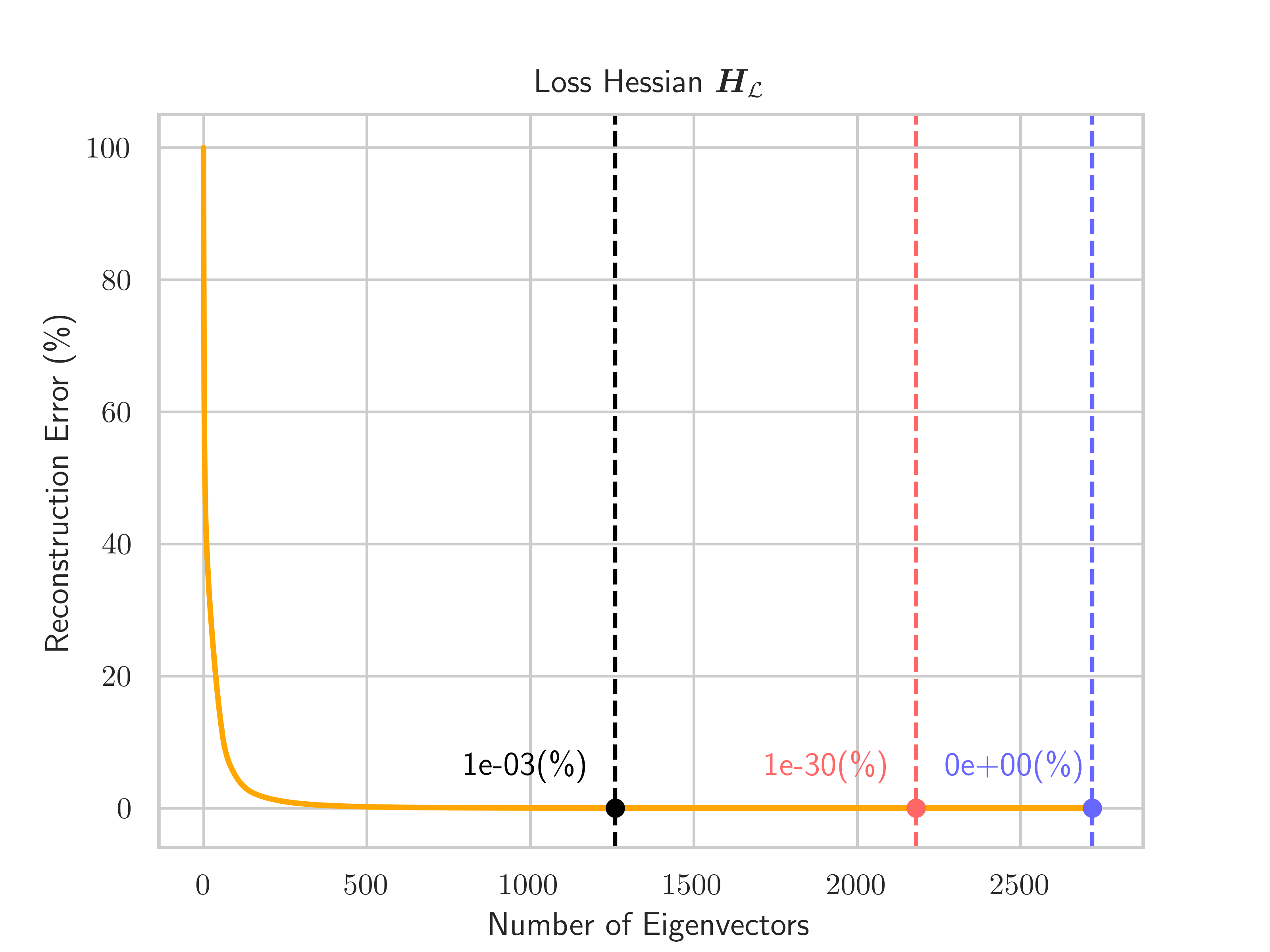}%
    \includegraphics[width=0.3\textwidth]{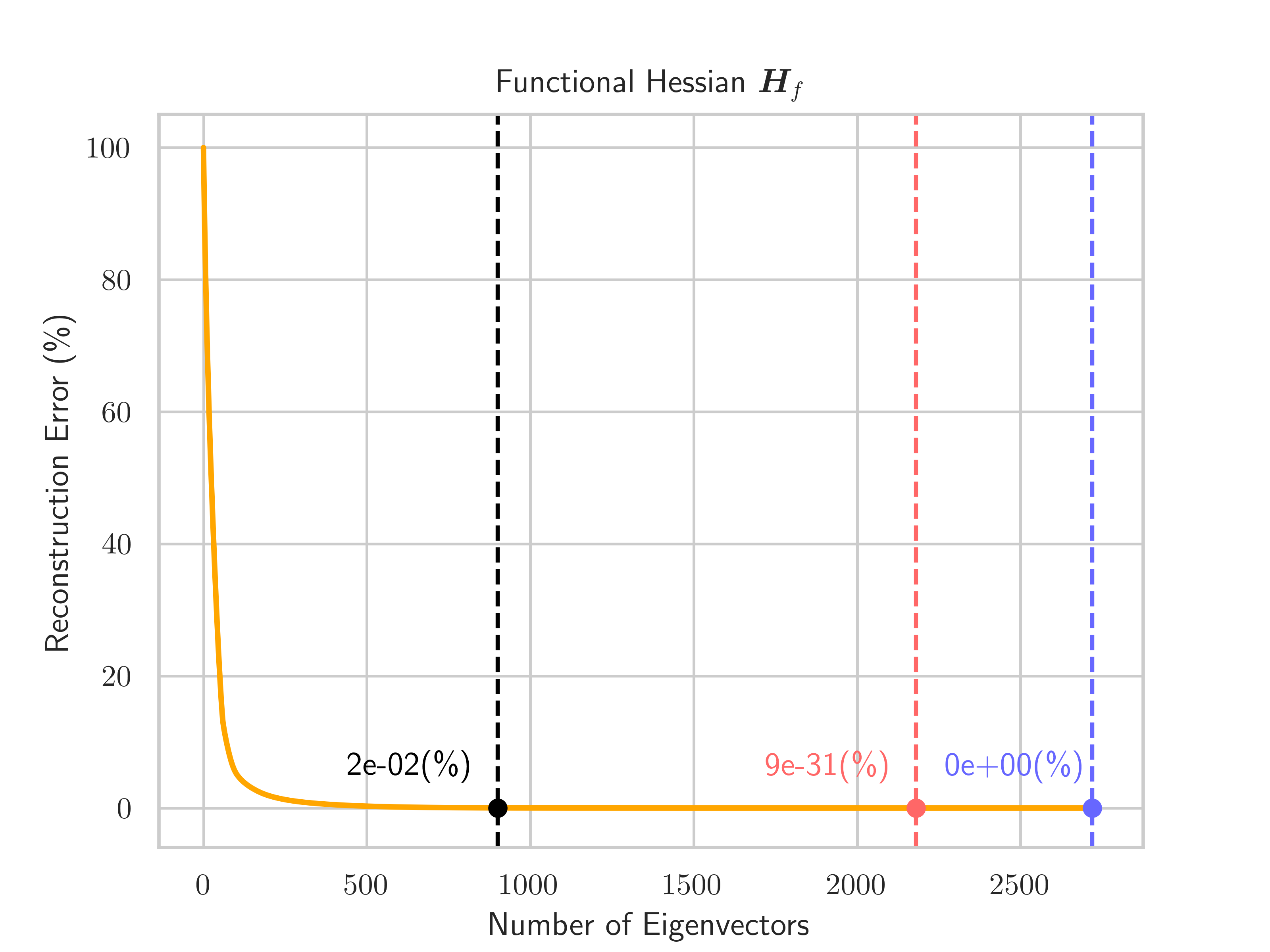}%
    \includegraphics[width=0.3\textwidth]{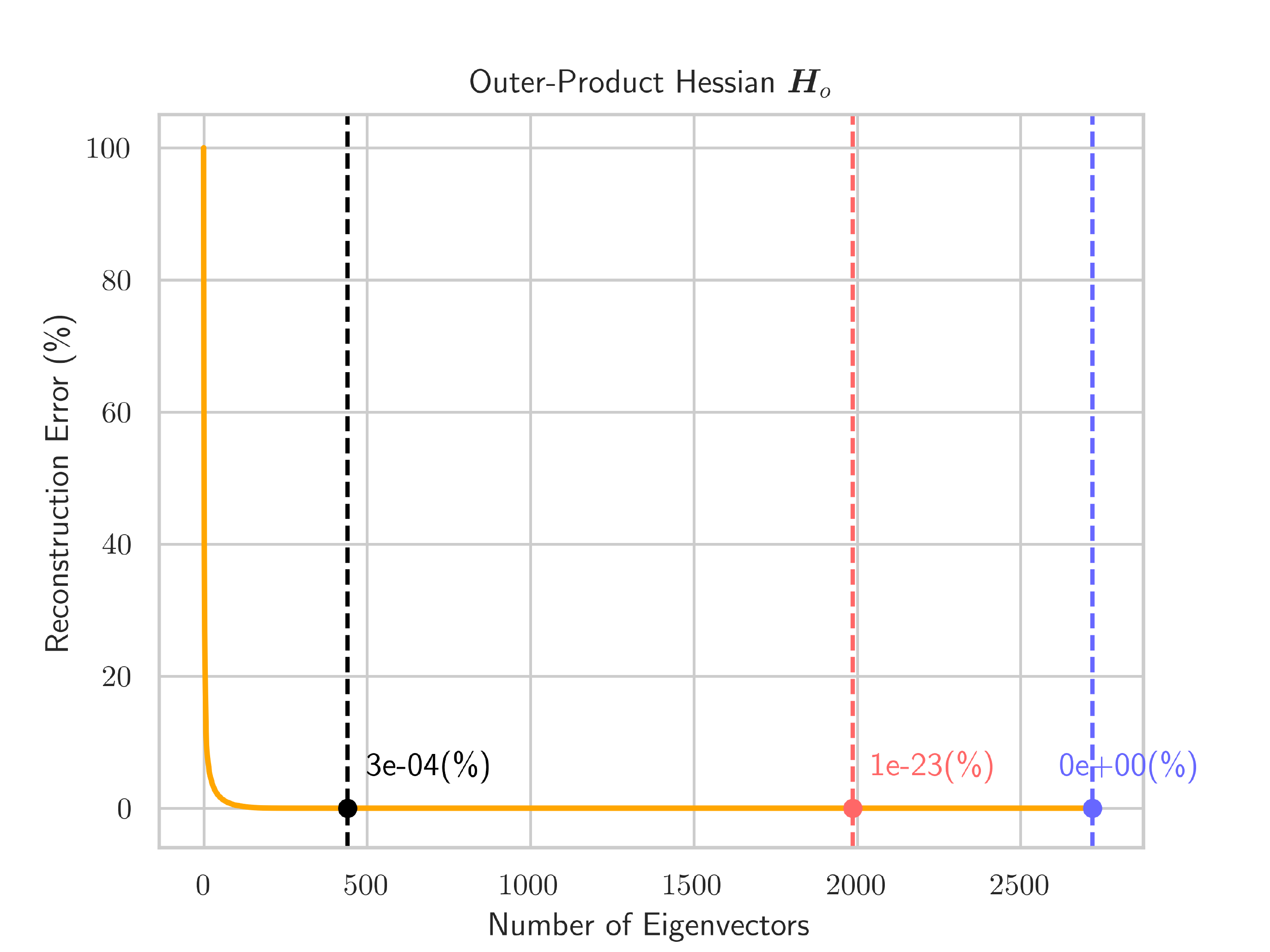}
    \caption{Hessian reconstruction error for \textbf{ELU} under \textbf{MSE} as the rank of the approximation is increased. The x-axis represents the number of top eigenvectors that form the low-rank approximation. The $y$-axis displays the \textit{reconstruction error in percentage} (100 $\%$ for zero eigenvectors used). The dashed vertical lines indicate the cut-off at various values of the rank: \textbf{first line} at the prediction based on the linear model, \textcolor{red}{\textbf{second line}} at the empirical measurement of rank, and \textcolor{blue}{\textbf{third line}} based on upper bounds from \citep{jacot2019asymptotic}, which become too coarse to be of any use (actually even greater than the \# of parameters but not marked there for visualization purposes). The hidden layer sizes are $30, 20$.}
    \label{fig:elu_recon}
\end{figure}
\FloatBarrier

\begin{figure}[!htb]
    \centering
    \includegraphics[width=0.33\textwidth]{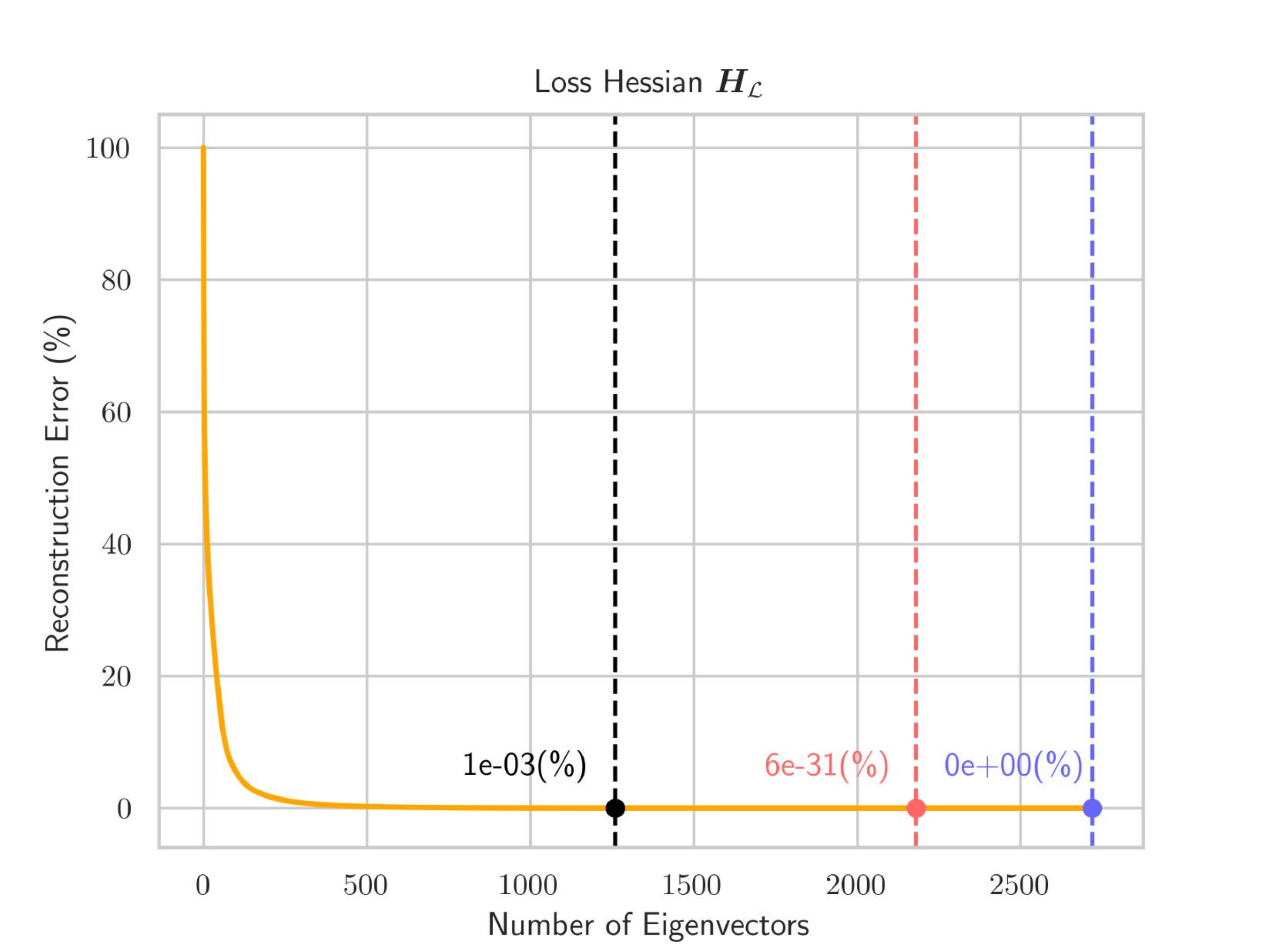}%
    \includegraphics[width=0.33\textwidth]{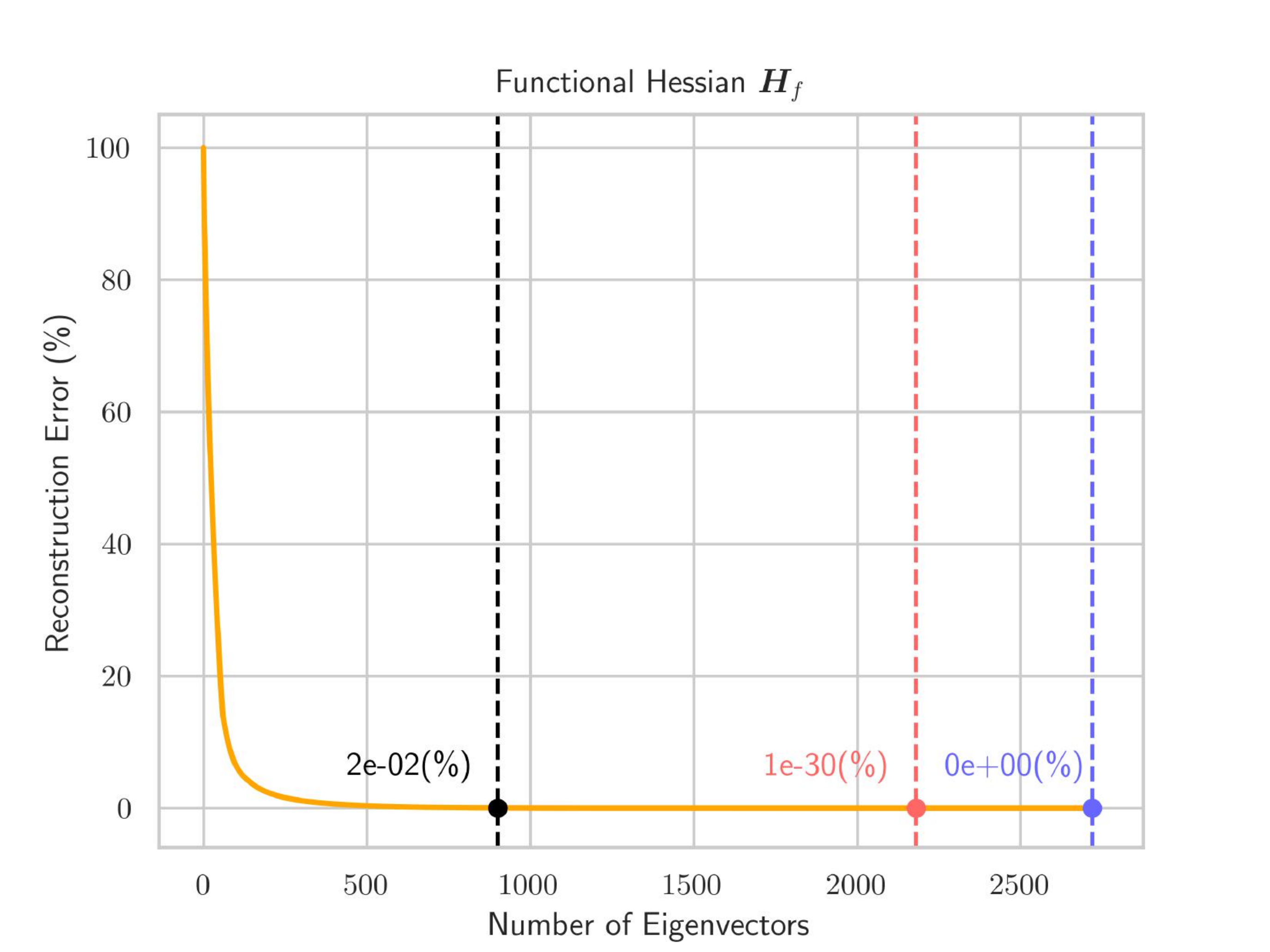}
    \includegraphics[width=0.33\textwidth]{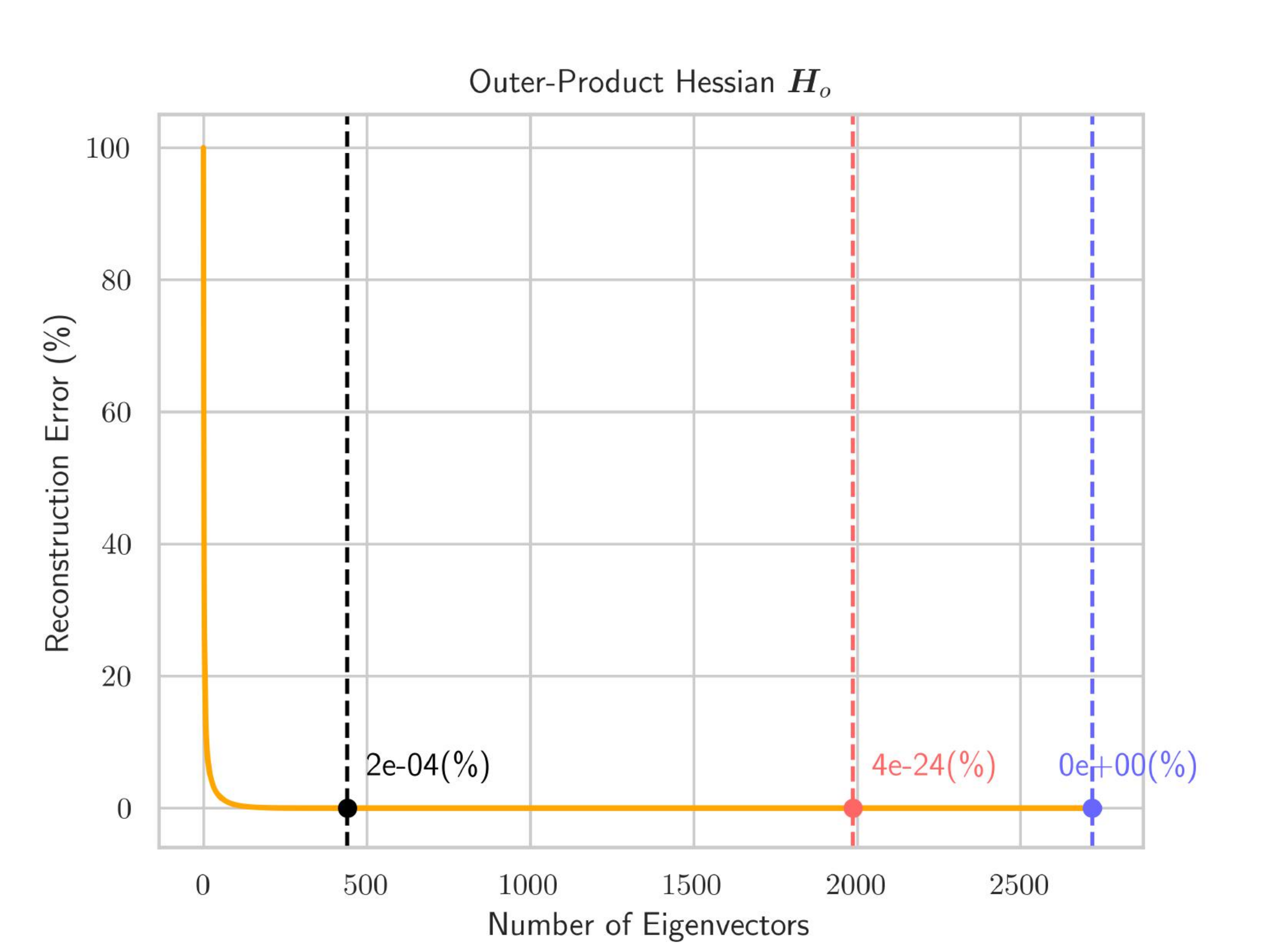}
    \caption{Hessian reconstruction error for \textbf{Tanh} under \textbf{MSE} as the rank of the approximation is increased. The x-axis represents the number of top eigenvectors that form the low-rank approximation. The $y$-axis displays the \textit{reconstruction error in percentage} (100 $\%$ for zero eigenvectors used). The dashed vertical lines indicate the cut-off at various values of the rank: \textbf{first line} at the prediction based on the linear model, \textcolor{red}{\textbf{second line}} at the empirical measurement of rank, and \textcolor{blue}{\textbf{third line}} based on upper bounds from \citep{jacot2019asymptotic}, which become too coarse to be of any use (actually even greater than the \# of parameters but not marked there for visualization purposes). The hidden layer sizes are $30, 20$.}
    \label{fig:tanh_reconstruction}
\end{figure}
\FloatBarrier
\begin{figure}[!htb]
    \centering
    \includegraphics[width=0.3\textwidth]{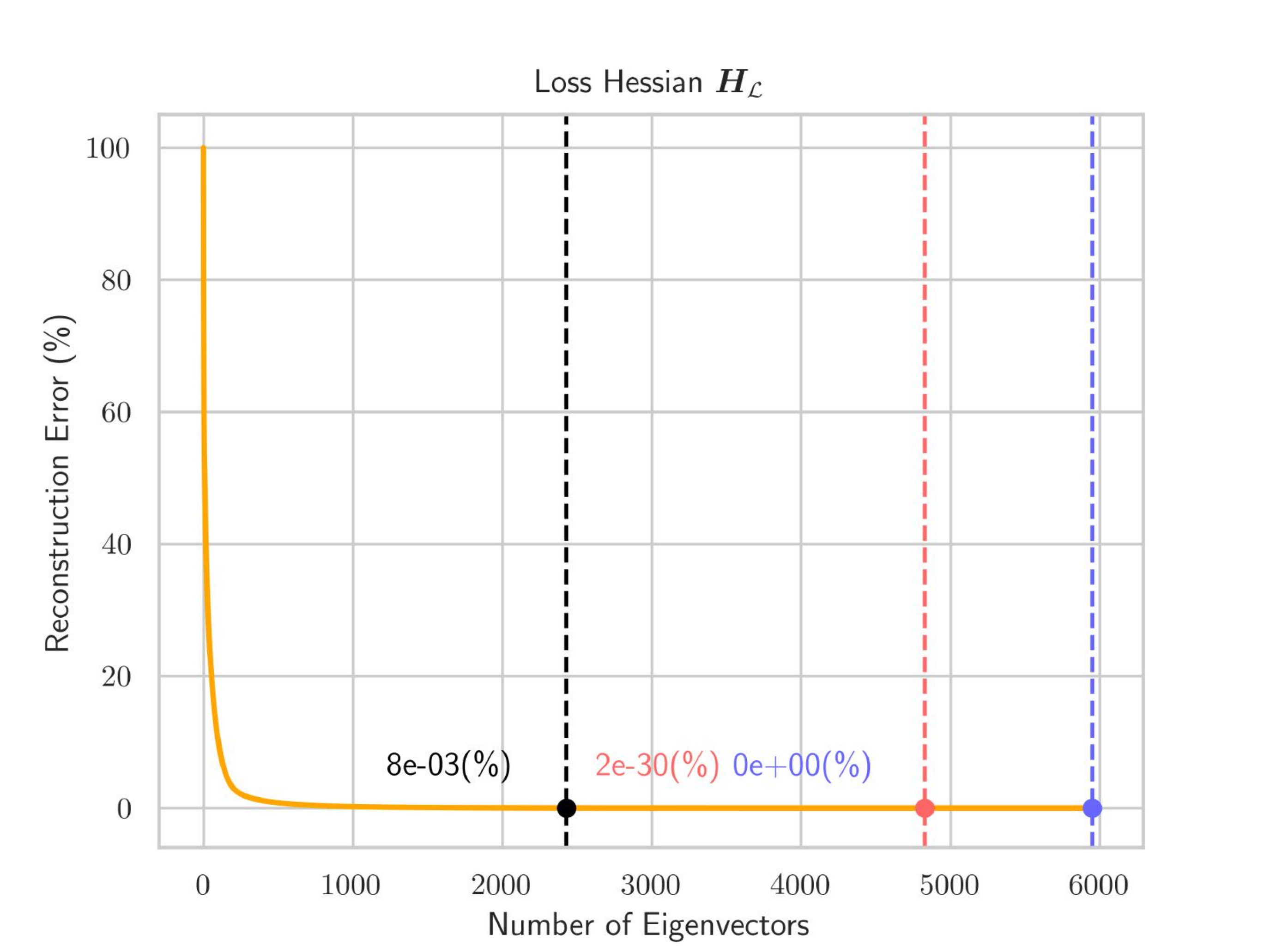}%
    \includegraphics[width=0.3\textwidth]{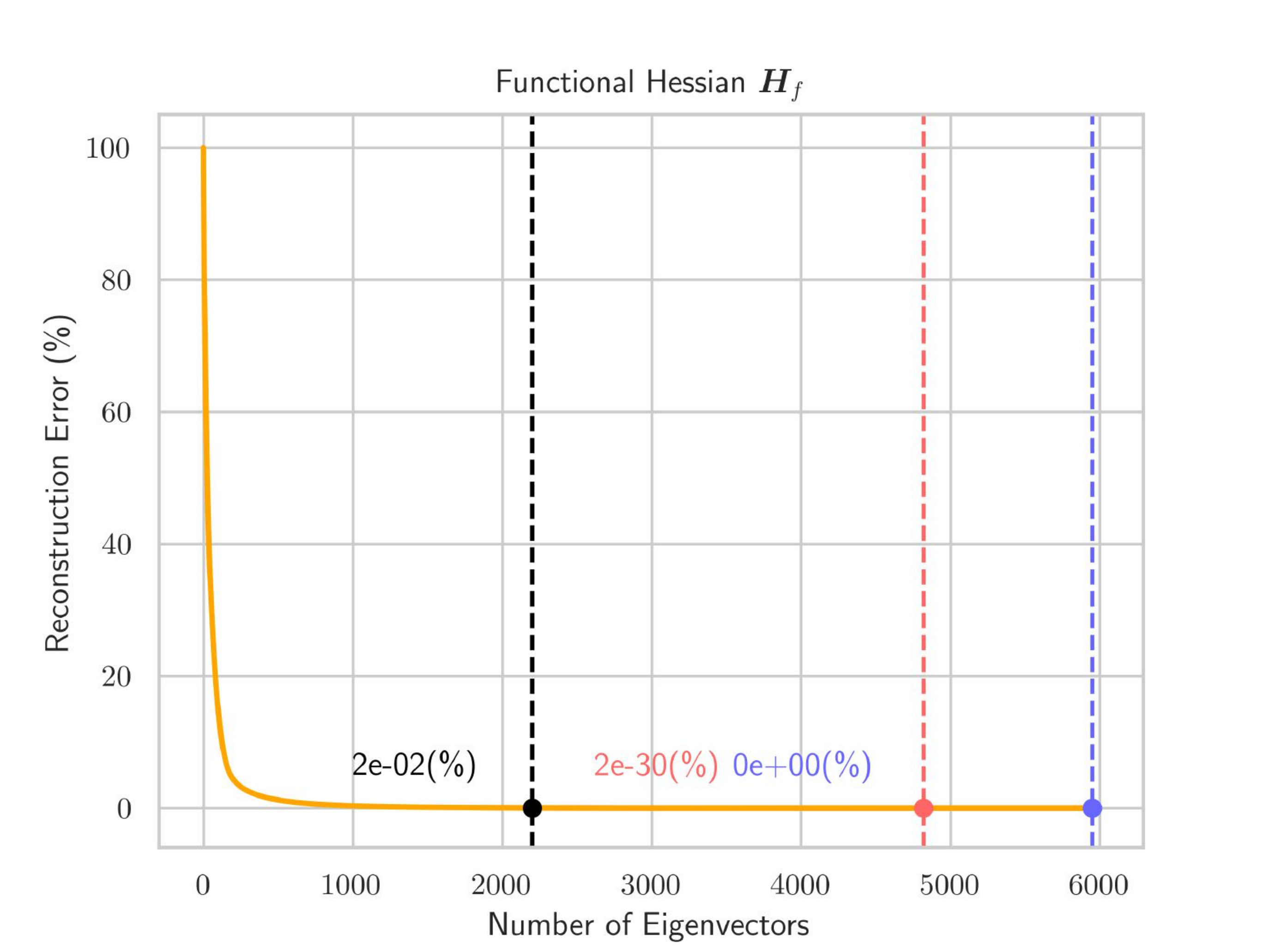}%
    \includegraphics[width=0.3\textwidth]{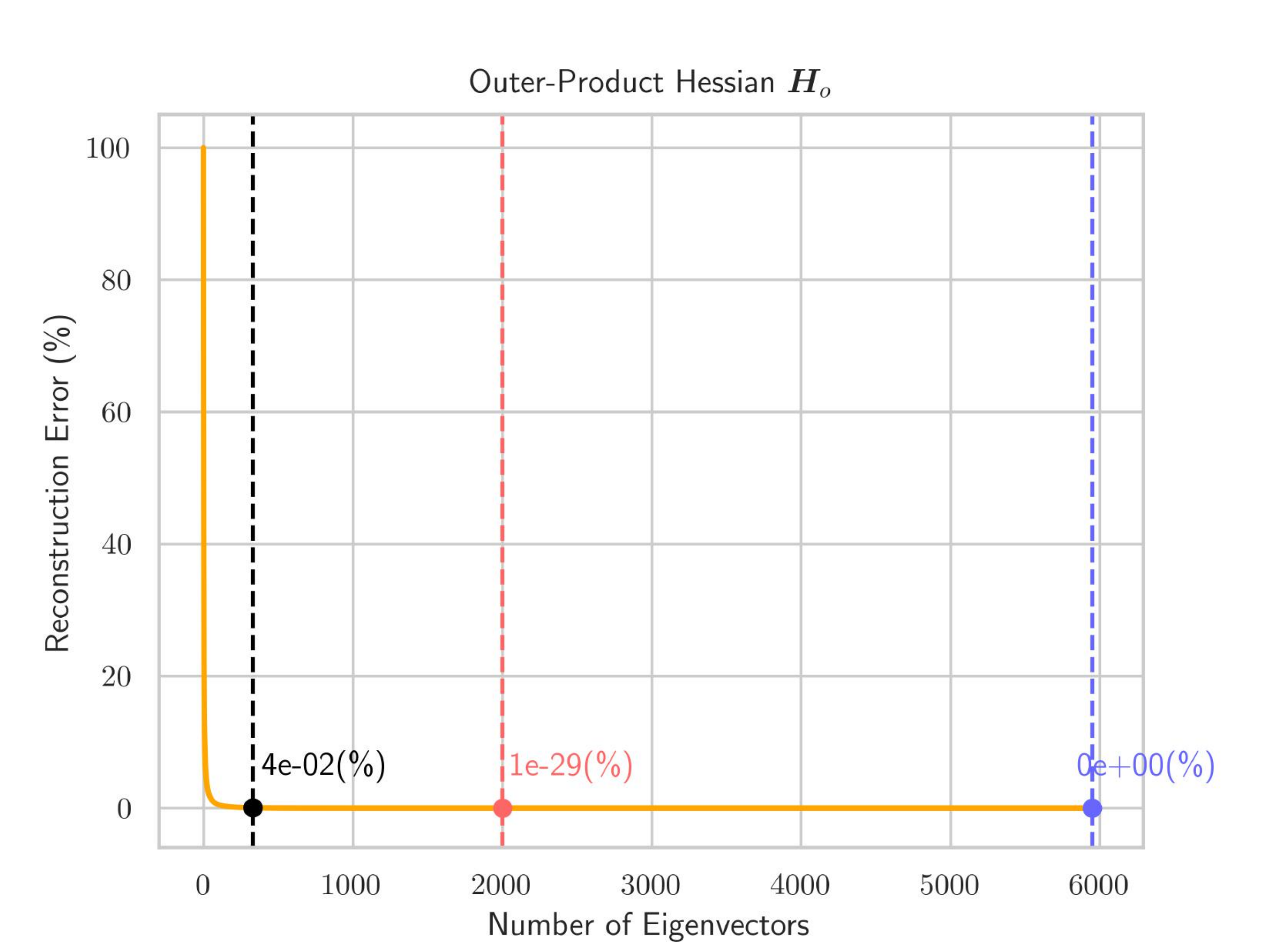}
    \caption{Hessian reconstruction error for \textbf{ReLU} under \textbf{MSE} as the rank of the approximation is increased. The x-axis represents the number of top eigenvectors that form the low-rank approximation. The $y$-axis displays the \textit{reconstruction error in percentage} (100 $\%$ for zero eigenvectors used). The dashed vertical lines indicate the cut-off at various values of the rank: \textbf{first line} at the prediction based on the linear model, \textcolor{red}{\textbf{second line}} at the empirical measurement of rank, and \textcolor{blue}{\textbf{third line}} based on upper bounds from \citep{jacot2019asymptotic}, which become too coarse to be of any use (actually even greater than the \# of parameters but not marked there for visualization purposes). The hidden layer sizes are $50, 40, 30$.}
    \label{fig:relu_recon_bigger}
\end{figure}
\begin{figure}[!htb]
    \centering
    \includegraphics[width=0.3\textwidth]{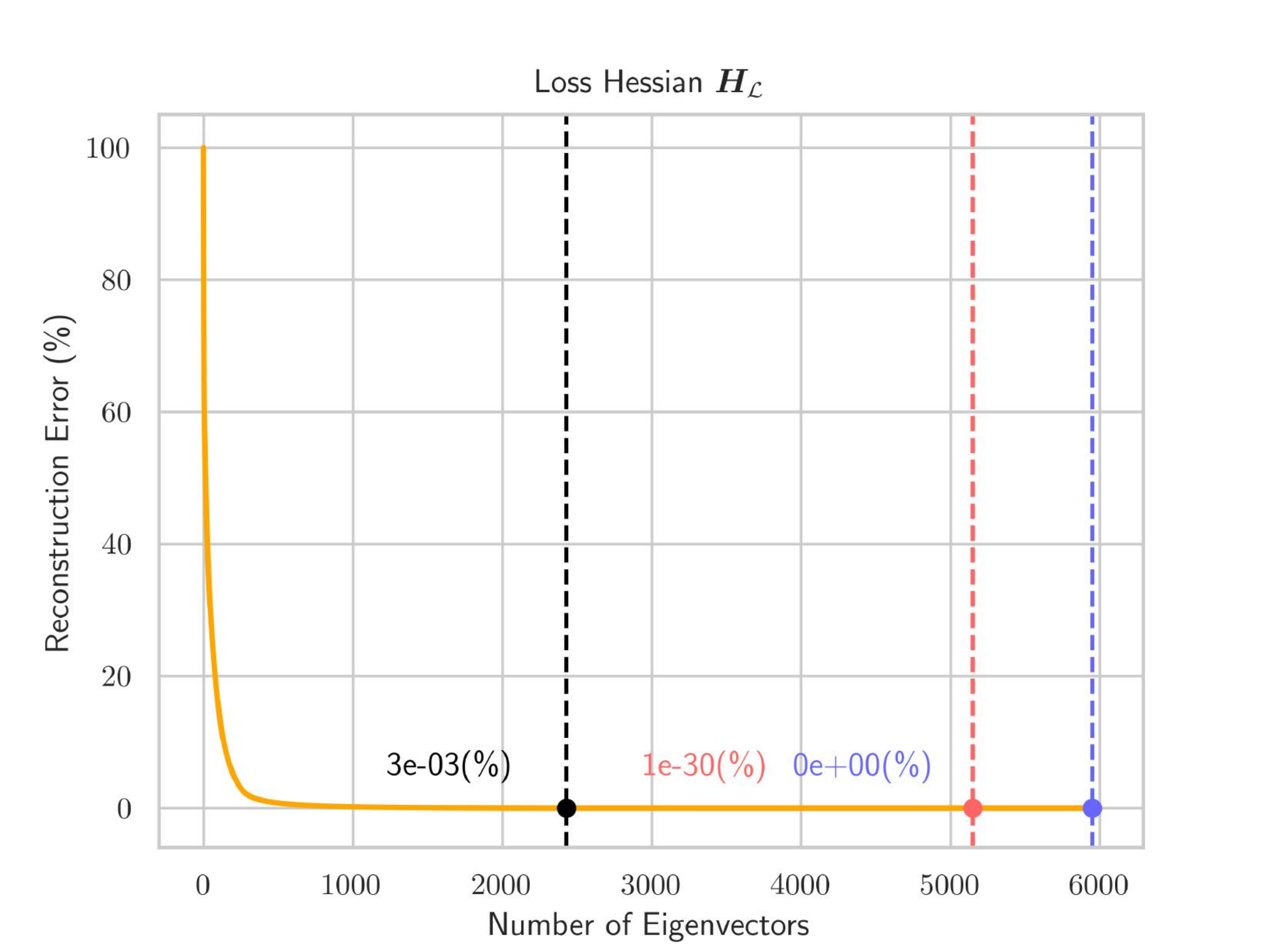}%
    \includegraphics[width=0.3\textwidth]{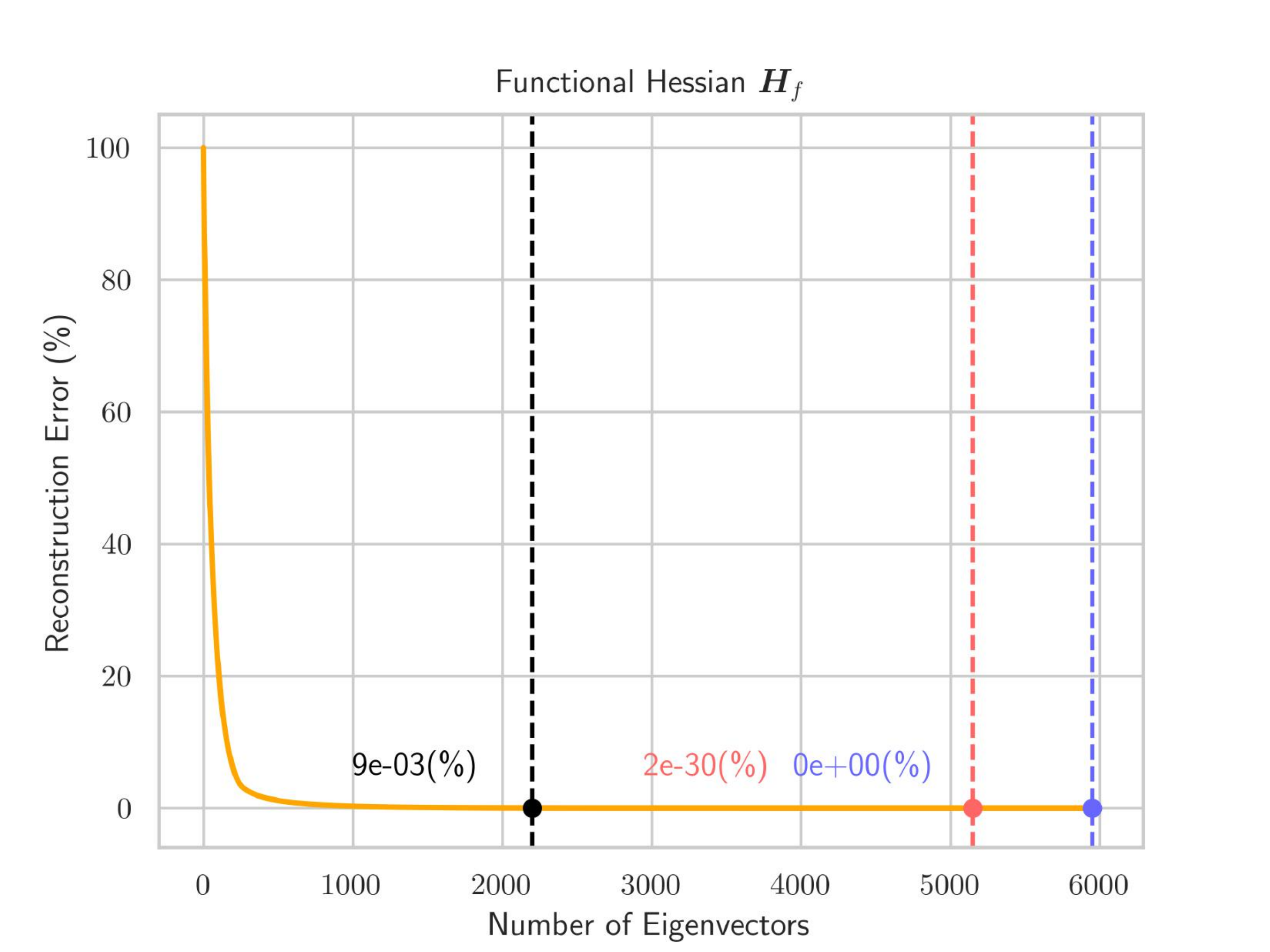}%
    \includegraphics[width=0.3\textwidth]{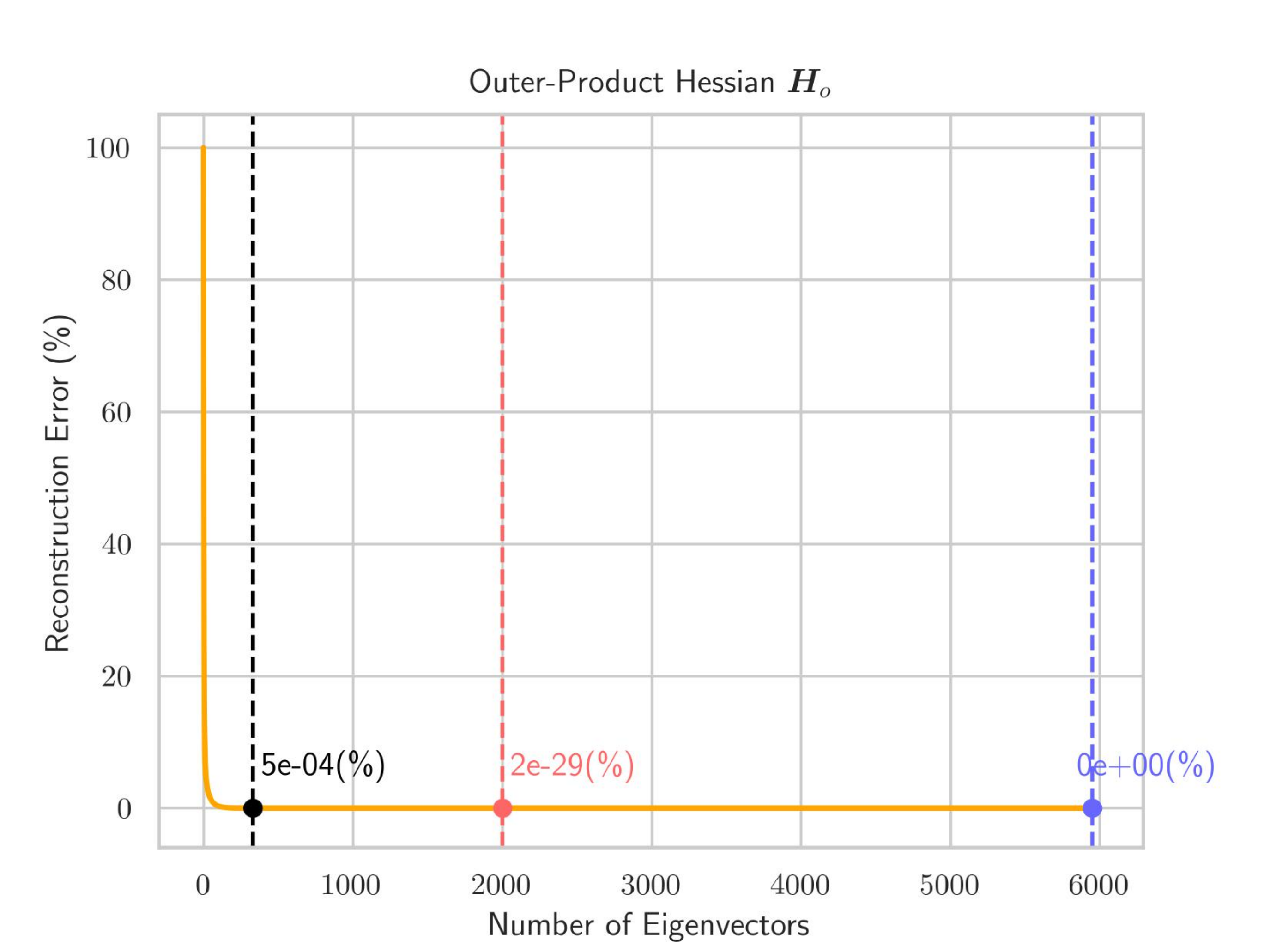}
    \caption{Hessian reconstruction error for \textbf{ELU} under \textbf{MSE} as the rank of the approximation is increased. The x-axis represents the number of top eigenvectors that form the low-rank approximation. The $y$-axis displays the \textit{reconstruction error in percentage} (100 $\%$ for zero eigenvectors used). The dashed vertical lines indicate the cut-off at various values of the rank: \textbf{first line} at the prediction based on the linear model, \textcolor{red}{\textbf{second line}} at the empirical measurement of rank, and \textcolor{blue}{\textbf{third line}} based on upper bounds from \citep{jacot2019asymptotic}, which become too coarse to be of any use (actually even greater than the \# of parameters but not marked there for visualization purposes). The hidden layer sizes are $50, 40, 30$.}
    \label{fig:elu_recon_bigger}
\end{figure}
\FloatBarrier
\begin{figure}[!htb]
    \centering
    \includegraphics[width=0.33\textwidth]{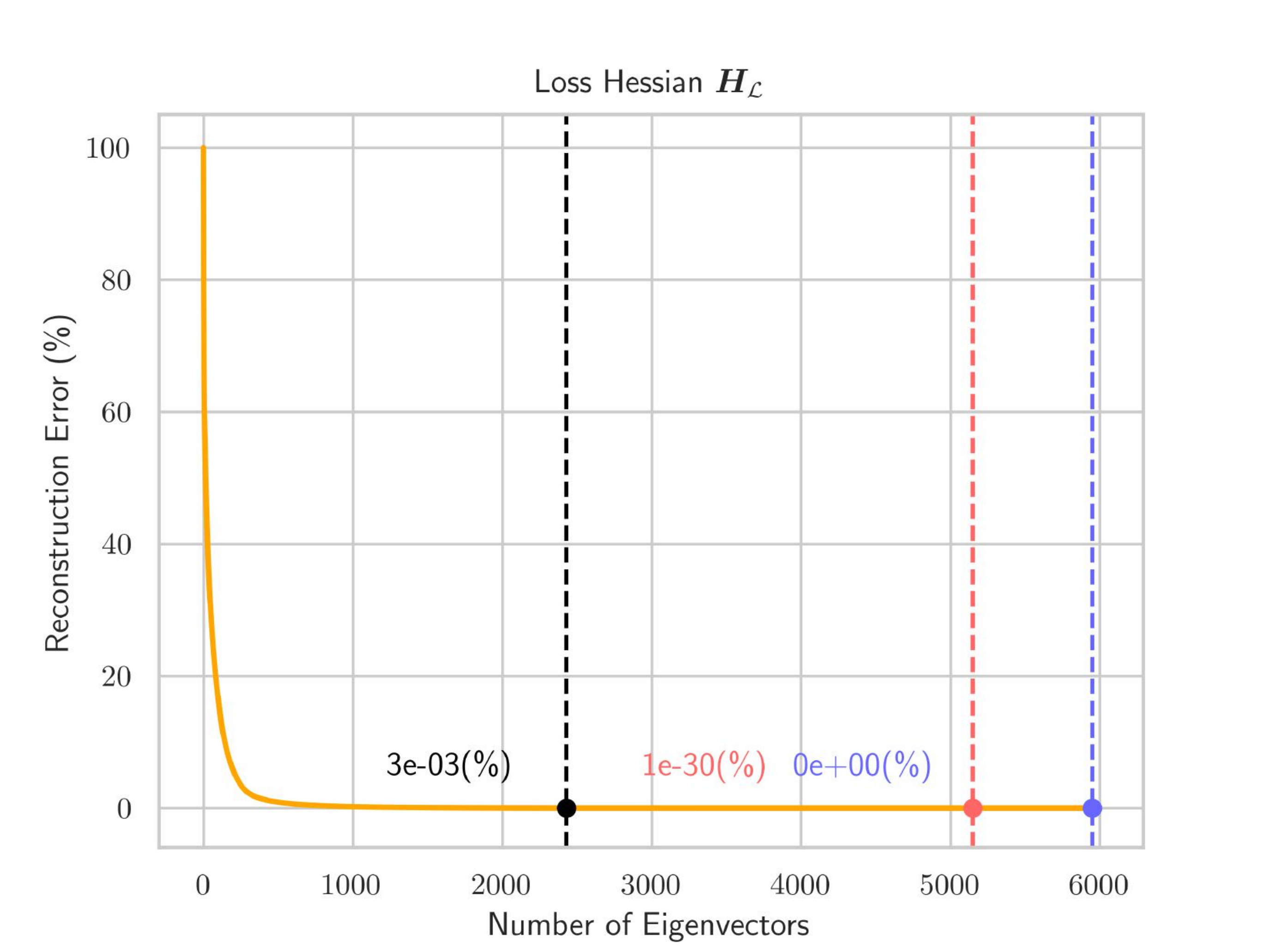}%
    \includegraphics[width=0.33\textwidth]{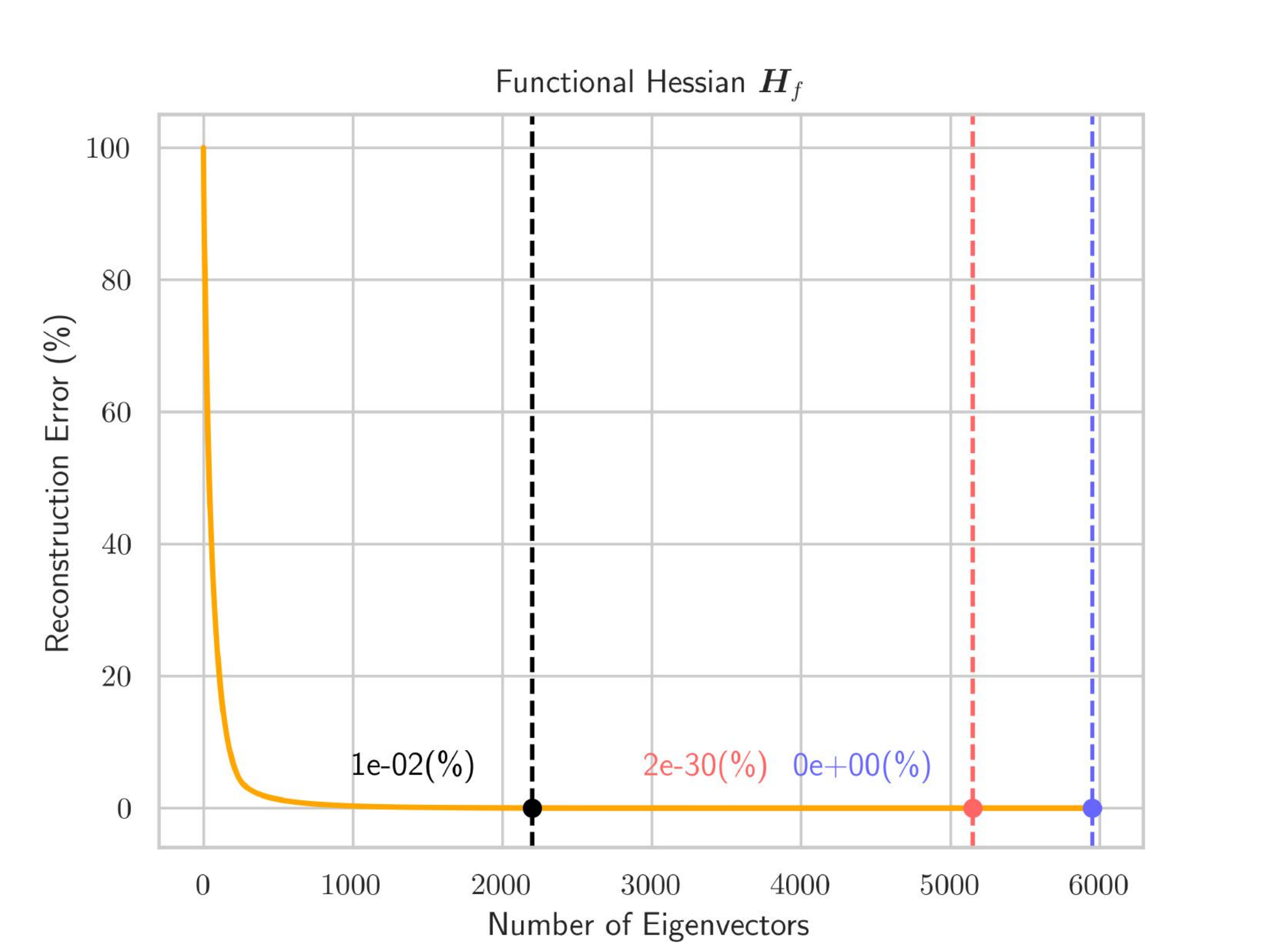}
    \includegraphics[width=0.33\textwidth]{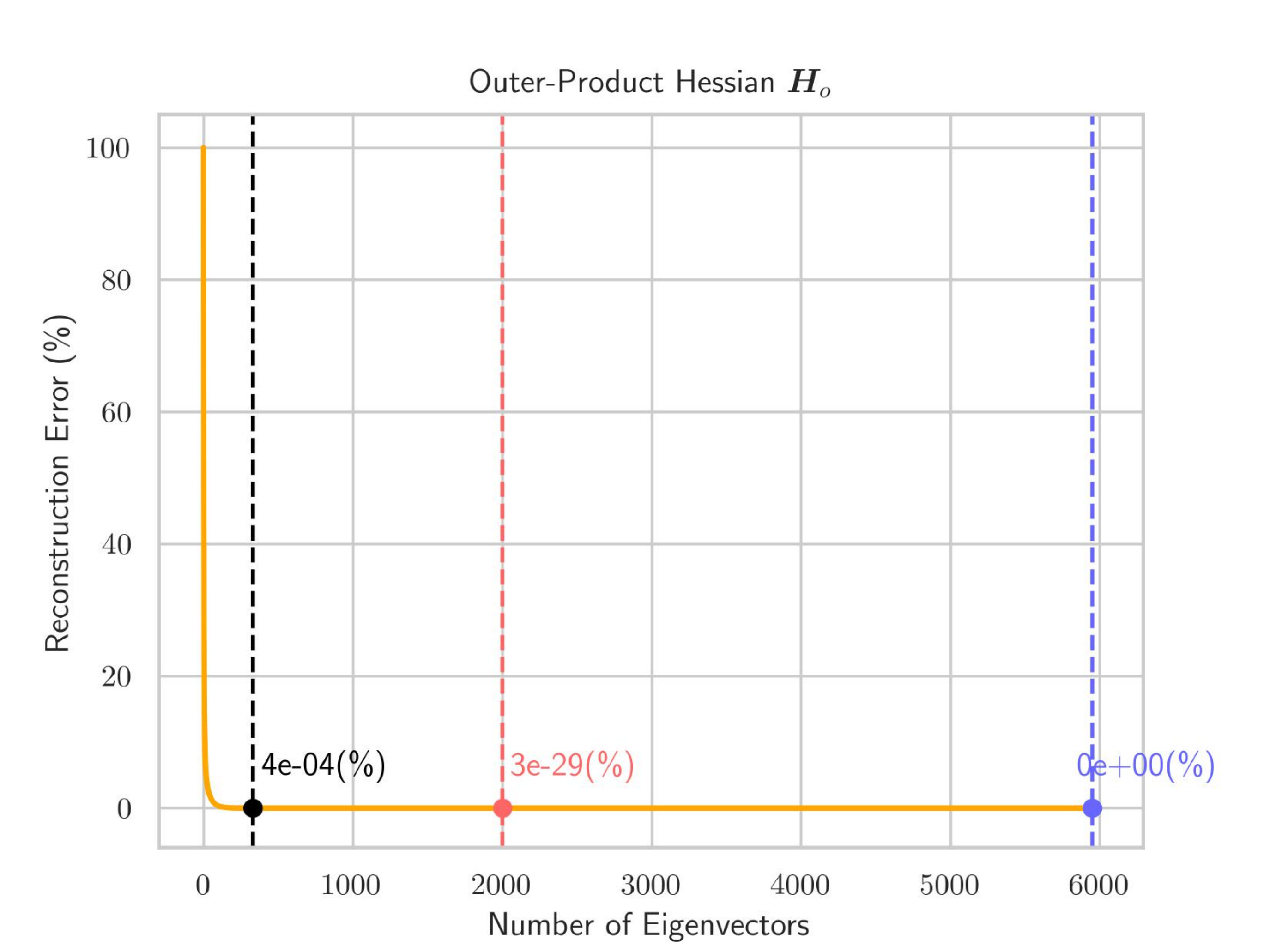}
    \caption{Hessian reconstruction error for \textbf{Tanh} under \textbf{MSE} as the rank of the approximation is increased. The x-axis represents the number of top eigenvectors that form the low-rank approximation. The $y$-axis displays the \textit{reconstruction error in percentage} (100 $\%$ for zero eigenvectors used). The dashed vertical lines indicate the cut-off at various values of the rank: \textbf{first line} at the prediction based on the linear model, \textcolor{red}{\textbf{second line}} at the empirical measurement of rank, and \textcolor{blue}{\textbf{third line}} based on upper bounds from \citep{jacot2019asymptotic}, which become too coarse to be of any use (actually even greater than the \# of parameters but not marked there for visualization purposes). The hidden layer sizes are $50, 40, 30$.}
    \label{fig:tanh_reconstruction_bigger}
\end{figure}
\FloatBarrier
\subsubsection{Cross Entropy}
Here we repeat the same experiments for cross entropy loss. We use the adjusted formula for the linear rank predictions, i.e. we replace $K$ by $K-1$. We display the results for ReLU in Figure \ref{fig:relu_reconstruction_cross}, for ELU in Figure \ref{fig:elu_recon_cross} and for tanh in Figure \ref{fig:tanh_recon_cross}. We also obtain excellent approximations for the numerical rank in this setting, showing that our predictions also extend to other losses under non-linearities. 
\begin{figure}[!htb]
    \centering
    \includegraphics[width=0.33\textwidth]{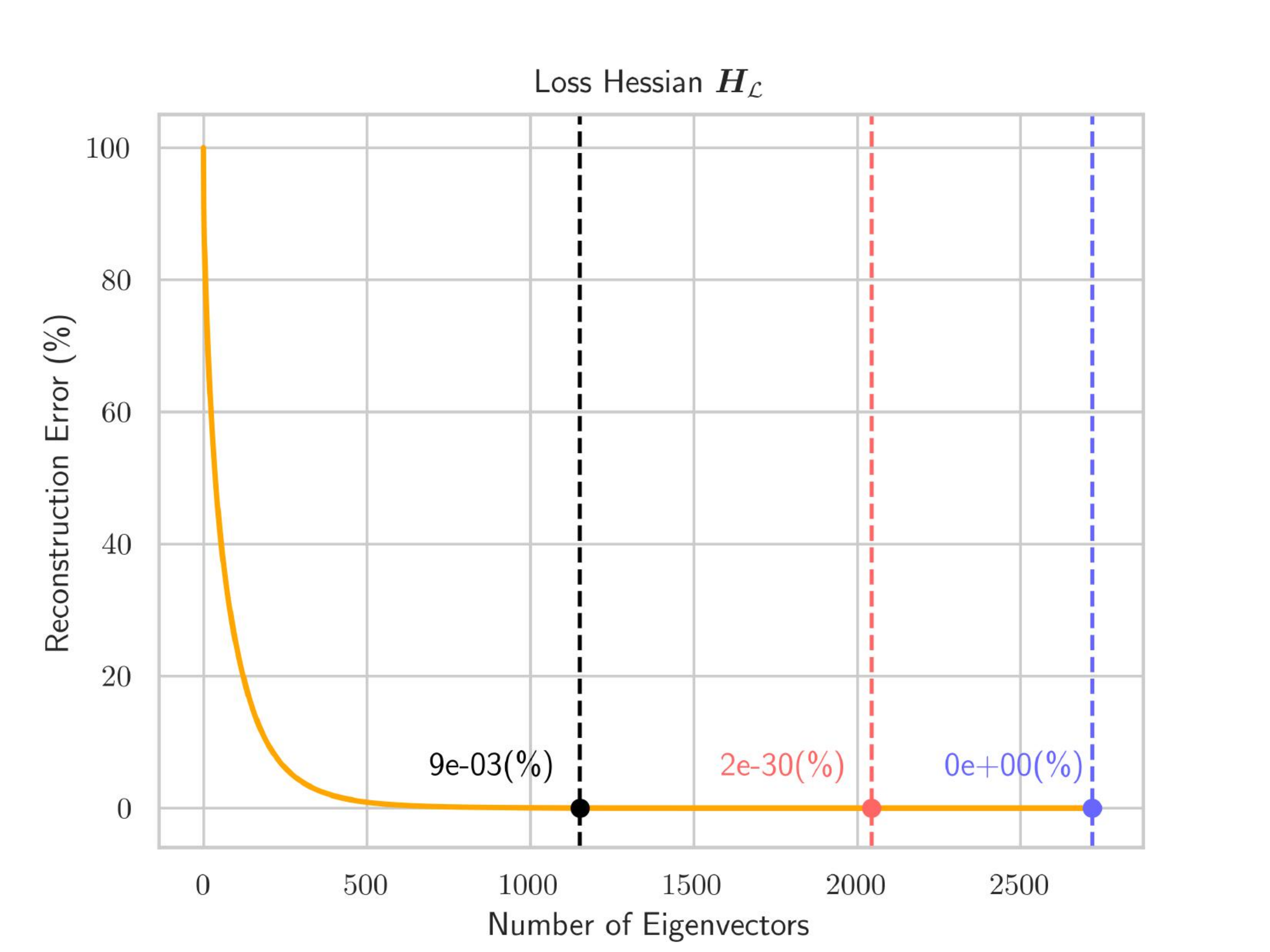}%
    \includegraphics[width=0.33\textwidth]{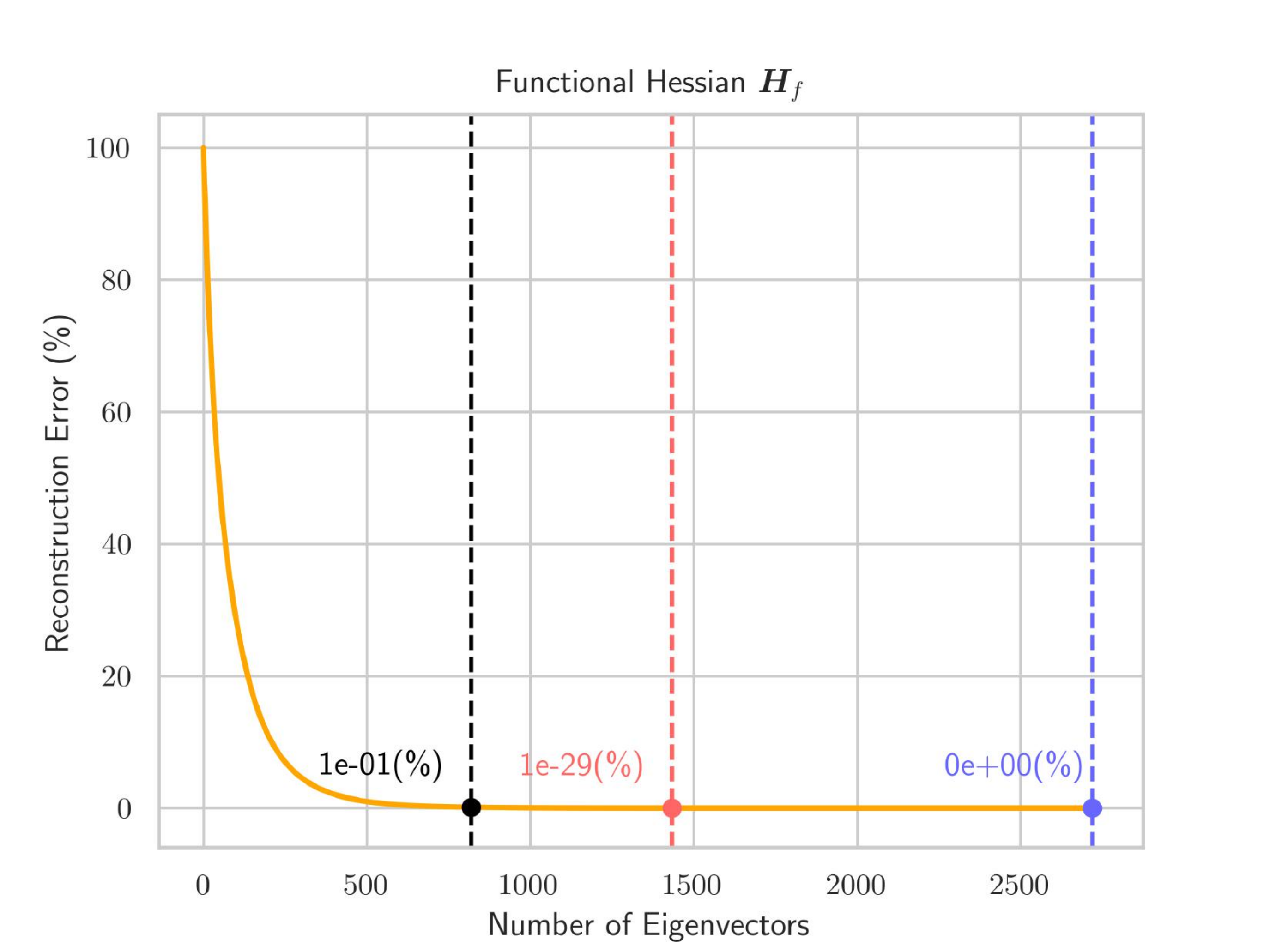}
    \includegraphics[width=0.33\textwidth]{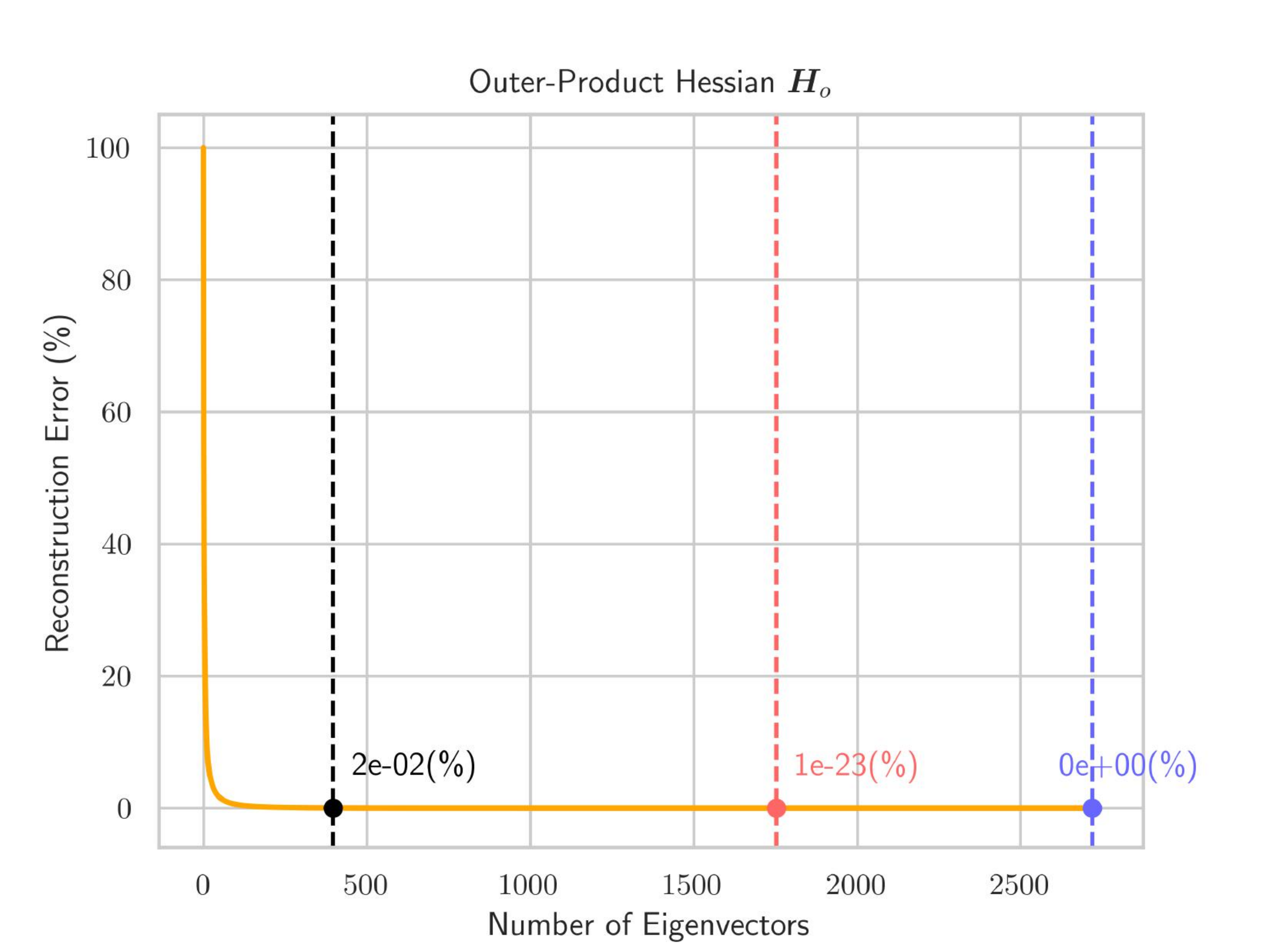}
    \caption{Hessian reconstruction error for \textbf{ReLU} as the rank of the approximation is increased under \textbf{cross entropy loss}. The x-axis represents the number of top eigenvectors that form the low-rank approximation. The $y$-axis displays the \textit{reconstruction error in percentage} (100 $\%$ for zero eigenvectors used). The dashed vertical lines indicate the cut-off at various values of the rank: \textbf{first line} at the prediction based on the linear model, \textcolor{red}{\textbf{second line}} at the empirical measurement of rank, and \textcolor{blue}{\textbf{third line}} based on upper bounds from \citep{jacot2019asymptotic}, which become too coarse to be of any use (actually even greater than the \# of parameters but not marked there for visualization purposes). The hidden layer sizes are $30, 20$.}
    \label{fig:relu_reconstruction_cross}
\end{figure}
\FloatBarrier

\begin{figure}[!htb]
    \centering
    \includegraphics[width=0.33\textwidth]{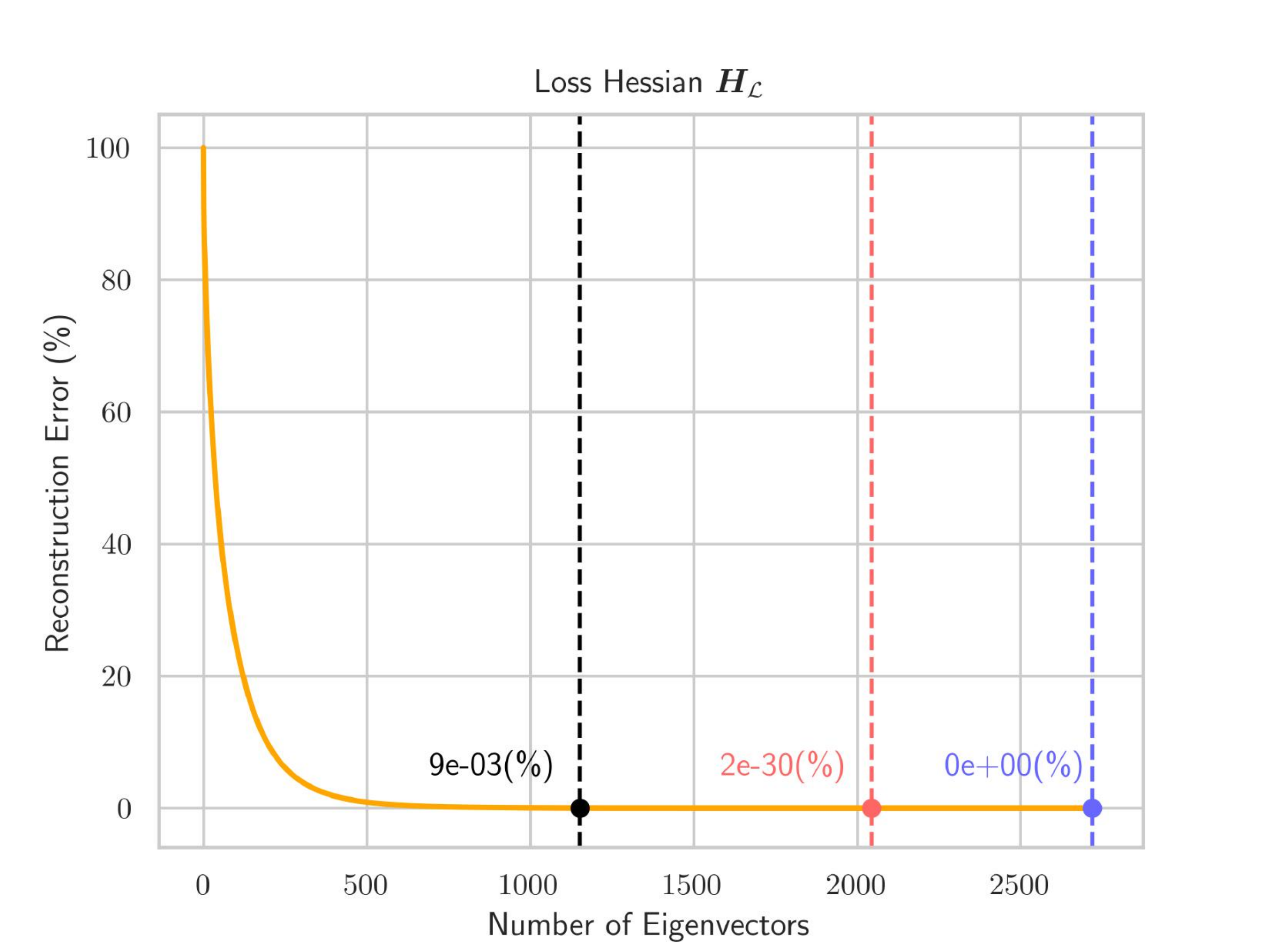}%
    \includegraphics[width=0.33\textwidth]{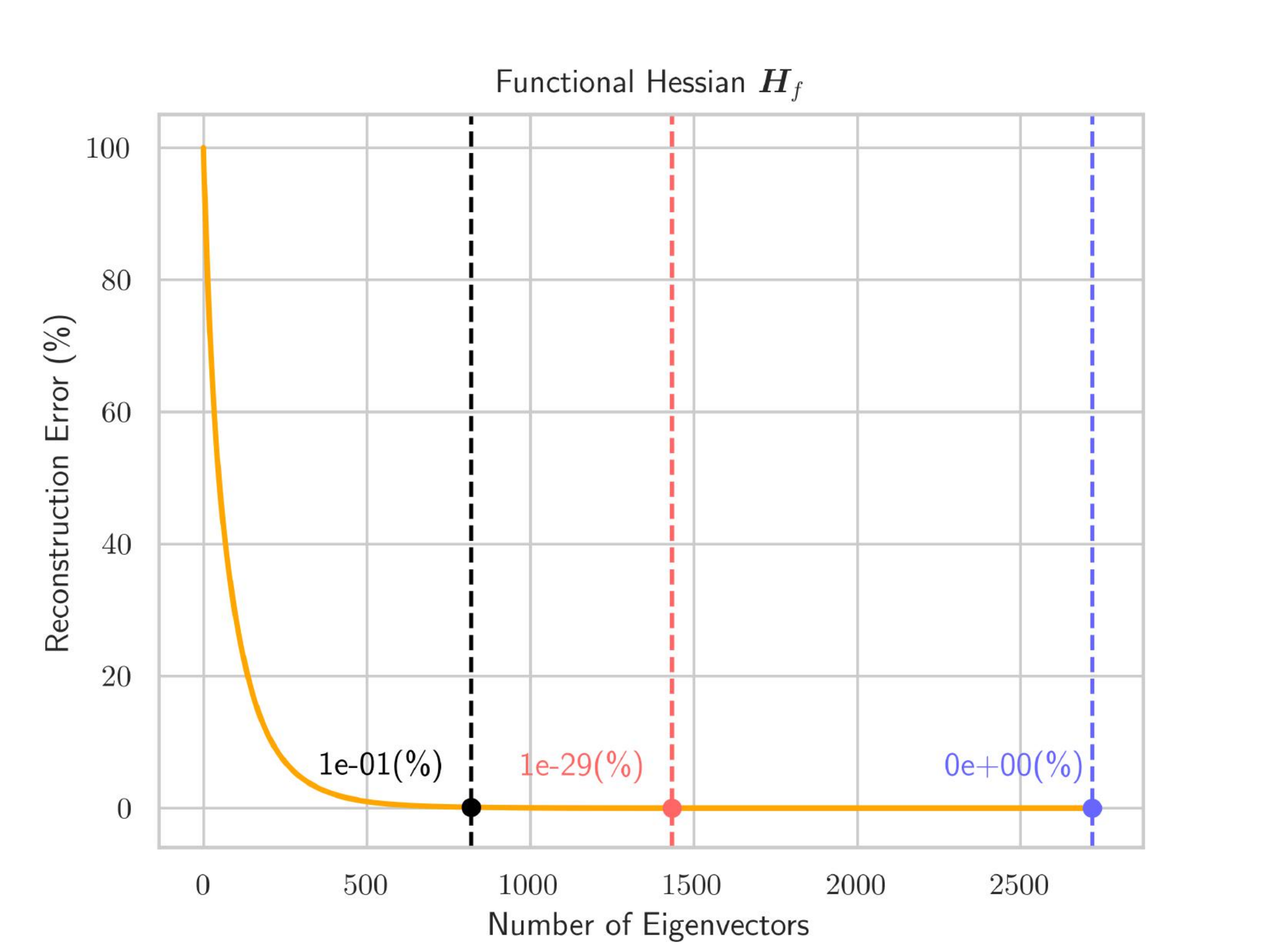}
    \includegraphics[width=0.33\textwidth]{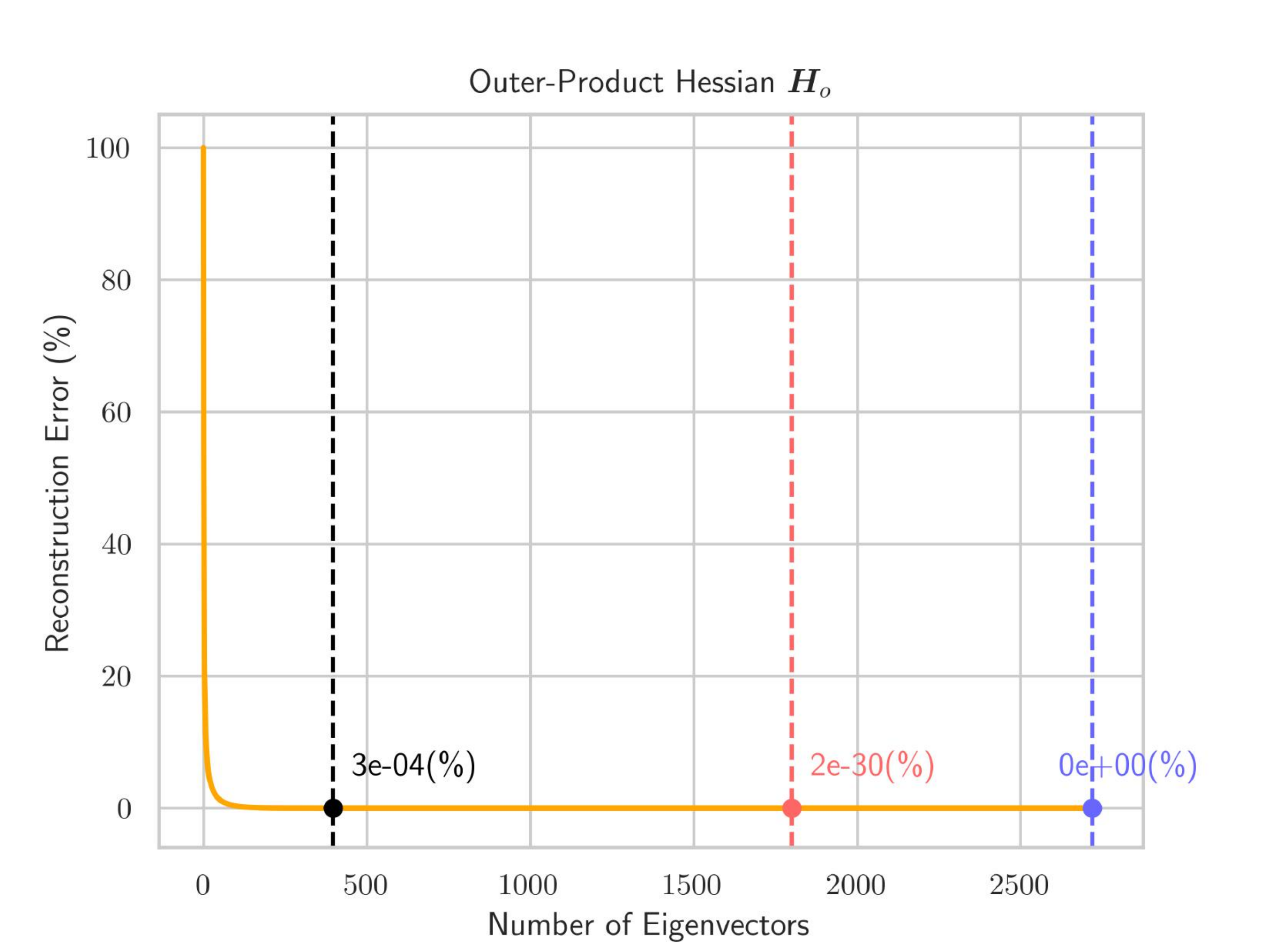}
    \caption{Hessian reconstruction error for \textbf{ELU} as the rank of the approximation is increased under \textbf{cross entropy loss}. The x-axis represents the number of top eigenvectors that form the low-rank approximation. The $y$-axis displays the \textit{reconstruction error in percentage} (100 $\%$ for zero eigenvectors used). The dashed vertical lines indicate the cut-off at various values of the rank: \textbf{first line} at the prediction based on the linear model, \textcolor{red}{\textbf{second line}} at the empirical measurement of rank, and \textcolor{blue}{\textbf{third line}} based on upper bounds from \citep{jacot2019asymptotic}, which become too coarse to be of any use (actually even greater than the \# of parameters but not marked there for visualization purposes). The hidden layer sizes are $30, 20$.}
    \label{fig:elu_recon_cross}
\end{figure}
\FloatBarrier
\begin{figure}[!htb]
    \centering
    \includegraphics[width=0.33\textwidth]{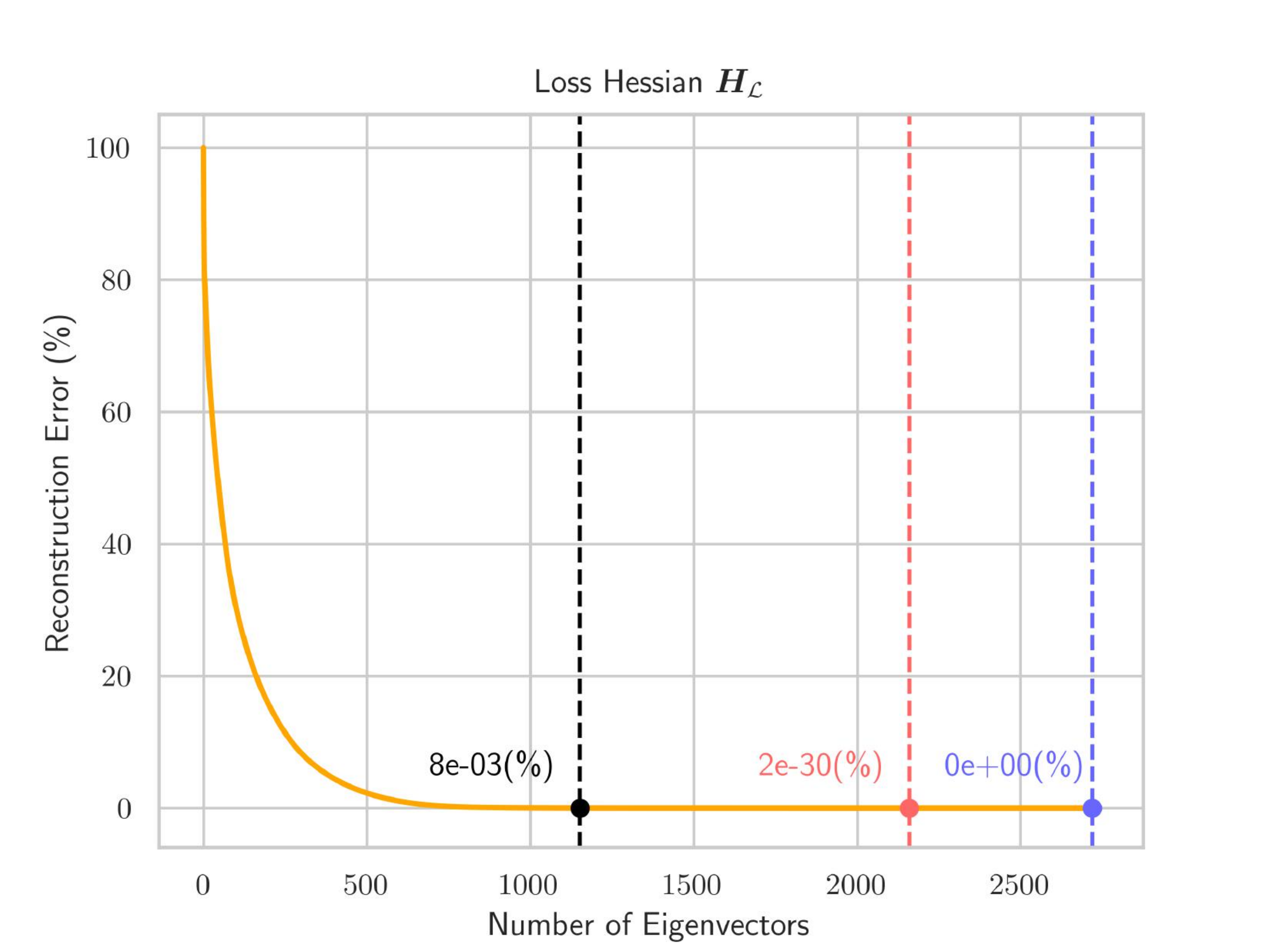}%
    \includegraphics[width=0.33\textwidth]{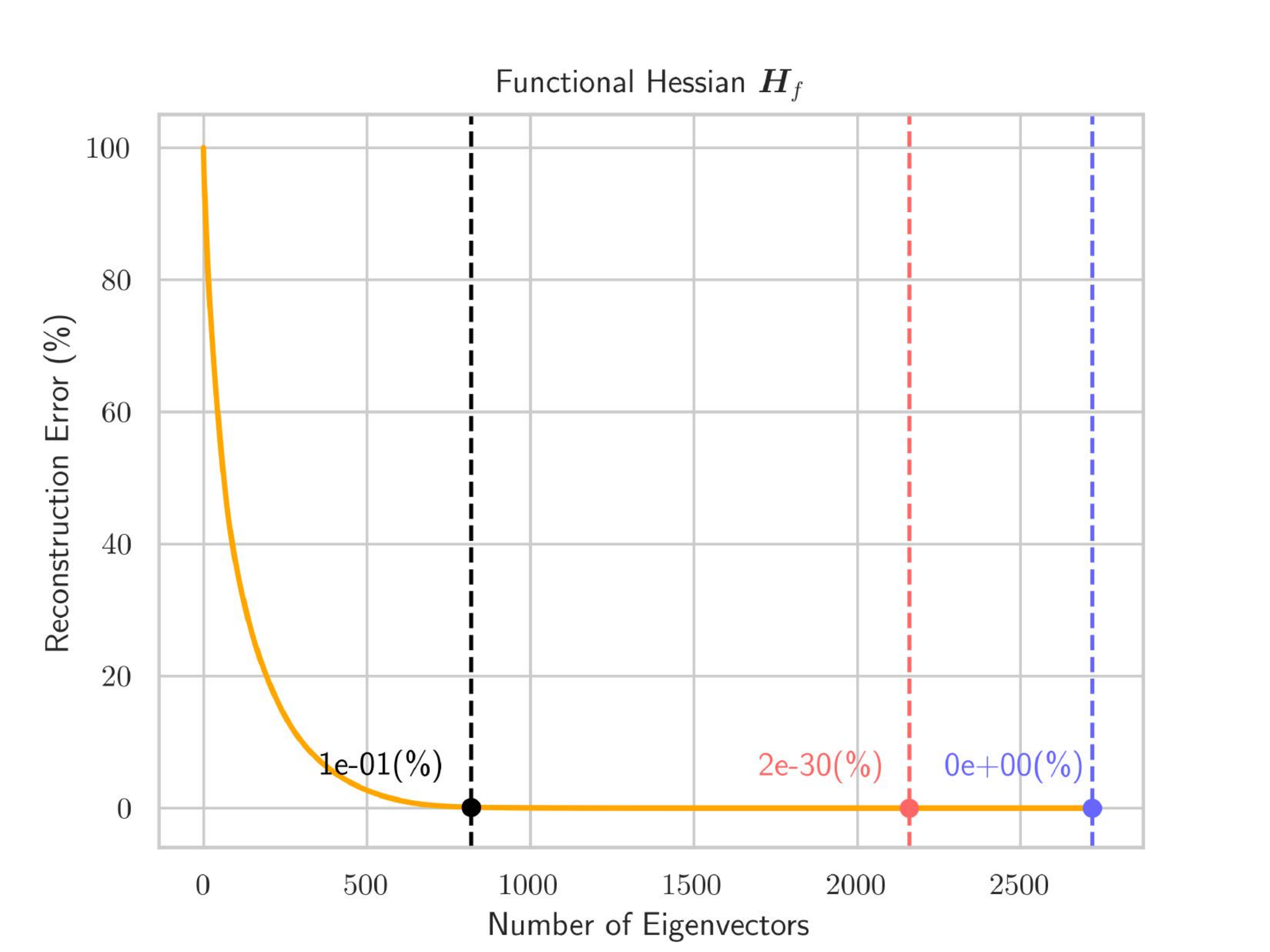}
    \includegraphics[width=0.33\textwidth]{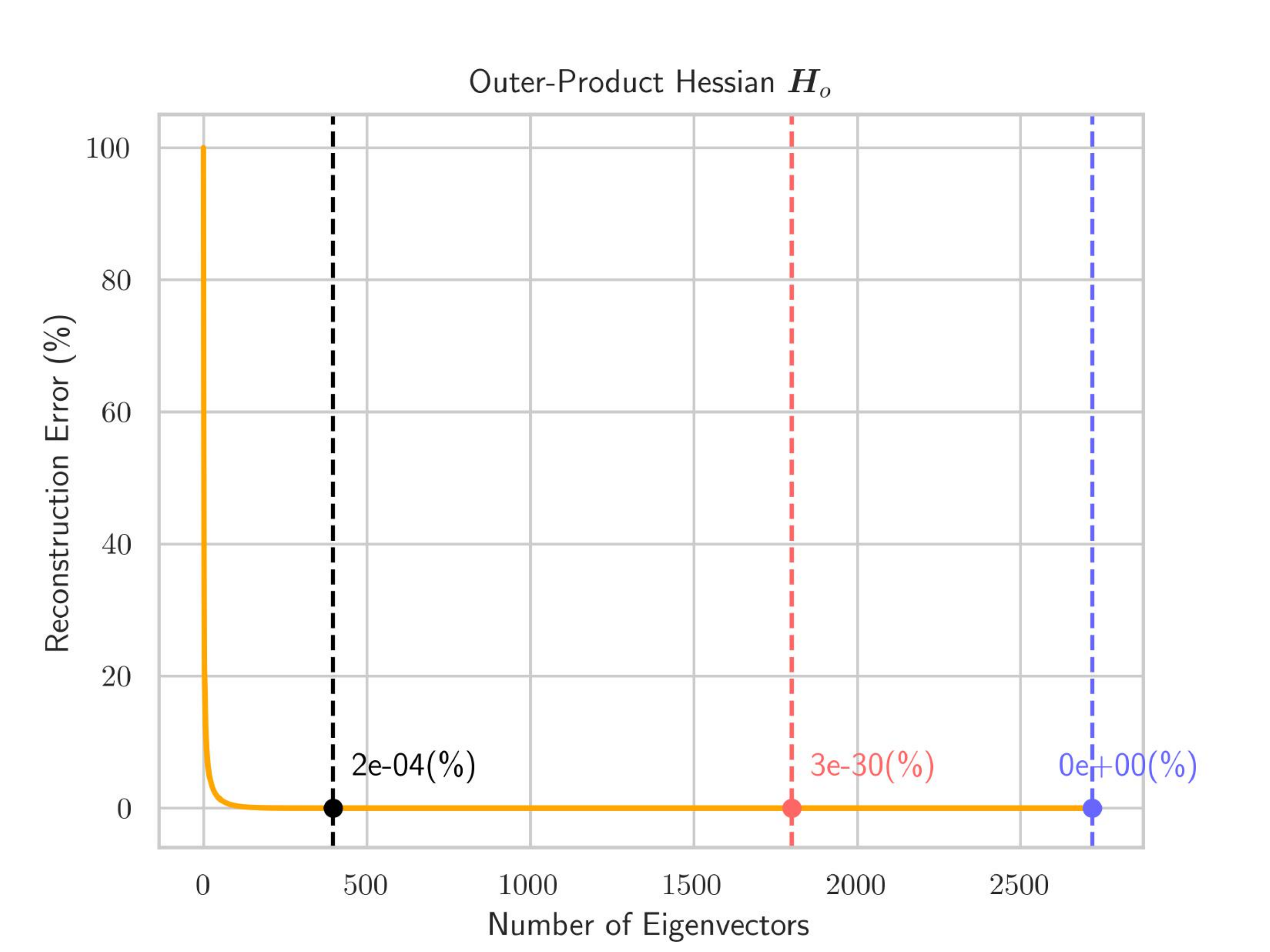}
    \caption{Hessian reconstruction error for \textbf{Tanh} as the rank of the approximation is increased under \textbf{cross entropy loss}. The x-axis represents the number of top eigenvectors that form the low-rank approximation. The $y$-axis displays the \textit{reconstruction error in percentage} (100 $\%$ for zero eigenvectors used). The dashed vertical lines indicate the cut-off at various values of the rank: \textbf{first line} at the prediction based on the linear model, \textcolor{red}{\textbf{second line}} at the empirical measurement of rank, and \textcolor{blue}{\textbf{third line}} based on upper bounds from \citep{jacot2019asymptotic}, which become too coarse to be of any use (actually even greater than the \# of parameters but not marked there for visualization purposes). The hidden layer sizes are $30, 20$.}
    \label{fig:tanh_recon_cross}
\end{figure}
\FloatBarrier
\subsection{Spectral Plots for More Non-Linearities}
\label{spectral_app}
To further underline the utility of our theoretical results in the non-linear setting, we present more spectral plots, super-imposing the linear and corresponding non-linear spectrum for more non-linearities and loss functions. For all experiments we use $d=64$ and $N=200$. 
\subsubsection{Mean Squared Error}
In Figure~\ref{fig:relu_spectrum}, we show the results for ReLU non-linearity, observing that the plateau of the spectrum is accurately described by our predictions. The same also holds for ELU activation, as can be readily seen in Figure \ref{fig:elu_spectrum}.
\begin{figure}[!htb]
    \centering
    \includegraphics[width=0.33\textwidth]{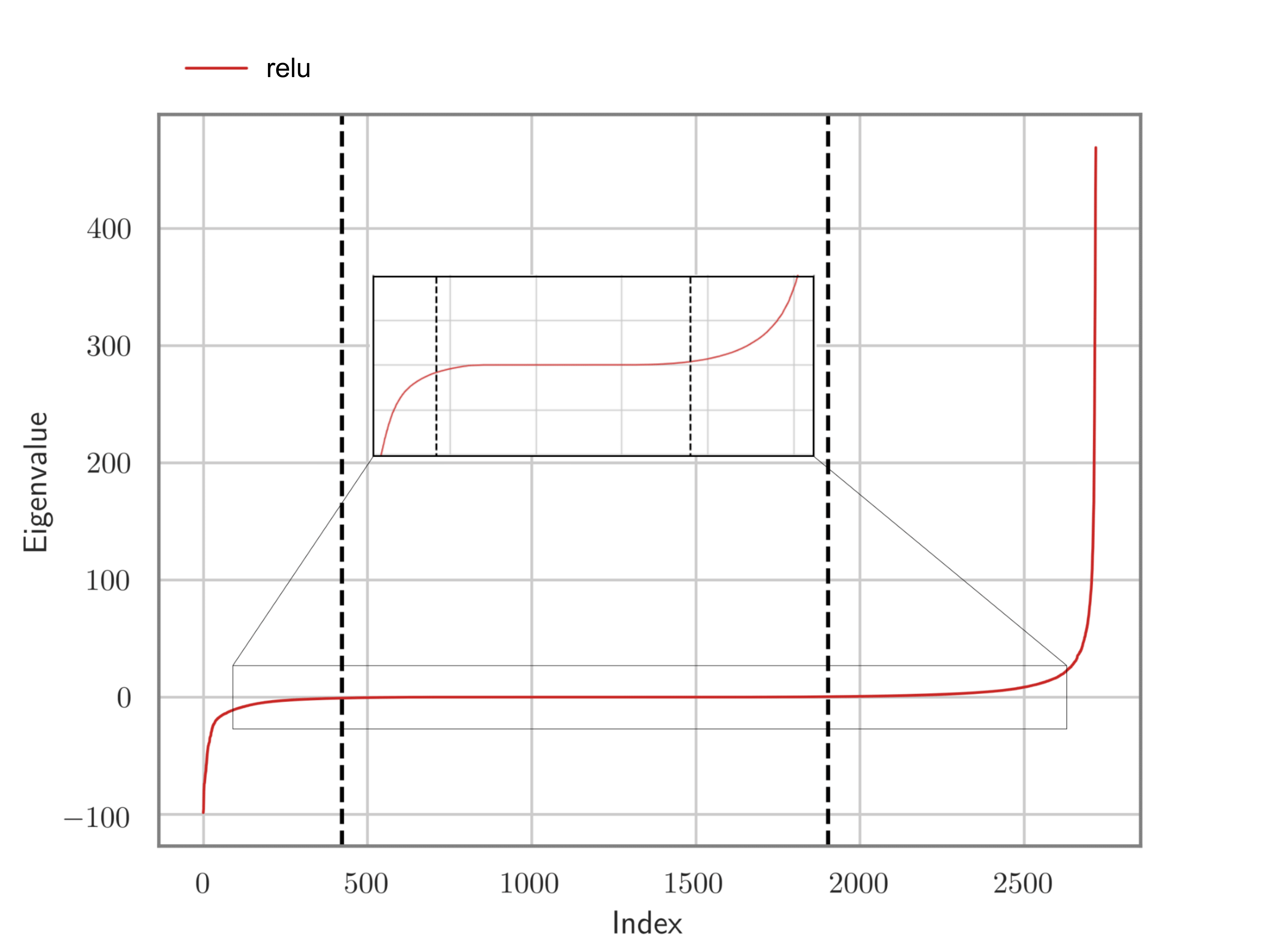}%
    \includegraphics[width=0.33\textwidth]{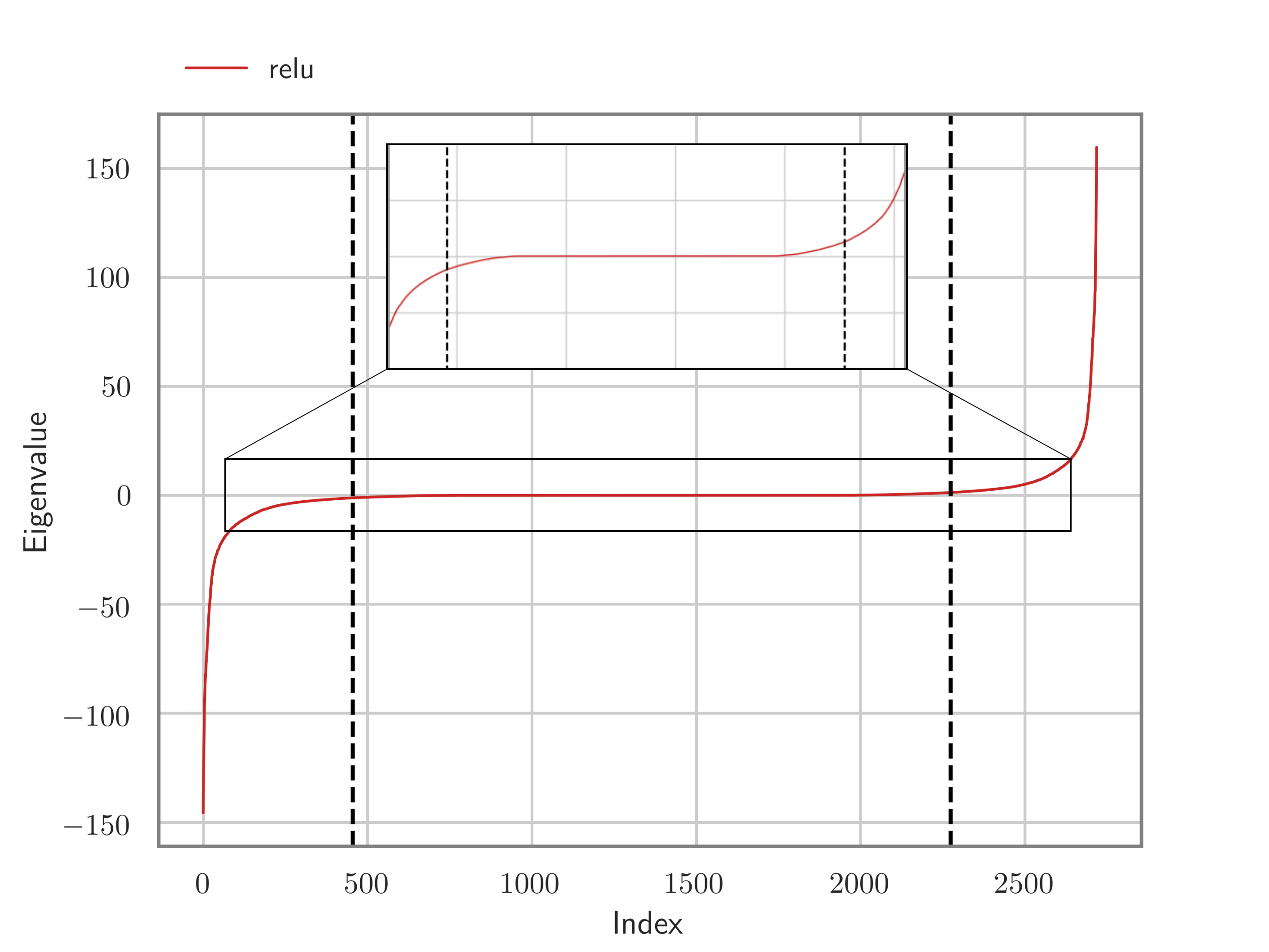}%
    \includegraphics[width=0.33\textwidth]{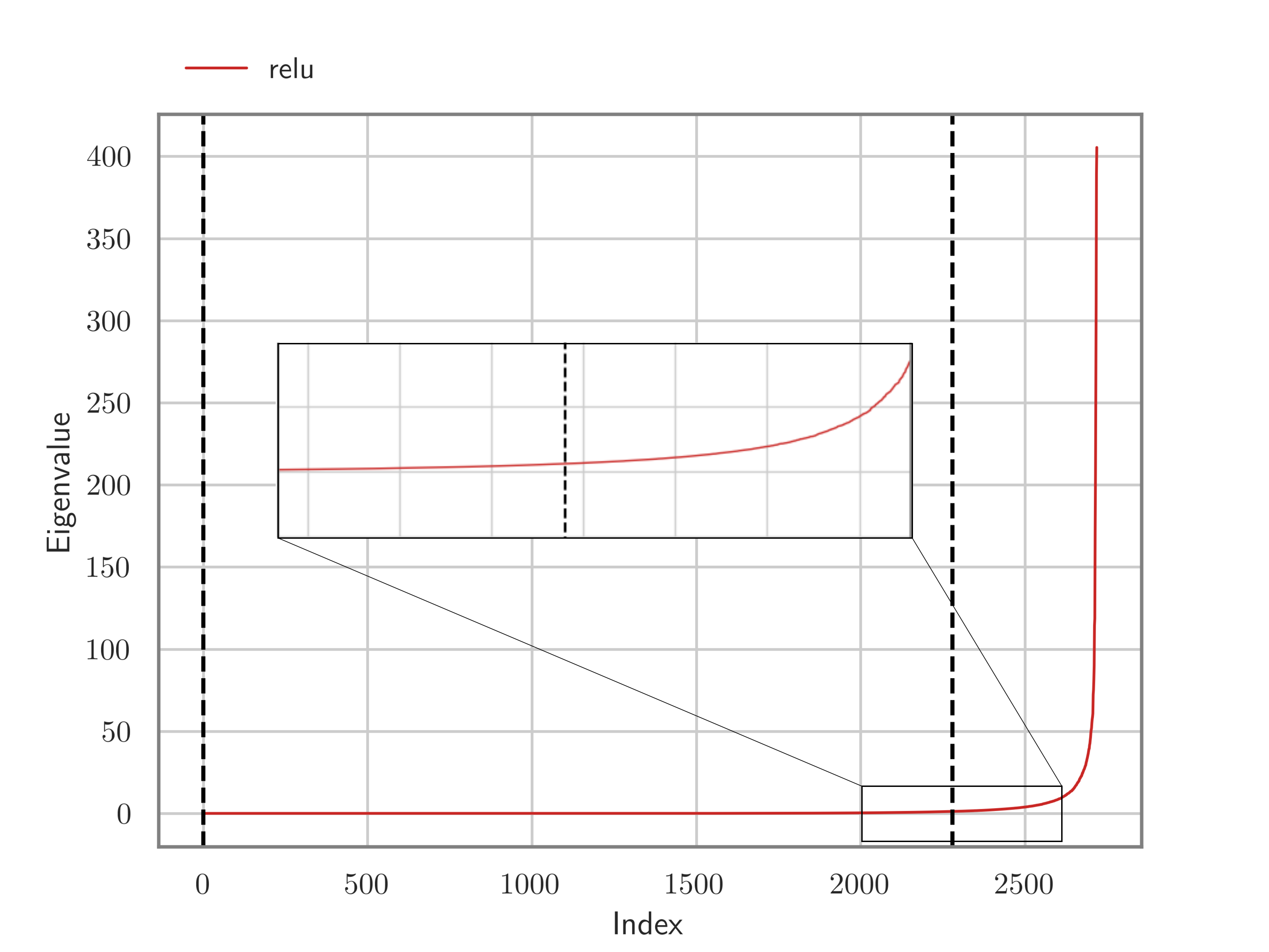}
    \caption{Spectrum of the loss Hessian $\HL$ (left), functional Hessian $\HF$ (middle) and outer product $\HO$ (right), for a \textbf{ReLU} network. Black dashed lines are the predictions of the bulk size via our rank formulas. We use 2 hidden layers of size $30,20$ on \textsc{MNIST}.}
    \label{fig:relu_spectrum}
\end{figure}
\FloatBarrier
\begin{figure}[ht!]
    \centering
    \includegraphics[width=0.33\textwidth]{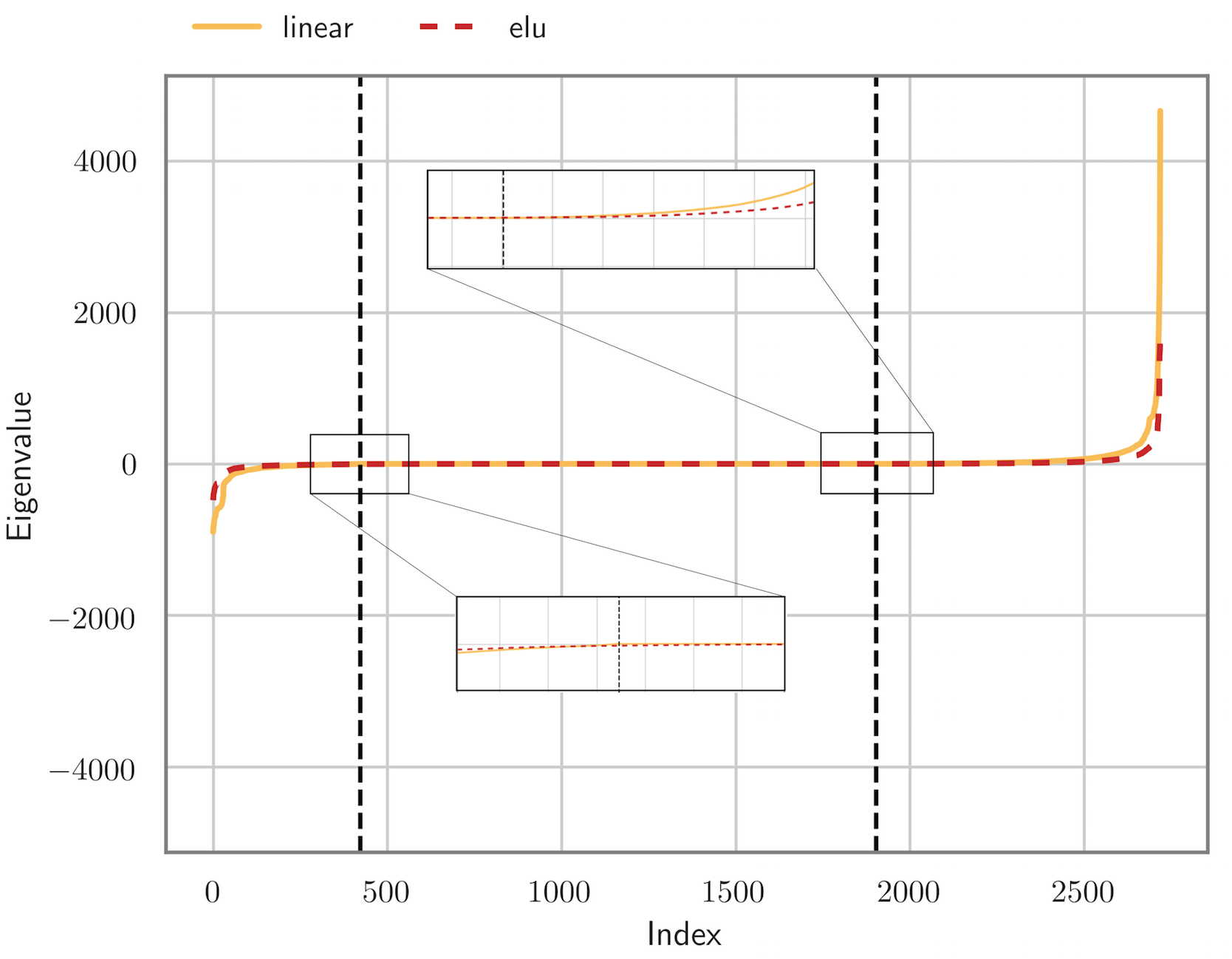}%
    \includegraphics[width=0.33\textwidth]{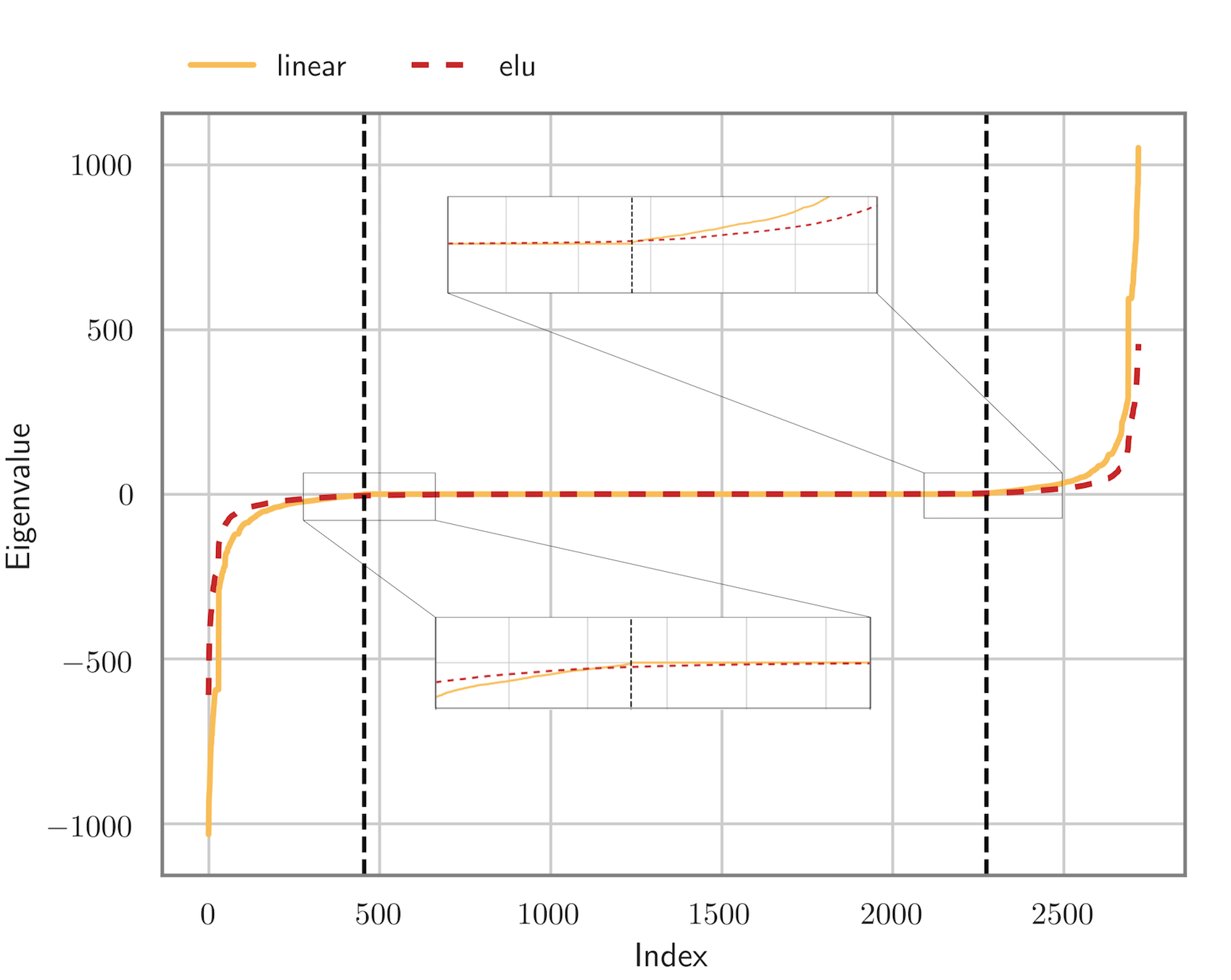}%
    \includegraphics[width=0.33\textwidth]{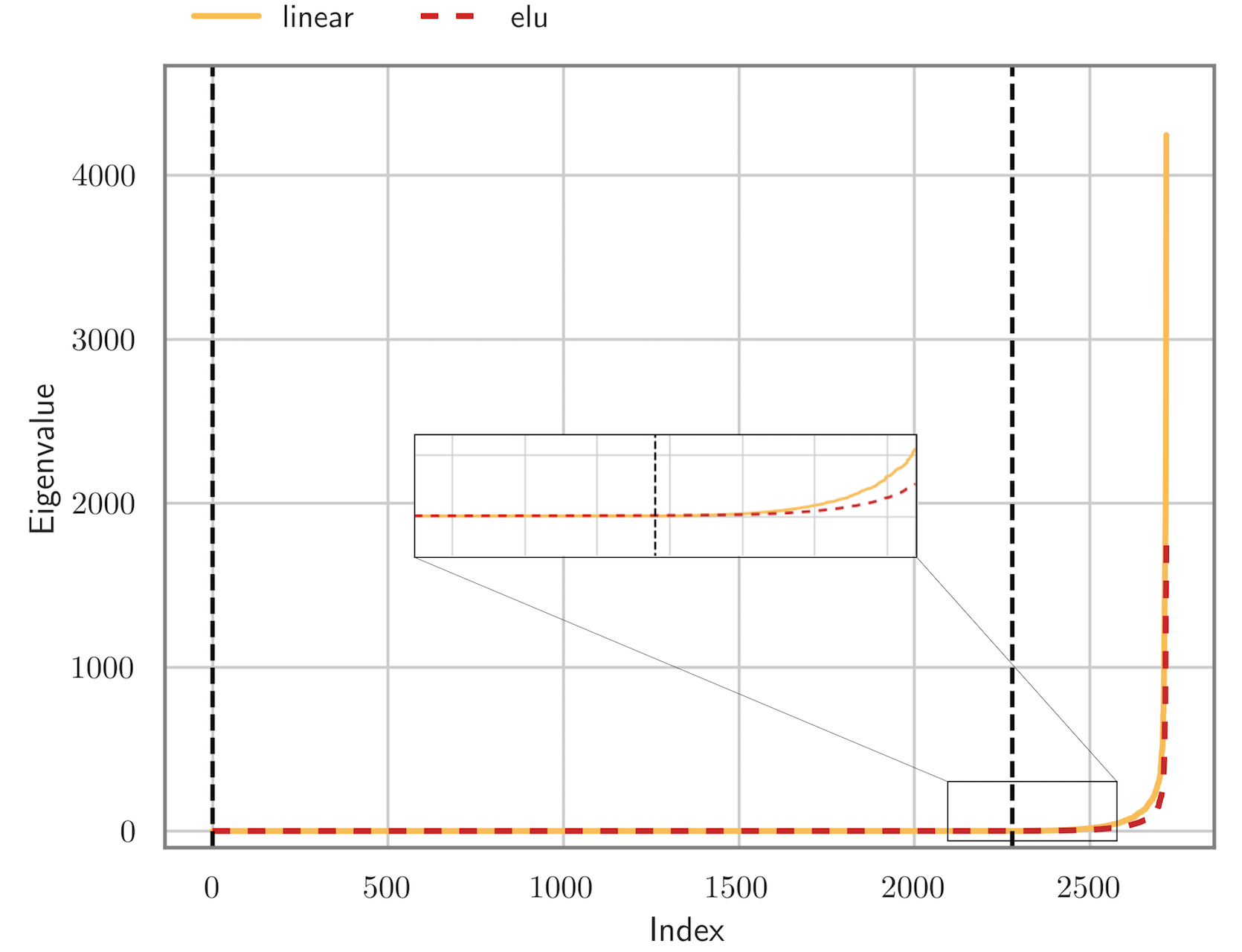}
    \caption{Spectrum of the loss Hessian $\HL$ (left), functional Hessian $\HF$ (middle)  and outer product $\HO$ (right), for \textcolor{orange}{\textbf{linear}} and \textcolor{red}{\textbf{non-linear}} networks. Black dashed lines are the predictions of the bulk size via our rank formulas. We use 2 hidden layers of size $30,20$ with \textbf{ELU} activation on \textsc{MNIST}.}
    \label{fig:elu_spectrum}
\end{figure}
\FloatBarrier
\subsubsection{Cross Entropy}
Here we show that the spectrum also behaves very similar if cross entropy is employed, again regardless of the non-linearity used. We show the results for ReLU non-linearity in Figure \ref{fig:relu_spectrum_cross}, for tanh in Figure \ref{fig:tanh_spectrum_cross} and ELU in Figure \ref{fig:elu_spectrum_cross}.
\begin{figure}[ht!]
    \centering
    \includegraphics[width=0.33\textwidth]{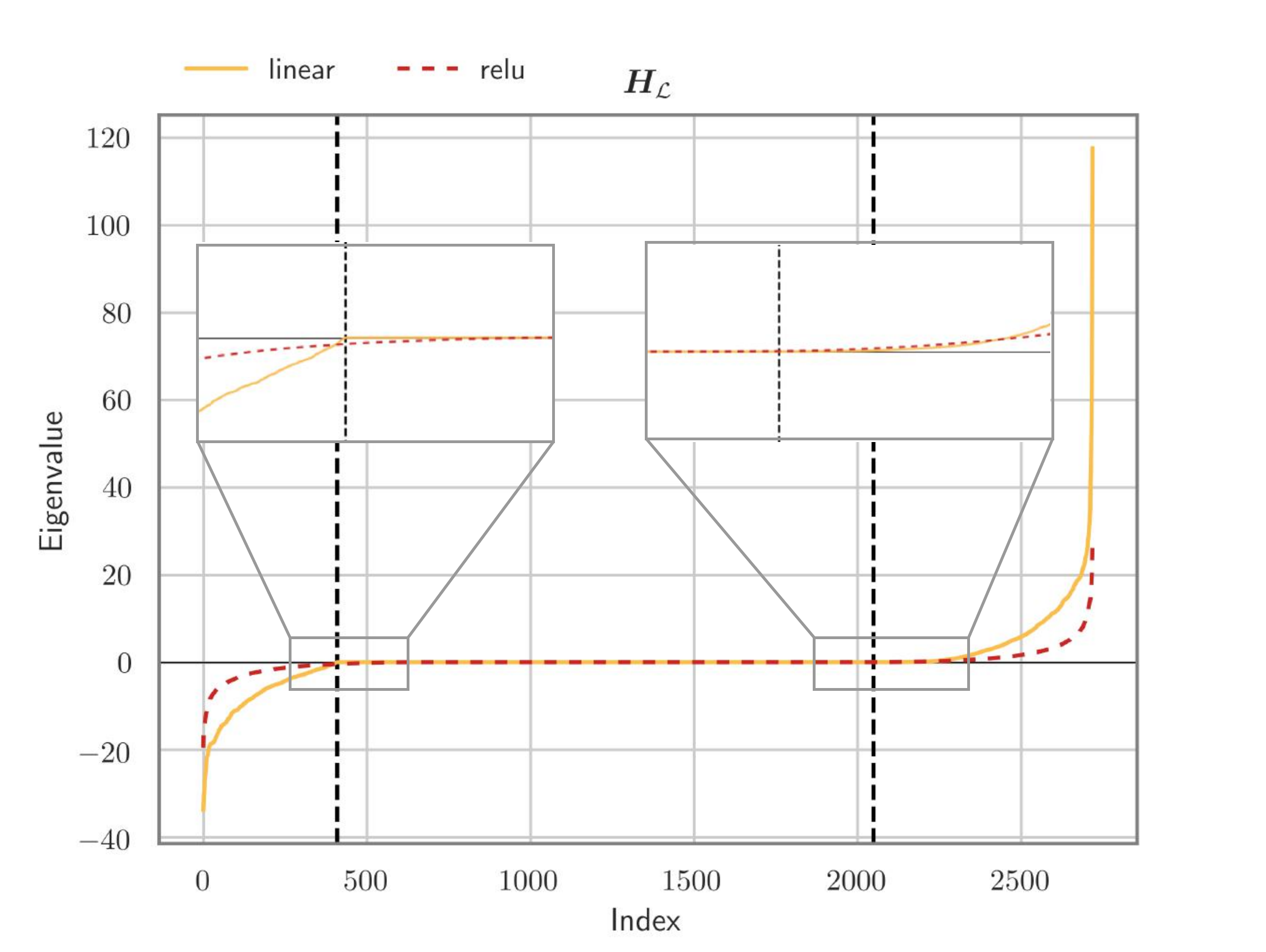}%
    \includegraphics[width=0.33\textwidth]{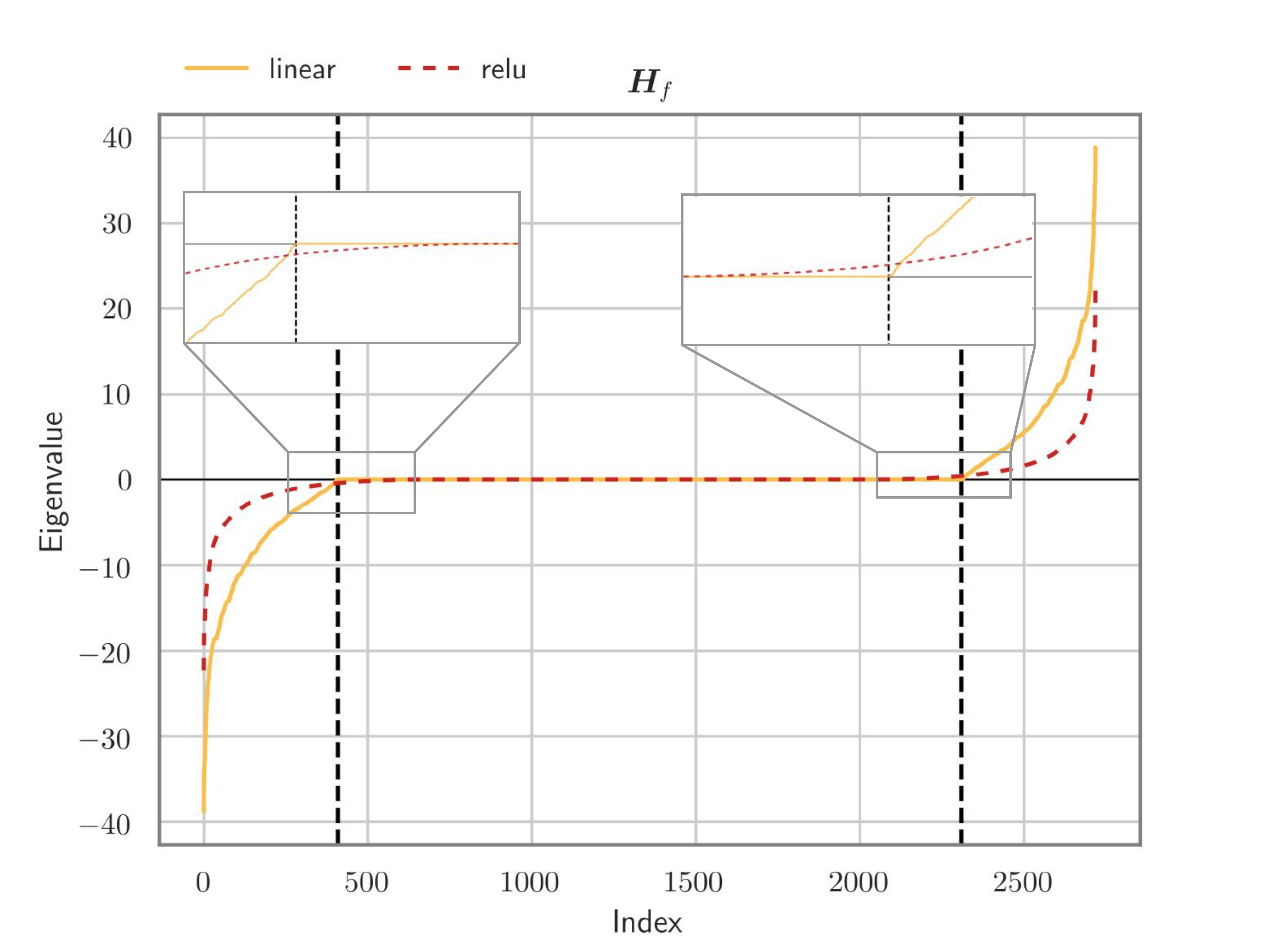}%
    \includegraphics[width=0.33\textwidth]{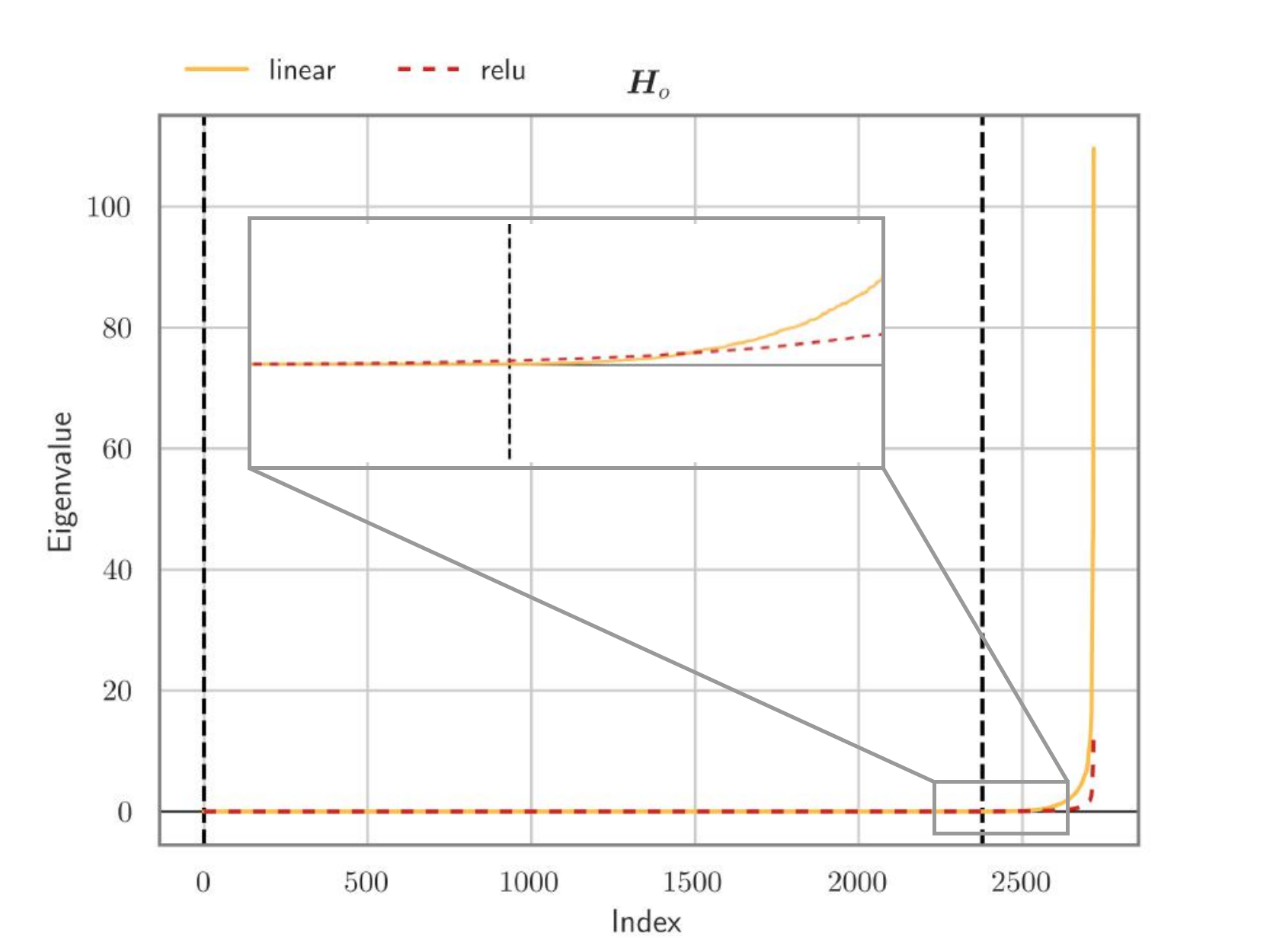}
    \caption{Spectrum of the loss Hessian $\HL$ (left), functional Hessian $\HF$ (middle)  and outer product $\HO$ (right), for \textcolor{orange}{\textbf{linear}} and \textcolor{red}{\textbf{non-linear}} networks. Black dashed lines are the predictions of the bulk size via our rank formulas. We use 2 hidden layers of size $30,20$ with \textbf{ReLU} activation on \textsc{MNIST} under \textbf{cross entropy} loss.}
    \label{fig:relu_spectrum_cross}
\end{figure}
\FloatBarrier
\begin{figure}[ht!]
    \centering
    \includegraphics[width=0.33\textwidth]{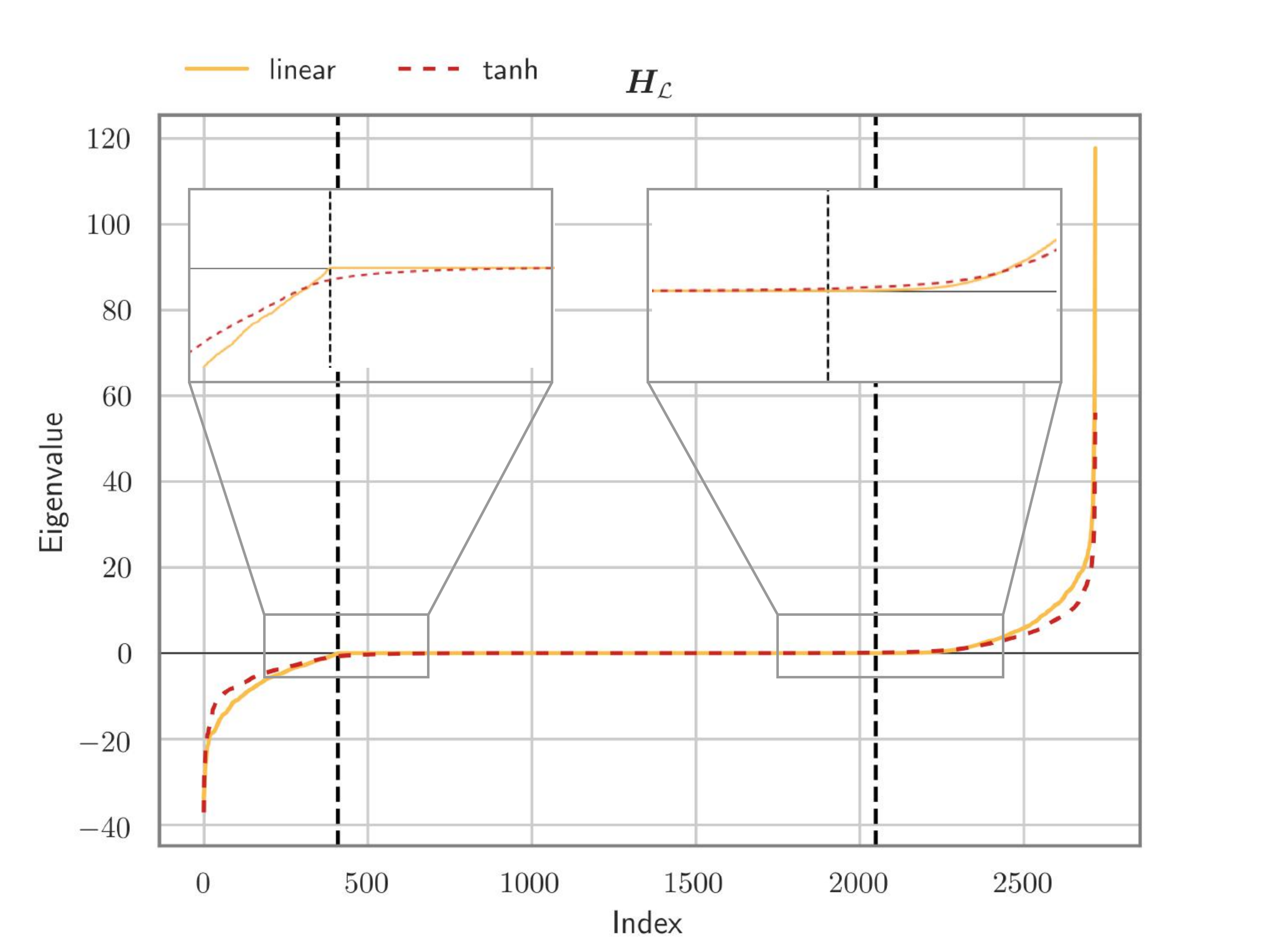}%
    \includegraphics[width=0.33\textwidth]{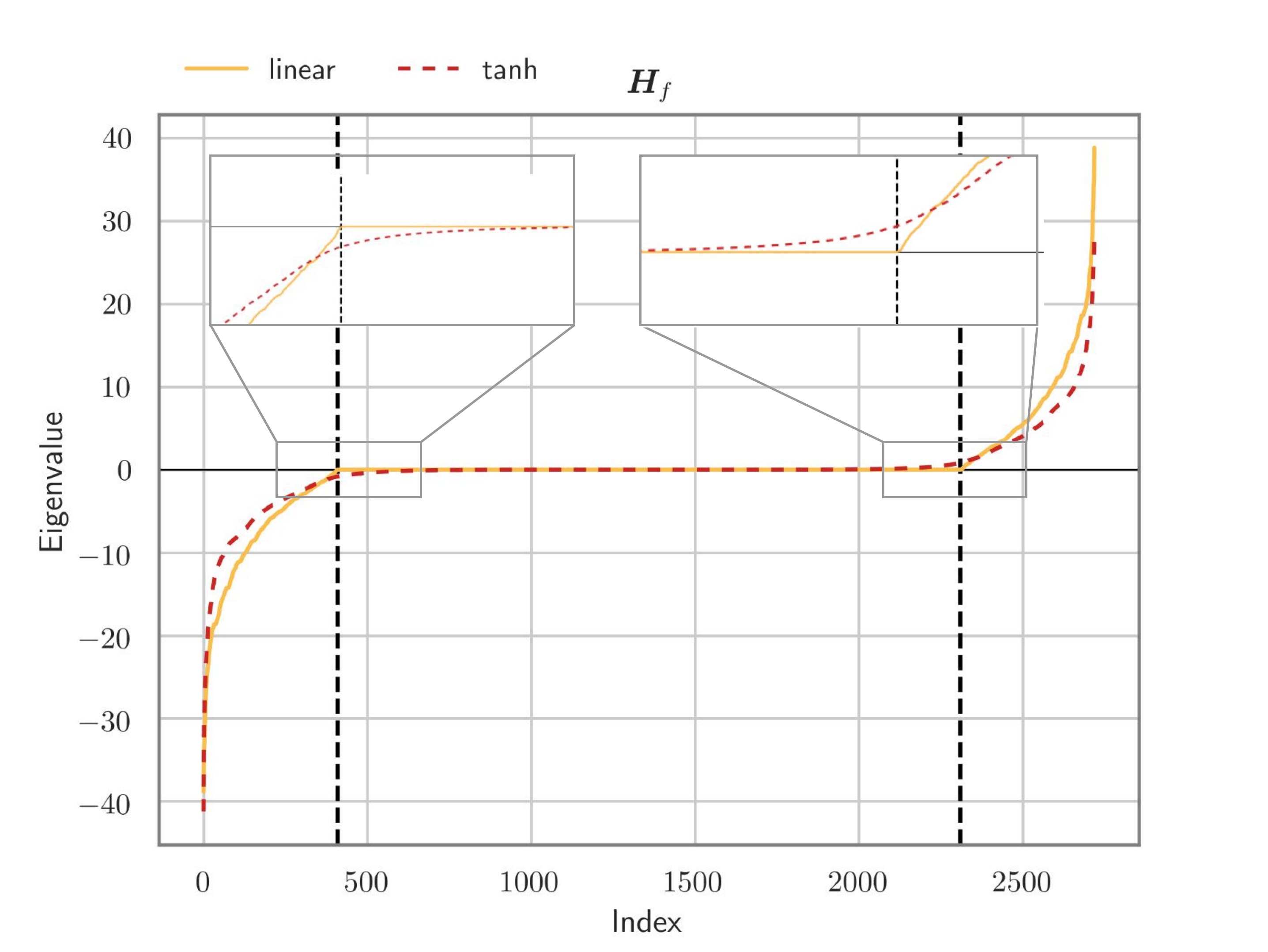}%
    \includegraphics[width=0.33\textwidth]{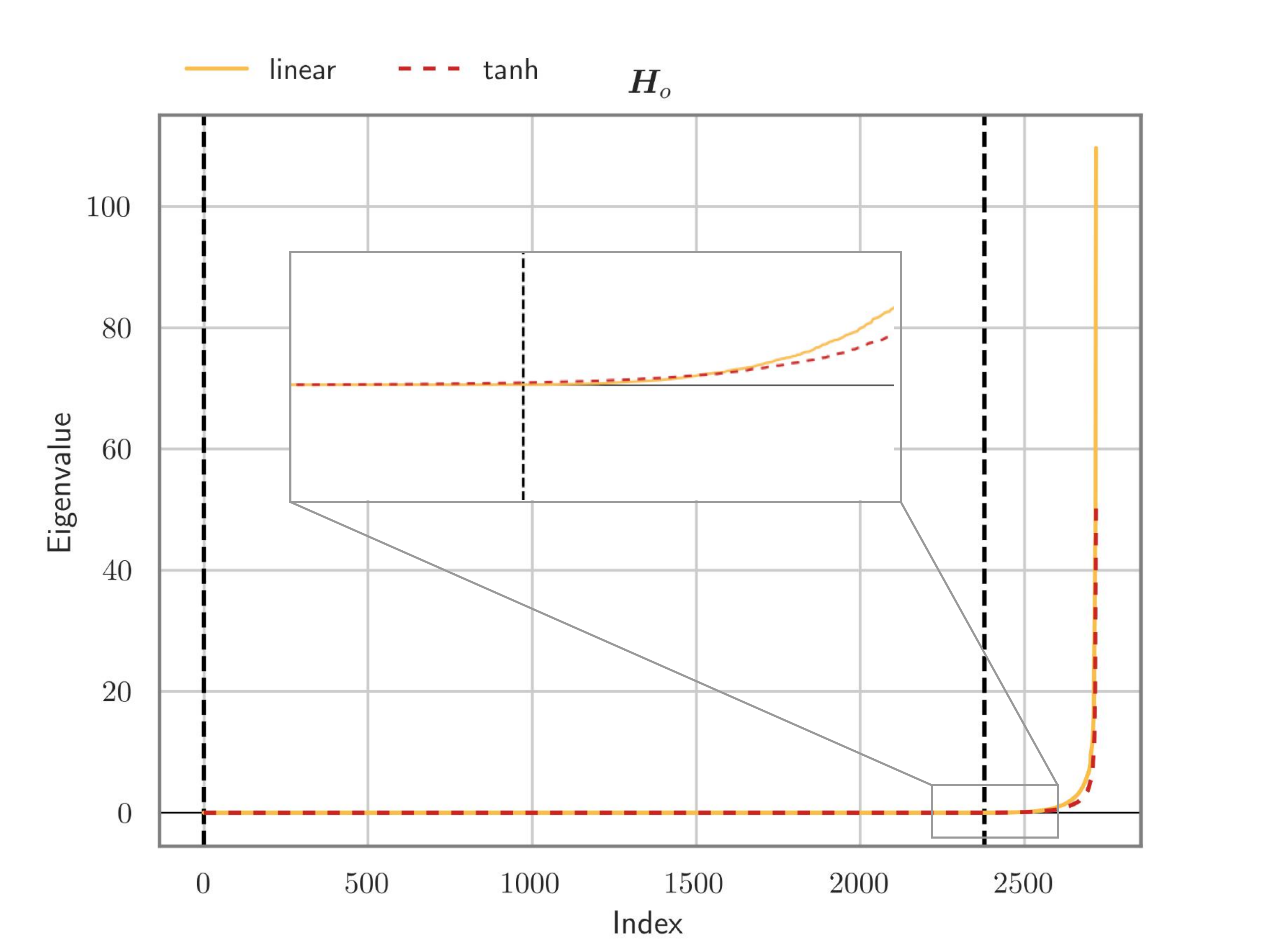}
    \caption{Spectrum of the loss Hessian $\HL$ (left), functional Hessian $\HF$ (middle)  and outer product $\HO$ (right), for \textcolor{orange}{\textbf{linear}} and \textcolor{red}{\textbf{non-linear}} networks. Black dashed lines are the predictions of the bulk size via our rank formulas. We use 2 hidden layers of size $30,20$ with \textbf{tanh} activation on \textsc{MNIST} under \textbf{cross entropy} loss.}
    \label{fig:tanh_spectrum_cross}
\end{figure}
\FloatBarrier
\begin{figure}[ht!]
    \centering
    \includegraphics[width=0.33\textwidth]{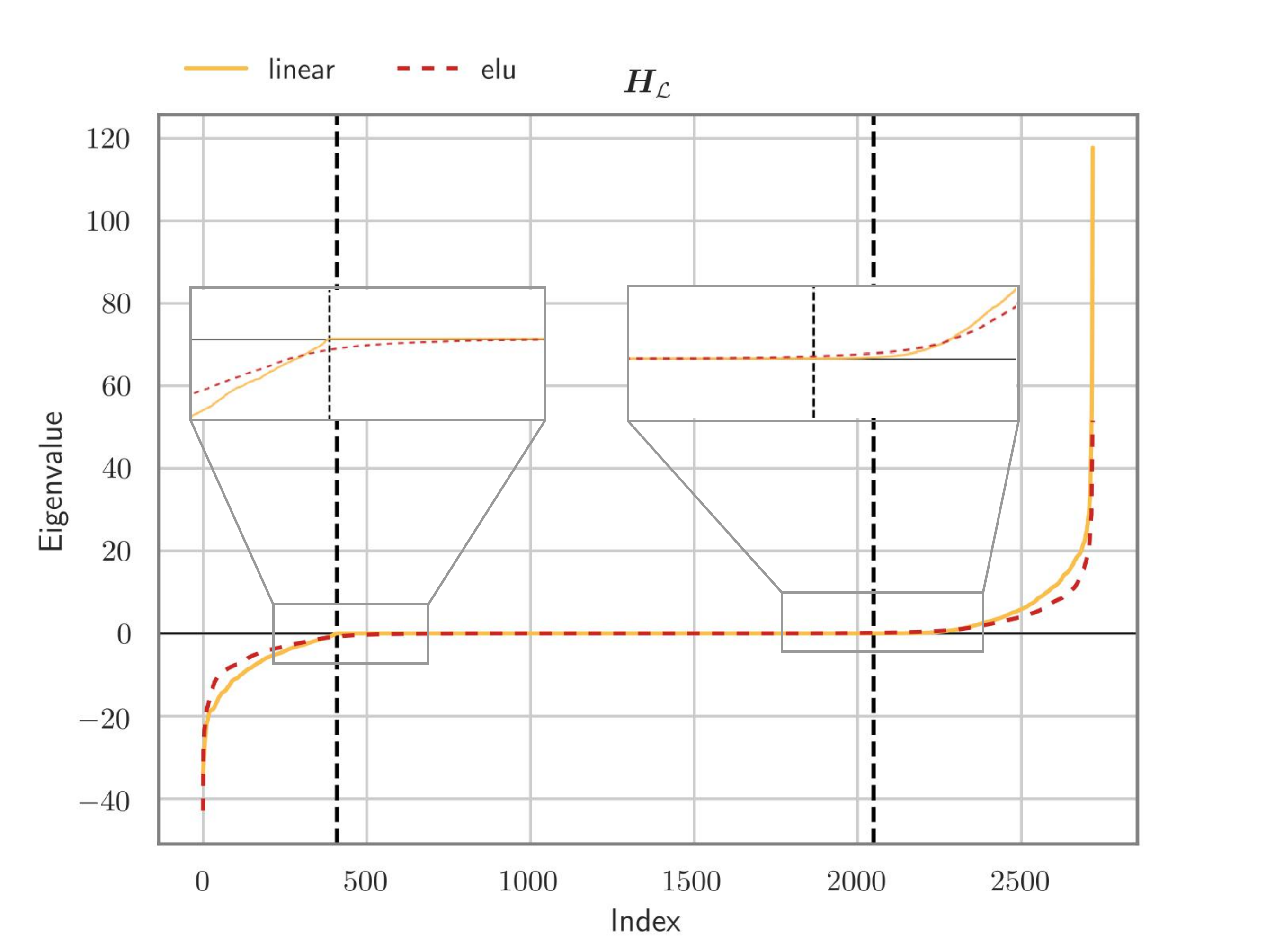}%
    \includegraphics[width=0.33\textwidth]{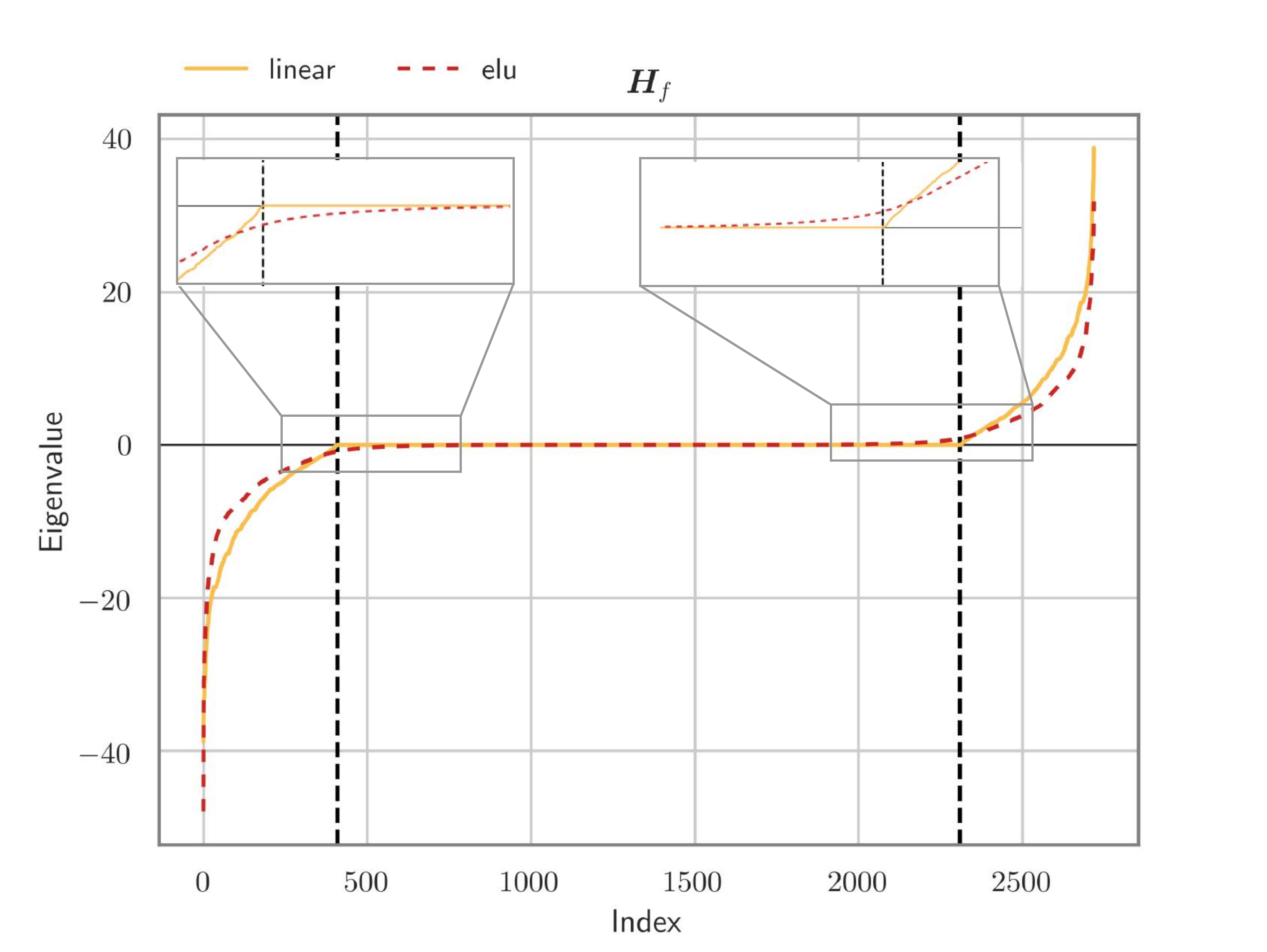}%
    \includegraphics[width=0.33\textwidth]{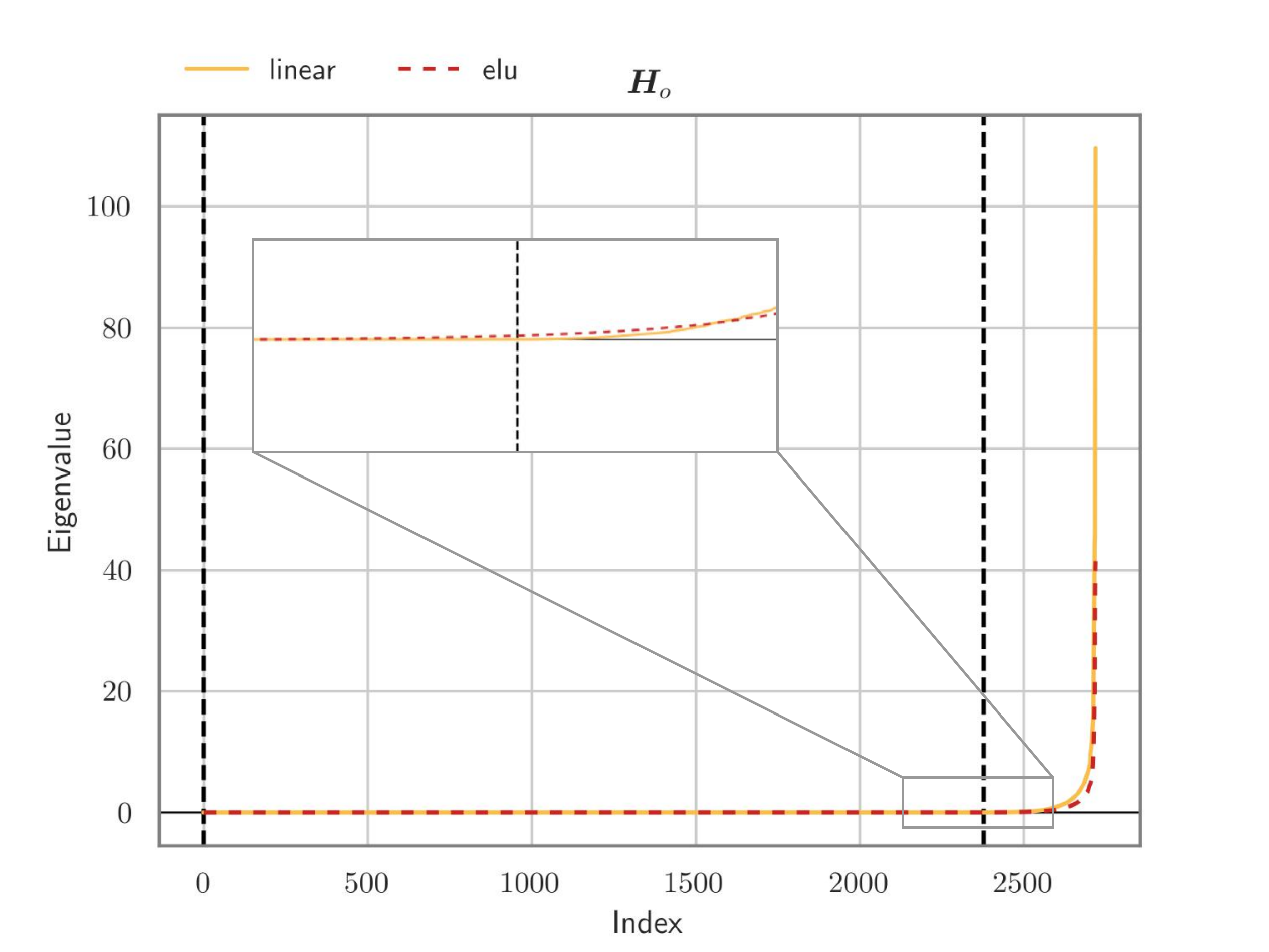}
    \caption{Spectrum of the loss Hessian $\HL$ (left), functional Hessian $\HF$ (middle)  and outer product $\HO$ (right), for \textcolor{orange}{\textbf{linear}} and \textcolor{red}{\textbf{non-linear}} networks. Black dashed lines are the predictions of the bulk size via our rank formulas. We use 2 hidden layers of size $30,20$ with \textbf{ELU} activation on \textsc{MNIST} under \textbf{cross entropy} loss.}
    \label{fig:elu_spectrum_cross}
\end{figure}
\FloatBarrier
\subsection{Rank Results for Neural Tangent Kernel}\label{supp:ntk}
In this section we show that our formulas also allow for insights into rank of the Gram matrix induced by the neural tangent kernel (NTK) \citep{jacot2018neural} at initialization. The NTK is a matrix defined entry-wise as
\begin{equation*}
    \hat{\Sigma}_{ij} = \left(\nabla_{\bm{\theta}}F_{\bm{\theta}}(\x_i)\right)^{\top}\nabla_{\bm{\theta}}F_{\bm{\theta}}(\x_j)
\end{equation*}
In Figure~\ref{fig:ntk} we display the rank dynamics  of the Gram matrix as a function of sample size. We use the predictions based on the outer-product Hessian. We observe an exact match for all datasets and sample sizes.
\begin{figure}[!htb]
    \centering
    \includegraphics[width=0.33\textwidth]{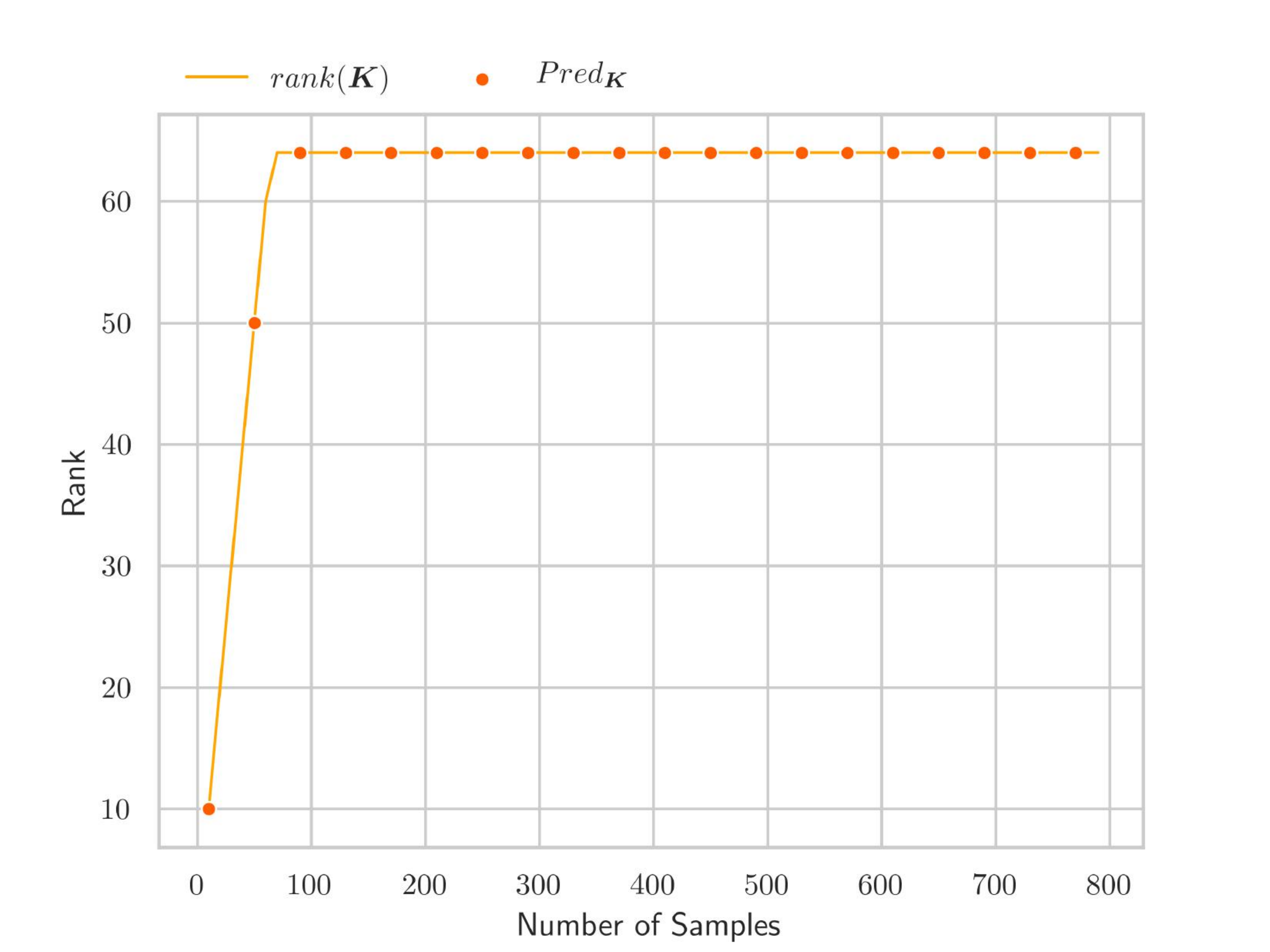}%
    \includegraphics[width=0.33\textwidth]{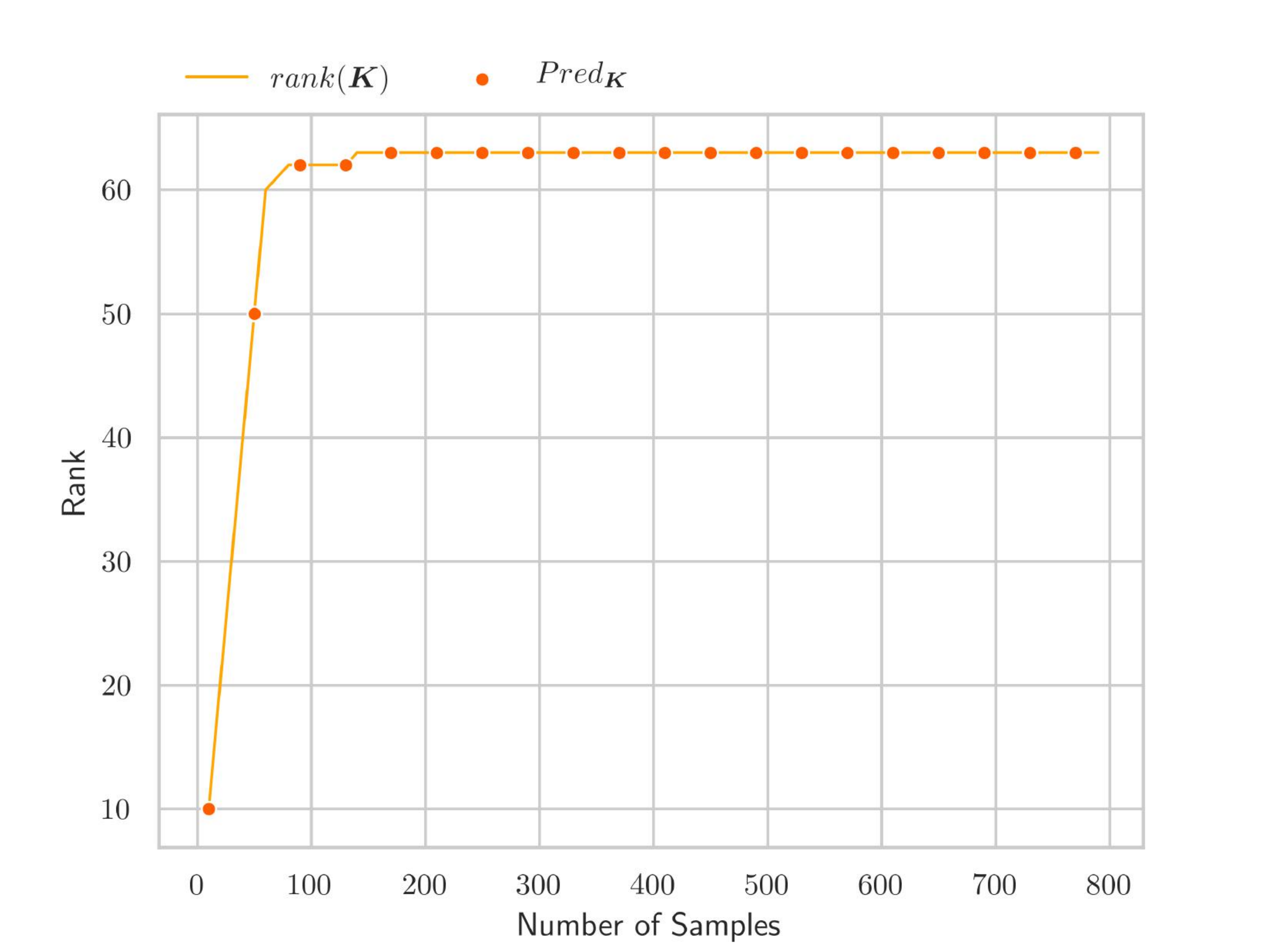}%
    \includegraphics[width=0.33\textwidth]{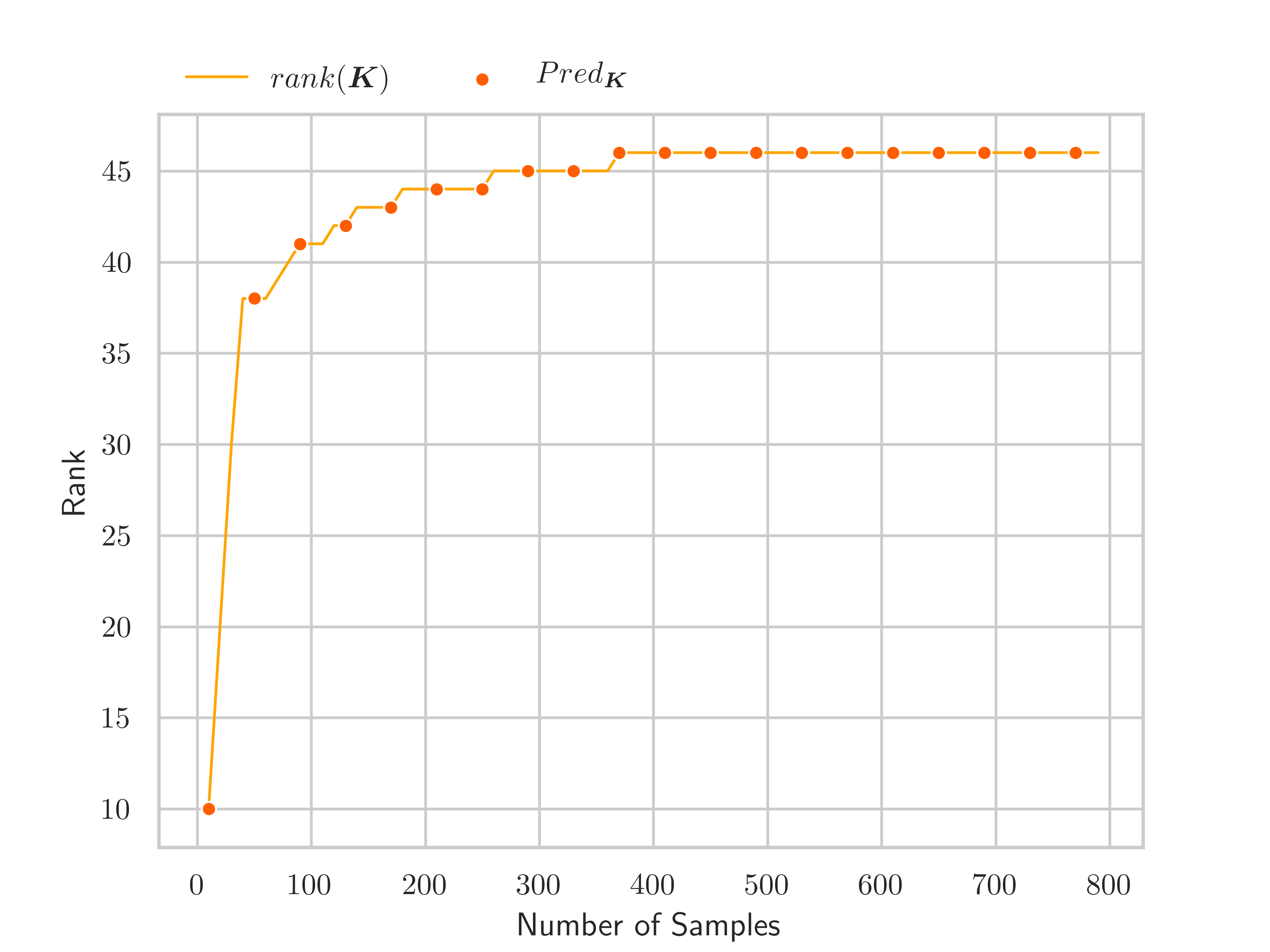}
    \caption{Rank of the empirical NTK versus sample size $n$ for architecture $20,20$. We display the predictions based on the outer-product $\HO$ as dots, using \textsc{CIFAR10} (left), Fashion\textsc{MNIST} (middle) and \textsc{MNIST} (right).}
    \label{fig:ntk}
\end{figure}
\FloatBarrier